\documentclass{article}
\PassOptionsToPackage{numbers, sort&compress}{natbib}
\usepackage[final]{neurips_2024}

\usepackage[utf8]{inputenc} 
\usepackage[T1]{fontenc}    
\usepackage{hyperref}       
\usepackage{url}            
\usepackage{booktabs}       
\usepackage{amsfonts}       
\usepackage{nicefrac}       
\usepackage{microtype}      
\usepackage{xcolor}         

\usepackage{amsmath}
\usepackage{amssymb}
\usepackage{mathtools}
\usepackage{amsthm}
\usepackage{enumitem}
\usepackage{math_commands}
\usepackage{multirow}
\usepackage{nicefrac}
\usepackage{subcaption}
\usepackage{algorithm}
\usepackage{algorithmic} 
\usepackage{multicol}
\usepackage{tablefootnote}
\usepackage{threeparttable}

\usepackage[capitalize]{cleveref}

\theoremstyle{plain}
\newtheorem{thm}{Theorem}[section]
\newtheorem{prop}[thm]{Proposition}
\newtheorem{lem}[thm]{Lemma}

\theoremstyle{definition}

\newtheorem{asmp}[thm]{Assumption}
\theoremstyle{remark}
\newtheorem{rmk}[thm]{Remark}

\allowdisplaybreaks[1]

\newcommand{\inprod}[2]{\left\langle{#1},{#2}\right\rangle}
\newcommand{\commentout}[1]{}

\newcommand{\sgdrr}{\textsf{SGD-RR}}
\newcommand{\sgdaus}{\textsf{SGDA-US}}
\newcommand{\sgdarr}{\textsf{SGDA-RR}}
\newcommand{\segus}{\textsf{SEG-US}}
\newcommand{\segrr}{\textsf{SEG-RR}}
\newcommand{\segff}{\textsf{SEG-FF}}
\newcommand{\segffa}{\textsf{SEG-FFA}}

\newcommand{\krasmann}{Krasnosel'ski\u{\i}-Mann}

\newcounter{symbolmarkcounter} 
\DeclareRobustCommand{\symbcntr}[1]{{\setcounter{symbolmarkcounter}{#1}\normalshape\textsuperscript{\fnsymbol{symbolmarkcounter}}}}
\DeclareRobustCommand{\normalsymbcntr}[1]{\setcounter{symbolmarkcounter}{#1}\fnsymbol{symbolmarkcounter}}

\usepackage{pifont} 
\definecolor{tabgreen}{rgb}{0.17, 0.63, 0.17}
\definecolor{tabred}{rgb}{0.84, 0.15, 0.16}
\definecolor{lightblue}{rgb}{0.16, 0.15, 0.84}
\newcommand{\cmark}{\ding{51}}%
\newcommand{\xmark}{\ding{55}}%
\newcommand{\yesmark}{{\color{tabred}\cmark}}
\newcommand{\nomark}{{\color{tabgreen}\xmark}}

\definecolor{revisionblue}{rgb}{0.0,0.36,1}

\definecolor{tabgray}{rgb}{0.5,0.5,0.5}

\makeatletter
\DeclareRobustCommand{\thmprime}{%
  \begingroup
  \expandafter\in@\expandafter b\expandafter{\f@series}%
  \ifin@ \boldmath \fi
  $\m@th{}^{\prime}$%
  \endgroup
}
\makeatother
\newcommand{\titleref}[2]{%
    \texorpdfstring{\Cref{#2}}{#1~\ref{#2}}%
}
\newcommand{\titleEqref}[1]{%
    \texorpdfstring{\eqref{#1}}{(\ref{#1})}%
}
\Crefname{algocf}{Algorithm}{Algorithms}
  
\title{Stochastic Extragradient with Flip-Flop Shuffling \& Anchoring: Provable~Improvements}

\author{ 
  Jiseok Chae \\
  Department of Mathematical Sciences\\
  KAIST\\
  Daejeon, Republic of Korea\\
  \texttt{jsch@kaist.ac.kr} \\
  \And
  Chulhee Yun \\
  Kim Jaechul Graduate School of AI \\
  KAIST\\
  Seoul, Republic of Korea\\
  \texttt{chulhee.yun@kaist.ac.kr} \\
  \AND
  Donghwan Kim \\
  Department of Mathematical Sciences\\
  KAIST\\
  Daejeon, Republic of Korea\\
  \texttt{donghwankim@kaist.ac.kr} \\
}

\begin{document}

\maketitle

\begin{abstract}
In minimax optimization, the extragradient (EG) method has been extensively studied because it outperforms the gradient descent-ascent method in convex-concave (C-C) problems.
Yet, stochastic EG (SEG) has seen limited success in C-C problems, especially for unconstrained cases.
Motivated by the recent progress of shuffling-based stochastic methods, we investigate the convergence of \emph{shuffling-based SEG} in unconstrained \emph{finite-sum} minimax problems, in search of %
convergent
shuffling-based SEG.  
Our analysis reveals that both random reshuffling and the recently proposed flip-flop shuffling %
alone can suffer divergence in C-C problems. 
However, with an additional simple trick called anchoring, we develop the \emph{SEG with flip-flop anchoring} (\segffa{}) method which successfully converges in C-C problems.  
We also show upper and lower bounds in the strongly-convex-strongly-concave setting, demonstrating that \segffa{} has a provably faster convergence rate compared to other shuffling-based methods.
\end{abstract}

\section{Introduction}

\looseness=-1
Minimax problems with a finite-sum structure, which are optimization problems of the form \begin{equation} \label{eqn:minimax_problem}
  \min_{\vx} \max_{\vy} \, f(\vx, \vy) \coloneqq \frac{1}{n} \sum_{i=1}^n f_i(\vx, \vy),
\end{equation} can be found in many interesting applications, such as generative adversarial networks \citep{Good14}, refining diffusion models \citep{Kim24CTM}, adversarial training \citep{Madr18}, %
optimal transport based generative models \citep{Rout22}, multi-agent reinforcement learning \citep{Wai18}, and so on. 
Deterministic methods for minimax problems, such as \emph{gradient descent-ascent}~(GDA) \citep{Arro56} and \emph{extragradient}~(EG) \citep{Korp76}, have been extensively studied in the literature.
It is though known that, unlike \emph{gradient descent} (GD) for minimization problems, GDA may diverge
even when $f$ is convex on $\vx$ and concave on $\vy$. 
On the other hand, EG employs a two-step update procedure, named \emph{extrapolation} and \emph{update} steps (see Section~\ref{sec:related-works} for details), which allows it to find
an optimum %
under this \emph{convex-concave} setting \citep{Korp76, Solo99}, and moreover, attains a convergence rate faster than GDA \citep{Aziz20} when $f$ is strongly convex on $\vx$ and strongly concave on $\vy$. %

In contrast, attempts to construct stochastic variants of these algorithms have not been so fruitful. 
When $f$ is convex-concave, \emph{stochastic gradient descent-ascent} (SGDA) clearly %
may diverge, just as in the deterministic GDA.
To make matters worse, \emph{stochastic extragradient} (SEG) methods have also had limited success on unconstrained convex-concave problems. %
As we elaborate in \Cref{sec:related-works} in more detail, existing versions of SEG and their analyses have limitations that hinder its application to general unconstrained 
finite-sum
convex-concave problems, 
requiring additional assumptions
such as 
bounded domain, increasing batch size, %
convex-concavity of each component $f_i$,
uniformly bounded gradient variance,
and/or absence of convergence rates.

In the context of 
finite-sum optimization, most %
of the
theoretical studies on stochastic methods %
have long been based on the \emph{with-replacement} sampling scheme, 
where
an index $i(t)$ is independently and uniformly sampled among $\{1, \dots, n\}$ 
at each iteration $t$. 
Such a sampling scheme is relatively easy to theoretically analyze, because the sampled $f_{i(t)}$ is an unbiased estimator of the full objective function~$f$.
In practice, however, inspired by the empirical observations of faster convergence in finite-sum minimization \citep{Bott09, Rech13}, the \emph{without-replacement} sampling schemes have been the de facto standard. Among them, the most popular is the \emph{random reshuffling} (RR) scheme, where in every \emph{epoch} consisting of $n$ iterations, the indices are chosen exactly once in a randomly shuffled order. %

This gap between theory and practice in minimization problems is being closed 
by the recent breakthroughs in stochastic gradient descent (SGD), namely that SGD with RR leads to a provably faster convergence compared to with-replacement SGD when the number of epochs is large enough~\citep{Naga19, Ahn20, Mish20SGD, Nguy21, Yun21, Yun22}.
This has motivated further studies on finding other shuffling-based sampling schemes that can improve upon RR, resulting in the discoveries such as the \emph{flip-flop} scheme \citep{Rajp22} and \emph{gradient balancing}~(GraB)~\citep{Lu22,Cha23}.
The flip-flop scheme is a particularly simple yet interesting modification of RR with improved rates in \emph{quadratic} problems: 
a random permutation is used twice in a single epoch (i.e., two passes over $n$ components in an epoch), but the order is reversed in the second pass.
 
The aforesaid progress in minimization also triggered the study of stochastic minimax methods with shuffling. 
Similar to minimization problems, SGDA with RR indeed converges faster than the with-replacement SGDA, under assumptions such as strongly-convex-strongly-concave objectives~\citep{Das22} or $f$ satisfying the Polyak-{\L}ojasiewicz condition \citep{Cho22}.
Despite the superiority of EG over GDA, the \emph{SEG with shuffling} %
{has not been shown to have a solid theoretical advantage over
the SGDA with shuffling yet. This motivated} %
us to study the following question: %

\begingroup
\renewenvironment{quote}{%
  \list{}{%
    \leftmargin0.5cm   %
    \rightmargin\leftmargin
  }
  \item\relax
}
{\endlist}
\begin{quote} 
\textbf{\textit{Can shuffling schemes provide 
convergence guarantees for SEG, 
improved upon SGDA with shuffling,
in
unconstrained finite-sum
(strongly-)convex-(strongly-)concave
settings?}}
\end{quote} 
\endgroup
There are two types of SEG: \emph{same-sample} SEG, where a sample chosen is used both for the extrapolation step and the update step, and \emph{independent-sample} SEG, where two independently chosen samples are used in %
each step.
We will particularly focus on the same-sample SEG because it combines more naturally with shuffling-based schemes than independent-sample SEG.
Therefore, to be more specific,
we are interested in developing shuffling-based variants of \emph{same-sample} SEG
in unconstrained finite-sum minimax problems with minimal modifications to the algorithm.
We 
show %
that \textbf{(a)} in convex-concave settings, our new method %
reaches an optimum %
with a guarantee on the rate of convergence, overcoming the 
limitations of existing results; \textbf{(b)} in strongly-convex-strongly-concave settings, the method converges faster than other SGDA/SEG variants.

\subsection{Our Contributions}

In this paper, %
we 
study various same-sample SEG algorithms under different 
shuffling schemes, and
propose the 
\emph{stochastic extragradient with flip-flop anchoring} (\segffa) method, which is SEG amended with the techniques of \emph{flip-flop} %
shuffling scheme and \emph{anchoring}.
Here, by \emph{anchoring} we refer to a step of taking a convex combination between the initial and final iterates of an epoch, resembling the celebrated \emph{\krasmann{} iteration} \citep{Kras55, Mann53} as we discuss in \Cref{section:method}.
With such minimal modifications to SEG,
we show that \segffa{} achieves provably improved convergence guarantees.
More precisely, our contributions can be listed as follows (see \Cref{tab:results} for a summary). For clarity, %
we use \segus{} to refer to with-replacement SEG, where US stands for uniform sampling.
\begin{itemize}
    \item We first study the same-sample versions of \segus{}, SEG with RR (\segrr), and SEG with flip-flop (\segff). %
    We show that they all can diverge when $f$ is convex-concave,\footnote{ 
    This does not contradict the result in \citep{Hsie20}, which shows that the \emph{independent-sample} %
    SEG %
    with carefully designed step sizes rule
    converges
    to optima for convex-concave settings, albeit without a convergence rate.}
    by constructing an explicit counterexample (\Cref{thm:rr-ff-bad}). 
    This shows that shuffling alone \emph{cannot} fix the divergence issue of \segus. %
    \item %
    We next investigate the underlying cause
    for the nonconvergence of \segus{}, \segrr{}, and \segff{}.
    In particular, we identify
    that either they fail to match the update equation of the reference method EG
    beyond \emph{first-order} Taylor expansion terms, or attempting to match both the \emph{first-} and \emph{second-order} Taylor expansion terms results in divergence (\Cref{thm:ff-alone}).
    \item
    By adopting a simple technique of \emph{anchoring} on top of flip-flop shuffling, we devise our algorithm \segffa{}, whose epoch-wise update deterministically matches EG up to second-order Taylor expansion terms~(\Cref{thm:ffa-cubic-error}).
    We prove that \segffa{} enjoys improved convergence guarantees, as anticipated by our design principle. %
    Most importantly, we show that \segffa{} achieves a convergence
    rate of $\tilde{\gO}(\nicefrac{1}{K^{1/3}})$ when $f$ is convex-concave, where $K$ denotes the number of epochs.
    This is in stark contrast to other baseline algorithms that diverge under this setting (see the last column of \Cref{tab:results}).
    \item Moreover, we show that when $f$ is strongly-convex-strongly-concave, \segffa{} achieves a convergence rate of $\tilde{\gO}(\nicefrac{1}{nK^{4}})$~(\Cref{thm:scsc-ffa}). 
    In addition, by proving $\Omega(\nicefrac{1}{nK^3})$ lower bounds for the convergence rates of \sgdarr{} and \segrr{} under the same setting~(\Cref{thm:scsc-lower-bound}), we show that \segffa{} has a \emph{provable advantage} over these baseline algorithms.
\end{itemize}

\begin{table*}
\centering
\begin{threeparttable}[t]
\caption{\small Summary of upper/lower convergence rate bounds
of \emph{same-sample} SEG for unconstrained finite-sum minimax problems,
without requiring increasing batch size, convex-concavity of each component, and uniformly bounded gradient variance. Pseudocode of algorithms can be found in \Cref{sec:pseudocode}. We only display terms that become dominant for sufficiently large $T$ and $K$. To compare the with-replacement versions \hbox{(\textsf{-US})} against shuffling-based versions, one can substitute $T = nK$. The optimality measure used for SC-SC problems is $\E[\norm{\hat \vz - \vz^*}^2]$ for the last iterate $\hat \vz$. For C-C problems, we consider $\min_{t=0,\dots,T} \E [\norm{\mF \vz_t}^2]$ for with-replacement methods and $\min_{k=0,\dots,K} \E [\norm{\mF \vz_0^k}^2]$ for shuffling-based methods. %
}
\label{tab:results} 
\renewcommand*{\arraystretch}{1.05}
\setlength{\tabcolsep}{3pt}
\begin{small}
\begin{sc}
\begin{tabular}{c|c|c|c|c}
\toprule
& \multicolumn{2}{c|}{Strongly-Convex-Strongly-Concave} & \multicolumn{2}{c}{Convex-Concave}\\  \cmidrule{2-5} Method & Upper Bound & Lower Bound & Upper Bound & Lower Bound \\
\midrule
\sgdaus{} & $\gO ( \frac{1}{T} )$ {\scriptsize \citep{Loiz21}} & $\Omega ( \frac{1}{T} )$ {\scriptsize \citep{Cho22}} & N/A & $\Omega(1)$ {\scriptsize (as GDA)}\; \\
\segus{} & $\gO ( \frac{1}{T} )$ {\tiny \citep{Gorb22SEG}} & $\Omega ( \frac{1}{T} )$ {\tiny \citep{Bezn20}} & N/A\symbcntr{2}\symbcntr{3}\!\! & $\Omega(1)$ {\scriptsize (Thm.~\ref{thm:rr-ff-bad})}  \\ 
\midrule
\sgdarr{} & $\tilde \gO ( \frac{1}{nK^2} )$ {\scriptsize \citep{Das22}} & $\Omega ( \frac{1}{nK^3} )$ {\scriptsize (Thm.~\ref{thm:scsc-lower-bound})} & N/A & $\Omega(1)$ {\scriptsize (as GDA)}\;\\
\segrr{} & $\tilde \gO ( \frac{1}{nK^2} )$ {\scriptsize \citep{Emma24}} & $\Omega ( \frac{1}{nK^3} )$ {\scriptsize (Thm.~\ref{thm:scsc-lower-bound})} & {\color{lightgray} $\gO(\frac{1}{(nK)^{1/3}})$? {\scriptsize\citep{Emma24}}}\symbcntr{4} & $\Omega(1)$ {\scriptsize (Thm.~\ref{thm:rr-ff-bad})}\\
\segff{} & {$\tilde{\gO} ( \frac{1}{nK^2} )$} {\scriptsize (Thm.~\ref{thm:scsc-convergence-result})} & -- & N/A & $\Omega(1)$ {\scriptsize (Thm.~\ref{thm:rr-ff-bad})}\\
\midrule
\segffa{} & {$\tilde{\gO} ( \frac{1}{nK^4} )$} {\scriptsize (Thm.~\ref{thm:scsc-ffa})} & -- & $\tilde{\gO}(\frac{1}{K^{1/3}})$ {\scriptsize (Thm.~\ref{thm:cc-ffa})} & --\\
\bottomrule
\end{tabular}
\end{sc}
\end{small} 
\scriptsize 
\begin{tablenotes}
\item[\dag] \citep{Diak21,Gorb22SEG} show upper bounds for \segus{}, but they require increasing batch sizes as well as other assumptions (see \Cref{appx:comparison-table}).
\item[\ddag] \citep{Hsie20} shows that \emph{independent-sample} \segus{} converges for stepsizes $\alpha_t,\beta_t$ decaying at different rates, but gives no conv.\@ rate.
\item[\textsection] Unfortunately, the proof of this convergence bound in this recent AISTATS 2024 paper seems to be incorrect: see \Cref{sec:segrr-flaw}.
\end{tablenotes}
 \end{threeparttable}
\end{table*}

\section{Related Works}  \label{sec:related-works}
 
\paragraph{Extragradient and EG+} 
Extragradient (EG) method \citep{Korp76} is a widely used minimax optimization method, well-known for resolving the nonconvergence issue of GDA on convex-concave problems. 
In this paper, we also consider EG+ \citep{Diak21}, which is a generalization of EG. The update rule of EG+ is defined, for stepsizes $\{\eta_{1,k}\}_{k\geq0}$ and $\{\eta_{2,k}\}_{k\geq0}$,  as
\begin{align}  \label{eqn:detEG}
    &\left\{ \begin{aligned}
        \vu^{k} &\gets \vx^{k} - \eta_{1, k} \gradx f(\vx^{k}, \vy^{k})  \\ 
        \vv^{k} &\gets \vy^{k} + \eta_{1, k} \grady f(\vx^{k}, \vy^{k})
    \end{aligned}  \right.{,} %
    &\left\{ \begin{aligned}
        \vx^{k+1} &\gets \vx^{k} - \eta_{2, k} \gradx f(\vu^{k}, \vv^{k})   \\ 
        \vy^{k+1} &\gets \vy^{k} + \eta_{2, k} \grady f(\vu^{k}, \vv^{k}) 
    \end{aligned} \right.{.} %
\end{align} 
The first step is called the \emph{extrapolation step}, and the second step is called the \emph{update step}.
If $f$ is convex-concave, \citet{Diak21} show that EG+ reaches an optimum when $\eta_{1, k} \geq \eta_{2, k}$.
In particular, when $\eta_{1, k} = \eta_{2, k}$, we recover the standard EG by \citet{Korp76}. 

\paragraph{Stochastic Variants of Extragradient}
 
 In \eqref{eqn:detEG}, if the \emph{stochastic estimators} of $\gradx f$ and $\grady f$ are used instead of the gradients themselves, we get the standard SEG. 
 If an estimator chosen is used for both the extrapolation and the update steps, we get the \emph{same-sample} SEG, 
which we focus on
in this paper; see \Cref{sec:pseudocode} for the pseudocode. 

While EG improves upon GDA, unfortunately, SEG has not been able to show a clear advantage over SGDA. 
  On one hand, analyses of SEG on strongly-convex-strongly-concave problems have shown some success; see, \textit{e.g.,} \cite{Emma24, Gorb22SEG}.    
Yet, on the other hand, for general unconstrained convex-concave problems, to the best of our knowledge, the existing stochastic variants of EG and their analyses face several limitations.\footnote{Most of these results are carried out assuming access to a stochastic oracle of $f$, which indeed subsumes the finite-sum setting as a special case. 
However, it seems unlikely that %
these %
limitations of the existing studies will
be easily resolved by simply narrowing the focus down to the finite-sum setting; see \Cref{appx:finite-sum-is-general}. %
}\footnote{
Recently, %
\citet{Emma24} claimed
the
convergence of \segrr{} in the %
convex-concave
setting. 
Unfortunately, however, there seems to be a flaw in their 
proof. %
We defer a %
discussion on this to \Cref{sec:segrr-flaw}.
}
Assumptions commonly imposed in the existing literature include: %
\textbf{(i)} the domain is bounded, either explicitly or implicitly~\citep{Judi11,Mish20SEG},
\textbf{(ii)} one must increase the batch size to achieve convergence~\citep{Diak21, Cai22, Gorb22SEG},\footnote{In fact, for the methods studied in \citep{Diak21, Gorb22SEG} it is possible to show that increasing the batch size is \emph{strictly necessary and unavoidable} for convergence; see \Cref{sec:appx:segus-monotone-lb}.}
and 
\textbf{(iii)} each component $f_i$ is convex-concave~\citep{Mish20SEG, Gorb22SEG}, 
and \textbf{(iv)} the components have uniformly bounded gradient variance~\citep{Diak21, Cai22, Peth23}.
For further details, see \Cref{appx:comparison-table} and \Cref{tab:related-works} therein.
Notably, \citet{Hsie20} prove convergence of the \textit{independent-sample} SEG without these four restrictions, but the result lacks an explicit convergence rate. 

Our proposed \segffa{} overcomes all the aforementioned limitations, and reaches an optimum
with an explicit rate in unconstrained convex-concave problems, under relatively mild conditions.
The readers may also refer to \citep{Bezn23a} for a comprehensive overview on this topic.

Meanwhile, %
under
the finite-sum setting, variance reduction %
schemes
have also been considered, achieving some promising results \citep{Carm19, Alac22}. 
Yet, 
although theoretically appealing, variance reduction 
is %
less widely used in practice due to their curiously inferior performance in training neural networks \citep{Defa19}. 
On top of this practical issue, variance reduction techniques share the aforementioned limitation~\text{(ii)}, as accessing full gradients can be viewed as increasing the batch size. In contrast, our main goal in this paper is to study how a carefully chosen sampling scheme,
with minimal modifications to the algorithm,
can improve the convergence of SEG without the need for increased batch size; therefore, we believe that our work is not directly comparable to variance reduction-based EG.

\paragraph{Taylor Expansion Matching and Convergence Guarantees} 

It has been repeatedly reported that the convergence %
of an optimization method is deeply related to the degree to which the Taylor expansion (with respect to the step size) of its update equation matches with %
that
of an already known convergent method.
For example, \citet{Mokh20} observed that the advantage of EG over GDA comes from the Taylor expansion of update equations of EG matching that of the \emph{proximal point}~(PP) method~\citep{Mart70} up to second-order terms, 
whereas GDA matches PP only up to first-order terms. %

The advantages of the shuffling scheme over the with-replacement sampling can be explained in a similar way. %
One key property of shuffling-based methods is that, while the individual estimators are biased as they are dependent to other estimators within the same epoch, the overall stochastic error across the epoch decreases dramatically compared to using $n$ independent unbiased estimators. 
For instance, in SGD with RR \citep{Ahn20} and in SGDA with RR \citep{Das22}, the overall progress made within each epoch exactly matches their deterministic counterparts up to the first-order, leaving an error as small as $\gO(\eta^2)$, where $\eta$ is the stepsize. 
\citet{Rajp22} observed that, when each component functions are convex quadratics, then using flip-flop on SGD can reduce the error further to $\gO(\eta^3)$, resulting in an even faster convergence. 
As we further elaborate in \Cref{section:method}, the motivation behind our design principle of \segffa{} is also based on this line of observations.

\section{Notations and Problem Settings} \label{sec:problem-settings}
Let $[n] \subset \sZ$ denote the set $\{1, \dots, n\}$. 
The set of all permutations on $[n]$ will be denoted by $\gS_n$. 
For the finite-sum minimax problem \eqref{eqn:minimax_problem}, 
we denote the \emph{saddle gradient} operators by \[
\mF (\,\cdot\,) \coloneqq \begin{bmatrix} \gradx f (\,\cdot\,)  \\ -\grady f (\,\cdot\,)  \end{bmatrix}, \enspace \mF_i (\,\cdot\,) \coloneqq \begin{bmatrix} \gradx f_i (\,\cdot\,)  \\ -\grady f_i (\,\cdot\,)  \end{bmatrix}, \quad\! i \in [n] .  
\]
The derivative of an operator will be denoted with a prefix~$D$. For example, the derivative of $\mF$ is denoted by $D\mF$.
Often a single vector will be used to denote the minimization and the maximization variable at once. 
For instance, for $\vz \in \rr^{d_1 + d_2}$ which is a concatenation of $\vx \in \rr^{d_1}$ and $\vy \in \rr^{d_2}$, we simply write $\mF \vz$ to denote $\mF(\vx, \vy)$. 

It is well known that, if $f$ is $\mu$-strongly convex on $\vx$ and $\mu$-strongly concave on $\vy$ for some $\mu > 0$ (respectively, $\mu = 0$), then its saddle gradient $\mF$ is $\mu$-strongly monotone (respectively, monotone), in the following sense. For a proof of this standard fact, see, \textit{e.g.}, \citep{Grim23}.
\begin{asmp}[Monotonicity \& Strong Monotonicity]\label{asmp:monotone} For $\mu > 0$, we say that an operator $\mF$ is $\mu$-strongly monotone if, for any $\vz, \vw \in \rr^{{d_1}+{d_2}}$, it holds that \begin{equation}\label{eqn:def-strongly-monotone}
    \inprod{\mF \vz - \mF \vw}{\vz - \vw} \geq \mu\norm{\vz - \vw}^2. 
\end{equation} 
If \eqref{eqn:def-strongly-monotone} holds for $\mu = 0$, then we say that $\mF$ is monotone. 
\end{asmp}

Thus, from now on, we will use the term \emph{strongly monotone} (respectively, \emph{monotone}) problems rather than strongly-convex-strongly-concave (respectively, convex-concave) problems. 
Notice that we only assume that the full saddle gradient $\mF$ is (strongly) monotone, not the individual $\mF_i$'s.

In addition, we remark that
our convergence analysis under the monotonicity of $\mF$, \Cref{thm:cc-ffa},
in fact requires only a relaxed version of monotonicity, 
known as
\emph{star}-monotonicity.
This condition imposes the inequality~\eqref{eqn:def-strongly-monotone}
with $\mu=0$, but only when $\vw=\vz^*$, where $\vz^*$ is a point such that $\mF\vz^*=\zero$.
This relaxation allows for a certain degree of nonconvex-nonconcavity in $f$. 
For a more detailed discussion on the star-monotonicity condition, see~\Cref{appx:star-monotone}.

Other three underlying assumptions we make on the problem \eqref{eqn:minimax_problem} can be listed as follows. 

\begin{asmp}[Existence of an Optimal Solution]\label{asmp:solution}
    An optimal solution of the problem \eqref{eqn:minimax_problem}, which is a point we denote by $\vz^* = (\vx^*, \vy^*)$ that satisfies 
 \[%
        f(\vx^*, {\vy}) \leq f(\vx^*, \vy^*) \leq f({\vx}, \vy^*)
 \]%
 for any ${\vx} \in \rr^{d_1}$ and ${\vy} \in \rr^{d_2}$, exists in $\rr^{d_1+ d_2}$. 
\end{asmp}
    Because the problem is unconstrained and $f$ is convex-concave, a point $\vz^*$ is an optimum if and only if $\mF \vz^* = \zero$. 
For strongly monotone problems, \Cref{asmp:solution} is not explicitly required, %
as it is
guaranteed \textit{a priori} \citep[Proposition~22.11]{Baus17}. 
For monotone problems, we explicitly impose \Cref{asmp:solution} in order to exclude pathological problems such as $f(x,y) = x - y$.

\begin{asmp}[Smoothness]\label{asmp:smoothness} Each $f_i$ is $L$-smooth, and each $\mF_i$ is $M$-smooth. That is, for any $\vz, \vw \in \rr^{{d_1}+{d_2}}$, 
\begin{enumerate}[label=(\roman*)]
    \item $\norm{\mF_i \vz - \mF_i \vw} \leq L \norm{\vz - \vw}$, %
    \item $\norm{D\mF_i \vz - D\mF_i \vw} \leq M \norm{\vz - \vw}$.
\end{enumerate}
\end{asmp}  
It is worth mentioning that the gradient operator $\mF_i$ arising from a quadratic function $f_i$ is $M$-smooth with $M = 0$.
Notice also that, by the finite-sum structure $\mF = \frac{1}{n} \sum_{i=1}^n \mF_i$, it is clear that \Cref{asmp:smoothness} implies $f$ being $L$-smooth and $\mF$ being $M$-smooth. 

The $L$-smoothness assumption on the objective functions is standard in the optimization literature, while the $M$-smoothness assumption on the saddle gradients may look %
less standard.  
This smoothness assumption on the saddle gradient, 
in other words the Lipschitz Hessian condition,
for analyzing~\segffa{}
stems from the analysis of the \mbox{flip-flop} sampling scheme~\citep{Rajp22}.
In particular, 
this is needed for bounding the high-order error terms
between the (deterministic) EG and \segffa{}
in \Cref{sec:why-segffa}.
The existing analysis of flip-flop sampling~\citep{Rajp22}
is limited to quadratic functions that trivially have $0$-Lipschitz Hessians ($M=0$),
so our analysis is a step forward.

\begin{asmp}[Component Variance]\label{asmp:bounded-variance} There exist constants $\rho \geq 0$ and $\sigma \geq 0$ such that %
\begin{equation}\label{eqn:variance-control}
    \frac{1}{n}\sum_{i=1}^n\, \norm{\mF_i \vz - \mF\vz}^2 \leq (\rho \norm{\mF\vz} + \sigma)^2 \qquad \forall \vz. 
\end{equation}
\end{asmp} 
For strongly monotone problems, \Cref{asmp:bounded-variance}
is not explicitly required, because it can be obtained as %
a consequence of the preceding assumptions: see \Cref{appx:variance-redundant}.
Nevertheless, %
for convenience, we will keep the notations $\rho$ and $\sigma$ as in \eqref{eqn:variance-control} for the strongly monotone setting as well.

In many existing works on stochastic optimization methods for minimax problems, \Cref{asmp:bounded-variance} with $\rho = 0$ is imposed. This \emph{uniform} bound on the variance simplifies the convergence analyses,  
but it is also fairly restrictive especially in the unconstrained settings.
Already for bilinear finite-sum minimax problems $f(\vx,\vy) =\frac{1}{n}\sum_{i=1}^n \vx^\top \mB_i \vy$, one can easily check that setting $\rho = 0$ forces the matrices $\mB_i$ to be exactly equal to each other.
For machine learning applications, %
it has been also reported that the assumption with $\rho = 0$ often fails to hold \citep{Bezn23a}. 
Therefore, allowing the variance to grow with the gradient $\mF \vz$ makes the assumption much more realistic. %

The Lipschitz Hessian condition and the component variance assumption for monotone problems may still look rather strong. 
We leave the study on how one can relax such assumptions to prove \emph{upper} bounds for convergence rates as an interesting future direction.
On the other hand, while our lower bound results in Theorems~\ref{thm:rr-ff-bad} and \ref{thm:scsc-lower-bound} are derived under those strong assumptions, they still serve as lower bound results also for larger function classes that do not have those assumptions. In other words, the value of those results are not limited because of those assumptions being imposed. 

\section{Shuffling Alone Is Not Enough} \label{sec:shuffling-not-enough} %
Under the settings we have discussed, 
we study the SEG with shuffling-based sampling schemes. %
First we describe the precise methods of our consideration, namely the \segrr{} and \segff{}. %

For $k \geq 0$, in the beginning of %
an epoch, a random permutation $\tau_k$ is sampled from a uniform distribution over $\gS_n$.
Then, for $n$ iterations, we use each of the component functions once, in the order determined by $\tau_k$. 
That is, for $i=0, 1, \dots, n-1$ we do 
\begin{equation} \label{eqn:segrr} 
\begin{aligned}
\vw_i^k &\gets \vz_i^k - \alpha_k \mF_{\tau_k(i+1)} \vz_i^k, \\
\vz_{i+1}^{k} &\gets \vz_i^k - \beta_k \mF_{\tau_k(i+1)} \vw_i^k,
\end{aligned} 
\end{equation} for some stepsizes $\alpha_k$ and $\beta_k$. In case of \segrr{}, the epoch is completed here, and we set $\vz_{0}^{k+1} \gets \vz_n^k$ as the initial point for the next epoch. 
 
In case of \segff{}, we additionally perform $n$ more iterations in the epoch, as proposed in \citet{Rajp22}. 
In these additional iterations, the component functions are each used once more, but in the reverse order. 
That is, for $i = n, n+1, \dots, 2n-1$, we do \begin{equation} \label{eqn:segff} 
\begin{aligned}
\vw_i^k &\gets \vz_i^k - \alpha_k \mF_{\tau_k(2n-i)} \vz_i^k, \\
\vz_{i+1}^{k} &\gets \vz_i^k - \beta_k \mF_{\tau_k(2n-i)} \vw_i^k. 
\end{aligned} 
\end{equation}
Then we set $\vz_{0}^{k+1} \gets \vz_{2n}^k$ as the initial point for the next epoch. 
The full pseudocode of these methods can be found in \Cref{sec:pseudocode}. 

When $\mF$ is strongly monotone, it is possible to show that both \segrr{} and \segff{} indeed provide speed-up over \segus{}.
The well-known %
rate of \segus{} under strong monotonicity of $\mF$ is $\Theta(\nicefrac{1}{T})$, where $T$ is the total number of iterations \citep{Bezn20,Gorb22SEG}.  
Translating this rate to our shuffling-based setting, where there are $\Theta(n)$ iterations per epoch, this rate amounts to $\Theta(\nicefrac{1}{nK})$. 
Recently, \citet{Emma24} have shown that \segrr{}, under the same setting as ours, attains a convergence rate of $\tilde{\gO}(\nicefrac{1}{nK^2})$,
on par with the rate of \sgdarr{} \citep{Das22}.
In \Cref{sappx:sm-analysis}, we also show that \segff{} attains a similar rate of convergence. %

However, it turns out that the benefit of shuffling does not extend further beyond the %
strongly monotone setting. In fact, when $\mF$ is merely monotone, then in the worst case, \segrr{} and \segff{} %
suffers from nonconvergence, just as in the case of 
\segus{}.

\begin{thm}\label{thm:rr-ff-bad}
    For $n = 2$, there exists a minimax problem with $f(x,y) = \frac{1}{2} \sum_{i=1}^2 f_i(x,y)$ having a monotone $\mF$, consisting of $L$-smooth quadratic $f_i$'s satisfying Assumption~\ref{asmp:bounded-variance} with $(\rho, \sigma) = (1,0)$, such that \segus{}, \segrr{} and \segff{} diverge in expectation for any positive stepsizes.
\end{thm}
    We provide the explicit counterexample and the proof of divergence in \Cref{ssec:divergence}. 
    Note that \Cref{thm:rr-ff-bad} and its proof in \Cref{ssec:divergence} imply that $\min_{t=0,\dots,T} \E [\norm{\mF \vz_t}^2]=\Omega(1)$ for \segus{} and $\min_{k=0,\dots,K} \E [\norm{\mF \vz_0^k}^2] = \Omega(1)$ for \segrr{} and \segff{}, as summarized in \Cref{tab:results}.

\section{\segffa{}: SEG with Flip-Flop Anchoring} \label{section:method}

In this section, we investigate the underlying cause for nonconvergence of \segrr{} and \segff{} from the perspective of how accurately they match the convergent EG or PP methods
in terms of the Taylor expansions of updates.
We then propose adding a simple \emph{anchoring} step at the end of each epoch of \segff{}. 
It turns out that adding the \emph{anchoring} step, which is a step of taking a convex combination of an iterate with a previously computed iterate, reduces the stochastic noise and leads to a method with improved convergence properties.

\subsection{Design Principle: Second-Order Matching} \label{sec:why-segffa}

As observed by \cite{Mokh20}, the key feature of EG behind its superior convergence properties compared to GDA is its update rule closely resembling PP, while the ``error'' of GDA as an approximation of PP is so large that it hinders convergence. 
The difference between the updates of EG and PP, in the Taylor expansion, is as small as $\gO(\eta^3)$ per iteration, where $\eta$ is the stepsize. On the other hand, GDA and PP show a difference of $\gO(\eta^2)$, and this greater ``error'' explains why GDA diverges while EG and PP converge. 
Of course, EG and PP are not the only two algorithms that converge in the monotone setting;
let us recall the update rule of EG+ method \citep{Diak21}, and Taylor-expand it as the following:
\begin{equation} \label{eqn:egplus}
    \begin{aligned} 
        \vz^{+} &\coloneqq \vz - \eta_2 \mF (\vz - \eta_1 \mF \vz) \\
        &\phantom{:}= \vz - \eta_2 \mF \vz + \eta_1 \eta_2  D \mF (\vz) \mF \vz + O(\eta_1^2 \eta_2).     
    \end{aligned} 
\end{equation}
EG+ is known to converge for unconstrained monotone problems if $\eta_1 \geq \eta_2$. When $\eta_1 = \eta_2$, it recovers EG and matches PP up to second-order terms. 

Based on these observations, we now state our key principle for designing a convergent version of SEG: \emph{second-order matching}. 
We would like to choose proper stepsizes, sampling scheme, and anchoring scheme so that our without-replacement SEG can \emph{deterministically} match the update equation of a convergent algorithm (EG/PP or EG+) up to the $O(\eta^2)$ terms (i.e., \emph{second-order} terms in the Taylor expansion), thereby satisfying a small $\gO(\eta^3)$ approximation error. 
We show that \textbf{(a)} this \emph{second-order matching} can be achieved with \emph{flip-flop anchoring}, 
but not solely by permutation-based sampling such as RR and flip-flop (without anchoring), 
and \textbf{(b)} second-order matching indeed grants convergence for monotone problems. 
In particular, we demonstrate that 
\begin{enumerate}
    \item \segrr{} suffers a poor approximation error of $\gO(\eta^2)$ as an approximation of EG/EG+. 
    \item \segff{} can match EG+ up to second-order terms, but it results in a choice of stepsizes ($\eta_2 = 2\eta_1$) that make EG+ diverge (\Cref{thm:ff-alone}).
    \item \segffa, the method we propose, matches EG up to second-order terms to get an error of $\gO(\eta^3)$ (\Cref{thm:ffa-cubic-error}), achieving convergence in monotone problems (\Cref{thm:cc-ffa}). 
\end{enumerate}

To this end, let us consider a general form of SEG that incorporates %
any arbitrary sampling scheme. 
More precisely, in the $k$-th ``epoch'' consisted of $N$ iterations, the components are chosen in the order of $\mT_{0}^k, \mT_{1}^k, \cdots, \mT_{N-1}^k$, where $\mT_i^k \in \{\mF_1, \dots, \mF_n\}$ for each~$i$. For our purpose, we assume that $N$~is some multiple of $n$ (e.g., $N = n$ for \segrr{}, $N= 2n$ for \segff{}).
Then, given $\alpha$ and $\beta$ we perform SEG updates,
for $i = 0, 1, \dots, N-1$,
\begin{equation} \label{eqn:update_rule1-m} 
\begin{aligned}
    \vw_i^k     &\gets \vz_i^k - \alpha \mT_{i}^k \vz_i^k, \\ \qquad 
    \vz_{i+1}^k &\gets \vz_i^k - \beta \mT_{i}^k \vw_i^k.
\end{aligned} 
\end{equation}

\subsubsection{Necessity of Flip-Flop Sampling}
The general method in \eqref{eqn:update_rule1-m}
that sets the initial point for the next epoch as $\vz_0^{k+1} \gets \vz_N^k$ satisfies the following property.
\begin{prop}\label{prop:unravelling-m} 
Suppose that \Cref{asmp:smoothness} holds.
For some $\veps_N^k = o\left((\alpha+\beta)^2\right)$, it holds that \begin{equation} 
    \vz_0^{k+1} =  \vz_0^k - {\beta} \sum_{j=0}^{N-1} \mT_j^k \vz_0^k + {\alpha\beta} \sum_{j=0}^{N-1} D\mT_j^k (\vz_0^k) \mT_j^k \vz_0^k +   {\beta^2} \sum_{ i < j }  D\mT_j^k (\vz_0^k) \mT_i^k \vz_0^k +  {\veps_N^k}. \label{eqn:approx-m-no-anchor}
 \end{equation} 
\end{prop}
See \Cref{sappx:unravelling} for the proof.
To make \eqref{eqn:egplus} and \eqref{eqn:approx-m-no-anchor}
match up to the second-order, both the equations
\begin{align} 
    \label{eqn:first-order}
    & \frac{\eta_2}{n} \sum_{j=1}^n \mF_i \vz_0^k
        = {\beta} \sum_{j=0}^{N-1} \mT_j^k \vz_0^k
        \qquad \text{ and} \\ 
   & \frac{\eta_1 \eta_2}{n^2} \Big(\sum_{j=1}^n  D \mF_j (\vz_0^k) \mF_j \vz_0^k + 
    \sum_{i \neq j}  D \mF_j (\vz_0^k) \mF_i \vz_0^k\Big) %
    = {\alpha\beta} \sum_{j=0}^{N-1} D\mT_j^k (\vz_0^k) \mT_j^k \vz_0^k + {\beta^2} \sum_{ i < j }  D\mT_j^k (\vz_0^k) \mT_i^k \vz_0^k   \label{eqn:second-order}
\end{align}  
must hold. %
Clearly, without-replacement sampling will make \eqref{eqn:first-order} hold. 
However, it is easy to check 
that random reshuffling falls short of making \eqref{eqn:second-order} hold. 
This is because, if RR is used, then $\mT_0^k, \mT_1^k, \dots, \mT_{n-1}^k$ is nothing but a reordering of $\mF_1, \dots, \mF_n$ into $\mF_{\tau(1)}, \dots, \mF_{\tau(n)}$, so the RHS of \eqref{eqn:second-order} can only contain terms $D\mF_{\tau(j)}(\vz_0^k) \mF_{\tau(i)} \vz_0^k$ with $i \leq j$.
This observation %
motivates the use of flip-flop sampling, because choosing $\mT_i^k = \mT_{2n-1-i}^k$ lets all the required terms $D\mF_j(\vz_0^k) \mF_i \vz_0^k$ to appear in the RHS of \eqref{eqn:second-order}. 

\subsubsection{Designing \segffa} 
Flip-flop does resolve the aforesaid issue, but still {another} complication remains for plain \segff{}. 

\begin{prop} \label{thm:ff-alone}
    Suppose we use flip-flop sampling (without anchoring).
    In order to make \eqref{eqn:first-order} and \eqref{eqn:second-order} hold, we must choose $\beta = \nicefrac{\eta_1}{n}$ and $\alpha = \nicefrac{\beta}{2}$. However, this leads to $\eta_2 = 2 \eta_1$, which is the set of parameters that fails to make EG+ converge. 
\end{prop} 
    For the proof, see \Cref{sec:ff-alone}. %
    This shows that a modification is necessary to develop a stochastic method that achieves second-order matching to %
\emph{convergent} EG/EG+ methods. 

We thus propose %
to add an \emph{anchoring} step: 
\begin{equation} \label{eqn:anchoring}
\vz_0^{k+1}  \gets \tfrac{1}{2} \left(\vz_{N}^k + \vz_0^k \right), %
\end{equation}
after finishing the $N$ updates \eqref{eqn:update_rule1-m},
instead of $\vz_0^{k+1} \gets \vz_N^k$.
This is our \emph{Stochastic ExtraGradient with Flip-Flop Anchoring} (\segffa) method, named after the design of combining the flip-flop sampling scheme and the anchoring step. 
We note that this idea of taking a convex combination %
has originally appeared in the \krasmann{} iteration \citep{Kras55, Mann53}, and also under the name of \emph{Lookahead} methods \citep{Chav21, Peth23LA}. 
This slightly differs from the more widely used \emph{Halpern iteration}~\citep{Halp67} based anchoring~(\textit{cf.\@}~\citep{yoon21}), which would have used the initial point $\vz_0^0$ instead of $\vz_0^k$ in \eqref{eqn:anchoring}. 

This anchoring step changes \eqref{eqn:approx-m-no-anchor} accordingly, and essentially amounts to dividing the right-hand sides of \eqref{eqn:first-order} and \eqref{eqn:second-order} each by $2$ (see \Cref{sec:proof-segffa} for the detailed derivations). 
We show that choosing $\alpha = \nicefrac{\beta}{2}$ in fact leads to the second-order matching to EG, i.e., EG+ with $\eta_1 = \eta_2$. 
\begin{prop} \label{thm:ffa-cubic-error}
  Suppose that Assumptions~\ref{asmp:smoothness} and \ref{asmp:bounded-variance} hold.
  Then, for $\beta_k = \eta$ and $\alpha_k = \nicefrac{\beta_k}{2}$, 
  \segffa{} becomes 
  an approximation of EG with error at most $\gO(\eta^3)$. In other words, we achieve \[
   \norm{ \vz_0^k - \eta n \mF (\vz_0^k - \eta n \mF \vz_0^k)  - \vz_0^{k+1} } = \gO(\eta^3). 
  \]
\end{prop}
In other words, adding the anchoring step allows us to get a method that well approximates the convergent EG
with an error as small as $\gO(\eta^3)$. For a more in-depth discussion, see \Cref{appx:within-epoch}.

\subsection{Convergence Analysis of \segffa}

As a result of the second-order matching, we obtain \segffa, a stochastic method that has an error of $\gO(\eta^3)$ as an approximation of EG. 
Achieving this order of magnitude for the approximation error turns out to be the key to the exact convergence to an optimum under the monotone setting. 
\begin{thm}\label{thm:cc-ffa}
    Suppose that $\mF$ is (star-)monotone, 
    Assumptions~\ref{asmp:solution},~\ref{asmp:smoothness}, and \ref{asmp:bounded-variance} hold, and we are running \segffa. 
    Then, for any $K \geq 1$, by choosing the stepsizes sufficiently small and decaying 
    as $\beta_k = {\gO}(\nicefrac{1}{k^{1/3} \log k})$ and $\alpha_k = \nicefrac{\beta_k}{2}$,
    the iterates generated by \segffa{} achieves the bound \[
        \min_{k= 0, 1, \dots, K} \expt \norm{\mF \vz_0^k}^2  = \gO\left( \frac{(\log K)^2}{K^{1/3}} \right) . 
    \]
\end{thm}
    For the full statement of the theorem and its proof, see \Cref{sappx:meta-analysis}.
We note that, although \Cref{thm:cc-ffa},
and also \Cref{thm:scsc-ffa} below, are stated specifically for \segffa,
our analyses show that
both theorems can be applied to
any method that achieves the second-order matching in terms of \Cref{thm:ffa-cubic-error}.

The reduced error also shows a gain in the rate of convergence under the strongly monotone setting. This aligns with the intuition that error hinders convergence, hence having a smaller error 
is beneficial.

\begin{thm}\label{thm:scsc-ffa}
    Suppose that $\mF$ is $\mu$-strongly monotone with $\mu>0$ and 
    \Cref{asmp:smoothness} holds. 
    Then, there exists a choice of $\eta > 0$ such that, 
    when \segffa{} is run for $K$ epochs with constant stepsizes $\beta_k = \eta$ and $\alpha_k =\nicefrac{\eta}2$, 
    for some constant $\omega$ independent of $\eta$, the iterates generated by \segffa{} achieves the bound 
\[ 
    \expt \norm{\vz_0^K - \vz^*}^2 \leq \exp\left( - \frac{1}{2} \mu \omega n K \right)  \norm{\vz_0^{0} - \vz^*}^2 + \gO\left(\frac{\left(\log(n^{1/4} K)\right)^{4}}{ n  K^{4}} \right) .  
\] \end{thm}  
    
\Cref{thm:scsc-ffa} actually stems from a unified analysis that encompasses all the shuffling-based SEG methods introduced in this paper, including \segrr{} and \segff{}. See \Cref{sappx:sm-analysis} for the details. 

Notice the exponent $4$ of the number of epochs $K$ in the convergence rate, which is twice as large as the exponent $2$ of \sgdarr{} and %
\segrr{}.
In fact, this gain in the rate of convergence turns out to be fundamental. 
As we show in the following theorem, the theoretical lower bounds of convergence for \sgdarr{} and \segrr{} with constant stepsize are both $\Omega(\nicefrac{1}{nK^3})$. This exhibits that there is a \emph{provable gap} between those methods and \segffa{}, which attains $\tilde{\gO}(\nicefrac{1}{nK^{4}})$.

\begin{thm}
\label{thm:scsc-lower-bound}
    Suppose $n \geq 2$. For both \sgdarr{} with constant stepsize $\alpha_k = \alpha > 0$ and \segrr{} with constant stepsize $\alpha_k = \alpha > 0$, $\beta_k = \beta > 0$,
    there exists a $\mu$-strongly monotone minimax problem $f(\vz) = \frac{1}{n} \sum_{i=1}^n f_i(\vz)$ with $\mu>0$ such that regardless of stepsizes, we have
    \begin{equation*}
        \E \left [\norm{\vz^K_0 - \vz^*}^2 \right ] 
        = 
        \begin{cases}
        \Omega \left ( 
        \frac{\sigma^2}{L \mu n K}
        \right )    & \text{ if }K \leq L/\mu,  \\
        \Omega \left ( 
        \frac{L\sigma^2}{\mu^3 n K^3}
        \right )    & \text{ if }K > L/\mu. \\
        \end{cases}
    \end{equation*}
\end{thm} 

\begin{proof}
    The full statement and the proof are presented in \Cref{sec:appx:lower-bounds}. 
\end{proof}

\section{Experiments} \label{sec:experiments-shortlist}
We consider 
randomly generated quadratic problems of the form \begin{equation}\label{eqn:experiment-cc}
\min_{\vx \in \sR^{d_x}}\!\max_{\vy \in \sR^{d_y}} \: 
\frac{1}{n} 
\sum_{i=1}^{n} \ 
\begin{bmatrix}
    \vx \\ \vy
\end{bmatrix}^{\!\top}\! \!
\begin{bmatrix}
    \mA_i \! & \! \mB_i \\ \mB_i^\top \! & \! -\mC_i
\end{bmatrix} \!
\begin{bmatrix}
    \vx \\ \vy
\end{bmatrix} - \vt_i^\top \!
\begin{bmatrix}
    \vx \\ \vy
\end{bmatrix}\!.
\end{equation} 
In particular, we sample the random components 
so that the full objective is 
either monotone %
or strongly monotone, respectively, while each of the components may be nonmonotone. For the exact descriptions on how we constructed the problems, see \Cref{appx:problem-sampling}. 

\begin{figure}[tb]
    \centering
    \phantom{.} \hfill
    \includegraphics[height=5cm]{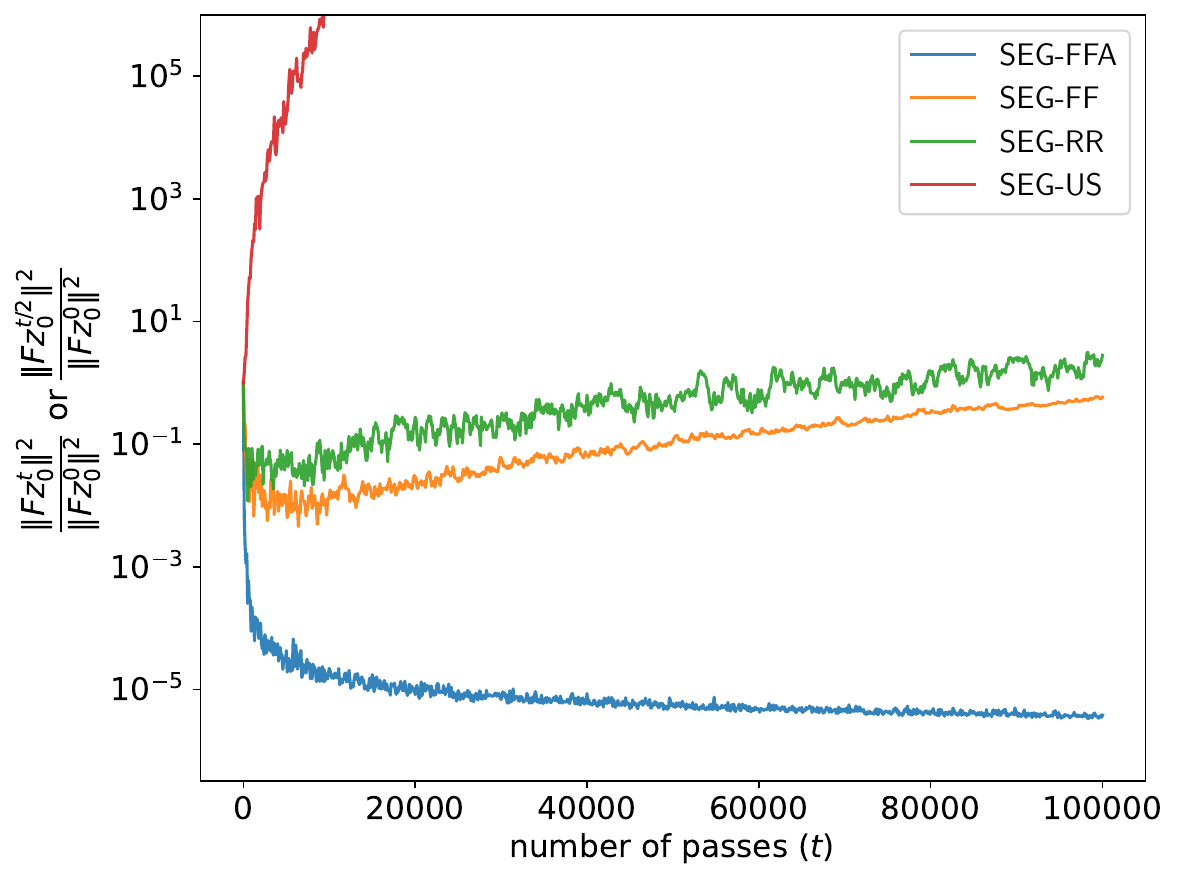}
    \hfill
    \includegraphics[height=5cm]{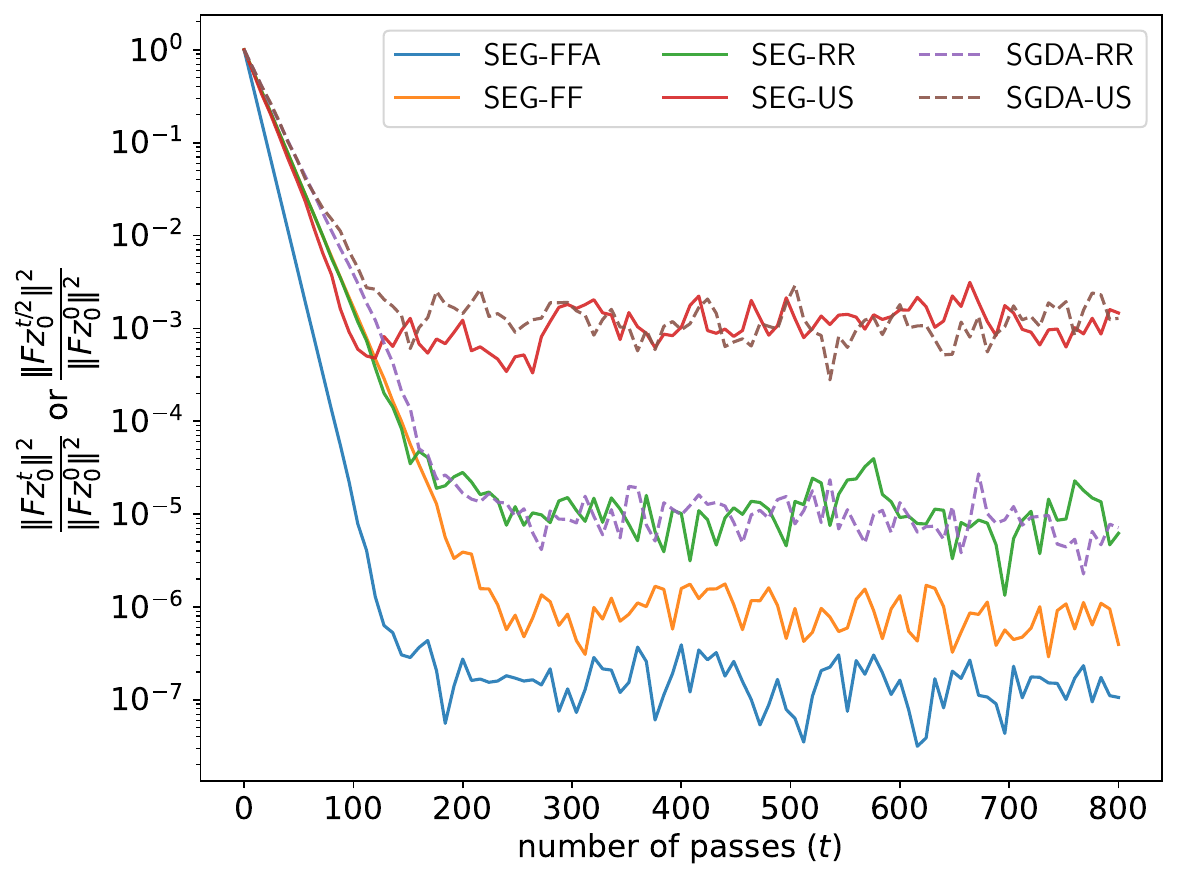}
    \hfill \phantom{.}
    \caption{Experimental results on the (left) monotone and (right) strongly monotone examples, comparing the variants of SEG. For a fair comparison, we take the number of passes over the full dataset as the abscissae. 
    In other words, we plot $\nicefrac{{\|{\mF \vz_0^{t/2}}\|^2}}{\|{\mF \vz_0^0}\|^2}$ for \segffa{} and \segff{}, as they pass through the whole dataset twice every epoch, and $\nicefrac{{\|{\mF \vz_0^{t}}\|^2}}{\|{\mF \vz_0^0}\|^2}$ for the other methods, as they pass once every epoch. }
    \label{fig:cc-shortlist}
\end{figure}

\paragraph{Monotone Case} 

\looseness=-1
We ran the experiment on $5$ random instances of \eqref{eqn:experiment-cc} with the stepsizes scheduled as $\eta_k = \nicefrac{\eta_0}{(1+k/10)^{0.34}}$ where $\eta_0 = \min\{0.01, \frac{1}{L}\}$ for \segffa, and $\alpha_k = \beta_k = \eta_k$ for \segus, \segrr, and \segff. 
The exponent $0.34$ is to ensure a sufficient decay rate required by \Cref{thm:cc-ffa}, and the convergence of \segffa{} under such a stepsize scheduling is validated in \Cref{rmk:poly-stepsize-validation}. 
The value of $\eta_0$ is, however, a heuristically determined small number. 
The results of the geometric mean over the $5$ runs
are plotted in \Cref{fig:cc-shortlist}.
As expected by our theory,
\segffa{} successfully shows convergence, while all of \segff, \segrr, and \segus{} diverge in the long~run. 

\paragraph{Strongly monotone case}

Along with the variants of SEG, we also compare the performances of \sgdarr{} and \sgdaus{}. 
We ran the experiment on $5$ random instances of \eqref{eqn:experiment-cc} with stepsizes $\eta_k = 0.001$, and the results
are plotted in \Cref{fig:cc-shortlist}. Additional results obtained from using other stepsizes can be found in \Cref{appx:ssec-scsc-longlist}. 
We again observe an agreement between the empirical results and our theory; 
\segffa{} eventually finds the point with the smallest gradient norm among the methods that are considered. 

Further additional experiments and ablation studies we have conducted can be found in~\Cref{appx:experiments}. 

\section{Conclusion}
We proposed \segffa{}, a new stochastic variant of EG that uses flip-flop sampling and anchoring. 
While being a minimal modification from the vanilla SEG, \segffa{} attains the crucial ``second-order matching'' property to the deterministic EG, leading to a two-fold improved convergence. 
On one hand, \segffa{} reaches
an optimum in the monotone setting, unlike many baseline methods such as \segus{}, \segrr{}, and \segff{} that diverge.
Moreover, in the strongly monotone setting, \segffa{} shows a faster convergence with a provable gap from the other methods.  

An interesting future direction would be to extend our work to %
more general nonconvex-nonconcave
problems, further exploring the potentials of the second-order matching technique.
It would also be appealing to further study whether it is possible to devise a new method that achieves second-order (or higher) matching without the anchoring step, potentially enhancing our understanding of the effectiveness of the matching technique.

\acksection

This work was supported in part by the National Research Foundation of Korea (NRF) grant funded by the Korea government (MSIT) (No.\@~RS-2019-NR040050). 
JC and DK acknowledge support from the NRF grant (No.\@~RS-2022-NR071715) funded by the Korea government (MSIT), and the Samsung Science \& Technology Foundation grant (No.\@~SSTF-BA2101-02). 
CY acknowledges support from the NRF grant (No.\@~RS-2023-00211352) funded by the Korea government (MSIT).

\bibliographystyle{plainnat}
\bibliography{biblio.bib}

\begin{thebibliography}{56}
\providecommand{\natexlab}[1]{#1}
\providecommand{\url}[1]{\texttt{#1}}
\expandafter\ifx\csname urlstyle\endcsname\relax
  \providecommand{\doi}[1]{doi: #1}\else
  \providecommand{\doi}{doi: \begingroup \urlstyle{rm}\Url}\fi

\bibitem[Ahn et~al.(2020)Ahn, Yun, and Sra]{Ahn20}
Kwangjun Ahn, Chulhee Yun, and Suvrit Sra.
\newblock {SGD} with shuffling: optimal rates without component convexity and
  large epoch requirements.
\newblock \emph{Advances in Neural Information Processing Systems},
  33:\penalty0 17526--17535, 2020.

\bibitem[Alacaoglu and Malitsky(2022)]{Alac22}
Ahmet Alacaoglu and Yura Malitsky.
\newblock Stochastic variance reduction for variational inequality methods.
\newblock In \emph{Conference on Learning Theory}, pages 778--816. PMLR, 2022.

\bibitem[Arrow and Hurwicz(1956)]{Arro56}
Kenneth~J. Arrow and Leonid Hurwicz.
\newblock Reduction of constrained maxima to saddle-point problems.
\newblock In \emph{{Proc. Third Berkeley Symp. on Math. Statist. and Prob.}},
  volume~5, pages {1--20}, 1956.
\newblock Univ. of Calif. Press.

\bibitem[Azizian et~al.(2020)Azizian, Mitliagkas, Lacoste-Julien, and
  Gidel]{Aziz20}
Wa{\"\i}ss Azizian, Ioannis Mitliagkas, Simon Lacoste-Julien, and Gauthier
  Gidel.
\newblock A tight and unified analysis of gradient-based methods for a whole
  spectrum of differentiable games.
\newblock In \emph{International Conference on Artificial Intelligence and
  Statistics}, pages 2863--2873. PMLR, 2020.

\bibitem[Bauschke and Combettes(2017)]{Baus17}
Heinz~H. Bauschke and Patrick~L. Combettes.
\newblock \emph{Convex Analysis and Monotone Operator Theory in Hilbert
  Spaces}.
\newblock Springer, 2nd edition, 2017.

\bibitem[Beznosikov et~al.(2020)Beznosikov, Samokhin, and Gasnikov]{Bezn20}
Aleksandr Beznosikov, Valentin Samokhin, and Alexander Gasnikov.
\newblock Distributed saddle-point problems: Lower bounds, near-optimal and
  robust algorithms.
\newblock \emph{arXiv preprint arXiv:2010.13112}, 2020.

\bibitem[Beznosikov et~al.(2023)Beznosikov, Polyak, Gorbunov, Kovalev, and
  Gasnikov]{Bezn23a}
Aleksandr Beznosikov, Boris Polyak, Eduard Gorbunov, Dmitry Kovalev, and
  Alexander Gasnikov.
\newblock Smooth monotone stochastic variational inequalities and saddle point
  problems: A survey.
\newblock \emph{European Mathematical Society Magazine}, 127:\penalty0 15--28,
  2023.

\bibitem[Bottou(2009)]{Bott09}
L{\'e}on Bottou.
\newblock Curiously fast convergence of some stochastic gradient descent
  algorithms.
\newblock In \emph{Proceedings of the symposium on learning and data science,
  Paris}, volume~8, pages 2624--2633. Citeseer, 2009.

\bibitem[Cai et~al.(2022)Cai, Song, Guzm{\'a}n, and Diakonikolas]{Cai22}
Xufeng Cai, Chaobing Song, Crist{\'o}bal Guzm{\'a}n, and Jelena Diakonikolas.
\newblock Stochastic {Halpern} iteration with variance reduction for stochastic
  monotone inclusions.
\newblock \emph{Advances in Neural Information Processing Systems},
  35:\penalty0 24766--24779, 2022.

\bibitem[Carmon et~al.(2019)Carmon, Jin, Sidford, and Tian]{Carm19}
Yair Carmon, Yujia Jin, Aaron Sidford, and Kevin Tian.
\newblock Variance reduction for matrix games.
\newblock \emph{Advances in Neural Information Processing Systems}, 32, 2019.

\bibitem[Cha et~al.(2023)Cha, Lee, and Yun]{Cha23}
Jaeyoung Cha, Jaewook Lee, and Chulhee Yun.
\newblock Tighter lower bounds for shuffling {SGD}: Random permutations and
  beyond.
\newblock In \emph{International Conference on Machine Learning}, pages
  3855--3912. PMLR, 2023.

\bibitem[Chavdarova et~al.(2021)Chavdarova, Pagliardini, Stich, Fleuret, and
  Jaggi]{Chav21}
Tatjana Chavdarova, Matteo Pagliardini, Sebastian~U Stich, Fran{\c{c}}ois
  Fleuret, and Martin Jaggi.
\newblock Taming {GANs} with {Lookahead}-minmax.
\newblock In \emph{The Ninth International Conference on Learning
  Representations}, 2021.

\bibitem[Cho and Yun(2023)]{Cho22}
Hanseul Cho and Chulhee Yun.
\newblock {SGDA} with shuffling: faster convergence for nonconvex-{P{\L}}
  minimax optimization.
\newblock In \emph{The Eleventh International Conference on Learning
  Representations}, 2023.

\bibitem[Choudhury et~al.(2023)Choudhury, Gorbunov, and Loizou]{Chou23}
Sayantan Choudhury, Eduard Gorbunov, and Nicolas Loizou.
\newblock Single-call stochastic extragradient methods for structured
  non-monotone variational inequalities: Improved analysis under weaker
  conditions.
\newblock \emph{Advances in Neural Information Processing Systems},
  36:\penalty0 64918--64956, 2023.

\bibitem[Das et~al.(2022)Das, Sch{\"o}lkopf, and Muehlebach]{Das22}
Aniket Das, Bernhard Sch{\"o}lkopf, and Michael Muehlebach.
\newblock Sampling without replacement leads to faster rates in finite-sum
  minimax optimization.
\newblock \emph{Advances in Neural Information Processing Systems},
  35:\penalty0 6749--6762, 2022.

\bibitem[Defazio and Bottou(2019)]{Defa19}
Aaron Defazio and L{\'e}on Bottou.
\newblock On the ineffectiveness of variance reduced optimization for deep
  learning.
\newblock \emph{Advances in Neural Information Processing Systems}, 32, 2019.

\bibitem[Diakonikolas et~al.(2021)Diakonikolas, Daskalakis, and Jordan]{Diak21}
Jelena Diakonikolas, Constantinos Daskalakis, and Michael~I. Jordan.
\newblock Efficient methods for structured nonconvex-nonconcave min-max
  optimization.
\newblock In \emph{International Conference on Artificial Intelligence and
  Statistics}, pages 2746--2754. PMLR, 2021.

\bibitem[Emmanouilidis et~al.(2024)Emmanouilidis, Vidal, and Loizou]{Emma24}
Konstantinos Emmanouilidis, Ren{\'e} Vidal, and Nicolas Loizou.
\newblock Stochastic extragradient with random reshuffling: Improved
  convergence for variational inequalities.
\newblock In \emph{International Conference on Artificial Intelligence and
  Statistics}, pages 3682--3690. PMLR, 2024.

\bibitem[Goodfellow et~al.(2014)Goodfellow, Pouget-Abadie, Mirza, Xu,
  Warde-Farley, Ozair, Courville, and Bengio]{Good14}
Ian Goodfellow, Jean Pouget-Abadie, Mehdi Mirza, Bing Xu, David Warde-Farley,
  Sherjil Ozair, Aaron Courville, and Yoshua Bengio.
\newblock Generative adversarial nets.
\newblock \emph{Advances in neural information processing systems}, 27, 2014.

\bibitem[Gorbunov et~al.(2022{\natexlab{a}})Gorbunov, Berard, Gidel, and
  Loizou]{Gorb22SEG}
Eduard Gorbunov, Hugo Berard, Gauthier Gidel, and Nicolas Loizou.
\newblock Stochastic extragradient: General analysis and improved rates.
\newblock In \emph{International Conference on Artificial Intelligence and
  Statistics}, pages 7865--7901. PMLR, 2022{\natexlab{a}}.

\bibitem[Gorbunov et~al.(2022{\natexlab{b}})Gorbunov, Loizou, and
  Gidel]{Gorb22EG}
Eduard Gorbunov, Nicolas Loizou, and Gauthier Gidel.
\newblock Extragradient method: {$O(1/K)$} last-iterate convergence for
  monotone variational inequalities and connections with cocoercivity.
\newblock In \emph{International Conference on Artificial Intelligence and
  Statistics}, pages 366--402. PMLR, 2022{\natexlab{b}}.

\bibitem[Grimmer et~al.(2023)Grimmer, Lu, Worah, and Mirrokni]{Grim23}
Benjamin Grimmer, Haihao Lu, Pratik Worah, and Vahab Mirrokni.
\newblock The landscape of the proximal point method for nonconvex--nonconcave
  minimax optimization.
\newblock \emph{Mathematical Programming}, 201\penalty0 (1-2):\penalty0
  373--407, 2023.

\bibitem[Halpern(1967)]{Halp67}
Benjamin Halpern.
\newblock {Fixed points of nonexpanding maps}.
\newblock \emph{Bulletin of the American Mathematical Society}, 73\penalty0
  (6):\penalty0 957--961, 1967.

\bibitem[Harris et~al.(2020)Harris, Millman, van~der Walt, Gommers, Virtanen,
  Cournapeau, Wieser, Taylor, Berg, Smith, Kern, Picus, Hoyer, van Kerkwijk,
  Brett, Haldane, del R{\'{i}}o, Wiebe, Peterson, G{\'{e}}rard-Marchant,
  Sheppard, Reddy, Weckesser, Abbasi, Gohlke, and Oliphant]{NumPy}
Charles~R. Harris, K.~Jarrod Millman, St{\'{e}}fan~J. van~der Walt, Ralf
  Gommers, Pauli Virtanen, David Cournapeau, Eric Wieser, Julian Taylor,
  Sebastian Berg, Nathaniel~J. Smith, Robert Kern, Matti Picus, Stephan Hoyer,
  Marten~H. van Kerkwijk, Matthew Brett, Allan Haldane, Jaime~Fern{\'{a}}ndez
  del R{\'{i}}o, Mark Wiebe, Pearu Peterson, Pierre G{\'{e}}rard-Marchant,
  Kevin Sheppard, Tyler Reddy, Warren Weckesser, Hameer Abbasi, Christoph
  Gohlke, and Travis~E. Oliphant.
\newblock Array programming with {NumPy}.
\newblock \emph{Nature}, 585\penalty0 (7825):\penalty0 357--362, September
  2020.

\bibitem[Hsieh et~al.(2020)Hsieh, Iutzeler, Malick, and Mertikopoulos]{Hsie20}
Yu-Guan Hsieh, Franck Iutzeler, J{\'e}r{\^o}me Malick, and Panayotis
  Mertikopoulos.
\newblock Explore aggressively, update conservatively: Stochastic extragradient
  methods with variable stepsize scaling.
\newblock \emph{Advances in Neural Information Processing Systems},
  33:\penalty0 16223--16234, 2020.

\bibitem[Hunter(2007)]{Matplotlib}
J.~D. Hunter.
\newblock Matplotlib: A 2d graphics environment.
\newblock \emph{Computing in Science \& Engineering}, 9\penalty0 (3):\penalty0
  90--95, 2007.

\bibitem[Juditsky et~al.(2011)Juditsky, Nemirovski, and Tauvel]{Judi11}
Anatoli Juditsky, Arkadi Nemirovski, and Claire Tauvel.
\newblock Solving variational inequalities with stochastic mirror-prox
  algorithm.
\newblock \emph{Stochastic Systems}, 1\penalty0 (1):\penalty0 17--58, 2011.

\bibitem[Kim et~al.(2024)Kim, Lai, Liao, Murata, Takida, Uesaka, He, Mitsufuji,
  and Ermon]{Kim24CTM}
Dongjun Kim, Chieh-Hsin Lai, Wei-Hsiang Liao, Naoki Murata, Yuhta Takida,
  Toshimitsu Uesaka, Yutong He, Yuki Mitsufuji, and Stefano Ermon.
\newblock Consistency trajectory models: Learning probability flow {ODE}
  trajectory of diffusion.
\newblock In \emph{International Conference on Learning Representations}, 2024.

\bibitem[Korpelevich(1976)]{Korp76}
Galina~M. Korpelevich.
\newblock The extragradient method for finding saddle points and other
  problems.
\newblock \emph{Matecon}, 12:\penalty0 747--756, 1976.

\bibitem[{Krasnosel'ski\u{\i}}(1955)]{Kras55}
M.~A. {Krasnosel'ski\u{\i}}.
\newblock Two remarks on the method of successive approximations.
\newblock \emph{Uspekhi Matematicheskikh Nauk}, 10:\penalty0 123--127, 1955.

\bibitem[Loizou et~al.(2021)Loizou, Berard, Gidel, Mitliagkas, and
  Lacoste-Julien]{Loiz21}
Nicolas Loizou, Hugo Berard, Gauthier Gidel, Ioannis Mitliagkas, and Simon
  Lacoste-Julien.
\newblock Stochastic gradient descent-ascent and consensus optimization for
  smooth games: Convergence analysis under expected co-coercivity.
\newblock \emph{Advances in Neural Information Processing Systems},
  34:\penalty0 19095--19108, 2021.

\bibitem[Lu et~al.(2022)Lu, Guo, and De~Sa]{Lu22}
Yucheng Lu, Wentao Guo, and Christopher De~Sa.
\newblock {GraB}: Finding provably better data permutations than random
  reshuffling.
\newblock \emph{Advances in Neural Information Processing Systems},
  35:\penalty0 8969--8981, 2022.

\bibitem[Mann(1953)]{Mann53}
W.~Robert Mann.
\newblock Mean value methods in iteration.
\newblock \emph{Proceedings of the American Mathematical Society}, 4\penalty0
  (3):\penalty0 506--510, 1953.

\bibitem[Martinet(1970)]{Mart70}
Bernard Martinet.
\newblock Regularisation d'in{\'e}quations variationelles par approximations
  succesives.
\newblock \emph{Revue Fran{\c{c}}aise d'informatique et de Recherche
  op{\'e}rationelle}, 1970.

\bibitem[Mishchenko et~al.(2020{\natexlab{a}})Mishchenko, Khaled, and
  Richt{\'a}rik]{Mish20SGD}
Konstantin Mishchenko, Ahmed Khaled, and Peter Richt{\'a}rik.
\newblock Random reshuffling: Simple analysis with vast improvements.
\newblock \emph{Advances in Neural Information Processing Systems},
  33:\penalty0 17309--17320, 2020{\natexlab{a}}.

\bibitem[Mishchenko et~al.(2020{\natexlab{b}})Mishchenko, Kovalev, Shulgin,
  Richt{\'a}rik, and Malitsky]{Mish20SEG}
Konstantin Mishchenko, Dmitry Kovalev, Egor Shulgin, Peter Richt{\'a}rik, and
  Yura Malitsky.
\newblock Revisiting stochastic extragradient.
\newblock In \emph{International Conference on Artificial Intelligence and
  Statistics}, pages 4573--4582. PMLR, 2020{\natexlab{b}}.

\bibitem[{M\k{a}dry} et~al.(2018){M\k{a}dry}, Makelov, Schmdit, Tsipras, and
  Vladu]{Madr18}
Aleksander {M\k{a}dry}, Aleksandar Makelov, Ludwig Schmdit, Dimitris Tsipras,
  and Adrian Vladu.
\newblock Towards deep learning models resistant to adversarial attacks.
\newblock In \emph{International Conference on Learning Representations}, 2018.

\bibitem[Mokhtari et~al.(2020)Mokhtari, Ozdaglar, and Pattathil]{Mokh20}
Aryan Mokhtari, Asuman Ozdaglar, and Sarath Pattathil.
\newblock A unified analysis of extra-gradient and optimistic gradient methods
  for saddle point problems: Proximal point approach.
\newblock In \emph{International Conference on Artificial Intelligence and
  Statistics}, pages 1497--1507. PMLR, 2020.

\bibitem[Nagaraj et~al.(2019)Nagaraj, Jain, and Netrapalli]{Naga19}
Dheeraj Nagaraj, Prateek Jain, and Praneeth Netrapalli.
\newblock {SGD} without replacement: Sharper rates for general smooth convex
  functions.
\newblock In \emph{International Conference on Machine Learning}, pages
  4703--4711. PMLR, 2019.

\bibitem[Nesterov(2018)]{Nest18}
Yurii Nesterov.
\newblock \emph{Lectures on convex optimization}, volume 137 of \emph{Springer
  Optimization and Its Applications}.
\newblock Springer, second edition, 2018.

\bibitem[Nguyen et~al.(2021)Nguyen, Tran-Dinh, Phan, Nguyen, and
  Van~Dijk]{Nguy21}
Lam~M. Nguyen, Quoc Tran-Dinh, Dzung~T. Phan, Phuong~Ha Nguyen, and Marten
  Van~Dijk.
\newblock A unified convergence analysis for shuffling-type gradient methods.
\newblock \emph{The Journal of Machine Learning Research}, 22\penalty0
  (1):\penalty0 9397--9440, 2021.

\bibitem[Pethick et~al.(2023{\natexlab{a}})Pethick, Fercoq, Latafat, Patrinos,
  and Cevher]{Peth23}
Thomas Pethick, Olivier Fercoq, Puya Latafat, Panagiotis Patrinos, and Volkan
  Cevher.
\newblock Solving stochastic weak {Minty} variational inequalities without
  increasing batch size.
\newblock In \emph{International Conference on Learning Representations},
  2023{\natexlab{a}}.

\bibitem[Pethick et~al.(2023{\natexlab{b}})Pethick, Xie, and Cevher]{Peth23LA}
Thomas Pethick, Wanyun Xie, and Volkan Cevher.
\newblock Stable nonconvex-nonconcave training via linear interpolation.
\newblock \emph{Advances in Neural Information Processing Systems}, 37,
  2023{\natexlab{b}}.

\bibitem[Popov(1980)]{Popo80}
L.~D. Popov.
\newblock A modification of the {Arrow}-{Hurwitz} method of search for saddle
  points.
\newblock \emph{Matematicheskie Zametki}, 28\penalty0 (5):\penalty0 777--784,
  1980.

\bibitem[Rajput et~al.(2020)Rajput, Gupta, and Papailiopoulos]{Rajp20}
Shashank Rajput, Anant Gupta, and Dimitris Papailiopoulos.
\newblock Closing the convergence gap of {SGD} without replacement.
\newblock In \emph{International Conference on Machine Learning}, pages
  7964--7973. PMLR, 2020.

\bibitem[Rajput et~al.(2022)Rajput, Lee, and Papailiopoulos]{Rajp22}
Shashank Rajput, Kangwook Lee, and Dimitris Papailiopoulos.
\newblock Permutation-based {SGD}: Is random optimal?
\newblock In \emph{International Conference on Learning Representations}, 2022.

\bibitem[Recht and R{\'e}(2013)]{Rech13}
Benjamin Recht and Christopher R{\'e}.
\newblock Parallel stochastic gradient algorithms for large-scale matrix
  completion.
\newblock \emph{Mathematical Programming Computation}, 5\penalty0 (2):\penalty0
  201--226, 2013.

\bibitem[Rout et~al.(2022)Rout, Korotin, and Burnaev]{Rout22}
Litu Rout, Alexander Korotin, and Evgeny Burnaev.
\newblock Generative modeling with optimal transport maps.
\newblock In \emph{International Conference on Learning Representations}, 2022.

\bibitem[Safran and Shamir(2020)]{Safr20}
Itay Safran and Ohad Shamir.
\newblock How good is {SGD} with random shuffling?
\newblock In \emph{Conference on Learning Theory}, pages 3250--3284. PMLR,
  2020.

\bibitem[Safran and Shamir(2021)]{Safr21}
Itay Safran and Ohad Shamir.
\newblock Random shuffling beats {SGD} only after many epochs on
  ill-conditioned problems.
\newblock \emph{Advances in Neural Information Processing Systems},
  34:\penalty0 15151--15161, 2021.

\bibitem[Solodov and Svaiter(1999)]{Solo99}
Mikhail~V. Solodov and Benar~F. Svaiter.
\newblock A hybrid approximate extragradient--proximal point algorithm using
  the enlargement of a maximal monotone operator.
\newblock \emph{Set-Valued Analysis}, 7\penalty0 (4):\penalty0 323--345, 1999.

\bibitem[Virtanen et~al.(2020)Virtanen, Gommers, Oliphant, Haberland, Reddy,
  Cournapeau, Burovski, Peterson, Weckesser, Bright, {van der Walt}, Brett,
  Wilson, Millman, Mayorov, Nelson, Jones, Kern, Larson, Carey, Polat, Feng,
  Moore, {VanderPlas}, Laxalde, Perktold, Cimrman, Henriksen, Quintero, Harris,
  Archibald, Ribeiro, Pedregosa, {van Mulbregt}, and {SciPy 1.0
  Contributors}]{SciPy}
Pauli Virtanen, Ralf Gommers, Travis~E. Oliphant, Matt Haberland, Tyler Reddy,
  David Cournapeau, Evgeni Burovski, Pearu Peterson, Warren Weckesser, Jonathan
  Bright, St{\'e}fan~J. {van der Walt}, Matthew Brett, Joshua Wilson, K.~Jarrod
  Millman, Nikolay Mayorov, Andrew R.~J. Nelson, Eric Jones, Robert Kern, Eric
  Larson, CJ~Carey, {\.I}lhan Polat, Yu~Feng, Eric~W. Moore, Jake {VanderPlas},
  Denis Laxalde, Josef Perktold, Robert Cimrman, Ian Henriksen, E.~A. Quintero,
  Charles~R. Harris, Anne~M. Archibald, Ant{\^o}nio~H. Ribeiro, Fabian
  Pedregosa, Paul {van Mulbregt}, and {SciPy 1.0 Contributors}.
\newblock {{SciPy} 1.0: Fundamental Algorithms for Scientific Computing in
  Python}.
\newblock \emph{Nature Methods}, 17:\penalty0 261--272, 2020.

\bibitem[Wai et~al.(2018)Wai, Yang, Wang, and Hong]{Wai18}
Hoi-To Wai, Zhuoran Yang, Zhaoran Wang, and Mingyi Hong.
\newblock Multi-agent reinforcement learning via double averaging primal-dual
  optimization.
\newblock \emph{Advances in Neural Information Processing Systems}, 31, 2018.

\bibitem[Yoon and Ryu(2021)]{yoon21}
TaeHo Yoon and Ernest~K. Ryu.
\newblock {Accelerated Algorithms for Smooth Convex-Concave Minimax Problems
  with {$\mathcal{O} (1/k^2)$} Rate on Squared Gradient Norm}.
\newblock In \emph{International Conference on Machine Learning}, pages
  12098--12109. PMLR, 2021.

\bibitem[Yun et~al.(2021)Yun, Sra, and Jadbabaie]{Yun21}
Chulhee Yun, Suvrit Sra, and Ali Jadbabaie.
\newblock Open problem: Can single-shuffle {SGD} be better than reshuffling
  {SGD} and {GD}?
\newblock In \emph{Conference on Learning Theory}, pages 4653--4658. PMLR,
  2021.

\bibitem[Yun et~al.(2022)Yun, Rajput, and Sra]{Yun22}
Chulhee Yun, Shashank Rajput, and Suvrit Sra.
\newblock Minibatch vs local {SGD} with shuffling: Tight convergence bounds and
  beyond.
\newblock In \emph{International Conference on Learning Representations}, 2022.

\end{thebibliography}

\medskip

\appendix

\newpage 
\newpage 

\tableofcontents
\newpage

\section{Pseudocode of the Algorithms} \label{sec:pseudocode}

We present the pseudocode of the algorithms we consider in this paper in \Cref{alg:segrr,,alg:segff,,alg:segffa}, with the pseudocode of the with-replacement stochastic methods in \Cref{alg:segus}.  

\begin{algorithm}[H]
\caption{\segus{} / {}\sgdaus}
\label{alg:segus}
\begin{algorithmic}
    \STATE {\bfseries Input:} {The number of components $n$; stepsize sequences $\{\alpha_t\}_{t \geq 0}$ and $\{\beta_t\}_{t \geq 0}$}
    \STATE {\bfseries Initialize:} {$\vz_0 \in \sR^{d_1+d_2}$}
    \FOR{$t= 0, 1, \dots$} 
        \STATE {sample $i(t)$ uniformly from $\{1, \dots, n\}$}\;
        \IF{\sgdaus} 
            \STATE $\vz_{t+1} \gets \vz_t - \alpha_t \mF_{i(t)} \vz_t$\;
        \ELSIF{\segus} 
           \STATE $\vw_t \gets \vz_t - \alpha_t \mF_{i(t)} \vz_t$\;
           \STATE $\vz_{t+1} \gets \vz_t - \beta_t \mF_{i(t)} \vw_t$\;
        \ENDIF
    \ENDFOR
    \end{algorithmic}
\end{algorithm} 

\begin{algorithm}[H]
    \caption{\segrr{} / {}\sgdarr}
    \label{alg:segrr}
    \begin{algorithmic} 
    \STATE {\bfseries Input:} {The number of components $n$; stepsize sequences $\{\alpha_k\}_{k \geq 0}$ and $\{\beta_k\}_{k \geq 0}$}
    \STATE {\bfseries Initialize:} {$\vz_0^0 \in \sR^{d_1+d_2}$} 
    \FOR{$k= 0, 1, \dots$}
        \STATE {sample $\tau_k$ uniformly from $\gS_n$}\;
        \FOR {$i = 0$ {\bfseries to} $n-1$} 
            \IF{\sgdarr} 
                \STATE $\vz_{i+1}^k \gets \vz_i^k - \alpha_k \mF_{\tau_k(i+1)} \vz_i^k$\;
            \ELSIF{\segrr} 
                \STATE $\vw_i^k \gets \vz_i^k - \alpha_k \mF_{\tau_k(i+1)} \vz_i^k$\;
                \STATE $\vz_{i+1}^k \gets \vz_i^k - \beta_k \mF_{\tau_k(i+1)} \vw_i^k$\;
            \ENDIF
        \ENDFOR
        \STATE $\vz_{0}^{k+1} \gets \vz_{n}^k$\;
    \ENDFOR 
    \end{algorithmic}
\end{algorithm} 
     
\begin{algorithm}[H]
    \caption{\segff}
    \label{alg:segff}
    \begin{algorithmic} 
    \STATE {\bfseries Input:} {The number of components $n$; stepsize sequences $\{\alpha_k\}_{k \geq 0}$ and $\{\beta_k\}_{k \geq 0}$}
    \STATE {\bfseries Initialize:} {$\vz_0^0 \in \sR^{d_1+d_2}$} 
    \FOR{$k= 0, 1, \dots$}
        \STATE {sample $\tau_k$ uniformly from $\gS_n$}\;
        \FOR {$i = 0$ {\bfseries to} $n-1$} 
            \STATE $\vw_i^k \gets \vz_i^k - \alpha_k \mF_{\tau_k(i+1)} \vz_i^k$\;
            \STATE $\vz_{i+1}^k \gets \vz_i^k - \beta_k \mF_{\tau_k(i+1)} \vw_i^k$\;
        \ENDFOR
        \FOR {$i = n$ {\bfseries to} $2n-1$} 
            \STATE $\vw_i^k \gets \vz_i^k - \alpha_k \mF_{\tau_k(2n-i)} \vz_i^k$\;
            \STATE $\vz_{i+1}^k \gets \vz_i^k - \beta_k \mF_{\tau_k(2n-i)} \vw_i^k$\; 
        \ENDFOR
        \STATE $\vz_{0}^{k+1} \gets \vz_{2n}^k$\;
    \ENDFOR 
    \end{algorithmic}
\end{algorithm} 

\begin{algorithm}[H]
    \caption{\segffa}
    \label{alg:segffa}
    \begin{algorithmic} 
    \STATE {\bfseries Input:} {The number of components $n$; stepsize sequences $\{\eta_k\}_{k \geq 0}$}
    \STATE {\bfseries Initialize:} {$\vz_0^0 \in \sR^{d_1+d_2}$} 
    \FOR{$k= 0, 1, \dots$} 
        \STATE {sample $\tau_k$ uniformly from $\gS_n$}\;
        \FOR {$i = 0$ {\bfseries to} $n-1$} 
            \STATE $\vw_i^k \gets \vz_i^k - \frac{\eta_k}{2} \mF_{\tau_k(i+1)} \vz_i^k$\;
            \STATE $\vz_{i+1}^k \gets \vz_i^k - \eta_k \mF_{\tau_k(i+1)} \vw_i^k$\;
        \ENDFOR
        \FOR {$i = n$ {\bfseries to} $2n-1$} 
            \STATE $\vw_i^k \gets \vz_i^k - \frac{\eta_k}{2} \mF_{\tau_k(2n-i)} \vz_i^k$\;
            \STATE $\vz_{i+1}^k \gets \vz_i^k - \eta_k \mF_{\tau_k(2n-i)} \vw_i^k$\;
        \ENDFOR
        \STATE $\vz_{0}^{k+1} \gets \frac{\vz_0^k + \vz_{2n}^k}{2}$\;
    \ENDFOR 
    \end{algorithmic}
\end{algorithm} 

\section{Further Details and Discussions on the Related Works} \label{appx:more-related-works}

\subsection{A Summary of the Limitations of the Existing Works in the Monotone Setting} \label{appx:comparison-table}
In \Cref{tab:related-works}, we have summarized the settings considered in each of the previous works on stochastic variants of EG discussed in \Cref{sec:related-works}, and compare them with our settings.
Please note that we focus on the monotone $\mF$ setting in the table.
Entries that are worth further discussions are marked, with the corresponding explanations below.

\begin{table}[thpb] 
      \centering 
      \caption{Comparison of the underlying settings between ours and the existing works.}
\label{tab:related-works} 
\centerline{
\begin{tabular}{|c|c|c|c|c|c|} 
\hline %
    & \begin{tabular}[c]{@{}c@{}}{same} \\{sample?} \end{tabular}
    & \begin{tabular}[c]{@{}c@{}}{required} \\{batch size} \end{tabular} %
    & \begin{tabular}[c]{@{}c@{}}{bounded} \\{domain?} \end{tabular} %
    & \begin{tabular}[c]{@{}c@{}}{uniform} \\{gradient variance?} \end{tabular} %
    & \begin{tabular}[c]{@{}c@{}}{monotone} \\ {components?}\end{tabular} \\ \hline
Ours                     & \cmark & {\color{tabgreen} constantly 1}  & \nomark   & \nomark   & \nomark   \\ \hline
\citet{Cai22}\symbcntr{1}  & N/A & {\color{tabred}   increasing}    & \nomark   & \yesmark  & \nomark   \\ \hline
\citet{Chou23}\symbcntr{1}  & N/A & {\color{tabred}   increasing}    & \nomark   & \nomark   & \nomark   \\ \hline
\citet{Diak21}            & \xmark & {\color{tabred}   increasing}\symbcntr{3} & \nomark   & \yesmark  & \nomark   \\ \hline
\multirow{2}{*}{\citet{Gorb22SEG}}         & \cmark & {\color{tabred}   increasing}\symbcntr{3} & \nomark  & \nomark   & (star-)\yesmark\symbcntr{2} \\ \cline{2-6}
& \xmark & {\color{tabred}   increasing}\symbcntr{3} & \nomark  & \yesmark   & \nomark \\ \hline
\citet{Hsie20}\symbcntr{4}  & \xmark & {\color{tabgreen} constant}  & \nomark   & \nomark   & \nomark   \\ \hline
\citet{Judi11}            & \xmark & {\color{tabgreen} constant}      & \yesmark  & \yesmark  & \nomark   \\ \hline
\citet{Mish20SEG}         & \cmark & {\color{tabgreen} constant}      & \yesmark\symbcntr{5} & \nomark\symbcntr{5}  & \yesmark  \\ \hline
\citet{Peth23}            & \xmark & {\color{tabgreen} constant}      & \nomark   & \yesmark  & \nomark   \\ \hline
\end{tabular} 
}
\end{table}

\textbf{(\normalsymbcntr{1})}\quad 
The methods proposed in these works are not stochastic variants of EG in a strict sense. 
The method introduced by~\citet{Cai22} is rather a hybrid of EG and the Halpern iteration~\citep{Halp67}, while the method by~\citet{Chou23} is a stochastic version of the so-called \emph{optimistic gradient} method~\citep{Popo80}.
Hence, determining whether these methods fall into the category of same-sample methods or not is unnecessary.
Nonetheless, as these works focus on solving a similar problem to ours, we include them as references. 
        
\textbf{(\normalsymbcntr{2})}\quad {Under the assumptions that \citet{Gorb22SEG} make in their paper, one can show that each of the components must necessarily be (star-)monotone when the full $\mF$ is (star-)monotone. For further explanations on why this is the case, %
         see the following \Cref{sect:gorb22a-asmp}.} 
      
\textbf{(\normalsymbcntr{3})}\quad Yet, to be precise, what \citet{Gorb22SEG} have shown in the monotone case is that \segus{} can find an optimal solution if we \emph{increase} the batch size each iteration. If the batch size is fixed as a constant, then they were only able to show that the iterates will be {\emph{bounded} in the (star-)monotone setting. In particular, they did not provide a guarantee that the iterates will be necessarily convergent.} 
      
In fact, as we demonstrate with an explicit counterexample in \Cref{sec:appx:segus-monotone-lb}, if we do not increase the batch size each iteration, then it is possible to show that 
      \segus{} in the worst case will \emph{never} converge to an optimal point. 
      This nonconvergence result in fact holds for any \segus{} whose extrapolation and update stepsizes differ by a constant factor. Hence, it not only applies to \citep{Gorb22SEG}, but also to \citep{Diak21}.

\textbf{(\normalsymbcntr{4})}\quad 
\citet{Hsie20} show 
that \emph{independent-sample} \segus{} converges 
for stepsizes $\alpha_t,\beta_t$ decaying at certain different rates, 
but gives no convergence rates.
      
\textbf{(\normalsymbcntr{5})}\quad \citet{Mish20SEG} assume a uniformly bounded gradient variance in the strongly monotone case. In the monotone case, the bound they derived depends on the supremum of the gradient variance over the domain that is under consideration. Hence, in the monotone case, either the domain has to be (implicitly) bounded, or the uniform gradient variance assumption should be imposed.

\subsection{On the Assumptions Made by \texorpdfstring{\citet{Gorb22SEG}}{Gorbunov et al.}} \label{sect:gorb22a-asmp}
We would like to first clarify that in \citep{Gorb22SEG}, the requirement to increase the batch size is utilized only in the monotone setting: see, \textit{e.g.}, Corollary~E.4 therein.  

\citet{Gorb22SEG} use a generalized notion of $\mu$-strong monotonicity, namely the \emph{$\mu$-quasi strong monotonicity}, which requires the operator $\mF$ to satisfy \begin{equation}\label{eqn:def-quasi-monotone}
    \inprod{\mF \vz}{\vz - \vz^*} \geq \mu\norm{\vz - \vz^*}^2. 
\end{equation}  In the notion of $\mu$-quasi strong monotonicity they also allow $\mu \leq 0$. 
In particular, if \eqref{eqn:def-quasi-monotone} holds with $\mu = 0$, then $\mF$ is called a \emph{star-monotone} operator. 
In \Cref{appx:star-monotone} we further discuss on star-monotone operators.

Meanwhile, let us further elaborate on why in the (star-)monotone setting, the assumptions made by the authors of \citep{Gorb22SEG} lead to each component being star-monotone. In their work the authors require, as equation $\text{(10)}$ therein, that \begin{equation}\label{eqn:gorbunov-is-monotone} 
    \frac{1}{n} \sum_{i : \mu_i \geq 0} \mu_i + \frac{4}{n} \sum_{i : \mu_i < 0} \mu_i \geq 0.
\end{equation} Observe that this amounts to \begin{equation}\label{eqn:gorbunov-alter} 
\mu \coloneqq \frac{1}{n} \sum_{i=1}^n \mu_i \geq -\frac{3}{n} \sum_{i:\mu_i < 0} \mu_i = \frac{3}{n} \sum_{i:\mu_i < 0} \card{\mu_i}.
\end{equation} However, if any of $\mu_i$ is strictly negative, then the rightmost sum in \eqref{eqn:gorbunov-alter} becomes strictly positive, hence cannot be less than or equal to $\mu$ if $\mu = 0$. Therefore, the only possible case is when the rightmost sum is an empty sum. In other words, \eqref{eqn:gorbunov-is-monotone} can hold with $\mu = 0$ only when $\mu_i \geq 0$ for all $i$, so that each $\mF_i$ is star-monotone. We would like to remind the readers that our analyses, on the other hand, do not have any restrictions on the individual components.

\subsection{Finite Sum Structure vs. General Stochastic Setting} \label{appx:finite-sum-is-general}

The works mentioned in \Cref{sec:related-works} usually assume that we have access to a stochastic oracle that returns a stochastic estimator of $\mF$. 
Indeed, having a finite sum structure is a special case of having a stochastic oracle, as each $\mF_i$ can be seen as an estimator of $\mF$.   
One might then ask whether assuming the finite sum structure can help the works mentioned in \Cref{sec:related-works} overcome the mentioned limitations. 
We strongly believe that this is not the case. %
Recall \Cref{thm:rr-ff-bad}, where we have constructed an explicit counterexample that \segus, \segrr, and \segff{} all diverge. 
Because the set of problems with a finite sum structure is a subset of the set of problems with a stochastic oracle, the (counter-)example in \Cref{thm:rr-ff-bad} also works as an example that displays the nonconvergence of \segus, \segrr, and \segff{} in the general stochastic setting.
That is, a variant of SEG that only modifies the stepsizes and/or the sampling scheme into a without-replacement based one will suffer from nonconvergence, due to the counterexample in \Cref{thm:rr-ff-bad}.
It is also true that there are some methods that cannot exactly be classified as one of \segus, \segrr, or \segff, but this counterexample demonstrates that, unless explicitly proven otherwise, there is not a good reason to believe that the existing convergence analyses will be easily extended beyond the assumptions they are each based on.

\subsection{On the Claimed Convergence of \segrr{} in the Monotone Setting by \texorpdfstring{\citet{Emma24}}{Emmanouilidis et al.}} \label{sec:segrr-flaw}
Recently, a paper focusing on the study of \segrr{} \citep{Emma24} has been published. 
As we have briefly introduced in \Cref{tab:results} with a discussion in \Cref{sec:shuffling-not-enough}, the authors have established a convergence rate $\tilde{\gO}(\nicefrac{1}{nK^2})$ of \segrr{} in the strongly monotone setting, using an independent analysis of ours. 

On the other hand, the authors of \citep{Emma24} furthermore claim that \segrr{} is capable of finding an optimum in the monotone setting, which is seemingly contradictory to our analyses. 
We assert that this is not the case, as their proof, at least in their AISTATS 2024 version, seems to have a flaw. 

In establishing equation (85) in \citep{Emma24}, the authors claim that the inequality 
\[
\frac{1}{K} \sum_{k=0}^K \frac{1}{G^k} \expt\left[\lVert \mF(\vz_0^k)\rVert\right] \geq \expt\left[\left\lVert \mF\left(\frac{1}{K}\sum_{k=0}^{K} \frac{1}{G^k} \vz_0^k\right)\right\rVert^2\right]
\]
holds by Jensen's inequality, where $G \geq 6$ is a fixed constant. 
However, Jensen's inequality cannot be applied here, because not only $\lVert \mF(\cdot)\rVert^2$ is possibly nonconvex, but also the weights multiplied to the iterates, namely $\nicefrac{1}{KG^k}$, do not sum up to $1$. 
Hence, the “averaged” iterate is not in the form of a convex combination.
So, even if $\lVert \mF(\cdot)\rVert^2$ was convex, if we were to properly apply Jensen's inequality, at least the averaged iterate should be multiplied by $\frac{1}{\sum_{k=0}^K 1/G^k}$ instead of $\frac{1}{K}$. 
Yet then, the sum ${\sum_{k=0}^K \frac{1}{G^k}} \leq \frac{G}{G-1}$ is bounded above by a constant independent of $K$, and the right hand side of the equation right above (85) in \citep{Emma24} shall no longer be divided by $K$. Therefore, their claimed convergence is unobtainable. 

We would also like to remark that the linear decay rate of $1/G^k$ can make the series $\sum_{k=0}^\infty \frac{1}{G^k} \expt[\lVert \mF(\vz_0^k)\rVert^2]$ convergent even when $\expt[\lVert \mF(\vz_0^k)\rVert^2]$ grows exponentially as $k \to \infty$, as long as its rate of exponential growth is less than $G$. In particular, once their (85) is corrected, there is no contradiction with our divergence result in \Cref{thm:rr-ff-bad}.

\interdisplaylinepenalty=10000
\section{Useful Lemmata} \label{sec:toolbox}
\begin{lem}[Polarization identity] \label{lem:polarization}
    For any two vectors $\va$ and $\vb$, it holds that \begin{align*}
   2\inprod{\va}{\vb} &= \norm{\va}^2  + \norm{\vb}^2 - \norm{\va - \vb}^2 \\ 
   &= \norm{\va + \vb}^2 - \norm{\va}^2 - \norm{\vb}^2 . 
    \end{align*}
\end{lem}
\begin{proof}  
    The identities are immediate from $\norm{\va \pm \vb}^2 = \norm{\va}^2 \pm 2\inprod{\va}{\vb} + \norm{\vb}^2$.
\end{proof}

\begin{lem}[Weighted AM-GM inequality] \label{lem:w-amgm-ineq}
    For any $\gamma > 0$ and two vectors $\va$ and $\vb$ in $\mathbb{R}^d$, \[
        2 \card{\inprod{\va}{\vb}} \leq \gamma \norm{\va}^2 + \frac{1}{\gamma} \norm{\vb}^2.     
    \]
\end{lem} 
\begin{proof}  
    Notice that \begin{align*}
2 \card{\inprod{\va}{\vb}} &\leq 2\left(|a_1 b_1| + \dots + |a_d b_d|\right) \\ 
&\leq {\left(\gamma a_1^2 + \frac{b_1^2}{\gamma} \right) + \dots + \left(\gamma a_d^2 + \frac{b_d^2}{\gamma} \right)} = {\gamma \norm{\va}^2 + \frac{1}{\gamma} \norm{\vb}^2. } \qedhere
    \end{align*}  
\end{proof}

\begin{lem}[Young's inequality] \label{lem:young-ineq}
    For any $\gamma > 0$ and two vectors $\va$ and $\vb$,  \begin{equation} \label{eqn:young-general}
    \norm{\va+ \vb}^2 \leq (1+\gamma) \norm{\va}^2 + \left( 1+\frac{1}{\gamma} \right) \norm{\vb}^2.     
    \end{equation} In particular, as a special case where $\gamma = 1$, it holds that \begin{align}\label{eqn:young-two}
        \norm{\va+ \vb}^2 \leq 2 \norm{\va}^2 + 2 \norm{\vb}^2.     
    \end{align} 
\end{lem} \begin{proof}
    The left hand side of \eqref{eqn:young-general} is $\norm{\va}^2 + 2\inprod{\va}{\vb} + \norm{\vb}^2$. Applying \Cref{lem:w-amgm-ineq} then suffices.
\end{proof}

\begin{lem}\label{lem:young-cor} For any two vectors $\va$ and $\vb$, it holds that  \[
        \norm{\va - \vb}^2 \geq \frac{1}{2}\norm{\va}^2 - \norm{\vb}^2. 
    \]
\end{lem} \begin{proof}
    From \eqref{eqn:young-two} it follows that \begin{align*}
        \norm{\va}^2 = \norm{(\va-\vb) + \vb}^2 \leq 2\norm{\va - \vb}^2  + 2\norm{\vb}^2.
    \end{align*} Simply rearranging the terms gives us the result.
\end{proof}

\begin{lem}[Generalized Young's inequality]
    For any nonnegative scalars $p_1, \dots, p_n$ such that $p_1 + \dots + p_n = 1$ and vectors $\va_1, \dots, \va_n$, it holds that  \[
    \norm{p_1 \va_1 + \dots + p_n \va_n}^2 \leq p_1 \norm{\va_1}^2 + \dots + p_n \norm{\va_n}^2.     
    \] In particular, setting $p_1 = \dots = p_n = \frac{1}{n}$ and multiplying both sides by $n^2$ yields \[
        \norm{\va_1 + \dots + \va_n}^2 \leq n\left( \norm{\va_1}^2 + \dots + \norm{\va_n}^2 \right) .     
        \] 
\end{lem}
\begin{proof}
    We use induction on $n$. If $n = 1$ then $p_1 = 1$, so there is nothing to show. For the inductive step, suppose that the statement holds for some $n \geq 1$. Say we are given nonnegative scalars $p_1, \dots, p_{n+1}$ such that $p_1+ \dots + p_{n+1}= 1$, and vectors $\va_1, \dots, \va_{n+1}$. For the moment, suppose that $p_{n+1} < 1$. Applying \Cref{lem:young-ineq} with $\gamma = \frac{p_{n+1}}{1-p_{n+1}}$ and using the induction hypothesis, we get
    \begin{align*}%
    \MoveEqLeft \norm{p_1 \va_1 + \dots + p_n \va_n + p_{n+1} \va_{n+1}}^2  \\ 
    &\leq \frac{1}{1-p_{n+1}}  \norm{p_1 \va_1 + \dots + p_n \va_n }^2 + \frac{1}{p_{n+1}} \norm{ p_{n+1} \va_{n+1}}^2 \\ 
    &= (1-p_{n+1})  \norm{\frac{p_1}{1-p_{n+1}} \va_1 + \dots + \frac{p_n}{1-p_{n+1}} \va_n }^2 + {p_{n+1}} \norm{ \va_{n+1}}^2 \\ 
    &\leq  p_1\norm{\va_1 }^2 + \dots + p_n \norm{ \va_n }^2 + {p_{n+1}} \norm{ \va_{n+1}}^2 
    \end{align*}  where in last line we used that $p_1 + \dots + p_n = 1-p_{n+1}$. Now, if $p_{n+1} = 1$, then we must have $p_1 = \dots = p_n = 0$, so the claimed inequality holds in this case also. This completes the proof.
\end{proof}

\begin{lem}\label{lem:hess-bound}
    Suppose that $\mF$ is $M$-smooth. Then for any $\vz$ and $\vw$ it holds that \[
    \norm{\mF\vw - \mF\vz - D\mF(\vz)(\vw-\vz) } \leq \frac{M}{2}\norm{\vw - \vz}^2.   
    \]
\end{lem} 
\begin{proof} The proof closely follows the arguments used for Lemma~1.2.4 in \citep{Nest18}, by replacing the gradients therein by saddle gradients. The fundamental theorem of calculus with the $M$-smoothness of $\mF$ gives us \begin{align*}
    \norm{\mF\vw - \mF\vz - D\mF(\vz)(\vw-\vz)} &= \norm{\int_0^1 D\mF(\vz+t(\vw - \vz)) \,\mathrm{d}t\, (\vw-\vz) - D\mF(\vz)(\vw-\vz)} \\
    &\leq  \norm{\vw-\vz} \int_0^1 \norm{ D\mF(\vz+t(\vw - \vz))  - D\mF(\vz) } \,\mathrm{d}t\,  \\
&\leq  \norm{\vw-\vz} \int_0^1 Mt \norm{\vw - \vz} \,\mathrm{d}t\,  \\
&= \frac{M}{2}\norm{\vw - \vz}^2. \qedhere
\end{align*} \end{proof}

\begin{lem}\label{lem:inprod-lower-bound} Let $\mF$ be a $\mu$-strongly monotone operator. Let $\vz^*$ be a point such that $\mF\vz^* = \zero$, and let $\eta > 0$. Then, for any point $\vz$ in the domain of $\mF$ and $\vw \coloneqq \vz - \eta \mF \vz$, it holds that \[
    \inprod{\mF \vw}{\vw - \vz^*} \geq \frac{\mu}{2}\norm{\vz - \vz^*}^2 - \eta^2 \mu \norm{\mF\vz}^2. 
    \]
\end{lem} \begin{proof}
    By the $\mu$-strong monotonicity of $\mF$ and \Cref{lem:young-cor} it holds that \begin{align*}
        \inprod{\mF \vw}{\vw - \vz^*} &\geq \mu \norm{\vw - \vz^*}^2  \\ 
        &= \mu \norm{\vz - \eta \mF \vz - \vz^*}^2  \\
        &\geq \frac{\mu}{2} \norm{\vz - \vz^*}^2 - \mu\norm{\eta \mF \vz}^2
    \end{align*} so we are done.
\end{proof}

The following lemma generalizes Lemma~3.2 in \cite{Gorb22EG} shown for monotone $\mF$ to $\mu$-strongly monotone~$\mF$ with $\mu > 0$.

\begin{lem}\label{lem:norm-decrease}  Let $\mF$ be a $\mu$-strongly monotone $L$-Lipschitz operator, and let $\vz$ be any point in the domain of $\mF$. Then for any $0 < \eta < \frac{1}{L\sqrt{2}}$, it holds that \[
    \norm{\mF(\vz - \eta \mF(\vz - \eta \mF \vz))}^2 \leq \left( 1-\frac{2 \eta \mu}{5} \right)\norm{\mF\vz}^2 . 
    \]
\end{lem} \begin{proof}
   For convenience, let us define $\vw \coloneqq \vz - \eta \mF \vz$ and $\vz^+ \coloneqq \vz - \eta \mF(\vz - \eta \mF \vz) = \vz - \eta \mF \vw$. 
   Because $\mF$~is $\mu$-strongly monotone, we have 
   \begin{equation} \label{eqn:norm-decrease-1}
      \begin{aligned}
         \mu \norm{\vz^{+} - \vz}^2 &\leq \inprod{\mF \vz^+ - \mF \vz}{\vz^+ - \vz} \\
         &= \eta \inprod{\mF \vz - \mF \vz^+}{\mF \vw}. 
      \end{aligned}
   \end{equation}
   Also from the $\mu$-strong monotonicity of $\mF$ we get 
   \begin{equation} \label{eqn:norm-decrease-2}
      \begin{aligned}
         \mu \norm{\vw - \vz^+}^2 &\leq \inprod{\mF \vw - \mF \vz^+}{\vw - \vz^+} \\
         &= \eta \inprod{\mF \vw - \mF \vz^+}{\mF \vw - \mF \vz}. 
      \end{aligned}
   \end{equation}
   Meanwhile, %
   from the $L$-Lipschitzness of $\mF$ we have
   \begin{align} 
      \norm{\mF\vw - \mF\vz^+}^2 &\leq \eta^2 L^2 \norm{\mF\vw - \mF\vz}^2. \label{eqn:norm-decrease-3}
   \end{align} 
   Summing up the inequalities \eqref{eqn:norm-decrease-1}, \eqref{eqn:norm-decrease-2}, \eqref{eqn:norm-decrease-3} with weights $\nicefrac{2}{\eta}$, $\nicefrac{1}{2\eta}$, and $\nicefrac{3}{2}$ respectively, we obtain \begin{align*} 
      \MoveEqLeft \frac{\mu}{\eta} \left( 2 \norm{\vz^{+} - \vz}^2 + \frac{1}{2} \norm{\vw - \vz^+}^2 \right) + \frac{3}{2} \norm{\mF\vw - \mF\vz^+}^2 \\
      &\leq 2 \inprod{\mF \vz - \mF \vz^+}{\mF \vw} + \frac{1}{2} \inprod{\mF \vw - \mF \vz^+}{\mF \vw - \mF \vz} + \frac{3\eta^2 L^2}{2} \norm{\mF\vw - \mF\vz}^2. 
   \end{align*} 
   From this inequality, we can exactly follow the arguments used in the proof of Lemma~D.4 in \citep{Gorb22EG} to derive that \begin{equation} \label{eqn:norm-decrease-4}
      \frac{\mu}{\eta} \left( 2 \norm{\vz^{+} - \vz}^2 + \frac{1}{2} \norm{\vw - \vz^+}^2 \right) + \norm{\mF\vz^+}^2 \leq \norm{\mF\vz}^2. 
   \end{equation} 
   Meanwhile, Young's inequality (\Cref{lem:young-ineq}) tells us that \[
      \eta^2 \norm{\mF \vz}^2 = \norm{\vw - \vz}^2 \leq \left( 1+\frac{1}{4} \right) \norm{\vw - \vz^+}^2 + (1+4) \norm{\vz^+ - \vz}^2.
   \] 
   Using this to lower bound the left hand side of \eqref{eqn:norm-decrease-4}, we get that \[
      \frac{2 \eta \mu}{5} \norm{\mF \vz}^2 + \norm{\mF\vz^+}^2 \leq \norm{\mF\vz}^2. 
   \]
   It remains to simply rearrange the terms.
\end{proof}

\begin{lem} \label{appx:variance-redundant}
       Suppose that $\mF_i$ is $L$-Lipschitz for all $i = 1, \dots, n$, and that $\mF \coloneqq \frac{1}{n} \sum_{i=1}^n \mF_i$ is $\mu$-strongly monotone with $\mu > 0$. 
       Define $\kappa \coloneqq \nicefrac{L}\mu$ and $\sigma_*^2 \coloneqq \frac{1}{n}\sum_{i=1}^n \norm{\mF_i \vz^*}^2$. 
       Then, for any $\vz \in \rr^{{d_1}+{d_2}}$ it holds that \[
    \frac{1}{n}\sum_{i=1}^n\, \norm{\mF_i \vz - \mF\vz}^2 \leq \left(\sqrt{3(1+\kappa^2)} \norm{\mF\vz} + \sqrt{3}\sigma_* \right)^2. 
    \] 
\end{lem}
\begin{proof}
For any $\vz, \vw\in \rr^{{d_1}+{d_2}}$, as \Cref{asmp:monotone} holds with $\mu > 0$, by Cauchy-Schwarz inequality \[
\mu \norm{\vz - \vw}^2 \leq \inprod{\mF \vz - \mF \vw}{\vz - \vw} \leq \norm{\mF \vz - \mF \vw}\norm{\vz - \vw},
\] and as a consequence, $\norm{\vz - \vw} \leq \nicefrac{1}{\mu}\norm{\mF \vz - \mF \vw}$. 
    Thus, for any $i \in [n]$, it holds that \begin{align*}
         \norm{\mF_i \vz  - \mF \vz }^2 &\leq 3\norm{\mF_i \vz  - \mF_i \vz^* }^2 + 3\norm{\mF \vz  - \mF \vz^* }^2 + 3\norm{\mF_i \vz^*  - \mF \vz^* }^2 \\ 
         &\leq 3 L^2 \norm{\vz  - \vz^* }^2 + 3\norm{\mF \vz  - \mF \vz^* }^2 + 3\norm{\mF_i \vz^*}^2 \\ 
         &\leq 3 \left(\frac{L^2}{\mu^2} + 1\right)  \norm{\mF \vz  - \mF \vz^* }^2 + 3\norm{\mF_i \vz^*}^2 \\ 
         &= 3 (1+\kappa^2)  \norm{\mF \vz }^2 + 3\norm{\mF_i \vz^*}^2. 
    \end{align*} Summing this inequality over $i=1, \dots, n$ and then dividing by $n$ leads to \begin{align*}
        \frac{1}{n} \sum_{i=1}^n \norm{\mF_i \vz  - \mF \vz }^2 &\leq 3 (1+\kappa^2)  \norm{\mF \vz }^2 + \frac{3}{n} \sum_{i=1}^n\norm{\mF_i \vz^*}^2 \\
        &= 3 (1+\kappa^2)  \norm{\mF \vz }^2 + 3 \sigma_*^2.
    \end{align*} 
    The conclusion follows from the basic inequality $a^2 + b^2 \leq (a + b)^2$ which holds for any $a , b \geq 0$.
\end{proof}

\begin{lem}[Nonexpansiveness of the EG operator]\label{lem:nonexpansive} Let $\mF$ be a monotone $L$-Lipschitz operator, and $\vz^*$ be a point such that $\mF\vz^* = \zero$. Then, for any point $\vz$ in the domain of $\mF$ and $\eta > 0$, \[
    \norm{\vz - \eta \mF(\vz - \eta \mF \vz) - \vz^*}^2 \leq \norm{\vz - \vz^*}^2 - \eta^2 (1-\eta^2 L^2)\norm{\mF\vz}^2. 
    \]
\end{lem} \begin{proof}
    This classical result dates back to the original paper on EG by \citet{Korp76}. Here, for completeness, we replicate the proof using our notations. 
    
    Expanding the left hand side of the inequality stated, we obtain \begin{equation}\label{eqn:Gorb23lemma}\begin{aligned}
     \norm{\vz - \eta \mF(\vz - \eta \mF \vz) - \vz^*}^2 &= \norm{\vz  - \vz^*}^2 - 2\inprod{\eta \mF(\vz - \eta \mF \vz)}{\vz-\vz^*} + \norm{\eta \mF(\vz - \eta \mF \vz)}^2 \\ 
        &= \norm{\vz  - \vz^*}^2 - 2\eta \inprod{\mF(\vz - \eta \mF \vz)}{\vz - \eta \mF \vz -\vz^*} \\ 
        &\phantom{= \norm{\vz  - \vz^*}^2 \ }- 2\eta ^2 \inprod{ \mF(\vz - \eta \mF \vz)}{ \mF \vz } + \eta^2 \norm{\mF(\vz - \eta \mF \vz)}^2.
    \end{aligned}\end{equation} 
    For the first inner product term, we have \[
    - 2\eta \inprod{\mF(\vz - \eta \mF \vz)}{\vz - \eta \mF \vz -\vz^*} \leq 0
    \] because $\mF$ is monotone. For the second inner product term, we use the polarization identity (\Cref{lem:polarization}) and the $L$-Lipschitzness of $\mF$ to get \begin{align*}
      -2 \inprod{\mF (\vz - \eta \mF \vz)}{\mF \vz } &= \norm{\mF (\vz - \eta \mF \vz) - \mF \vz }^2 - \norm{\mF (\vz - \eta \mF \vz)}^2 - \norm{\mF \vz }^2 \\ 
      &\leq L^2 \norm{ - \eta \mF \vz }^2 - \norm{\mF (\vz - \eta \mF \vz)}^2 - \norm{\mF \vz }^2 \\ 
      &= -(1-\eta^2 L^2)\norm{\mF \vz}^2  - \norm{\mF (\vz - \eta \mF \vz)}^2. 
   \end{align*} Applying these two bounds on \eqref{eqn:Gorb23lemma} completes the proof.
\end{proof}

\begin{lem}\label{lem:recurrence}
    Let $\{a_k\}_{k\geq 0}$, $\{b_k\}_{k\geq 0}$, $\{c_k\}_{k\geq 0}$, and $\{d_k\}_{k\geq 0}$ be sequences of nonnegative numbers satisfying the recurrence relation \[
    b_k \leq (1+a_k)d_k - d_{k+1} + c_k \qquad \forall k \geq 0. 
    \] Then for any $k\geq 0$ it holds that \[
    d_{k+1} + \sum_{j=0}^k b_j \leq \left( \prod_{j=0}^{k} (1+a_j) \right) \left( d_0 + \sum_{j=0}^k c_j \right). 
    \]
\end{lem}\begin{proof}
    Because $a_k \geq 0$, it suffices to show that \begin{equation}\label{eqn:unrolled-recurrence}
        \sum_{j=0}^k (b_j-c_j) \prod_{i=j+1}^{k} (1+a_i) \leq -d_{k+1} + d_0 \prod_{j=0}^{k} (1+a_j) , 
    \end{equation} as this implies  \begin{align*}
        \sum_{j=0}^k b_j &\leq \sum_{j=0}^k  b_j \prod_{i=j+1}^{k} (1+a_i)  \\ 
        &\leq \left( \sum_{j=0}^k c_j \prod_{i=j+1}^{k} (1+a_i) \right) -d_{k+1} + d_0 \prod_{j=0}^{k} (1+a_j) \\ 
        &\leq -d_{k+1} + \left( d_0 + \sum_{j=0}^k c_k  \right)   \prod_{j=0}^{k} (1+a_j). 
    \end{align*}
    So, we show that \eqref{eqn:unrolled-recurrence} holds, by induction on $k$. For the base case $k = 0$, the recurrence relation tells us that \[
    b_0 - c_0 \leq (1+a_0) d_0 - d_1  
    \] which is exactly \eqref{eqn:unrolled-recurrence} when $k = 0$. Now suppose that \eqref{eqn:unrolled-recurrence} holds for some $k \geq 0$. Using the induction hypothesis and the recurrence relation we get \begin{align*}
        \sum_{j=0}^{k+1} (b_j-c_j) \prod_{i=j+1}^{k+1} (1+a_i) &= b_{k+1} - c_{k+1} + (1+a_{k+1}) \left( \sum_{j=0}^k (b_j-c_j) \prod_{i=j+1}^{k} (1+a_i) \right) \\ 
        &\leq b_{k+1} - c_{k+1} -(1+a_{k+1})d_{k+1} + d_0 \prod_{j=0}^{k+1} (1+a_j) \\ 
        &\leq - d_{k+2} + d_0 \prod_{j=0}^{k+1} (1+a_j) .
    \end{align*} This shows that \eqref{eqn:unrolled-recurrence} holds also for $k+1$, and we are done.
\end{proof}

The subsequent lemma is technical, but it can be derived from elementary calculus.
 
\begin{lem}\label{lem:cbrt-log-growth-bound} 
For any $K\geq 1$,   \[
     \sum_{k=2}^{K+2} \frac{1}{k^{2/3}(\log k)^2} \geq \frac{(K+3)^{1/3}}{(\log (K+3))^2}. 
   \]
\end{lem}  \begin{proof}
    Consider the function $h(x) \coloneqq \frac{1}{x^{2/3}(\log x)^2}$ over the interval $[2, K+3]$. 
    As \[
    h'(x) = -\frac{2}{x^{5/3} (\log x)^3}-\frac{2}{3 x^{5/3} (\log x)^2} < 0,
    \] $h$ is decreasing. 
    Hence, an upper Riemann sum becomes an upper bound for the integral, so we have \begin{equation}\label{eqn:upper-riemann-sum}
    \sum_{k=2}^{K+2} \frac{1}{k^{2/3}(\log k)^2} \geq \int_{2}^{K+3}\frac{1}{x^{2/3}(\log x)^2}\,\mathrm{d}x  .
    \end{equation}
    
    Now consider a function $g:[1, \infty) \to \rr$, defined as \[
        g(y) \coloneqq  \int_{2}^{y+3}\frac{1}{x^{2/3}(\log x)^2}\,\mathrm{d}x - \frac{(y+3)^{1/3}}{(\log (y+3))^2}. 
    \] Differentiating, we get  \[
        g'(y) = \frac{2}{(y+3)^{2/3} (\log(y+3))^3}+\frac{2}{3 (y+3)^{2/3} (\log(y+3))^2} > 0
    \] whenever $y \geq 1$. That is, $g$ is increasing on $ y \geq 1$.
    We then show that $g(1) \geq 0$. To this end, let us begin with observing that \[
        h''(x) = \frac{6}{x^{8/3} (\log x)^4}+\frac{14}{3 x^{8/3} (\log x)^3}+\frac{10}{9 x^{8/3} (\log x)^2} > 0,
    \] from which we get that $h$ is convex.
    In particular, it holds that \[
        \int_2^3 h(x)\,\mathrm{d}x \geq \int_2^3 h'\left(\frac{5}{2}\right) \left(x - \frac{5}{2}\right) + h\left(\frac{5}{2}\right)\,\mathrm{d}x = h\left(\frac{5}{2}\right),
    \] and similarly, $\int_3^4 h(x) \,\mathrm{d}x \geq h(\nicefrac{7}{2})$. 
    Thus we indeed have \begin{align*}
        g(1) &= \int_2^4 h(x)\,\mathrm{d}x - \frac{4^{1/3}}{(\log 4)^2} \\
            &\geq \frac{1}{(5/2)^{2/3}(\log (5/2))^2} + \frac{1}{(7/2)^{2/3}(\log (7/2))^2} - \frac{4^{1/3}}{(\log 4)^2} \ \geq 0 . 
    \end{align*} 
    
        Recalling that $g$ is increasing, we have $g(K) \geq g(1) \geq 0$ for all $K \geq 1$. 
        This, with \eqref{eqn:upper-riemann-sum}, implies that \[
            \sum_{k=2}^{K+2} \frac{1}{k^{2/3}(\log k)^2} \geq \int_{2}^{K+3}\frac{1}{x^{2/3}(\log x)^2}\,\mathrm{d}x  \geq \frac{(K+3)^{1/3}}{(\log (K+3))^2}
        \] holds whenever $K\geq 1$, which is exactly the claimed.  
    \end{proof}

\section{Missing Proofs for \titleref{Section}{section:method}} \label{sec:proof-segffa}
\subsection{Unravelling the Recurrence of the Generalized SEG in~\titleEqref{eqn:update_rule1-m} and~\titleEqref{eqn:anchoring}} \label{sappx:unravelling} 
In \Cref{sec:why-segffa}, we considered the method where, in a single epoch (hence omitting all superscripts that are used to denote the epoch number for convenience), the iterates are generated following the recurrence 
\begin{equation} \label{eqn:appx:update_rule1-m} \begin{aligned}
    \vw_i     &= \vz_i - \alpha \mT_{i} \vz_i \\
    \vz_{i+1} &= \vz_i - \beta \mT_{i} \vw_i 
\end{aligned} \end{equation}
    for $i = 0, 1, \dots, N-1$, where each $\mT_i$ are sampled from the set $\{\mF_1, \dots, \mF_n\}$, and an additional anchoring step \begin{equation}\label{eqn:appx:anchoring}
    \vz^{\sharp} \coloneqq \frac{\vz_{N} + \theta \vz_0}{1+\theta} 
    \end{equation} is performed so that $\vz^{\sharp}$ is used as the initial point of the next epoch. 
    Notice that \eqref{eqn:appx:anchoring} is a generalized anchoring step that incorporates all the settings we are considering, as the versions of SEG where anchoring is not used correspond to taking $\theta = 0$, and the anchoring step \eqref{eqn:anchoring} that is used in \segffa{} corresponds to taking $\theta = 1$. 
In this section we would like to prove the following statement regarding this update rule. 

\begin{prop}[\Cref{prop:unravelling-m}]\label{prop:appx:unravelling-m} It holds that \begin{equation}\label{eqn:appx:approx-m}
        \vz^\sharp =  \vz_0 - \frac{\beta}{1+\theta} \sum_{j=0}^{N-1} \mT_j \vz_0 + \frac{\alpha\beta}{1+\theta} \sum_{j=0}^{N-1} D\mT_j(\vz_0) \mT_j \vz_0 +  \frac{\beta^2}{1+\theta} \sum_{0\leq i < j \leq N-1}  D\mT_j(\vz_0) \mT_i \vz_0 + \frac{\veps_N}{1+\theta}  
    \end{equation} for some $\veps_N=o\left((\alpha+\beta)^2\right)$. \end{prop}
\begin{proof} 
Equation \eqref{eqn:appx:approx-m} immediately follows from \Cref{prop:unravelling-each-step}, with \eqref{eqn:appx-approx-error-m} giving us the precise definition of $\veps_{N}$. To show that $\veps_N=o\left((\alpha+\beta)^2\right)$, we begin with noting that both $\norm{\vz_j - \vz_0}$ and $\norm{\vw_j - \vz_0}$ are of $\gO(\alpha+\beta)$, because both $\vz_j$ and $\vw_j$ are obtained from $\vz_0$ by performing at most $j$ updates following \eqref{eqn:appx:update_rule1-m}. Thus, the first term in the right hand side of \eqref{eqn:appx-approx-error-m} is of $\gO(\beta(\alpha+\beta)^2)$ by \Cref{lem:hess-bound}, and the remaining terms are of $\gO((\alpha+\beta)^3)$ by the $L$-smoothness of the operators $\mF_1, \dots, \mF_n$.  
\end{proof}

    \begin{prop}\label{prop:unravelling-each-step}
        For any $i = 0, 1,\dots, N$, it holds that \begin{equation}\label{eqn:appx-approx-intermediate}
            \vz_i = \vz_0 - \beta \sum_{j=0}^{i-1} \mT_j \vz_0 + \alpha\beta \sum_{j=0}^{i-1} D\mT_j(\vz_0) \mT_j \vz_0 +  \beta^2 \sum_{0\leq k < j \leq i-1}  D\mT_j(\vz_0) \mT_k \vz_0   + \veps_i
        \end{equation} where we denote \begin{equation} \label{eqn:appx-approx-error-m}\begin{aligned}
            \veps_i \coloneqq & - \beta \sum_{j=0}^{i-1} \Bigl(\mT_j\vw_j - \mT_j \vz_0 -  D\mT_j(\vz_0)(\vw_j - \vz_0)\Bigr)    \\
        &\ \ + \alpha \beta \sum_{j=0}^{i-1}D\mT_j(\vz_0)(\mT_j \vz_j - \mT_j \vz_0) + \beta^2\sum_{j=0}^{i-1} D\mT_j (\vz_0) \sum_{k=0}^{j-1}(\mT_k \vw_k - \mT_k \vz_0).
        \end{aligned}\end{equation}
    \end{prop} 
    \begingroup 
    \begin{proof}
    We use induction on $i$. There is nothing to show for the base case $i= 0$. Now, suppose that \eqref{eqn:appx-approx-intermediate} and \eqref{eqn:appx-approx-error-m} hold for some $i < N$, and write \begin{align*}
        \vz_{i+1} &= \vz_i - \beta \mT_{i} \vw_i  \\
                    &= \vz_i - \beta \mT_{i} \vz_0 - \beta D\mT_i(\vz_0)(\vw_i - \vz_0) - \beta \Bigl(\mT_i \vw_i - \mT_i \vz_0 - D\mT_i(\vz_0)(\vw_i- \vz_0)\Bigr). 
        \end{align*} Here, notice that by the update rule we have \begin{align*}
            \vw_i &= \vz_i - \alpha \mT_i \vz_i \\
            &= \vz_0 - \beta \sum_{j=0}^{i-1}\mT_j \vw_j - \alpha \mT_i \vz_i. 
        \end{align*} 
        \begingroup 
        \allowdisplaybreaks
        Using this identity and the induction hypothesis we get \begin{align*}
            \vz_{i+1} &= \vz_0 - \beta \sum_{j=0}^{i-1} \mT_j \vz_0 + \alpha\beta \sum_{j=0}^{i-1} D\mT_j(\vz_0) \mT_j \vz_0 +  \beta^2 \sum_{0\leq k < j \leq i-1}  D\mT_j(\vz_0) \mT_k \vz_0   + \veps_i\\
            &\phantom{= \vz_i} \ - \beta \mT_{i} \vz_0 - \beta D\mT_i(\vz_0)\left( - \beta \sum_{j=0}^{i-1}\mT_j \vw_j - \alpha \mT_i \vz_i \right) \\
            &\phantom{= \vz_i} \ - \beta \Bigl(\mT_i \vw_i - \mT_i \vz_0 - D\mT_i(\vz_0)(\vw_i- \vz_0)\Bigr) \\
            &= \vz_0 - \beta \sum_{j=0}^{i-1} \mT_j \vz_0 + \alpha\beta \sum_{j=0}^{i-1} D\mT_j(\vz_0) \mT_j \vz_0 +  \beta^2 \sum_{0\leq k < j \leq i-1}  D\mT_j(\vz_0) \mT_k \vz_0   + \veps_i\\
            &\phantom{= \vz_i} \ - \beta \mT_{i} \vz_0  + \beta^2 D\mT_i(\vz_0) \sum_{j=0}^{i-1}\mT_j \vz_0 + \beta^2 D\mT_i(\vz_0) \sum_{j=0}^{i-1}(\mT_j \vw_j - \mT_j \vz_0)  \\
            &\phantom{= \vz_i} \ + \alpha\beta D\mT_i(\vz_0)(\mT_i \vz_i - \mT_i \vz_0) + \alpha\beta D\mT_i(\vz_0)\mT_i \vz_0 \\ 
            &\phantom{= \vz_i} \ - \beta \Bigl(\mT_i \vw_i - \mT_i \vz_0 - D\mT_i(\vz_0)(\vw_i- \vz_0)\Bigr)  \\
            &= \vz_0 - \beta \sum_{j=0}^{i} \mT_j \vz_0 + \alpha\beta \sum_{j=0}^{i} D\mT_j(\vz_0) \mT_j \vz_0 +  \beta^2 \sum_{0\leq k < j \leq i}  D\mT_j(\vz_0) \mT_k \vz_0   + \veps_i\\
            &\phantom{= \vz_i} \ + \beta^2 D\mT_i(\vz_0) \sum_{j=0}^{i-1}(\mT_j \vw_j - \mT_j \vz_0) + \alpha\beta D\mT_i(\vz_0)(\mT_i \vz_i - \mT_i \vz_0) \\ 
            &\phantom{= \vz_i} \ - \beta \Bigl(\mT_i \vw_i - \mT_i \vz_0 - D\mT_i(\vz_0)(\vw_i- \vz_0)\Bigr)  \\
            &= \vz_0 - \beta \sum_{j=0}^{i} \mT_j \vz_0 + \alpha\beta \sum_{j=0}^{i} D\mT_j(\vz_0) \mT_j \vz_0 +  \beta^2 \sum_{0\leq k < j \leq i}  D\mT_j(\vz_0) \mT_k \vz_0  + \veps_{i+1}
        \end{align*} which asserts that \eqref{eqn:appx-approx-intermediate} also holds for $i+1$. 
        \endgroup
    \end{proof}  
    \endgroup

\subsection{Insufficiency of Only Using Flip-Flop Sampling}\label{sec:ff-alone}
Here we prove the following. 
    \begin{prop}[\Cref{thm:ff-alone}]
         Suppose we use flip-flop sampling only. In order to make \eqref{eqn:first-order} and \eqref{eqn:second-order} hold, we must choose $\beta = \nicefrac{\eta_1}{n}$ and $\alpha = \nicefrac{\beta}{2}$. However, this leads to $\eta_2 = 2 \eta_1$, which is the set of parameters that fails to make EG+ converge. 
        \end{prop} \begin{proof}
    Suppose that we have already established the upcoming \Cref{lem:ff-rearranging}. 
    Then, we can see by setting $\theta = 0$ in the result of \Cref{lem:ff-rearranging} that for \eqref{eqn:second-order} to hold, 
    the following system of equations should be satisfied: \[
        \left\{ \begin{aligned}
       \eta_1 \eta_2   &=  2n^2 \beta^2 ,\\
       \eta_1 \eta_2   &=   n^2(2\alpha\beta + \beta^2) ,\\
              \eta_2   &=  2n \beta . 
        \end{aligned} \right.
        \] 
    Solving this system of equations, we get $\eta_1 = n \beta$, $\eta_2 = 2 n\beta$, and $\alpha = \nicefrac{\beta}{2}$. 
    
    For the latter part of the statement on the divergence of EG+ with $\eta_2 = 2\eta_1$, 
    consider the $(1+1)$-dimensional bilinear problem \[
\min_x \max_y \ xy    
\] whose unique optimum is $\vz^* = (0, 0)$. A simple computation shows that \[
\mF \vz = \begin{bmatrix}
    0 & 1 \\ -1 & 0
\end{bmatrix} \vz.
\] Consequently, for any $\eta > 0$, the update rule of EG+ with $\eta_1 = \eta$ and $\eta_2 = 2\eta$ amounts to \[
    \vz^+ = \vz - 2\eta \mF (\vz - \eta \mF \vz) = \begin{bmatrix}
        1 - 2\eta^2 & -2 \eta \\ 2\eta & 1 - 2\eta^2 
    \end{bmatrix} \vz. 
\] It follows that \begin{align*}
    \norm{\vz^+ - \vz^*}^2 &= \norm{\begin{bmatrix}
        1 - 2\eta^2 & -2 \eta \\ 2\eta & 1 - 2\eta^2 
    \end{bmatrix} \begin{bmatrix}
        x \\ y
    \end{bmatrix}}^2  \\
    &= \left( (1-2 \eta^2) x-2 \eta  y\right)^2+\left(2 \eta  x+ (1-2 \eta^2) y\right)^2 \\
    &= (1+4\eta^4) (x^2+y^2)\\
    &= (1+4\eta^4) \norm{\vz  - \vz^*}^2.
\end{align*} Therefore, the distance from the optimal solution strictly increases every iterate. 
\end{proof}

    It remains to actually prove \Cref{lem:ff-rearranging}. 
    
\begin{lem}\label{lem:ff-rearranging}
When flip-flop sampling is used with the generalized anchoring step \eqref{eqn:appx:anchoring}, it holds that
    \begin{align*}
          \MoveEqLeft  \frac{\alpha\beta}{1+\theta} \sum_{j=0}^{N-1} D\mT_j(\vz_0) \mT_j \vz_0 + \frac{\beta^2}{1+\theta} \sum_{0\leq i < j \leq N-1}  D\mT_j(\vz_0) \mT_i \vz_0 \\
          &= \frac{2\alpha\beta + \beta^2}{1+\theta} \sum_{j=1}^{n} D\mF_j(\vz_0) \mF_j \vz_0 + \frac{2\beta^2}{1+\theta} \sum_{i \neq j}  D\mF_j(\vz_0) \mF_i \vz_0.
        \end{align*}
\end{lem}
\begin{proof}
    As we are using flip-flop sampling, we have $N = 2n$, and it is clear that \[
\sum_{j=0}^{N-1} D\mT_j(\vz_0) \mT_j \vz_0 = 2 \sum_{j=1}^{n} D\mF_j(\vz_0) \mF_j \vz_0. 
    \] For the %
    second term,
    as $\mT_i = \mT_{2n-1-i}$, we have 
    \begingroup \allowdisplaybreaks
    \begin{align*} 
        \sum_{0\leq i < j \leq 2n-1}  D\mT_j(\vz_0) \mT_i \vz_0 &= \sum_{0\leq i < j \leq n-1}  D\mT_j(\vz_0) \mT_i \vz_0 + \sum_{n\leq i < j \leq 2n-1}  D\mT_j(\vz_0) \mT_i \vz_0 \\ 
        &\phantom{=} \ + \sum_{i=0}^{n-1} \sum_{j=n}^{2n-2-i} D\mT_j(\vz_0) \mT_i \vz_0  + \sum_{i=0}^{n-1} \sum_{j=2n-i}^{2n-1} D\mT_j(\vz_0) \mT_i \vz_0 \\ 
        &\phantom{=} \ + \sum_{i=0}^{n-1}  D\mT_{2n-1-i}(\vz_0) \mT_i \vz_0 \\ 
        &= \sum_{0\leq i < j \leq n-1}  D\mT_j(\vz_0) \mT_i \vz_0 + \sum_{0\leq j < i \leq n-1}  D\mT_j(\vz_0) \mT_i \vz_0 \\ 
        &\phantom{=} \ + \sum_{i=0}^{n-1} \sum_{j=i+1}^{n-1} D\mT_j(\vz_0) \mT_i \vz_0  + \sum_{i=0}^{n-1} \sum_{j=0}^{i-1} D\mT_j(\vz_0) \mT_i \vz_0 \\ 
        &\phantom{=} \ + \sum_{i=0}^{n-1}  D\mT_{i}(\vz_0) \mT_i \vz_0 \\ 
        &= 2 \sum_{0\leq i < j \leq n-1}  D\mT_j(\vz_0) \mT_i \vz_0 + 2 \sum_{0\leq j < i \leq n-1}  D\mT_j(\vz_0) \mT_i \vz_0 \\
        &\phantom{=} \ + \sum_{i=0}^{n-1}  D\mT_{i}(\vz_0) \mT_i \vz_0 . 
    \end{align*}
    \endgroup
    The claimed identity can be obtained by taking the weighted sum of the two results. 
\end{proof}

\section{Within-Epoch Error Analysis for Upper Bounds}\label{appx:within-epoch}
 
All the upper bounds for \segrr, \segff, and \segffa{} in this paper are established by following the two steps below. 

The first step is to decompose the cumulative updates made within an epoch by using the method into a sum of an exact EG update and a within-epoch error term, which we denote by $\vr^k$. In particular, we show that the error term $\vr^k$ occurring from any of \segrr, \segff, and \segffa{} can be expressed in a specific unified form (described in \Cref{thm:meta-error-bound}). This will be the main focus of this section. %

The second step is establishing a convergence rate that can be applied to \emph{any} method whose update can be decomposed into a sum of an exact EG update and an error term that is of the specific unified form mentioned above. By doing so, the convergence rates of \segrr, \segff, and \segffa{} will automatically follow as special cases of the general convergence result. This step will be dealt in Appendices~\ref{sappx:sm-analysis}~and~\ref{sappx:meta-analysis}.

To this end, for any of \segrr, \segff, and \segffa, let us decompose the cumulative updates made within an epoch into a sum of an exact EG update and a within-epoch error term $\vr^k$, as \begin{equation*} %
    \vz_0^{k+1}  = \vz_0^k - \eta_k n  \mF (\vz_0^k - \eta_k n \mF\vz_0^k) + \vr^k.
\end{equation*} 
The quality of the method will depend on how small the ``noise'' term $\vr^k$ is, as the noise will in general hinder the convergence.
As mentioned above, it turns out that, regardless of the method that is in use, the noise term can be bounded in a unified format, as follows. 
\begin{thm}\label{thm:meta-error-bound}
    Suppose that Assumptions~\ref{asmp:smoothness}~and~\ref{asmp:bounded-variance} hold.
    Then, for each of \segrr{}, \segff{}, and \segffa{}, there exists a choice of stepsizes that makes the following hold:
for an exponent~$a$ that depends on the method, 
    there exist constants $C_\textsf{1}$, $D_\textsf{1}$, $V_\textsf{1}$, $C_\textsf{2}$, $D_\textsf{2}$, and $V_\textsf{2}$, all independent of $\eta_k$ and~$n$, such that the error term $\vr^k$ satisfies a deterministic bound 
    \begin{equation}\label{eqn:rk1-ff}
        \norm{\vr^k} \leq \eta_k^a n^a C_\textsf{1} \norm{\mF \vz_0^k}  + \eta_k^a n^a D_\textsf{1} \norm{\mF \vz_0^k }^2 + \eta_k^a n^a V_\textsf{1} 
    \end{equation} and a bound that holds on expectation 
    \begin{align} %
        \expt\left[\norm{\vr}^2\middle|\vz_0^k\right] &\leq \eta_k^{2a} n^{2a} C_\textsf{2} \norm{\mF \vz_0^k}^2  + \eta_k^{2a} n^{2a} D_\textsf{2} \norm{\mF \vz_0^k }^4 + \eta_k^{2a} n^{2a-1} V_\textsf{2}.  \label{eqn:rk2-ff} 
    \end{align}
    Furthermore, the exponent is $a = 2$ for \segrr{} and \segff{}, and $a = 3$ for \segffa{}.
\end{thm}

In other words, \segffa{} has an error that is an order of magnitude smaller than other methods. 
Thus, it is now intuitively clear that \segffa{} should have an advantage in the convergence. 
The proof of \Cref{thm:meta-error-bound} is quite long and technical, so we defer it to \Cref{appx:r-bound-main}.

Within the remaining of this section only, although it is an abuse of notation, for convenience we will write $\mF_i$ to denote the saddle gradient of the component function chosen in the $i$\textsuperscript{th} iteration.  
More precisely, for indices $i =0, 1, \dots, n-1$ we denote $\mF_{\tau(i+1)}$ by $\mF_i$. 
Similarly, in cases of considering \segff{} or \segffa{}, for $i \geq n$ we denote $\mF_{\tau(2n-i)}$ by $\mF_i$. 
Also, we omit the superscripts and subscripts denoting the epoch number $k$ unless strictly necessary, as all the iterates that we consider will be from the same epoch. 

Let us reformulate the update rule \eqref{eqn:update_rule1-m} into \begin{align} \label{eqn:update-with-xi} \begin{aligned}
    \vw_i &= \vz_i - \xi \eta \mF_i \vz_i, \\
    \vz_{i+1} &= \vz_i - \eta \mF_i \vw_i. 
\end{aligned}\end{align} Note that $\xi = \nicefrac{1}{2}$ for \segffa{}, and $\xi = 1$ for \segrr{} and \segff{}.

\subsection{Auxiliary Lemmata}
For $j = 1, \dots, 2n$ we define \begin{align} 
    \vg_{j} &\coloneqq \sum_{i=0}^{j-1} \mF_i \vz_0, \\ 
    \delta_j &\coloneqq \norm{\vg_j - j \mF \vz_0}, \label{eqn:error-delta}\\ %
    \Sigma_{j} &\coloneqq \sum_{i=1}^j \delta_i, \label{eqn:error-sigma}\\ %
    \Psi_{j} &\coloneqq \sum_{i=1}^j \delta_i^2. \label{eqn:error-psi}
\end{align} 
We set $\Sigma_0 = \Psi_0 = 0$, as they are empty sums. Notice that $\delta_j$ is a random variable that depends on the permutation $\tau$.  

Meanwhile, by triangle inequality it is immediate that \[
    \norm{\vg_j} \leq j\norm{ \mF \vz_0} + \delta_j, 
\] and by Young's inequality it holds that  \[
    \norm{\vg_j}^2 \leq 2j^2\norm{ \mF \vz_0}^2 + 2\delta_j^2 .
\] 

\begin{lem} \label{lem:wi-bound} %
    For any index $i \geq 1$, it holds that  \begin{align}
        \begin{split}
        \norm{\vz_i - \vz_0} &\leq \eta \left(1+ {\xi \eta L} \right)\norm{\vg_i} \\
        &\phantom{\leq} \quad +  \eta^2 L \left( 2 \xi  + 2 \xi \eta L + \xi^2 {\eta^2 L^2} \right) \sum_{\ell=0}^{i-2} \left(1+ \eta L + {\xi \eta^2 L^2}\right)^{i - \ell - 2} \norm{\vg_{\ell+1}} , 
        \end{split} \label{eqn:zi-bound}\\ 
        \begin{split}
        \norm{\vw_i - \vz_0} &\leq \xi \eta \norm{\vg_{i+1}} + \xi \eta \left((1-\xi^{-1}) + 2\eta L + \xi {\eta^2 L^2}\right) \norm{\vg_i} \\
        &\phantom{\leq} \quad + \eta (1+\xi \eta L) \left( 2 \xi \eta L + 2 \xi \eta^2 L^2 + \xi^2 {\eta^3 L^3} \right)\sum_{\ell=0}^{i-2}\left(1+ \eta L + {\xi \eta^2 L^2}\right)^{i-\ell-2}  \norm{\vg_{\ell+1}}. 
        \end{split} \label{eqn:wi-bound}
    \end{align}  
\end{lem} 
\begin{proof}
    By the fundamental theorem of calculus for line integrals and the update rule \eqref{eqn:update-with-xi}, we have \begin{align*}
        \vw_i &= \vz_i -  \xi \eta \mF_{i} \vz_i \\
                &= \vz_i -  \xi \eta \mF_{i} \vz_0 -  \xi \eta(\mF_{i} \vz_i - \mF_{i} \vz_0) \\
                &= \vz_i -  \xi \eta\mF_{i} \vz_0 - \xi \eta\int_0^1 D \mF_{i} (\vz_0 + t(\vz_i - \vz_0)) \,\mathrm{d}t\, (\vz_i - \vz_0)
    \end{align*} and similarly \begin{align*}
        \vz_{i+1} &= \vz_i - \eta \mF_{i} \vw_i \\
                    &= \vz_i - \eta \mF_{i} \vz_0 - \eta (\mF_{i} \vw_i - \mF_{i} \vz_0) \\
                    &= \vz_i - \eta \mF_{i} \vz_0 - \eta \int_0^1 D \mF_{i} (\vz_0 + t(\vw_i - \vz_0)) \,\mathrm{d}t\, (\vw_i - \vz_0).
    \end{align*}
    Hence, by defining \begin{equation}\label{eqn:def-A-B}\begin{aligned}
        \mA_i &\coloneqq \int_0^1 D \mF_{i} (\vz_0 + t(\vz_i - \vz_0)) \,\mathrm{d}t \\
        \mB_i &\coloneqq \int_0^1 D \mF_{i} (\vz_0 + t(\vw_i - \vz_0)) \,\mathrm{d}t\
    \end{aligned}\end{equation} the update rule can be rewritten using these quantities as \begin{align}
        \vw_i &= \vz_i -  \xi \eta \mF_{i} \vz_0 -  \xi \eta \mA_i (\vz_i - \vz_0), \label{eqn:wi}\\
        \vz_{i+1} &=  \vz_i -  \eta \mF_{i} \vz_0 -  \eta \mB_i (\vw_i - \vz_0). \label{eqn:zi1}
    \end{align} Subtracting $\vz_0$ from both sides of \eqref{eqn:wi} we get \begin{equation}\label{eqn:wi-and-zi}\begin{aligned}
        \vw_i - \vz_0 &= \vz_i - \vz_0 - \xi \eta \mF_{i} \vz_0 - \xi \eta \mA_i (\vz_i - \vz_0) \\
        &= \left(\id - \xi \eta \mA_i\right) (\vz_i - \vz_0 ) - \xi \eta \mF_{i} \vz_0,  
        \end{aligned}\end{equation} and plugging this into \eqref{eqn:zi1} gives us \begin{equation}\label{eqn:recurrence} \begin{aligned}
            \vz_{i+1} - \vz_0 &=  \vz_i - \vz_0 - \eta \mF_{i} \vz_0 - \eta \mB_i (\vw_i - \vz_0) \\
            &= \vz_i - \vz_0 - \eta \mF_{i} \vz_0 - \eta \mB_i \left( \left(\id - \xi \eta \mA_i\right) (\vz_i - \vz_0 ) - \xi \eta \mF_{i} \vz_0 \right) \\
            &= \left(\id -  \eta \mB_i + \xi{\eta^2} \mB_i \mA_i\right) (\vz_i - \vz_0) - \eta\left(\id - \xi \eta \mB_i\right) \mF_{i} \vz_0.    
            \end{aligned} \end{equation} For convenience let us define \begin{align*}
            \mC_i &\coloneqq \id -  \eta \mB_i + \xi{\eta^2} \mB_i \mA_i, \\
        \mP_{i, \ell} &\coloneqq \mC_i \mC_{i-1} \dots \mC_{\ell+2} \mC_{\ell+1}    
            \end{align*} and $\mP_{i, i} \coloneqq \id$ as it denotes an empty product. Observe that for any $j$ we have \begin{equation} \label{eqn:ci-bound-1}
                \norm{\mC_j} = \norm{\id -  \eta \mB_i + \xi{\eta^2} \mB_i \mA_i} \leq 1+ \eta L + {\xi \eta^2 L^2}.
                \end{equation}  Also note that for any $\ell$ it holds that 
            \begin{align*}
                \MoveEqLeft \left(\id - \xi \eta \mB_ {\ell+1}\right) - \mC_{\ell+1} \left(\id - \xi \eta\mB_ {\ell}\right) \\ 
                &= \left(\id - \xi \eta \mB_ {\ell+1}\right) - \left(\id -  \eta \mB_{\ell+1} + \xi \eta^2 \mB_{\ell+1} \mA_{\ell+1}\right) \left(\id - \xi \eta \mB_ {\ell}\right)  \\
                &= \xi \eta (\mB_{\ell+1} + \mB_{\ell}) - {\xi \eta^2} \mB_{\ell+1}(\mA_{\ell+1} + \mB_{\ell}) + \xi^2 \eta^3 \mB_{\ell+1}\mA_{\ell+1}\mB_\ell
            \end{align*} and hence  \begin{equation} \label{eqn:ci-bound-2}
            \norm{ \left(\id - \xi \eta \mB_ {\ell+1}\right) - \mC_{\ell+1} \left(\id - \xi \eta\mB_ {\ell}\right)  }  \leq  2 \xi \eta L + 2 \xi \eta^2 L^2 + \xi^2 {\eta^3 L^3}.  %
            \end{equation} 
            
            Unravelling the recurrence relation \eqref{eqn:recurrence} we get \begin{align*}
 \vz_{i+1} - \vz_0 &= \mC_i (\vz_i - \vz_0) - \eta\left(\id - \xi \eta \mB_i\right) \mF_{i} \vz_0 \\ 
        &= \mC_i \bigl( \mC_{i-1} (\vz_{i-1} - \vz_0) - \eta\left(\id - \xi \eta \mB_{i-1}\right) \mF_{i-1} \vz_0 \bigr) - \eta\left(\id - \xi \eta \mB_i\right) \mF_{i} \vz_0 \\
        &=  \mP_{i, i-2} (\vz_{i-1} - \vz_0) - \eta \sum_{\ell = i-1}^i \mP_{i, \ell} \left(\id - \xi \eta \mB_{\ell}\right) \mF_{\ell} \vz_0\\
        &=  \mP_{i, i-2} \bigl( \mC_{i-2} (\vz_{i-2} - \vz_0) - \eta\left(\id - \xi \eta \mB_{i-2}\right) \mF_{i-2} \vz_0\bigr) -  \eta \sum_{\ell = i-1}^i \mP_{i, \ell} \left(\id - \xi \eta \mB_{\ell}\right) \mF_{\ell} \vz_0\\
        &=  \mP_{i, i-3} (\vz_{i-2} - \vz_0) -  \eta \sum_{\ell = i-2}^i \mP_{i, \ell} \left(\id - \xi \eta \mB_{\ell}\right) \mF_{\ell} \vz_0\\
        &{} \ \,  \vdots \\
        &= \mP_{i, -1} (\vz_{0} - \vz_0) -  \eta \sum_{\ell = 0}^i \mP_{i, \ell} \left(\id - \xi \eta \mB_{\ell}\right) \mF_{\ell} \vz_0
    \end{align*} and therefore \begin{equation} \label{eqn:zi-general-form}
        \vz_{i} - \vz_0 = -  \eta \sum_{\ell = 0}^{i-1} \mP_{i-1, \ell} \left(\id - \xi \eta \mB_{\ell}\right) \mF_{\ell} \vz_0.    
    \end{equation} 
    In order to compute the bound for $\norm{\vz_{i} - \vz_0}$, we use summation by parts to get \begin{align*}
        \frac{1}{ \eta} (\vz_0 - \vz_i) &= \sum_{\ell = 0}^{i-1} \mP_{i-1, \ell} \left(\id - \xi \eta \mB_{\ell}\right) \mF_{\ell} \vz_0 \\
        &= \mP_{i-1, i-1} \left(\id - \xi \eta \mB_{i-1}\right) \sum_{\ell=0}^{i-1}\mF_\ell \vz_0 \\
        &\phantom{=}\qquad - \sum_{\ell=0}^{i-2} \left(\mP_{i-1, \ell+1} \left(\id - \xi \eta \mB_{\ell+1}\right) - \mP_{i-1, \ell} \left(\id - \xi \eta \mB_{\ell}\right)\right) \sum_{j=0}^{\ell}\mF_\ell \vz_0 \\
        &= \left(\id - \xi \eta \mB_{i-1}\right) \vg_i - \sum_{\ell=0}^{i-2} \left(\mP_{i-1, \ell+1} \left(\id - \xi \eta \mB_{\ell+1}\right) - \mP_{i-1, \ell} \left(\id - \xi \eta \mB_{\ell}\right)\right) \vg_{\ell+1}.
    \end{align*} 
     Here, observe that \begin{align*}
       \MoveEqLeft[4] \mP_{i-1, \ell+1} \left(\id - \xi \eta \mB_{\ell+1}\right) - \mP_{i-1, \ell} \left(\id - \xi \eta \mB_{\ell}\right) \\
       &= \mC_{i-1}\mC_{i-2}\dots\mC_{\ell+2} \left( \left(\id - \xi \eta \mB_{\ell+1}\right) - \mC_{\ell+1}\left(\id - \xi \eta \mB_{\ell}\right)\right)
     \end{align*} so by using \eqref{eqn:ci-bound-1} and \eqref{eqn:ci-bound-2} we obtain  \begin{align*}
      \MoveEqLeft \norm{\mP_{i-1, \ell+1} \left(\id - \xi \eta \mB_{\ell+1}\right) - \mP_{i-1, \ell} \left(\id - \xi \eta \mB_{\ell}\right)} \\
        &\leq \left( 2 \xi \eta L + 2 \xi \eta^2 L^2 + \xi^2 {\eta^3 L^3} \right) \left(1+ \eta L + {\xi \eta^2 L^2}\right)^{i - \ell - 2}. 
    \end{align*} Therefore, we conclude that \[
    \norm{\vz_i - \vz_0} \leq  \eta \left(1+ {\xi \eta L} \right)\norm{\vg_i} + \eta^2 L \left( 2 \xi  + 2 \xi \eta L + \xi^2 {\eta^2 L^2} \right) \sum_{\ell=0}^{i-2} \left(1+ \eta L + {\xi \eta^2 L^2}\right)^{i - \ell - 2} \norm{\vg_{\ell+1}} .   
    \]

    Meanwhile, substituting \eqref{eqn:zi-general-form} back to \eqref{eqn:wi-and-zi} gives us \begin{equation}\label{wi-preliminary-form}
        \vw_i - \vz_0 =  - \xi{\eta} \mF_{i} \vz_0 - \eta \sum_{\ell = 0}^{i-1} \left(\id - \xi{\eta} \mA_i\right) \mP_{i-1, \ell} \left(\id - \xi{\eta}\mB_{\ell}\right) \mF_{\ell} \vz_0 .  
    \end{equation} For $\ell < i$ let us define \begin{align*}
        \mR_{i, \ell} &\coloneqq \xi^{-1} \left(\id - \xi{\eta} \mA_i\right) \mP_{i-1, \ell} \left(\id - \xi{\eta} \mB_ {\ell}\right) \\ 
        &\phantom{:}= \xi^{-1} \left(\id - \xi{\eta} \mA_i\right) \mC_{i-1}\mC_{i-2}\dots\mC_{\ell+2}\mC_{\ell+1} \left(\id - \xi{\eta} \mB_ {\ell}\right)  
    \end{align*} and for convenience $\mR_{i, i} \coloneqq \id$ so that \eqref{wi-preliminary-form} can be rewritten as \begin{equation} \label{eqn:wi-general-form}
        \frac{1}{\xi \eta}(\vz_0 - \vw_i) = \sum_{\ell = 0}^{i} \mR_{i, \ell} \mF_{\ell} \vz_0 .  
    \end{equation} Applying summation by parts on the above, we obtain \begin{align*}
        \frac{1}{\xi \eta}(\vz_0 - \vw_i) &= \mR_{i, i} \sum_{\ell = 0}^i \mF_{\ell} \vz_0 - \sum_{\ell=0}^{i-1} (\mR_{i, \ell+1} - \mR_{i, \ell})\sum_{j=0}^\ell \mF_{j} \vz_0 \\ 
        &= \vg_{i+1} - \sum_{\ell=0}^{i-1} (\mR_{i, \ell+1} - \mR_{i, \ell}) \vg_{\ell+1} 
    \end{align*} and as a consequence we get \begin{equation}\label{eqn:wi-prebound}
        \frac{1}{\xi \eta}\norm{\vw_i - \vz_0} \leq \norm{\vg_{i+1}} + \sum_{\ell=0}^{i-1} \norm{\mR_{i, \ell+1} - \mR_{i, \ell}}\norm{\vg_{\ell+1}}.
    \end{equation}
    It remains to bound $\norm{\mR_{i, \ell+1} - \mR_{i, \ell}}$. For the special case where $\ell = i-1$, a direct computation leads to  \begin{align*}
        \mR_{i, i} - \mR_{i, i-1} &= \id - \xi^{-1} \left(\id - \xi{\eta} \mA_{i}\right)\left(\id - \xi{\eta} \mB_{i-1}\right) \\
        &= (1-\xi^{-1})\id +  \eta \mA_i+  \eta \mB_{i-1} - \xi {\eta^2} \mA_i \mB_{i-1} 
    \end{align*} and thus we have \begin{equation} \label{eqn:qi-diff-bound-1}
        \norm{\mR_{i, i} - \mR_{i, i-1}} \leq (1-\xi^{-1}) + 2\eta L + \xi {\eta^2 L^2}. 
    \end{equation} For the other cases; that is, when $\ell < i-1$, we have \[
        \mR_{i, \ell+1} - \mR_{i, \ell} = \xi^{-1} \left(\id - \xi \eta \mA_i\right) \mC_{i-1}\mC_{i-2}\dots\mC_{\ell+2}\left( \left(\id - \xi \eta\mB_ {\ell+1}\right) - \mC_{\ell+1} \left(\id - \xi \eta\mB_ {\ell}\right)  \right)
    \] so by using \eqref{eqn:ci-bound-1} and \eqref{eqn:ci-bound-2} we get the bound \begin{equation}\label{eqn:qi-diff-bound-2}
        \norm{\mR_{i, \ell+1} - \mR_{i, \ell} } \leq \xi^{-1} (1+\xi \eta L) \left( 2 \xi \eta L + 2 \xi \eta^2 L^2 + \xi^2 {\eta^3 L^3} \right)\left(1+ \eta L + {\xi \eta^2 L^2}\right)^{i-\ell-2} .
    \end{equation} Applying \eqref{eqn:qi-diff-bound-1} and \eqref{eqn:qi-diff-bound-2} on \eqref{eqn:wi-prebound} gives the bound for $\norm{\vw_i - \vz_0}$. 
\end{proof}

\begin{prop}\label{thm:general-iterate-bound}\label{thm:deterministic-iterate-bound} Suppose that \segffa{} is used, $\eta < \frac{1}{nL}$, and let $\nu \coloneqq 1+\frac{1}{2n}$. Then for any $i = 1, \dots, 2n-1$ we have the bounds 
    \begin{align*}
        \norm{\vz_i - \vz_0} &\leq \left( \eta \nu i + \frac{\eta \nu^2 e^2 i(i-1)}{2n} \right) \norm{ \mF \vz_0} + \eta \nu \delta_i  +  \eta^2 L \nu^2  e^2 \Sigma_{i-1}   , \\
        \norm{\vw_i - \vz_0} &\leq \frac{\eta}{2} \left(1+2\nu^2 i+\frac{ \nu^3 e^2 i(i-1)}{n}  \right) \norm{ \mF \vz_0} + \frac{\eta}{2} \delta_{i+1} + \frac{\eta (2\nu^2 - 1)}{2} \delta_i + \eta^2 L \nu^3 e^2 \Sigma_{i-1} , \\ 
        \norm{\vw_i - \vz_0}^2  &\leq \left( \frac{3\eta^2(i+1)^2}{2} +   \frac{3\eta^2 ( 2\nu^2 - 1 )^2 i^2}{2} + \frac{\eta^2 \nu^6 e^4 i (i-1)^2 (2i-1)}{n^2} \right) \norm{\mF \vz_0}^2 \\ 
        &\phantom{\leq} \qquad  + \frac{3\eta^2}{2} \delta_{i+1}^2 + \frac{3\eta^2(2\nu^2 - 1)^2}{2} \delta_i^2  +  \frac{6 \eta^2 \nu^6 e^4 (i-1)}{n^2} \Psi_{i-1}.  
    \end{align*}
\end{prop} 
\begin{proof}
    Using elementary calculus one can show that $x \mapsto (1+\frac{1}{x} + \frac{1}{2x^2})^{x}$ increases on $x > 0$ and is bounded above by $e$. Hence for all $0 \leq \ell < i \leq 2n$ we have \[
        \left(1+ \eta L + \frac{\eta^2 L^2}{2}\right)^{i-\ell-2} \leq \left( 1+ \frac{1}{n} + \frac{1}{2n^2}\right)^{2n} \leq e^2.    
    \] 
   Applying the definitions \eqref{eqn:error-delta} and \eqref{eqn:error-sigma} on \eqref{eqn:zi-bound} and then substituting $\xi = \nicefrac{1}{2}$ we get 
    \begin{align*}
        \norm{\vz_i - \vz_0} &\leq \eta \left(1+\frac{\eta L}{2} \right)\norm{\vg_i} + \eta^2 L \left(1+\frac{\eta L}{2}\right)^2 \sum_{\ell=0}^{i-2}  \left(1+ \eta L + \frac{\eta^2 L^2}{2}\right)^{i - \ell - 2} \norm{\vg_{\ell+1}} \\
        &\leq \eta \nu \left( i\norm{ \mF \vz_0} + \delta_i \right)+ \eta^2 L \nu^2 \sum_{\ell=0}^{i-2}  e^2 \left( (\ell+1)\norm{ \mF \vz_0} + \delta_{\ell+1} \right) \\
        &\leq \eta \nu \left( i\norm{ \mF \vz_0} + \delta_i \right)+ \frac{\eta \nu^2 e^2 i(i-1)}{2n} \norm{ \mF \vz_0} + \eta^2 L \nu^2  e^2 \Sigma_{i-1} .
    \end{align*} Similarly, from \eqref{eqn:wi-bound} we get
    \begin{align*}
        \norm{\vw_i - \vz_0} &\leq \frac{\eta}{2} \norm{\vg_{i+1}} + \frac{\eta}{2}\left(1+2\eta L + \frac{\eta^2 L^2}{2}\right) \norm{\vg_i} \\
        &\phantom{\leq} \qquad + \eta^2 L \left(1+\frac{\eta L}{2}\right)^3 \sum_{\ell=0}^{i-2} \left(1+ \eta L + \frac{\eta^2 L^2}{2}\right)^{i-\ell-2} \norm{\vg_{\ell+1}}  \\ 
        &\leq \frac{\eta}{2} \left( (i+1)\norm{ \mF \vz_0} + \delta_{i+1} \right) + \frac{\eta}{2}\left(1+\frac{2}{n} + \frac{1}{2n^2}\right) \left( i\norm{ \mF \vz_0} + \delta_i \right) \\
        &\phantom{\leq} \qquad + \eta^2 L \nu^3 \sum_{\ell=0}^{i-2} e^2 \left( (\ell+1)\norm{ \mF \vz_0} + \delta_{\ell+1} \right)  \\ 
        &\leq \frac{\eta}{2} (1+2i\nu^2) \norm{ \mF \vz_0} + \frac{\eta \nu^3 e^2 i(i-1)}{2n} \norm{ \mF \vz_0} + \frac{\eta}{2} \delta_{i+1} + \frac{\eta (2\nu^2 - 1)}{2} \delta_i + \eta^2 L \nu^3 e^2 \Sigma_{i-1}.
    \end{align*}
    Finally, applying generalized Young's inequality on \eqref{eqn:wi-bound} we get\begin{align*}
        \norm{\vw_i - \vz_0}^2 &\leq \frac{3\eta^2}{4} \norm{\vg_{i+1}}^2 + \frac{3\eta^2}{4}\left(1+2\eta L + \frac{\eta^2 L^2}{2}\right)^2 \norm{\vg_i}^2 \\
        &\phantom{\leq} \qquad + 3 \left( \eta^2 L \left(1+\frac{\eta L}{2}\right)^3 \sum_{\ell=0}^{i-2} \left(1+ \eta L + \frac{\eta^2 L^2}{2}\right)^{i-\ell-2} \norm{\vg_{\ell+1}}\right)^2. 
    \end{align*} 
    \begingroup \interdisplaylinepenalty=10000 
    Using generalized Young's inequality once more on the last term gives us \begin{align*}
        3 \left( \eta^2 L \left(1+\frac{\eta L}{2}\right)^3 \sum_{\ell=0}^{i-2} \left(1+ \eta L + \frac{\eta^2 L^2}{2}\right)^{i-\ell-2} \norm{\vg_{\ell+1}}\right)^2 &\leq 3 \left( \frac{\eta \nu^3 e^2}{n} \sum_{\ell=0}^{i-2} \norm{\vg_{\ell+1}}\right)^2 \\
        &\leq \frac{3 \eta^2 \nu^6 e^4 (i-1)}{n^2}   \sum_{\ell=0}^{i-2} \norm{\vg_{\ell+1}}^2.
\end{align*} \endgroup Plugging this back yields 
\begin{align*}
    \norm{\vw_i - \vz_0}^2 &\leq \frac{3\eta^2}{4} \norm{\vg_{i+1}}^2 + \frac{3\eta^2}{4}\left(1+2\eta L + \frac{\eta^2 L^2}{2}\right)^2 \norm{\vg_i}^2 +\frac{3 \eta^2 \nu^6 e^4 (i-1)}{n^2}   \sum_{\ell=0}^{i-2} \norm{\vg_{\ell+1}}^2 \\ 
    &\leq \frac{3\eta^2}{4} \left( 2(i+1)^2\norm{\mF \vz_0}^2 + 2\delta_{i+1}^2 \right) + \frac{3\eta^2}{4}\left( 2\nu^2 - 1\right)^2 \left( 2i^2\norm{\mF \vz_0}^2 + 2\delta_i^2 \right) \\ 
    &\phantom{\leq} \qquad  + \frac{3 \eta^2 \nu^6 e^4 (i-1)}{n^2}   \sum_{\ell=0}^{i-2} \left(  2(\ell+1)^2\norm{\mF \vz_0}^2 + 2\delta_{\ell+1}^2 \right) \\
    &\leq \frac{3\eta^2}{2} \left(  (i+1)^2\norm{\mF \vz_0}^2 + \delta_{i+1}^2 \right) + \frac{3\eta^2}{2}\left( 2\nu^2 - 1\right)^2 \left(  i^2\norm{\mF \vz_0}^2 + \delta_i^2 \right) \\ 
    &\phantom{\leq} \qquad  + \frac{\eta^2 \nu^6 e^4 i (i-1)^2 (2i-1)}{n^2}    \norm{\mF \vz_0}^2 +  \frac{6 \eta^2 \nu^6 e^4 (i-1)}{n^2} \Psi_{i-1}  .
\end{align*} Now the claimed inequalities can be obtained simply by rearranging the terms appropriately. 
\end{proof} 

\begin{prop}\label{thm:rr-general-iterate-bound}\label{thm:rr-deterministic-iterate-bound} Suppose that either \segrr{} or \segff{} is used with $\alpha = \beta = \eta < \frac{1}{nL}$, and let $\tilde{\nu} \coloneqq 1+\frac{1}{n}$. Then for any $i = 1, \dots, 2n-1$ we have the bounds 
    \begin{align*}
        \norm{\vz_i - \vz_0} &\leq \left( \eta \tilde{\nu} i + \frac{16 \eta \tilde{\nu}^2  i(i-1)}{n} \right) \norm{ \mF \vz_0} + \eta \tilde{\nu} \delta_i  + 32 \eta^2 L \tilde{\nu}^2 \Sigma_{i-1}  , \\
        \norm{\vw_i - \vz_0} &\leq {\eta} \left(1+i\tilde{\nu}^2 + \frac{16  \tilde{\nu}^3  i(i-1)}{n} \right) \norm{ \mF \vz_0} + \eta \delta_{i+1} + {\eta (\tilde{\nu}^2 - 1)} \delta_i + 32\eta^2 L \tilde{\nu}^3 \Sigma_{i-1} , \\ 
        \norm{\vw_i - \vz_0}^2  &\leq \left( 6{\eta^2(i+1)^2}  +   \frac{6\eta^2\left(  1+\tilde{\nu} \right)^2 i^2}{n^2} + \frac{1024\eta^2\tilde{\nu}^6 i(i-1)^2(2i-1) }{n^2 } \right) \norm{\mF \vz_0}^2 \\ 
        &\phantom{\leq} \qquad  + 6\eta^2 \delta_{i+1}^2 + \frac{6\eta^2\left(  1+\tilde{\nu} \right)^2 }{n^2} \delta_i^2  +  \frac{6144\eta^2\tilde{\nu}^6(i-1) }{n^2 } \Psi_{i-1}.  
    \end{align*}
\end{prop} 
\begin{proof}
    One can verify that $x \mapsto (1+\frac{4}{3x})^{x}$ increases on $x \geq 3$ and is bounded above by $e^{4/3} < 4$. With noting that $(1+\frac{1}{1} + \frac{1}{1^2})^{1} = 3 < 4$, $(1+\frac{1}{2} + \frac{1}{2^2})^{2} = \frac{49}{16} < 4$, and $1+\frac{1}{x} +\frac{1}{x^2} \leq 1+\frac{4}{3x}$ whenever $x \geq 3$, we see that for all $0 \leq \ell < i \leq 2n$ it holds that \[
        \left(1+ \eta L + {\eta^2 L^2}\right)^{i-\ell-2} \leq \left( 1+ \frac{1}{n} + \frac{1}{n^2}\right)^{2n} \leq 4^2 = 16.    
    \] Also, we have \[
    2 + 2\eta L + \eta^2 L^2 \leq 2+\frac{2}{n} + \frac{1}{n^2} = 1+\tilde{\nu}^2 \leq 2\tilde{\nu}^2.  
    \]
    Applying the definitions \eqref{eqn:error-delta} and \eqref{eqn:error-sigma} on \eqref{eqn:zi-bound} and then substituting $\xi = 1$ we get 
    \begin{align*}
        \norm{\vz_i - \vz_0} &\leq  \eta \left(1+ { \eta L} \right)\norm{\vg_i}   +  \eta^2 L \left( 2   + 2  \eta L +  {\eta^2 L^2} \right) \sum_{\ell=0}^{i-2} \left(1+ \eta L + {  \eta^2 L^2}\right)^{i - \ell - 2} \norm{\vg_{\ell+1}}  \\
        &\leq \eta \tilde{\nu} \left( i\norm{ \mF \vz_0} + \delta_i \right)+ 2 \eta^2 L \tilde{\nu}^2 \sum_{\ell=0}^{i-2}  16 \left( (\ell+1)\norm{ \mF \vz_0} + \delta_{\ell+1} \right) \\
        &\leq \eta \tilde{\nu} \left( i\norm{ \mF \vz_0} + \delta_i \right)+ \frac{16 \eta \tilde{\nu}^2  i(i-1)}{n} \norm{ \mF \vz_0} + 32 \eta^2 L \tilde{\nu}^2 \Sigma_{i-1} .
    \end{align*} Similarly, from \eqref{eqn:wi-bound} we get
    \begin{align*}
        \norm{\vw_i - \vz_0} &\leq \eta \norm{\vg_{i+1}} +  \eta^2 L \left(  2 +  {\eta L}\right) \norm{\vg_i} \\
        &\phantom{\leq} \qquad +  \eta^2 L (1+  \eta L) \left(2 + 2 \eta L + {\eta^2 L^2} \right)\sum_{\ell=0}^{i-2} \left(1+ \eta L + {  \eta^2 L^2}\right)^{i-\ell-2}  \norm{\vg_{\ell+1}}. \\ 
        &\leq \eta \left( (i+1)\norm{ \mF \vz_0} + \delta_{i+1} \right) + \frac{\eta}{n} \left(2+\frac{1}{n}\right) \left( i\norm{ \mF \vz_0} + \delta_i \right) \\
        &\phantom{\leq} \qquad +  \eta^2 L \tilde{\nu} \left(2\tilde{\nu}^2 \right)\sum_{\ell=0}^{i-2} 16 \left( \ell\norm{ \mF \vz_0} + \delta_\ell \right).  \\ 
        &\leq {\eta} (1+i\tilde{\nu}^2) \norm{ \mF \vz_0} + \frac{16\eta \tilde{\nu}^3  i(i-1)}{n} \norm{ \mF \vz_0} + \eta \delta_{i+1} + {\eta (\tilde{\nu}^2 - 1)} \delta_i + 32\eta^2 L \tilde{\nu}^3 \Sigma_{i-1}.
    \end{align*}
    Finally, applying Young's inequality on \eqref{eqn:wi-bound} we get
    \begin{align*}
        \norm{\vw_i - \vz_0}^2 &\leq 3 \eta^2 \norm{\vg_{i+1}}^2 +  \frac{3 \eta^2\left(  2 +  {\eta L}\right)^2 }{n^2} \norm{\vg_i}  \\
        &\phantom{\leq} \qquad + 3 \left( \eta^2 L (1+\eta L) \left( 2  + 2 \eta L +  {\eta^2 L^2} \right)\sum_{\ell=0}^{i-2}\left(1+ \eta L + {\eta^2 L^2}\right)^{i-\ell-2}  \norm{\vg_{\ell+1}} \right)^2 \\
        &\leq 3 \eta^2 \norm{\vg_{i+1}}^2 +  \frac{3 \eta^2\left(  2 +  {\eta L}\right)^2 }{n^2} \norm{\vg_i}   + 3 \left( \frac{2\eta\tilde{\nu}^3 }{n }   \sum_{\ell=0}^{i-2} 16  \norm{\vg_{\ell+1}} \right)^2. 
    \end{align*}
   Using Young's inequality once more on the last term gives us \begin{align*}
    3 \left( \frac{32\eta\tilde{\nu}^3 }{n }   \sum_{\ell=0}^{i-2} \norm{\vg_{\ell+1}} \right)^2 &\leq  \frac{3072\eta^2\tilde{\nu}^6(i-1) }{n^2 }   \sum_{\ell=0}^{i-2} \norm{\vg_{\ell+1}}^2. 
   \end{align*} 
   Plugging this back yields \begin{align*}
    \norm{\vw_i - \vz_0}^2 &\leq 3 \eta^2 \norm{\vg_{i+1}}^2 +  \frac{3 \eta^2\left(  2 +  {\eta L}\right)^2 }{n^2} \norm{\vg_i} + \frac{3072\eta^2\tilde{\nu}^6(i-1) }{n^2 }   \sum_{\ell=0}^{i-2} \norm{\vg_{\ell+1}}^2 \\
    &\leq 6 \eta^2 \left(  (i+1)^2\norm{\mF \vz_0}^2 + \delta_{i+1}^2 \right) +  \frac{6\eta^2\left(  2 +  {\eta L}\right)^2 }{n^2} \left( i^2\norm{\mF \vz_0}^2 + \delta_{i}^2 \right) \\
    &\phantom{\leq} \qquad + \frac{6144\eta^2\tilde{\nu}^6(i-1) }{n^2 }   \sum_{\ell=0}^{i-2} \left((\ell+1)^2\norm{\mF \vz_0}^2 +  \delta_{\ell+1}^2 \right)  \\
    &\leq 6 \eta^2 \left(  (i+1)^2\norm{\mF \vz_0}^2 + \delta_{i+1}^2 \right) +  \frac{6\eta^2\left(  1+\tilde{\nu} \right)^2 }{n^2} \left( i^2\norm{\mF \vz_0}^2 + \delta_{i}^2 \right) \\
    &\phantom{\leq} \qquad + \frac{1024\eta^2\tilde{\nu}^6 i(i-1)^2(2i-1) }{n^2 }\norm{\mF \vz_0}^2   + \frac{6144\eta^2\tilde{\nu}^6(i-1) }{n^2 } \Psi_{i-1} .
\end{align*}
 Now the claimed inequalities can be obtained simply by rearranging the terms appropriately. 
\end{proof} 

Let us now derive the upper bounds for the quantities related to $\delta_j$ and $\Sigma_j$, defined in \eqref{eqn:error-delta} and \eqref{eqn:error-sigma} respectively, using the upper bound of the variance of saddle gradients \eqref{eqn:variance-control}.

\begin{lem}\label{lem:deterministic-deviation-bound}
    For any $j = 1, \dots, 2n$, it \emph{deterministically} holds that \begin{align} 
        \delta_j &\leq n (\rho \norm{\mF\vz_0} + \sigma) \label{eqn:dev-bound}. 
    \end{align} 
 \end{lem} \begin{proof}
    For any set of indices $\mathcal{J} \subset \{0, \dots, n-1\}$, by \Cref{asmp:bounded-variance} it holds that  \[ 
  \sum_{i\in\mathcal{J}}\, \norm{\mF_i \vz_0 - \mF\vz_0}^2 \leq \sum_{i=0}^{n-1}\, \norm{\mF_i \vz_0 - \mF\vz_0}^2 \leq n  (\rho \norm{\mF\vz_0} + \sigma)^2 .   
\] %
Hence, for any $j = 1, \dots, n$ we have \begin{align*}
    \norm{\vg_j - j \mF \vz_0}^2 &=  \norm{\sum_{i=0}^{j-1} \mF_i\vz_0 - j\mF \vz_0}^2 \\
    &\leq j \sum_{i=0}^{j-1} \norm{\mF_i\vz_0 - \mF \vz_0}^2 \\
    &\leq j n  (\rho \norm{\mF\vz_0} + \sigma)^2  \\
    &\leq n^2  (\rho \norm{\mF\vz_0} + \sigma)^2 ,
 \end{align*} and for any $j = n+1, \dots, 2n$ we have \begin{equation} \label{eqn:flop-sum} \begin{aligned}
    \norm{\vg_j - j \mF \vz_0}^2 &=  \norm{\sum_{i=0}^{n-1} \mF_i\vz_0 + \sum_{i=n}^{j-1} \mF_i\vz_0 - j\mF \vz_0}^2 \\ 
    &=  \norm{\sum_{i=n}^{j-1} \mF_i\vz_0 - (j-n)\mF \vz_0}^2 \\
    &=  \norm{\sum_{i=2n-j}^{n-1} \mF_i\vz_0 - (j-n)\mF \vz_0}^2 \\
    &\leq (j-n) \sum_{i=0}^{j-1} \norm{\mF_i\vz_0 - \mF \vz_0}^2 \\
    &\leq n^2  (\rho \norm{\mF\vz_0} + \sigma)^2 . 
 \end{aligned} \end{equation}
    Therefore, in any case we have \[
        \norm{\vg_j - j \mF \vz_0}^2 \leq n^2  (\rho \norm{\mF\vz_0} + \sigma)^2 .
    \] Taking square roots on both sides gives us the desired bound. 
\end{proof}

\begin{lem}\label{lem:mish-lemma}
    For any $j = 1, \dots, 2n$, it holds that \begin{align} 
        \expt_\tau [ \delta_j^2 ] &\leq \frac{n  (\rho \norm{\mF\vz_0} + \sigma)^2 }{2}, \label{eqn:mish-bound} %
    \end{align} 
 \end{lem} \begin{proof}
If $n = 1$ then the left hand side is always $0$, so there is nothing to show. So, we may assume that $n \geq 2$. Then, for any $j = 1, \dots, n$, using {Lemma~1 in \citep{Mish20SGD}} we obtain  \[
        \expt_\tau \norm{\frac{1}{j}\vg_j - \mF\vz_0}^2 \leq \frac{n-j}{j(n-1)} (\rho \norm{\mF\vz_0} + \sigma)^2 .    
 \] Multiplying both sides by $j^2$ and applying AM-GM inequality leads to \[
    \expt_\tau \norm{ \vg_j - j \mF\vz_0}^2 \leq \frac{j(n-j)}{n-1} (\rho \norm{\mF\vz_0} + \sigma)^2  \leq \frac{n^2}{4(n-1)} (\rho \norm{\mF\vz_0} + \sigma)^2  \leq \frac{n}{2} (\rho \norm{\mF\vz_0} + \sigma)^2 . 
        \] Meanwhile, for $j = n+1, \dots, 2n$, following the first few steps in \eqref{eqn:flop-sum} we get \[
            \norm{\vg_j - j \mF \vz_0}^2 =  \norm{\sum_{i=2n-j}^{n-1} \mF_i\vz_0 - (j-n)\mF \vz_0}^2 
         \] Here, once more applying {Lemma~1 of \citep{Mish20SGD}}, we get \begin{align*}
            \expt_\tau  \norm{\vg_j - j \mF \vz_0}^2 &= \expt_\tau  \norm{\sum_{i=2n-j}^{n-1} \mF_i\vz_0 - (j-n)\mF \vz_0}^2 \\
     &= (j-n)^2 \expt_\tau  \norm{\frac{1}{j-n}\sum_{i=2n-j}^{n-1} \mF_i\vz_0 -  \mF \vz_0}^2  \\
     &\leq  (j-n)^2 \cdot \frac{n-(j-n)}{(j-n)(n-1)} (\rho \norm{\mF\vz_0} + \sigma)^2  \\ 
     &\leq  \frac{(j-n)(2n-j)}{n-1} (\rho \norm{\mF\vz_0} + \sigma)^2 . 
         \end{align*} Using AM-GM inequality on the last line gives us  \[
            \expt_\tau \norm{ \vg_j - j \mF\vz_0}^2 \leq \frac{(j-n)(2n-j)}{n-1} (\rho \norm{\mF\vz_0} + \sigma)^2  \leq \frac{n^2}{4(n-1)} (\rho \norm{\mF\vz_0} + \sigma)^2  \leq \frac{n}{2} (\rho \norm{\mF\vz_0} + \sigma)^2  .
                \] Thus, for any case, we have \eqref{eqn:mish-bound}. 
 \end{proof}

 \begin{lem}\label{lem:sigma-concentration}
    For any $k, \ell \in \{0, 1, \dots, 2n\}$, it holds that \begin{align} 
        \expt_\tau [ \Sigma_k \Sigma_\ell ] &\leq \frac{k\ell n (\rho \norm{\mF\vz_0} + \sigma)^2 }{2}.  \label{eqn:sigma-sigma-bound}  
    \end{align} 
 \end{lem}  \begin{proof}
    Expanding the product $\Sigma_k \Sigma_\ell$ and writing in terms of $\delta$, we get \begin{align*}
        \Sigma_k \Sigma_\ell = \left( \sum_{i=1}^k \delta_i \right) \left( \sum_{j=1}^\ell \delta_j \right) &= \sum_{i=1}^k \sum_{j=1}^\ell \delta_i  \delta_j   \\
        &\leq \sum_{i=1}^k \sum_{j=1}^\ell \frac{\delta_i^2 + \delta_j^2}{2}
    \end{align*} where the last line follows from the AM-GM inequality. Taking the expectation with respect to $\tau$ and using the bound from \Cref{lem:mish-lemma}, we obtain \begin{align*}
            \expt_{\tau}[ \Sigma_k \Sigma_\ell ] &\leq \frac{1}{2} \sum_{i=1}^k \sum_{j=1}^\ell \left( \expt_{\tau}[\delta_i^2] + \expt_{\tau}[\delta_j^2] \right) \\
            &\leq \frac{1}{2} \sum_{i=1}^k \sum_{j=1}^\ell \left(\frac{n (\rho \norm{\mF\vz_0} + \sigma)^2 }{2} + \frac{n (\rho \norm{\mF\vz_0} + \sigma)^2 }{2} \right) \\
            &= \frac{k\ell n (\rho \norm{\mF\vz_0} + \sigma)^2 }{2}
        \end{align*} which is exactly the claimed.
 \end{proof} 
 \begin{lem}\label{lem:sum-sigma-concentration}
    For any $k, \ell \in \{0, 1, \dots, 2n\}$, it holds that \begin{align} 
        \expt_\tau \left[ \left( \sum_{i=1}^k \Sigma_i \right) \left( \sum_{j=1}^{\ell}\Sigma_j \right) \right] &\leq \frac{k(k+1)\ell(\ell+1) n (\rho \norm{\mF\vz_0} + \sigma)^2 }{8}.  \label{eqn:sigma-sigma-sigma-sigma-bound}  
    \end{align} 
 \end{lem}  \begin{proof}
    Expanding the product in the left hand side of \eqref{eqn:sigma-sigma-sigma-sigma-bound} and applying~\eqref{eqn:sigma-sigma-bound}, we get \begin{align*}
        \expt_\tau \left[ \left( \sum_{i=1}^k \Sigma_i \right) \left( \sum_{j=1}^{\ell}\Sigma_j \right)\right] = \expt_\tau \left[ \sum_{i=1}^k \sum_{j=1}^{\ell} \Sigma_i \Sigma_j \right] &=  \sum_{i=1}^k \sum_{j=1}^{\ell} \expt_\tau [ \Sigma_i \Sigma_j ] \\
         &\leq \sum_{i=1}^k \sum_{j=1}^\ell \frac{ijn (\rho \norm{\mF\vz_0} + \sigma)^2 }{2} \\
         &\leq \frac{k(k+1)\ell(\ell+1) n (\rho \norm{\mF\vz_0} + \sigma)^2 }{8}. \qedhere
    \end{align*}  
 \end{proof}

 \subsection{Upper Bounds of the Within-Epoch Errors}\label{appx:r-bound-main}

The full proof of \Cref{thm:meta-error-bound} is quite long and technical, so we divide it into several parts. 
First we show that \eqref{eqn:rk1-ff} and \eqref{eqn:rk2-ff} holds with $a = 3$ when \segffa{} is in use. %
Then we show that \Cref{thm:meta-error-bound} also holds for \segff{} in \Cref{appx:r-bound-segff}, and for \segrr{} in \Cref{appx:r-bound-segrr}. 

Throughout the remaining of this section, we always assume that the variance of the saddle gradients satisfies \eqref{eqn:variance-control}.

\subsubsection{Proof of \titleref{Equation}{eqn:rk1-ff} for \segffa{}} \label{appx:eqn:rk1-ff}
 
In this section we prove the following. 

\begin{thm}\label{thm:rk1-ffa} Say we use \segffa{}. Then, as long as the stepsize used in an epoch satisfies $\eta < \frac{1}{nL}$, it holds that 
    \begin{equation}\label{eqn:rk1-ffa}
    \norm{\vr} \leq \eta^3 n^3 C_\textsf{1A} \norm{\mF \vz_0}  + \eta^3 n^3 D_\textsf{1A} \norm{\mF \vz_0 }^2 + \eta^3 n^3 V_\textsf{1A} 
\end{equation} 
for constants \begin{align} 
    C_\textsf{1A} &\coloneqq L^2 \left( \frac{1}{2} \left(  1 + \frac{2 e^2}{3} \right) + \frac{6 +e^2 }{3}  +  15 \rho \right), \label{eqn:def-c1a}\\
    D_\textsf{1A} &\coloneqq M \left( \frac{83}{4} + \frac{ 24 e^4 }{5} + \rho^2 \left( \frac{243}{16}  + {27 e^4} \right) \right), \label{eqn:def-d1a} \\
    V_\textsf{1A} &\coloneqq   M \sigma^2 \left( \frac{243}{16}  + {27 e^4} \right) + 15   L^2 \sigma. \label{eqn:def-v1a}
\end{align}
\end{thm}

We first list the intermediate results. The actual proof of \Cref{thm:rk1-ffa} is in page~\pageref{prf:rk1-ffa}, at the end of this section. 

 \begin{prop} %
    For using \segffa, the within-epoch update $\vz^\sharp$ as given by~\eqref{eqn:anchoring} satisfies
    \[
    \vz^\sharp  = \vz_0 - n \eta \mF (\vz_0 - n \eta\mF\vz_0) + \vr 
    \] where we denote \begin{subequations}\label{eqn:appx:rk} \begin{align}
        \vr \coloneqq&\ n\eta \mF (\vz_0 - n\eta\mF\vz_0) -n\eta \mF\vz_0 + n^2 \eta^2 D\mF(\vz_0) \mF \vz_0 \label{eqn:appx:rk1} \\
           &\ -\frac{\eta}{2} \sum_{j=0}^{2n-1} \Bigl(\mF_j\vw_j - \mF_j \vz_0 - D\mF_j(\vz_0)(\vw_j - \vz_0) \Bigr) \label{eqn:appx:rk2}  \\
           &\ +\frac{\eta^2}{4} \sum_{j=0}^{2n-1} D\mF_j(\vz_0) (\mF_j \vz_j - \mF_j \vz_0) \label{eqn:appx:rk3} \\
           &\ + \frac{\eta^2}{2} \sum_{j=0}^{2n-1} D\mF_j(\vz_0) \sum_{k=0}^{j-1}(\mF_k \vw_k - \mF_k \vz_0). \label{eqn:appx:rk4} 
    \end{align} \end{subequations}
\end{prop}
\begin{proof}
Setting $\alpha = \nicefrac{\eta}{2}$, $\beta = \eta$, and $\theta = 1$ in \eqref{eqn:appx:approx-m}, we get \begin{equation} \label{eqn:appx:z-plus-raw}
    \vz^\sharp = \vz_0 - \frac{\eta}{2} \sum_{j=0}^{2n-1} \mF_j \vz_0 + \frac{\eta^2}{4} \sum_{j=0}^{2n-1} D\mF_j(\vz_0) \mF_j \vz_0 +  \frac{\eta^2}{2} \sum_{0\leq k < j \leq 2n-1}  D\mF_j(\vz_0) \mF_k \vz_0   + \frac{1}{2} \veps_{2n}
    \end{equation} where $\veps_{2n}$ is defined as in \eqref{eqn:appx-approx-error-m}. 
Recall that $\mF_{i} = \mF_{2n-1-i}$ for all $i = 0, 1, \dots, 2n-1$, and moreover, $\sum_{i=0}^{n-1} \mF_i = \sum_{i=n}^{2n-1}\mF_i = n\mF$. Thus, the first sum in the above is equal to $2n\mF\vz_0$, and the second sum is equal to $2\sum_{j=0}^{n-1} D\mF_j(\vz_0) \mF_j \vz_0$. For the last sum, observe that \begin{align*}
\sum_{0\leq k < j \leq 2n-1}  D\mF_j(\vz_0) \mF_k \vz_0 &= \sum_{0\leq k < j \leq n-1}  D\mF_j(\vz_0) \mF_k \vz_0 + \sum_{n\leq k < j \leq 2n-1}  D\mF_j(\vz_0) \mF_k \vz_0  \\
&\hphantom{= \sum_{0\leq k < j \leq n-1}  D\mF_j(\vz_0) \mF_k \vz_0} \ + \sum_{\substack{0\leq k \leq n-1 \\[1pt]  n \leq j \leq 2n-1}}  D\mF_j(\vz_0) \mF_k \vz_0 \\
&= \sum_{0\leq k < j \leq n-1}  D\mF_j(\vz_0) \mF_k \vz_0 + \sum_{n-1 \geq k > j \geq 0}  D\mF_j(\vz_0) \mF_k \vz_0  \\
&\hphantom{= \sum_{0\leq k < j \leq n-1}  D\mF_j(\vz_0) \mF_k \vz_0} \ + \sum_{\substack{0\leq k \leq n-1 \\[1pt]  n-1 \geq j \geq 0}}  D\mF_j(\vz_0) \mF_k \vz_0 \\ 
&= 2 \sum_{k \neq j}  D\mF_j(\vz_0) \mF_k \vz_0 + \sum_{j=0}^{n-1}  D\mF_j(\vz_0) \mF_j \vz_0 .
\end{align*} Hence, \eqref{eqn:appx:z-plus-raw} is equivalent to \begin{align*} 
\vz^\sharp &= \vz_0 - n \eta \mF \vz_0 + \frac{\eta^2}{2} \sum_{j=0}^{n-1} D\mF_j(\vz_0) \mF_j \vz_0 +  \frac{\eta^2}{2} \sum_{0\leq k < j \leq 2n-1}  D\mF_j(\vz_0) \mF_k \vz_0   + \frac{1}{2} \veps_{2n} \\
&= \vz_0 - n \eta \mF \vz_0 + \eta^2 \sum_{j=0}^{n-1} D\mF_j(\vz_0) \mF_j \vz_0 +  \eta^2 \sum_{k \neq j}  D\mF_j(\vz_0) \mF_k \vz_0 + \frac{1}{2} \veps_{2n} \\
&= \vz_0 - n \eta \mF \vz_0 + \eta^2 \left( \sum_{j=0}^{n-1} D\mF_j(\vz_0) \right) \left(\sum_{j=0}^{n-1} \mF_j \vz_0 \right)+ \frac{1}{2} \veps_{2n} \\
&= \vz_0 - n \eta \mF \vz_0 + n^2 \eta^2 D\mF(\vz_0)\mF \vz_0 + \frac{1}{2} \veps_{2n}.
\end{align*}
Observing that the terms \eqref{eqn:appx:rk2}, \eqref{eqn:appx:rk3}, and \eqref{eqn:appx:rk4} add up to $\frac{1}{2}\veps_{2n}$ completes the proof. 
\end{proof} 
 \begin{prop}\label{prop:appx-general-r-bound}
    Suppose that $\eta < \frac{1}{nL}$, and let $\nu \coloneqq 1+\frac{1}{2n}$. Then the noise term satisfies the bound \begin{align*}
    \norm{\vr} &\leq  {\eta^3 n^3 L^2} \norm{\mF \vz_0} \left( \frac{1}{2 n}\left(  1 + \frac{2 e^2}{3} \right) + \frac{4 \nu +e^2 }{3}  \right) \\
    &\phantom{\leq} \qquad + {\eta^3 n^3 M} \norm{\mF \vz_0 }^2 \left(\frac{1}{2} + { 4 \nu^4 } + \frac{ 16 \nu e^4 }{5} \right) \\
    &\phantom{\leq} \qquad + \frac{3\eta^3 M}{8} \left( \Psi_{2n} + (2\nu^2 - 1)^2 \Psi_{2n-1}  +  \frac{4\nu^6 e^4}{n^2} \sum_{j=1}^{2n-2} j \Psi_j \right) \\  
    &\phantom{\leq} \qquad  +  \frac{\eta^3 L^2 (\nu+1)}{4} \Sigma_{2n-1} + \frac{\eta^3 L^2 \nu^2 (1 + \eta L e^2)}{2} \sum_{j=1}^{2n-2} \Sigma_{j} + \frac{\eta^4 L^3 \nu^3 e^2}{2} \sum_{k=1}^{2n-2} (2n-k-1) \Sigma_{k-1}  .
\end{align*}
\end{prop} \begin{proof}
We bound each line in equation \eqref{eqn:appx:rk}. For \eqref{eqn:appx:rk1}, we use \Cref{lem:hess-bound} to get \begin{align*}
   \norm{ n\eta \mF (\vz_0 - n\eta\mF\vz_0) -n\eta \mF\vz_0 + n^2 \eta^2 D\mF(\vz_0) \mF \vz_0 } &\leq \frac{n\eta M}{2} \norm{-n\eta \mF \vz_0 }^2 \\ 
   &= \frac{n^3\eta^3 M}{2} \norm{\mF \vz_0 }^2.
\end{align*} In bounding the remaining three lines we repeatedly use the bounds obtained in \Cref{thm:general-iterate-bound}. We will also use the following bounds, which follows from \eqref{eqn:update-with-xi}, \eqref{eqn:error-delta}, and Young's inequality: \begin{align*}
  \norm{\vw_0 - \vz_0}   &=  \frac{\eta}{2} \norm{\vg_1}  \leq \frac{\eta}{2} \norm{\mF\vz_0} + \frac{\eta}{2} \delta_1, \\
  \norm{\vw_0 - \vz_0}^2 &=  \frac{\eta^2}{4} \norm{\vg_1}^2  \leq \frac{\eta^2}{2} \norm{\mF\vz_0}^2 + \frac{\eta^2}{2} \delta_1^2.
    \end{align*} For \eqref{eqn:appx:rk2}, observe that \Cref{lem:hess-bound} gives us \[
    \norm { \mF_j\vw_j - \mF_j \vz_0 - D\mF_j(\vz_0)(\vw_j - \vz_0)} \leq \frac{ M }{2} \norm{\vw_j - \vz_0}^2.
    \] Thus, by using the bound obtained in \Cref{thm:general-iterate-bound}, we get 
    \begin{align*}
        \MoveEqLeft \norm{ -\frac{\eta}{2} \sum_{j=0}^{2n-1} \Bigl(\mF_j\vw_j - \mF_j \vz_0 - D\mF_j(\vz_0)(\vw_j - \vz_0) \Bigr) } \\
        &\leq \frac{\eta}{2} \sum_{j=0}^{2n-1} \norm{ \mF_j\vw_j - \mF_j \vz_0 - D\mF_j(\vz_0)(\vw_j - \vz_0) } \\  
        &\leq \frac{\eta M}{4} \sum_{j=0}^{2n-1} \norm{ \vw_j - \vz_0 }^2 \\ 
        &\leq \frac{\eta M}{4} \sum_{j=1}^{2n-1} \left( \frac{3\eta^2(j+1)^2}{2} +   \frac{3\eta^2 ( 2\nu^2 - 1 )^2 j^2}{2} + \frac{\eta^2 \nu^6 e^4 j (j-1)^2 (2j-1)}{n^2} \right) \norm{\mF \vz_0}^2 \\
        &\phantom{\leq} \qquad + \frac{\eta M}{4} \sum_{j=1}^{2n-1}  \left( \frac{3\eta^2}{2} \delta_{j+1}^2 + \frac{3\eta^2(2\nu^2 - 1)^2}{2} \delta_j^2  +  \frac{6 \eta^2 \nu^6 e^4 (j-1)}{n^2} \Psi_{j-1} \right) \\
        &\phantom{\leq} \qquad + \frac{\eta M}{4} \norm{\vw_0 - \vz_0}^2 \\ 
        &= \frac{\eta M}{4} \left( \frac{\eta^2 n (1+2n)(1+4n)-3\eta^2}{2} + \frac{\eta^2 ( 2\nu^2 - 1 )^2 n(2n-1)(4n-1)}{2} \right.\\
        &\phantom{= \frac{\eta M}{4} \left( \frac{\eta^2 n (1+2n)(1+4n)-3\eta^2}{2}  \right. \ } + \left. \frac{ \eta^2 \nu^6 e^4 (n-1)(2n-1)(32n^2-42n+11)}{5n} \right) \norm{\mF \vz_0}^2 \\
     &\phantom{\leq} \qquad + \frac{3\eta^3 M}{8} (\Psi_{2n} - \delta_1^2) + \frac{3\eta^3 M (2\nu^2 - 1)^2}{8} \Psi_{2n-1}  +  \frac{3 \eta^3 M \nu^6 e^4}{2n^2} \sum_{j=1}^{2n-1} (j-1) \Psi_{j-1}  \\
        &\phantom{\leq} \qquad + \frac{\eta M}{4} \left( \frac{\eta^2}{2} \norm{\mF\vz_0}^2 + \frac{\eta^2}{2} \delta_1^2 \right) \\
        &\leq {\eta^3 n^3 M} \left( { \nu^2  }+ { ( 2\nu^2 - 1 )^2 } + \frac{ 16 \nu e^4 }{5} \right) \norm{\mF \vz_0}^2  \\
     &\phantom{\leq} \qquad + \frac{3\eta^3 M}{8} \Psi_{2n} + \frac{3\eta^3 M (2\nu^2 - 1)^2}{8} \Psi_{2n-1}  +  \frac{3 \eta^3 M \nu^6 e^4}{2n^2} \sum_{j=1}^{2n-2} j \Psi_j   \\
        &\leq {\eta^3 n^3 M} \left( { 4 \nu^4 } + \frac{ 16 \nu e^4 }{5} \right) \norm{\mF \vz_0}^2 + \frac{3\eta^3 M}{8} \left( \Psi_{2n} + (2\nu^2 - 1)^2 \Psi_{2n-1}  +  \frac{4\nu^6 e^4}{n^2} \sum_{j=1}^{2n-2} j \Psi_j \right) 
      \end{align*} where along the derivation we used the inequality \[
\nu^5 (n-1)(2n-1)(32n^2-42n+11) \leq 64n^4
      \] which holds for all $n \geq 1$. From now on, we will keep on using similar techniques to reduce the exponents of $\nu$, without explicitly stating the inequalities used, but recovering the inequalities that are used should be clear from context.   

For \eqref{eqn:appx:rk3}, we use $L$-smoothness of $\mF_j$, and also the fact that it implies $\norm{D\mF_j (\vz_0)} \leq L$, to get \begin{align*}
    \MoveEqLeft \norm{ \frac{\eta^2}{4} \sum_{j=0}^{2n-1} D\mF_j(\vz_0) (\mF_j \vz_j - \mF_j \vz_0) } \\ 
    &\leq \frac{\eta^2}{4} \sum_{j=0}^{2n-1} \norm{  D\mF_j(\vz_0)} \norm{ \mF_j \vz_j - \mF_j \vz_0 } \\
    &\leq \frac{\eta^2 L^2}{4} \sum_{j=0}^{2n-1}  \norm{ \vz_j - \vz_0 } \\
    &\leq \frac{\eta^2 L^2}{4} \sum_{j=1}^{2n-1} \left( \left( \eta \nu j + \frac{\eta \nu^2 e^2 j(j-1)}{2n} \right) \norm{ \mF \vz_0} + \eta \nu \delta_j +  \eta^2 L \nu^2  e^2 \Sigma_{j-1} \right)  \\
    &= \frac{\eta^2 L^2}{4}   \left( \eta \nu n(2n-1) + \frac{2 \eta \nu^2 e^2 (n-1)(2n-1)}{3} \right) \norm{ \mF \vz_0} \\
    &\phantom{\leq  \frac{\eta^2 L^2}{4} } \  + \frac{\eta^3 L^2 \nu}{4} \Sigma_{2n-1}  + \frac{\eta^4 L^3 \nu^2  e^2}{4} \sum_{j=1}^{2n-1}  \Sigma_{j-1} \\ 
    &\leq \frac{\eta^3 n^2 L^2}{2}   \left( 1 + \frac{2 e^2}{3} \right) \norm{ \mF \vz_0} +  \frac{\eta^3 L^2 \nu}{4} \Sigma_{2n-1}  + \frac{\eta^4 L^3 \nu^2  e^2}{4} \sum_{j=1}^{2n-2}  \Sigma_j.
\end{align*}

By the same logic, each summand in \eqref{eqn:appx:rk4} with $j > 0$ can be bounded as \begin{align*}
    \MoveEqLeft \norm{ D\mF_j(\vz_0) \sum_{k=0}^{j-1}(\mF_k \vw_k - \mF_k \vz_0) } \\
     &\leq \norm{ D\mF_j(\vz_0)} \sum_{k=0}^{j-1} \norm{ \mF_k \vw_k - \mF_k \vz_0 } \\
     &\leq L^2 \sum_{k=0}^{j-1} \norm{ \vw_k - \vz_0 } \\ 
     &\leq L^2 \left( \frac{\eta}{2} \norm{\mF\vz_0} + \frac{\eta}{2} \delta_1 \right) \\
    &\phantom{\leq} \qquad + L^2 \sum_{k=1}^{j-1}  \frac{\eta}{2} \left(1+2\nu^2 k+\frac{ \nu^3 e^2 k(k-1)}{n}  \right) \norm{ \mF \vz_0} \\ 
    &\phantom{\leq} \qquad + L^2 \sum_{k=1}^{j-1} \left(  \frac{\eta}{2} \delta_{k+1} + \frac{\eta (2\nu^2 - 1)}{2} \delta_k + \eta^2 L \nu^3 e^2 \Sigma_{k-1} \right) \\ 
    &= \frac{\eta L^2}{2} \left(  \norm{\mF\vz_0} + \delta_1 \right) + \frac{\eta L^2}{2} \left(j-1+\nu^2 j(j-1)+\frac{ \nu^3 e^2 j(j-1)(j-2)}{3n}  \right) \norm{ \mF \vz_0} \\
        &\phantom{\leq} \qquad + \frac{\eta L^2}{2} (\Sigma_j - \delta_1) + \frac{\eta L^2 (2\nu^2 - 1)}{2} \Sigma_{j-1} + \eta^2 L^3 \nu^3 e^2 \sum_{k=1}^{j-1} \Sigma_{k-1} \\  
    &= \frac{\eta L^2}{2} \left(j +\nu^2 j(j-1)+\frac{ \nu^3 e^2 j(j-1)(j-2)}{3n}  \right) \norm{ \mF \vz_0} \\
        &\phantom{\leq} \qquad + \frac{\eta L^2}{2} \Sigma_j + \frac{\eta L^2 (2\nu^2 - 1)}{2} \Sigma_{j-1} + \eta^2 L^3 \nu^3 e^2 \sum_{k=1}^{j-1} \Sigma_{k-1} , 
 \end{align*} and when $j = 0$ the sum with respect to $k$ becomes an empty sum. 
 Thus, \eqref{eqn:appx:rk4} in total satisfies the bound \begin{align*}
    \MoveEqLeft \norm{\frac{\eta^2}{2} \sum_{j=0}^{2n-1} D\mF_j(\vz_0) \sum_{k=0}^{j-1}(\mF_k \vw_k - \mF_k \vz_0)} \\ 
    &\leq \frac{\eta^2}{2}  \sum_{j=0}^{2n-1} \norm{D\mF_j(\vz_0) \sum_{k=0}^{j-1}(\mF_k \vw_k - \mF_k \vz_0)} \\ 
    &\leq \frac{\eta^3 L^2}{4}  \sum_{j=1}^{2n-1} \left(j +\nu^2 j(j-1)+\frac{ \nu^3 e^2 j(j-1)(j-2)}{3n}  \right) \norm{ \mF \vz_0} \\  
    &\phantom{\leq} \qquad + \frac{\eta^2}{2}  \sum_{j=1}^{2n-1} \left( \frac{\eta L^2}{2} \Sigma_j + \frac{\eta L^2 (2\nu^2 - 1)}{2} \Sigma_{j-1} + \eta^2 L^3 \nu^3 e^2 \sum_{k=1}^{j-1} \Sigma_{k-1} \right)   \\ 
    &= \frac{\eta^3 L^2}{4} \left(n(2n-1) + \frac{4 \nu^2 n (n-1) (2n-1) }{3} + \frac{ \nu^3 e^2 (n-1)(2n-1)(2n-3)}{3}  \right) \norm{ \mF \vz_0} \\  
       &\phantom{\leq} \qquad + \frac{\eta^3 L^2}{4} \sum_{j=1}^{2n-1} \Sigma_j + \frac{\eta^3 L^2 (2\nu^2 - 1)}{4} \sum_{j=1}^{2n-1} \Sigma_{j-1} + \frac{\eta^4 L^3 \nu^3 e^2}{2} \sum_{j=1}^{2n-1} \sum_{k=1}^{j-1} \Sigma_{k-1} \\ 
    &\leq \frac{\eta^3 L^2}{2} \left( n^2 + \frac{4 n^3 }{3} + \frac{ 2 e^2 n^3}{3}  \right) \norm{ \mF \vz_0} \\  
       &\phantom{\leq} \qquad + \frac{\eta^3 L^2}{4} \sum_{j=1}^{2n-1} \Sigma_j + \frac{\eta^3 L^2 (2\nu^2 - 1)}{4} \sum_{j=1}^{2n-2} \Sigma_{j} + \frac{\eta^4 L^3 \nu^3 e^2}{2} \sum_{k=1}^{2n-2} \sum_{j=k+1}^{2n-1} \Sigma_{k-1} \\ 
    &\leq  {\eta^3 n^3 L^2} \left(\frac{4 \nu + e^2 }{3}  \right) \norm{ \mF \vz_0}  \\  
    &\phantom{\leq} \qquad + \frac{\eta^3 L^2}{4} \Sigma_{2n-1} + \frac{\eta^3 L^2 \nu^2}{2} \sum_{j=1}^{2n-2} \Sigma_{j} + \frac{\eta^4 L^3 \nu^3 e^2}{2} \sum_{k=1}^{2n-2} (2n-k-1) \Sigma_{k-1}  . 
 \end{align*} %
Simply collecting all the inequalities and rearranging the terms leads to the claimed bound. 
\end{proof}

Before we proceed, let us write \begin{align}
    X_1 &\coloneqq \frac{3\eta^3 M}{8} \left( \Psi_{2n} + (2\nu^2 - 1)^2 \Psi_{2n-1}  +  \frac{4\nu^6 e^4}{n^2} \sum_{j=1}^{2n-2} j \Psi_j \right), \label{eqn:error-x1} \\
    X_2 &\coloneqq \frac{\eta^3 L^2 (\nu+1)}{4} \Sigma_{2n-1} + \frac{\eta^3 L^2 \nu^2 (1 + \eta L e^2)}{2} \sum_{j=1}^{2n-2} \Sigma_{j} + \frac{\eta^4 L^3 \nu^3 e^2}{2} \sum_{k=1}^{2n-2} (2n-k-1) \Sigma_{k-1} \label{eqn:error-x2}
 \end{align} so that the bound on $\norm{\vr}$ obtained in \Cref{prop:appx-general-r-bound} can be written as \begin{equation}\label{eqn:r-bound-compact}
    \begin{aligned}
        \norm{\vr} &\leq  {\eta^3 n^3 L^2} \norm{\mF \vz_0} \left( \frac{1}{2 n}\left(  1 + \frac{2 e^2}{3} \right) + \frac{4 \nu +e^2 }{3}  \right) \\
    &\phantom{\leq} \qquad + {\eta^3 n^3 M} \norm{\mF \vz_0 }^2 \left(\frac{1}{2} + { 4 \nu^4 } + \frac{ 16 \nu e^4 }{5} \right) \\
    &\phantom{\leq} \qquad + X_1 + X_2.
    \end{aligned}
 \end{equation}

 \begin{thm} %
    \label{thm:appx-deterministic-r-bound}
    Suppose that $\eta < \frac{1}{nL}$, and let $\nu \coloneqq 1+\frac{1}{2n}$. Then the noise term \emph{deterministically} satisfies the bound \begin{align*}
   \norm{\vr} &\leq  {\eta^3 n^3 L^2} \norm{\mF \vz_0} \left( \frac{1}{2 n}\left(  1 + \frac{2 e^2}{3} \right) + \frac{4 \nu +e^2 }{3}  +  10 \nu \rho \right) \\
   &\phantom{\leq} \qquad + {\eta^3 n^3 M} \norm{\mF \vz_0 }^2 \left(\frac{1}{2} + { 4 \nu^4 } + \frac{ 16 \nu e^4 }{5} + \rho^2 \left( 3\nu^4  + {8\nu^3 e^4} \right)\right) \\
   &\phantom{\leq} \qquad +  \eta^3 n^3 M \sigma^2 \left( 3\nu^4  + {8\nu^3 e^4} \right) + 10 \nu \eta^3 n^3 L^2 \sigma.
\end{align*}
\end{thm}
\begin{proof} From \eqref{eqn:error-sigma}, \eqref{eqn:error-psi}, and \Cref{lem:deterministic-deviation-bound}, it holds that \begin{align}
   \Sigma_{j} &= \sum_{i=1}^j \delta_i \leq jn(\rho \norm{\mF\vz_0} + \sigma), \\
   \Psi_{j}   &= \sum_{i=1}^j \delta_i^2 \leq jn^2 (\rho \norm{\mF\vz_0} + \sigma)^2 \label{eqn:psi-bound}.   
   \end{align}
Plugging the bound for $\Psi_{j}$ into \eqref{eqn:error-x1} we get  \begin{align*}
   X_1 &\leq \frac{3\eta^3 M}{8} \!\left( 2n^3  (\rho \norm{\mF\vz_0} + \sigma)^2  + (2\nu^2 - 1)^2 (2n-1) 
   n^2 (\rho \norm{\mF\vz_0} + \sigma)^2  +  {4\nu^6 e^4} \sum_{j=1}^{2n-2} j^2  (\rho \norm{\mF\vz_0} + \sigma)^2  \right)  \\ 
   &= \frac{3\eta^3 M}{8} \left( \left(2n^3+ (2\nu^2 - 1)^2 (2n-1) 
   n^2\right) +  \frac{4\nu^6 e^4}{3} ( n-1) ( 2 n-1) ( 4 n-3) \right) (\rho \norm{\mF\vz_0} + \sigma)^2 \\
   &\leq \frac{3\eta^3 M}{8} \left( 4\nu^4 n^3 +  \frac{32\nu^3 e^4 n^3}{3} \right)  (\rho \norm{\mF\vz_0} + \sigma)^2 \\
   &= \frac{\eta^3 n^3 M  (\rho \norm{\mF\vz_0} + \sigma)^2 }{2} \left( 3\nu^4  + {8\nu^3 e^4} \right). 
\end{align*} By Young's inequality, it holds that \[
\frac{ (\rho \norm{\mF\vz_0} + \sigma)^2 }{2} \leq  \rho^2 \norm{\mF \vz_0}^2 + \sigma^2,   
\] from which we get \begin{equation}\label{eqn:deterministic-x1-bound}
   X_1 \leq \eta^3 n^3 M \rho^2 \norm{\mF \vz_0}^2 \left( 3\nu^4  + {8\nu^3 e^4} \right) + \eta^3 n^3 M \sigma^2 \left( 3\nu^4  + {8\nu^3 e^4} \right).
\end{equation} Meanwhile, plugging the bound for $\Sigma_{j}$ into \eqref{eqn:error-x2} we get  \begin{align*}
   X_2 &\leq \frac{\eta^3 L^2 (\nu+1)}{4} (2n-1)n(\rho \norm{\mF\vz_0} + \sigma) + \frac{\eta^3 L^2 \nu^2 (1 + \eta L e^2)}{2} \sum_{j=1}^{2n-2} {j}n (\rho \norm{\mF\vz_0} + \sigma)\\ 
   &\phantom{\leq} \qquad + \frac{\eta^4 L^3 \nu^3 e^2}{2} \sum_{k=1}^{2n-2} (2n-k-1) (k-1)  n (\rho \norm{\mF\vz_0} + \sigma) \\
   &= \frac{\eta^3 L^2 (\nu+1)}{4} (2n-1)n (\rho \norm{\mF\vz_0} + \sigma) + \frac{\eta^3 L^2 \nu^2 (1 + \eta L e^2)}{2} (n-1)(2n-1) n (\rho \norm{\mF\vz_0} + \sigma) \\ 
   &\phantom{\leq} \qquad + \frac{\eta^4 L^3 \nu^3 e^2}{6} \left(-3 + 11 n - 12 n^2 + 4 n^3 \right) n (\rho \norm{\mF\vz_0} + \sigma) \\
   &\leq \eta^3 L^2 (\rho \norm{\mF\vz_0} + \sigma) \left( n^2 + { (1 + \eta L e^2)} n^3  + \frac{2 \eta L e^2}{3} n^4 \right) \\ 
   &\leq \eta^3 n^3 L^2 (\rho \norm{\mF\vz_0} + \sigma) \left( \frac{1}{n} + { 1 + \frac{e^2}{n}} + \frac{2 e^2}{3} \right) 
\end{align*} where in the last line we used that $\eta < \frac{1}{nL}$. Because the inequality \[
   \frac{1}{n} + { 1 + \frac{e^2}{n}} + \frac{2 e^2}{3} \leq 10\nu
\] holds for all $n \geq 1$, continuing from above we obtain \begin{equation}\label{eqn:deterministic-x2-bound}
   \begin{aligned}
      X_2 &\leq 10 \nu \eta^3 n^3 L^2 (\rho \norm{\mF\vz_0} + \sigma)  \\ 
   &\leq 10 \nu \eta^3 n^3 L^2 \rho \norm{\mF \vz_0} + 10 \nu \eta^3 n^3 L^2 \sigma.
   \end{aligned}
\end{equation}Rearranging \eqref{eqn:r-bound-compact} with applying the bounds \eqref{eqn:deterministic-x1-bound} and \eqref{eqn:deterministic-x2-bound} gives us the claimed result. 
\end{proof} 

\begin{proof}[Proof of \Cref{thm:rk1-ffa}] \label{prf:rk1-ffa}
As $n\geq 1$, we notice that $\nicefrac{1}{n} \leq 1$ and $\nu \leq \nicefrac{3}{2}$ where $\nu = 1+\frac{1}{2n}$ following the notation of \Cref{thm:appx-deterministic-r-bound}. Then the bound \eqref{eqn:rk1-ffa} is immediate from \Cref{thm:appx-deterministic-r-bound}. 
\end{proof}

\subsubsection{Proof of \titleref{Equation}{eqn:rk2-ff} for \segffa} \label{appx:eqn:rk2-ff}

In this section, we prove the following. 
\begin{thm}\label{thm:rk2-ffa} Say we use \segffa{}. Then, as long as the stepsize used in an epoch satisfies $\eta < \frac{1}{nL}$, it holds that 
    \begin{equation}\label{eqn:rk2-ffa}
        \expt\left[\norm{\vr}^2\,\middle|\,\vz_0\right] \leq \eta^6 n^6 C_\textsf{2A} \norm{\mF \vz_0}^2 + \eta^6 n^6 D_\textsf{2A} \norm{\mF \vz_0 }^4 + \eta^6 n^5 V_\textsf{2A} 
\end{equation} 
for constants \begin{align} 
    C_\textsf{2A} &\coloneqq 4L^4 \left( \left( \frac{1}{2}\left(  1 + \frac{2 e^2}{3} \right) + \frac{6 +e^2 }{3}  \right)^2 + {36 \rho^2 e^4}   \right), \label{eqn:def-c2a} \\
    D_\textsf{2A} &\coloneqq 4M^2 \left( \left(\frac{83}{4} + \frac{ 24 e^4 }{5} \right)^2 +  {\rho^4 \left(\frac{243}{16}  + {27 e^4} \right)^2}  \right), \label{eqn:def-d2a}\\
    V_\textsf{2A} &\coloneqq 4  M^2 \sigma^4 \left( \frac{243}{16} + {27 e^4} \right)^2  + 144 e^4  L^4 \sigma^2.  \label{eqn:def-v2a}
\end{align}
\end{thm}
\begin{proof} 
    The bound is then immediate from the following \Cref{thm:appx-expected-r2-bound}, as $n\geq 1$ implies $\nicefrac{1}{n} \leq 1$ and $\nu \leq \nicefrac{3}{2}$ for $\nu$ defined in the statement of \Cref{thm:appx-expected-r2-bound}. 
\end{proof}

\begin{thm} \label{thm:appx-expected-r2-bound} Suppose that $\eta < \frac{1}{nL}$, and let $\nu \coloneqq 1+\frac{1}{2n}$. Then, \emph{in expectation}, the noise term satisfies the bound \begin{align*}
    \expt\left[\norm{\vr}^2\,\middle|\,\vz_0\right] &\leq  4 {\eta^6 n^6 L^4} \norm{\mF \vz_0}^2 \left( \left( \frac{1}{2 n}\left(  1 + \frac{2 e^2}{3} \right) + \frac{4 \nu +e^2 }{3}  \right)^2 + \frac{36 \rho^2 e^4}{n}  \right)\\
    &\phantom{\leq} \qquad + 4 {\eta^6 n^6 M^2} \norm{\mF \vz_0 }^4 \left( \left(\frac{1}{2} + { 4 \nu^4 } + \frac{ 16 \nu e^4 }{5} \right)^2 + \frac{\rho^4 \left( 3\nu^4  + {8\nu^3 e^4} \right)^2}{n} \right) \\
    &\phantom{\leq} \qquad + 4\eta^6 n^5 M^2 \sigma^4 \left( 3\nu^4  + {8\nu^3 e^4} \right)^2  + 144 e^4 \eta^6 n^5 L^4  \sigma^2 .
\end{align*}
\end{thm}
\begin{proof} 
    Notice that, when conditioned on $\vz_0$, the only source of randomness included in $\Psi_j$ is the random permutation $\tau$ selected for the epoch. Hence, we can use \Cref{lem:mish-lemma} to get 
    \[
        \expt\left[\Psi_j\,\middle|\,\vz_0\right] = \expt\left[ \sum_{i=1}^j \delta_i^2\,\middle|\,\vz_0\right] = \sum_{i=1}^j  \expt\left[ \delta_i^2\,\middle|\,\vz_0\right] \leq \frac{jn(\rho \norm{\mF\vz_0} + \sigma)^2}{2}. 
    \]
    Applying Young's inequality on \eqref{eqn:r-bound-compact} we get \begin{equation}\label{eqn:r2-bound-compact}
        \begin{aligned}
            \norm{\vr}^2 &\leq  4 {\eta^6 n^6 L^4} \norm{\mF \vz_0}^2 \left( \frac{1}{2 n}\left(  1 + \frac{2 e^2}{3} \right) + \frac{4 \nu +e^2 }{3}  \right)^2 \\
        &\phantom{\leq} \qquad + 4 {\eta^6 n^6 M^2} \norm{\mF \vz_0 }^4 \left(\frac{1}{2} + { 4 \nu^4 } + \frac{ 16 \nu e^4 }{5} \right)^2 \\
        &\phantom{\leq} \qquad + 4X_1^2 + 4X_2^2.
        \end{aligned}
     \end{equation} When conditioned on $\vz_0$, the first two lines are not random quantities. Thus, it suffices to derive the bounds for $\expt\left[X_i^2\,\middle|\, \vz_0\right]$, $i =1, 2$.
     
     Recall that the bound \eqref{eqn:deterministic-x1-bound} on $X_1$ holds \emph{deterministically}. Hence, it holds that \begin{align*}
       \expt\left[X_1^2 \,\middle|\, \vz_0 \right] &\leq   \expt\left[ X_1 \left( \eta^3 n^3 M \rho^2 \norm{\mF \vz_0}^2 \left( 3\nu^4  + {8\nu^3 e^4} \right) + \eta^3 n^3 M \sigma^2 \left( 3\nu^4  + {8\nu^3 e^4} \right) \right) \,\middle|\, \vz_0 \right] \\
       &= \eta^3 n^3 M \left( 3\nu^4  + {8\nu^3 e^4} \right)\left( \rho^2 \norm{\mF \vz_0}^2 + \sigma^2 \right)\expt\left[ X_1 \,\middle|\, \vz_0 \right].
     \end{align*}
     Now, to compute $\expt\left[X_1 \,\middle|\, \vz_0\right]$, we apply the linearity of expectation on \eqref{eqn:error-x1} to get \begin{align*} 
       \MoveEqLeft[1] \expt\left[ X_1 \,\middle|\, \vz_0 \right] \\ 
       &= \frac{3\eta^3 M}{8} \!\left( \expt\left[ \Psi_{2n} \,\middle|\, \vz_0 \right] + (2\nu^2 - 1)^2 \expt\left[ \Psi_{2n-1} \,\middle|\, \vz_0 \right] +  \frac{4\nu^6 e^4}{n^2} \sum_{j=1}^{2n-2} j \expt\left[ \Psi_{j} \,\middle|\, \vz_0 \right]\right) \\ 
        &\leq \frac{3\eta^3 M}{8} \!\left( n^2 (\rho \norm{\mF\vz_0} + \sigma)^2 + \frac{(2\nu^2 - 1)^2(2n-1) n (\rho \norm{\mF\vz_0} + \sigma)^2}{2}  +  \frac{4\nu^6 e^4}{n^2} \sum_{j=1}^{2n-2} \frac{j^2 n (\rho \norm{\mF\vz_0} + \sigma)^2}{2} \right) \\
    &= \frac{3\eta^3 M}{8}  \left( \frac{2n^2 + (2\nu^2 - 1)^2(2n^2-n)}{2}  +  \frac{2\nu^6 e^4(n-1)(2n-1)(4n-3)}{3n}\right) (\rho \norm{\mF\vz_0} + \sigma)^2 \\  
    &\leq \frac{3\eta^3 M}{8} \left( 2\nu^4 n^2 +  \frac{16 \nu^3 e^4 n^2}{3}\right) (\rho \norm{\mF\vz_0} + \sigma)^2 \\
    &= \frac{\eta^3 n^2 M (\rho \norm{\mF\vz_0} + \sigma)^2}{4} \left( 3\nu^4 + 8 \nu^3 e^4 \right).
     \end{align*}
Young's inequality gives us the bound \begin{equation} \label{eqn:sigma-young}
    \frac{(\rho \norm{\mF\vz_0} + \sigma)^2}{2} \leq  \rho^2 \norm{\mF \vz_0}^2 + \sigma^2 
\end{equation} which, with the inequality derived above, leads to \[
        \expt\left[ X_1 \,\middle|\, \vz_0 \right] \leq \frac{\eta^3 n^2 M}{2} \left( 3\nu^4 + 8 \nu^3 e^4 \right) \left( \rho^2 \norm{\mF \vz_0}^2 + \sigma^2 \right).
    \] As a consequence, with using Young's inequality once again, we obtain \begin{equation}\label{eqn:expected-x12-bound} \begin{aligned}
        \expt\left[X_1^2 \,\middle|\, \vz_0 \right] &\leq \frac{\eta^6 n^5 M^2}{2} \left( 3\nu^4  + {8\nu^3 e^4} \right)^2 \left( \rho^2 \norm{\mF \vz_0}^2 + \sigma^2 \right)^2  \\
        &\leq \eta^6 n^5 M^2 \left( 3\nu^4  + {8\nu^3 e^4} \right)^2 \left( \rho^4 \norm{\mF \vz_0}^4 + \sigma^4 \right) .
      \end{aligned} \end{equation}

To get the bound of $\expt\left[X_2^2 \,\middle|\, \vz_0\right]$, we begin by using 
\begin{align*}
    \eta L \nu^2 (2n-k-1) &\leq \left( 1+\frac{1}{2n} \right)^2\frac{2n-k-1}{n} \\ 
    &= -\frac{k}{4 n^3}-\frac{k}{n^2}-\frac{k}{n}-\frac{1}{4
    n^3}-\frac{1}{2 n^2}+\frac{1}{n}+2 \: \leq \: 2,
\end{align*}  
which holds for all $1 \leq k \leq 2n-2$,
to \eqref{eqn:error-x2} to obtain 
\begin{align*}
    X_2 &\leq \frac{\eta^3 L^2 (\nu+1)}{4} \Sigma_{2n-1} + \frac{\eta^3 L^2 \nu^2 (1 + \eta L e^2)}{2} \sum_{j=1}^{2n-2} \Sigma_{j} + \frac{\eta^3 L^2 \nu^3 e^2}{2} \sum_{k=1}^{2n-2} \frac{2n-k-1}{n} \Sigma_{k-1} \\
    &\leq \frac{\eta^3 L^2 (\nu+1)}{4} \Sigma_{2n-1} + \frac{\eta^3 L^2 \nu^2 (1 + \eta L e^2)}{2} \sum_{j=1}^{2n-2} \Sigma_{j} + {\eta^3 L^2 \nu e^2}  \sum_{k=1}^{2n-2}  \Sigma_{k-1} \\
    &\leq \frac{\eta^3 L^2 (\nu+1)}{4} \Sigma_{2n-1} + \left( \frac{\eta^3 L^2 \nu^2 (1 + \eta L e^2)}{2} + \eta^3 L^2 \nu e^2 \right)\sum_{j=1}^{2n-2} \Sigma_{j} \\
    &\leq \frac{\eta^3 L^2 (\nu+1)}{4} \Sigma_{2n-1} + 3 \eta^3 L^2 e^2 \sum_{j=1}^{2n-2} \Sigma_{j} .
\end{align*}
Then we directly square both sides and expand them to get %
\begin{align*}
    X_2^2 &\leq \left( \frac{\eta^3 L^2 (\nu+1)}{4} \Sigma_{2n-1} + 3 \eta^3 L^2 e^2 \sum_{j=1}^{2n-2} \Sigma_{j} \right)^2  \\  
    &= \frac{\eta^6 L^4 (\nu+1)^2}{16} \Sigma_{2n-1}^2 + 9 \eta^6 L^4 e^4 \left( \sum_{j=1}^{2n-2} \Sigma_{j} \right)^2 +  \frac{3 \eta^6 L^4 e^2 (\nu+1)}{2} \sum_{j=1}^{2n-2}  \Sigma_{2n-1} \Sigma_{j}. 
\end{align*} Here, using \Cref{lem:sigma-concentration} and \Cref{lem:sum-sigma-concentration} on the right hand side leads to \begin{align*} 
    \expt\left[X_2^2 \,\middle|\, \vz_0 \right] &\leq \frac{\eta^6 L^4 (\nu+1)^2 n (2n-1)^2 (\rho \norm{\mF\vz_0} + \sigma)^2}{32} + \frac{9 \eta^6 L^4 e^4  n (2n-2)^2(2n-1)^2(\rho \norm{\mF\vz_0} + \sigma)^2}{8} \\
    &\phantom{\leq} \qquad + \frac{3 \eta^6 L^4 e^2 (\nu+1)}{2} \sum_{j=1}^{2n-2} \frac{j n (2n-1) (\rho \norm{\mF\vz_0} + \sigma)^2}{2}  \\
    &\leq \frac{\eta^6 L^4 (\nu+1)^2 n (2n-1)^2 (\rho \norm{\mF\vz_0} + \sigma)^2}{32} + \frac{9 \eta^6 L^4 e^4  n (2n-2)^2(2n-1)^2(\rho \norm{\mF\vz_0} + \sigma)^2}{8} \\
    &\phantom{\leq} \qquad + \frac{3 \eta^6 L^4 e^2 (\nu+1) n (n-1) (2n-1)^2 (\rho \norm{\mF\vz_0} + \sigma)^2}{4} \\
    &\leq \frac{\eta^6 L^4 n^3 (\rho \norm{\mF\vz_0} + \sigma)^2}{2} + \frac{9 \eta^6 L^4 e^4  n (2n-2)^2(2n-1)^2 (\rho \norm{\mF\vz_0} + \sigma)^2}{8} \\
    &\phantom{\leq} \qquad +  {6 \eta^6 L^4 e^2 n^3 (n-1)  (\rho \norm{\mF\vz_0} + \sigma)^2} \\ 
    &= \eta^6 L^4 \left( \frac{ n^3 }{2} + \frac{9  e^4  n (2n-2)^2(2n-1)^2 }{8} + 6 e^2 n^3 (n-1) \right)  (\rho \norm{\mF\vz_0} + \sigma)^2 \\ 
    &\leq 18 e^4 \eta^6 L^4 n^5 (\rho \norm{\mF\vz_0} + \sigma)^2. 
\end{align*}
 
As a consequence, with using \eqref{eqn:sigma-young} once again, we obtain \begin{equation}\label{eqn:expected-x22-bound} \begin{aligned}
    \expt\left[X_2^2 \,\middle|\, \vz_0 \right] &\leq 36 e^4 \eta^6 L^4 n^5 \left( \rho^2 \norm{\mF \vz_0}^2 + \sigma^2  \right) .
          \end{aligned}\end{equation} Taking the conditional expectation on \eqref{eqn:r2-bound-compact}, applying the bounds \eqref{eqn:expected-x12-bound} and \eqref{eqn:expected-x22-bound}, and then rearranging the terms leads to the claimed inequality. 
\end{proof}

\subsubsection{Upper Bounds of the Within-Epoch Errors for \segff} \label{appx:r-bound-segff}
\begin{thm}\label{thm:rk-ff}  
Say we use \segff{} with $\alpha = \beta = \nicefrac{\eta}{2}$. Then, as long as the stepsize used in an epoch satisfies $\eta < \frac{1}{nL}$, it holds that 
    \begin{align*} 
        \norm{\vr} &\leq \eta^2 n^2 C_\textsf{1F} \norm{\mF \vz_0} + \eta^2 n^2 D_\textsf{1F} \norm{\mF \vz_0}^2 + \eta^2 n^2 V_\textsf{1F} \\
        \expt\left[\norm{\vr}^2\,\middle|\,\vz_0\right] &\leq \eta^4 n^4 C_\textsf{2F} \norm{\mF \vz_0}^2 + \eta^4 n^4 D_\textsf{2F} \norm{\mF \vz_0 }^4 + \eta^4 n^3 V_\textsf{2F}
\end{align*} 
for constants $C_\textsf{1F}$, $D_\textsf{1F}$, $V_\textsf{1F}$, $C_\textsf{2F}$, $D_\textsf{2F}$, and $V_\textsf{2F}$ to be determined later in \eqref{eqn:def-cdv1f} and \eqref{eqn:def-cdv2f}. 
\end{thm}\begin{proof}
    As we have discussed in \Cref{sec:why-segffa}, we already know that aiming to achieve $\gO(\eta^3)$ error without anchoring is futile. Instead, we show that error of magnitude $\gO(\eta^2)$ is possible with the chosen stepsizes. 

 By \Cref{prop:unravelling-each-step} and \Cref{lem:ff-rearranging} we have 
    For any $i = 0, 1,\dots, N$, it holds that \begin{align*}%
        \vz_{2n} &= \vz_0 - \frac{\eta}{2} \sum_{j=0}^{2n-1} \mT_j \vz_0 + \frac{\eta^2}{4} \sum_{j=0}^{2n-1} D\mT_j(\vz_0) \mT_j \vz_0 +  \frac{\eta^2}{4}  \sum_{0\leq k < j \leq 2n-1}  D\mT_j(\vz_0) \mT_k \vz_0   + \veps_{2n} \\ 
        &=  \vz_0 - \eta \sum_{j=0}^{n-1} \mF_j \vz_0 + \frac{3 \eta^2}{4} \sum_{j=1}^{n} D\mF_j(\vz_0) \mF_j \vz_0 +  \frac{\eta^2}{2}  \sum_{i \neq j}  D\mF_j(\vz_0) \mF_i \vz_0 + \veps_{2n} \\ 
        &=  \vz_0 - \eta n \mF \vz_0 + \eta^2 n^2 D\mF(\vz_0) \mF \vz_0 - \frac{\eta^2}{4} \sum_{j=1}^{n} D\mF_j(\vz_0) \mF_j \vz_0 - \frac{\eta^2}{2}  \sum_{i \neq j}  D\mF_j(\vz_0) \mF_i \vz_0 + \veps_{2n}
    \end{align*} where we denote \begin{equation} \label{eqn:ff-veps} \begin{aligned}
        \veps_{2n} \coloneqq & - \frac{\eta}{2} \sum_{j=0}^{2n-1} \Bigl(\mF_j\vw_j - \mF_j \vz_0 -  D\mF_j(\vz_0)(\vw_j - \vz_0)\Bigr)    \\
    &\ \ + \frac{\eta^2}{4} \sum_{j=0}^{2n-1}D\mF_j(\vz_0)(\mF_j \vz_j - \mF_j \vz_0) + \frac{\eta^2}{4} \sum_{j=0}^{2n-1} D\mF_j (\vz_0) \sum_{k=0}^{j-1}(\mF_k \vw_k - \mF_k \vz_0).
    \end{aligned} \end{equation} 
 
    Comparing $\vz_{2n}$ to a point that would have been the result of a deterministic EG update with stepsize~$\eta n$ we get \begin{align*}
        \vz_{2n} - \left( \vz_0 - \eta n \mF(\vz_0 - \eta n \mF \vz_0) \right) &= \eta n \mF(\vz_0 - \eta n \mF \vz_0) - \eta n \mF \vz_0 + \eta^2 n^2 D\mF(\vz_0) \mF \vz_0  + \veps_{2n} \\
        &\phantom{=} \qquad - \frac{\eta^2}{4} \sum_{j=1}^{n} D\mF_j(\vz_0) \mF_j \vz_0 - \frac{\eta^2}{2}  \sum_{i \neq j}  D\mF_j(\vz_0) \mF_i \vz_0.
    \end{align*}
    Let us define \begin{equation}  \label{eqn:ff-error}
    \tilde{\vr} \coloneqq \eta n \mF(\vz_0 - \eta n \mF \vz_0) - \eta n \mF \vz_0 + \eta^2 n^2 D\mF(\vz_0) \mF \vz_0  + \veps_{2n}.     
    \end{equation}
    Noticing the resemblence between \eqref{eqn:appx:rk} and the equations in \eqref{eqn:ff-veps} and \eqref{eqn:ff-error}, we can repeat the same reasoning used for \Cref{thm:rk1-ffa} and \Cref{thm:rk2-ffa}, 
    but with replacing the bounds given by \Cref{thm:general-iterate-bound} to those in \Cref{thm:rr-general-iterate-bound} (and plugging in $\nicefrac{\eta}2$ in place of $\eta$ in the statement of \Cref{thm:rr-general-iterate-bound}) to conclude that  \begin{align*}
        \norm{\tilde{\vr}} &\leq  \eta^3 n^3 \tilde{C}_\textsf{1A} \norm{\mF \vz_0}  +  \eta^3 n^3 \tilde{D}_\textsf{1A} \norm{\mF \vz_0 }^2 +  \eta^3 n^3 \tilde{V}_\textsf{1A} \\ 
        \expt\left[\norm{\tilde{\vr}}^2\,\middle|\,\vz_0\right] &\leq  \eta^6 n^6 \tilde{C}_\textsf{2A} \norm{\mF \vz_0}^2 +  \eta^6 n^6 \tilde{D}_\textsf{2A} \norm{\mF \vz_0 }^4 +  \eta^6 n^5 \tilde{V}_\textsf{2A} 
    \end{align*} for some constants $\tilde{C}_\textsf{1A}$, $\tilde{D}_\textsf{1A}$, $\tilde{V}_\textsf{1A}$, $\tilde{C}_\textsf{2A}$, $\tilde{D}_\textsf{2A}$, and $\tilde{V}_\textsf{2A}$. %
    Meanwhile, we also have \begin{align*}
     \MoveEqLeft \norm{ \frac{\eta^2}{4} \sum_{j=1}^{n} D\mF_j(\vz_0) \mF_j \vz_0 + \frac{\eta^2}{2}  \sum_{i \neq j}  D\mF_j(\vz_0) \mF_i \vz_0 } \\ 
     &= \norm{  \frac{\eta^2 n^2}{2}  D\mF (\vz_0) \mF  \vz_0 - \frac{\eta^2}{4} \sum_{j=1}^{n} D\mF_j(\vz_0) \mF_j \vz_0  } \\ 
     &\leq \frac{\eta^2 n^2}{2}  \norm{  D\mF (\vz_0) } \norm{ \mF \vz_0 } + \frac{\eta^2}{4} \sum_{j=1}^{n} \norm{D\mF_j(\vz_0)} \norm{ \mF_j \vz_0  } \\ 
     &\leq \frac{\eta^2 n^2}{2} L \norm{ \mF \vz_0 } + \frac{\eta^2}{4} \sum_{j=1}^{n} L\left(   \norm{ \mF_j \vz_0  - \mF \vz_0 } + \norm{ \mF \vz_0 } \right) \\ 
     &\leq \frac{\eta^2 (n^2+n) L}{2} \norm{ \mF \vz_0 } + \frac{\eta^2 L}{4} \sum_{j=1}^{n}   \norm{ \mF_j \vz_0  - \mF \vz_0 }  \\ 
     &\leq {\eta^2 n^2 L}  \norm{ \mF \vz_0 } + \frac{\eta^2 L}{4} \left( \sum_{j=1}^{n}   \norm{ \mF_j \vz_0  - \mF \vz_0 }^2 \right)^{1/2} \left( \sum_{j=1}^{n}1 \right)^{1/2} \\
     &= {\eta^2 n^2 L} \norm{ \mF \vz_0 } + \frac{\eta^2 n L}{4} (\rho \norm{\mF\vz_0} + \sigma) 
    \end{align*}
    where in the second to the last line we used the Cauchy-Schwarz inequality. Therefore, as $\eta \leq \nicefrac{1}{nL}$, we conclude that \[
        \norm{\vz_{2n} - \left( \vz_0 - \eta n \mF(\vz_0 - \eta n \mF \vz_0) \right)} \leq \eta^2 n^2 C_\textsf{1F} \norm{\mF \vz_0}  +  \eta^2 n^2 D_\textsf{1F} \norm{\mF \vz_0 }^2 +  \eta^2 n^2 V_\textsf{1F} 
    \] for constants \begin{equation} \label{eqn:def-cdv1f}  
        C_\textsf{1F} = L + \frac{\rho L}{4} + \frac{\tilde{C}_\textsf{1A}}{L}, \quad D_\textsf{1F} = \frac{\tilde{D}_\textsf{1A}}{L}, \quad  V_\textsf{1F} = \frac{\sigma L}{4} + \frac{\tilde{V}_\textsf{1A}}{L} .
    \end{equation} Moreover, using Young's inequality, we see that \begin{align*}
        \MoveEqLeft \norm{ \frac{\eta^2}{4} \sum_{j=1}^{n} D\mF_j(\vz_0) \mF_j \vz_0 + \frac{\eta^2}{2}  \sum_{i \neq j}  D\mF_j(\vz_0) \mF_i \vz_0 }^2 \\ 
        &\leq  3 {\eta^4 n^4 L^2} \norm{ \mF \vz_0 }^2 + \frac{3 \eta^4 n^2 L^2}{16} \rho^2 \norm{\mF\vz_0}^2 + \frac{3 \eta^4 n^2 L^2}{16} \sigma^2, 
       \end{align*} so we also conclude that \[
        \expt\left[ \norm{\vz_{2n} - \left( \vz_0 - \eta n \mF(\vz_0 - \eta n \mF \vz_0) \right)} ^2\,\middle|\,\vz_0\right] \leq  \eta^4 n^4 C_\textsf{2F} \norm{\mF \vz_0}^2 +  \eta^4 n^4 D_\textsf{2F} \norm{\mF \vz_0 }^4 +  \eta^4 n^3 V_\textsf{2F}  
    \] holds for constants \begin{equation} \label{eqn:def-cdv2f}
        C_\textsf{2F} = 6L^2 + \frac{3 \rho^2 L^2}{8} + \frac{2 \tilde{C}_\textsf{2A}}{L^2}, \quad D_\textsf{2F} = \frac{2 \tilde{D}_\textsf{1A}}{L^2}, \quad  V_\textsf{2F} = \frac{3\sigma^2 L^2}{8} + \frac{2 \tilde{V}_\textsf{1A}}{L^2}. \qedhere
    \end{equation} 
\end{proof}

\subsubsection{Upper Bounds of the Within-Epoch Errors for \segrr} \label{appx:r-bound-segrr}
\begin{thm}\label{thm:rk-rr}  
Say we use \segrr{} with $\alpha = \beta = {\eta}$. Then, as long as the stepsize used in an epoch satisfies $\eta < \frac{1}{nL}$, it holds that  
    \begin{align*} 
        \norm{\vr} &\leq \eta^2 n^2 C_\textsf{1R} \norm{\mF \vz_0} + \eta^2 n^2 D_\textsf{1R} \norm{\mF \vz_0}^2 + \eta^2 n^2 V_\textsf{1R} \\
        \expt\left[\norm{\vr}^2\,\middle|\,\vz_0\right] &\leq \eta^4 n^4 C_\textsf{2R} \norm{\mF \vz_0}^2 + \eta^4 n^4 D_\textsf{2R} \norm{\mF \vz_0 }^4 + \eta^4 n^3 V_\textsf{2R}
\end{align*} 
for constants $C_\textsf{1R}$, $D_\textsf{1R}$, $V_\textsf{1R}$, $C_\textsf{2R}$, $D_\textsf{2R}$, and $V_\textsf{2R}$ to be determined later in \eqref{eqn:def-cdv1r} and \eqref{eqn:def-cdv2r}. 
\end{thm}\begin{proof}
    As we have discussed in \Cref{sec:why-segffa}, we already know that aiming to achieve $\gO(\eta^3)$ error with only using random reshuffling is futile. Instead, we show that error of magnitude $\gO(\eta^2)$ is possible with the chosen stepsizes. 

 By \Cref{prop:unravelling-each-step} and \Cref{lem:ff-rearranging} we have 
    For any $i = 0, 1,\dots, N$, it holds that \begin{align*}%
        \vz_{n} &= \vz_0 - \eta \sum_{j=0}^{n-1} \mF_j \vz_0 + \eta^2 \sum_{j=0}^{n-1} D\mF_j(\vz_0) \mF_j \vz_0 +  \eta^2 \sum_{0\leq k < j \leq n-1}  D\mF_j(\vz_0) \mF_k \vz_0   + \veps_{n}  \\ 
        &=  \vz_0 - \eta n \mF \vz_0 + \eta^2 n^2 D\mF(\vz_0) \mF \vz_0 -  \eta^2 \sum_{0\leq j < k \leq n-1}  D\mF_j(\vz_0) \mF_k \vz_0   + \veps_{n} 
    \end{align*} where we denote \begin{equation} \label{eqn:rr-veps} \begin{aligned}
        \veps_{n} \coloneqq & - \eta \sum_{j=0}^{n-1} \Bigl(\mF_j\vw_j - \mF_j \vz_0 -  D\mF_j(\vz_0)(\vw_j - \vz_0)\Bigr)    \\
    &\ \ + \eta^2 \sum_{j=0}^{n-1}D\mF_j(\vz_0)(\mF_j \vz_j - \mF_j \vz_0) + \eta^2 \sum_{j=0}^{n-1} D\mF_j (\vz_0) \sum_{k=0}^{j-1}(\mF_k \vw_k - \mF_k \vz_0).
    \end{aligned} \end{equation} 
 
    Comparing $\vz_{n}$ to a point that would have been the result of a deterministic EG update with stepsize~$\eta n$ we get \begin{align*}
        \vz_{n} - \left( \vz_0 - \eta n \mF(\vz_0 - \eta n \mF \vz_0) \right) &= \eta n \mF(\vz_0 - \eta n \mF \vz_0) - \eta n \mF \vz_0 + \eta^2 n^2 D\mF(\vz_0) \mF \vz_0  + \veps_{n} \\
        &\phantom{=} \qquad - \eta^2 \sum_{0\leq j < k \leq n-1}  D\mF_j(\vz_0) \mF_k \vz_0 .
    \end{align*}
    Let us define \begin{equation}  \label{eqn:rr-error}
    \check{\vr} \coloneqq \eta n \mF(\vz_0 - \eta n \mF \vz_0) - \eta n \mF \vz_0 + \eta^2 n^2 D\mF(\vz_0) \mF \vz_0  + \veps_{n}.     
    \end{equation}
    Comparing the sums \eqref{eqn:appx:rk2}--\eqref{eqn:appx:rk4} to \eqref{eqn:rr-veps},  
    we can repeat the same reasoning used for \Cref{thm:rk1-ffa} and \Cref{thm:rk2-ffa}, 
    but with replacing the bounds given by \Cref{thm:general-iterate-bound} to those in \Cref{thm:rr-general-iterate-bound}, to conclude that  \begin{align*}
        \norm{\check{\vr}} &\leq  \eta^3 n^3 \check{C}_\textsf{1A} \norm{\mF \vz_0}  +  \eta^3 n^3 \check{D}_\textsf{1A} \norm{\mF \vz_0 }^2 +  \eta^3 n^3 \check{V}_\textsf{1A} \\ 
        \expt\left[\norm{\check{\vr}}^2\,\middle|\,\vz_0\right] &\leq  \eta^6 n^6 \check{C}_\textsf{2A} \norm{\mF \vz_0}^2 +  \eta^6 n^6 \check{D}_\textsf{2A} \norm{\mF \vz_0 }^4 +  \eta^6 n^5 \check{V}_\textsf{2A} 
    \end{align*} for some constants $\check{C}_\textsf{1A}$, $\check{D}_\textsf{1A}$, $\check{V}_\textsf{1A}$, $\check{C}_\textsf{2A}$, $\check{D}_\textsf{2A}$, and $\check{V}_\textsf{2A}$.  
    Meanwhile, we also have \begin{align*}
   \sum_{0\leq j < k \leq n-1}  D\mF_j(\vz_0) \mF_k \vz_0 &= \sum_{j=0}^{n-1} D\mF_j(\vz_0)  (n \mF \vz_0 - \vg_{j+1}) \\
   &= \sum_{j=0}^{n-1} (n - j-1) D\mF_j(\vz_0) \mF \vz_0 - \sum_{j=0}^{n-1} D\mF_j(\vz_0) (\vg_{j+1} - (j+1) \mF \vz_0)
    \end{align*} which leads to \begin{equation} \label{eqn:rr-cross-error} \begin{aligned}
        \norm{\sum_{0\leq j < k \leq n-1}  D\mF_j(\vz_0) \mF_k \vz_0} &\leq \sum_{j=0}^{n-1} (n - j-1) L \norm{\mF \vz_0} +  L \sum_{j=0}^{n-1} \delta_{j+1} \\ 
        &\leq \frac{n^2 L}{2} \norm{\mF \vz_0}  +  L \sum_{j=0}^{n-1} \delta_{j+1} . 
         \end{aligned} \end{equation} 
         Therefore, from  $\eta \leq \nicefrac{1}{nL}$ and \Cref{lem:deterministic-deviation-bound}, on one hand we obtain \[
    \norm{\vz_{n} - \left( \vz_0 - \eta n \mF(\vz_0 - \eta n \mF \vz_0) \right)} \leq \eta^2 n^2 C_\textsf{1R} \norm{\mF \vz_0}  +  \eta^2 n^2 D_\textsf{1R} \norm{\mF \vz_0 }^2 +  \eta^2 n^2 V_\textsf{1R} 
         \] for constants  \begin{equation} \label{eqn:def-cdv1r}  
            C_\textsf{1R} = \frac{L}{2} +  \rho L + \frac{\check{C}_\textsf{1A}}{L}, \quad D_\textsf{1R} = \frac{\check{D}_\textsf{1A}}{L}, \quad  V_\textsf{1R} = {\sigma L} + \frac{\check{V}_\textsf{1A}}{L} .
        \end{equation}
        On the other hand, applying Young's inequality on \eqref{eqn:rr-cross-error} we get \begin{align*}
            \norm{\sum_{0\leq j < k \leq n-1}  D\mF_j(\vz_0) \mF_k \vz_0}^2 &\leq {n^4 L^2 }\norm{\mF \vz_0}^2  + 2 L^2 \left(\sum_{j=0}^{n-1} \delta_{j+1} \right)^2 \\
            &\leq {n^4 L^2 }\norm{\mF \vz_0}^2  + 2 n L^2 \sum_{j=1}^n \delta_j^2. 
             \end{align*} 
             Taking the expectation conditioned on $\vz_0$ and applying \Cref{lem:mish-lemma},  we conclude that \[
        \expt\left[ \norm{\vz_{2n} - \left( \vz_0 - \eta n \mF(\vz_0 - \eta n \mF \vz_0) \right)} ^2\,\middle|\,\vz_0\right] \leq  \eta^4 n^4 C_\textsf{2R} \norm{\mF \vz_0}^2 +  \eta^4 n^4 D_\textsf{2R} \norm{\mF \vz_0 }^4 +  \eta^4 n^3 V_\textsf{2R}  
    \] holds for constants  \begin{equation} \label{eqn:def-cdv2r}  
        C_\textsf{2R} = 2L^2 + 4\rho^2 L^2 + \frac{2 \check{C}_\textsf{2A}}{L^2}, \quad D_\textsf{2R} = \frac{2 \check{D}_\textsf{2A}}{L^2}, \quad  V_\textsf{2R} = 4 \sigma^2 L^2 + \frac{2 \check{V}_\textsf{2A}}{L^2}. \qedhere
    \end{equation}
\end{proof}
 
\section{Convergence Bounds in the Strongly Monotone Setting} 
\label{sappx:sm-analysis}

In this section, we focus only on the iterates $\{\vz_0^k\}_{k \geq 0}$. So, we omit the subscript $0$ unless necessary, and simply write $\vz^k$ instead of~$\vz_0^k$. 

\subsection{Unified Analysis of the Upper Bounds for Shuffling-Based SEG Methods}

\begingroup 

When $\mF$ is $\mu$-strongly monotone with $\mu > 0$, all of \segrr, \segff, and \segffa{} do not diverge. 
In fact, it is possible to establish the following unified analysis of the methods. 
\begin{thm}[\Cref{thm:scsc-convergence-result}, simplified]\label{thm:rr-error-bound}
    Suppose that $\mF$ is $\mu$-strongly monotone with $\mu>0$, \Cref{asmp:smoothness} holds, 
    and an optimization method whose within-epoch error satisfies \eqref{eqn:rk1-ff} and \eqref{eqn:rk2-ff} for some constant $a > 0$ is run for $K$ epochs. 
    Then, for a sufficiently small constant $\omega$ that does not depend on $K$, we achieve the bound \begin{equation*} 
        \expt \norm{\vz^K - \vz^*}^2 \leq \exp\left( - \frac{1}{2} \mu \omega n K \right)  \norm{\vz^{0} - \vz^*}^2 + \tilde{\gO}\left(\frac{1}{ n  K^{2a-2}} \right) . 
 \end{equation*}
\end{thm}
The goal of this section is to prove this theorem, whose precise statement is in \Cref{thm:scsc-convergence-result}. %
As the polynomial decay will dominate the exponential decay for large enough $K$, the bound we get is essentially $\tilde{\gO}\left(\nicefrac{1}{ n  K^{2a-2}} \right)$. 
Recall that for \segff{} and \segrr{} we have $a = 2$ (by Theorems~\ref{thm:rk-ff} and \ref{thm:rk-rr}) which leads to an upper bound of $\tilde{\gO}(\nicefrac{1}{nK^2})$, whereas for \segffa{} we have $a = 3$ (by Theorems~\ref{thm:rk1-ffa} and \ref{thm:rk2-ffa}) which gives an upper bound of $\tilde{\gO}(\nicefrac{1}{nK^4})$.  

\endgroup

As also mentioned in the beginning of \Cref{appx:within-epoch}, for any of \segrr, \segff, and \segffa, we can decompose the update across the epoch into a deterministic EG update plus a noise. In this section, letting $\vw_\dagger^k \coloneqq \vz^k - \eta_k n \mF \vz^k$, we define $\widehat{\mF}^k$ by the relation $\eta_k n \widehat{\mF}^k = \eta_k n \mF \vw_\dagger^k + \vr^k$ so that \begin{equation}\label{eqn:appx-eg-with-error}  
 \vz^{k+1}  = \vz^k - \eta_k n \widehat{\mF}^k.
\end{equation}

\begin{prop}\label{prop:descent-lemma} Let $\mF$ be $\mu$-strongly monotone with $\mu > 0$. Then, for any $\eta_k > 0$, it holds that \begin{equation} \label{eqn:progress-of-strong-eg}\begin{aligned}
    \MoveEqLeft \eta_k^2 n^2 \left( 1-\frac{3}{2}\mu \eta_k n - \left(1+\frac{1}{2}\mu \eta_k n\right)\eta_k^2 n^2 L^2 \right) \norm{\mF\vz^k}^2  \\
    &\leq \left(1 - \frac{1}{2} \mu \eta_k n \right)\norm{\vz^{k} - \vz^*}^2 - \norm{\vz^{k+1} - \vz^*}^2 +  \frac{2+\mu \eta_k n}{\mu \eta_k n} \norm{\vr^k}^2.
\end{aligned}\end{equation}
\end{prop} 
\begin{proof}
From \eqref{eqn:appx-eg-with-error}, using \Cref{lem:inprod-lower-bound} we get \begin{align*}
    \norm{\vz^{k+1} - \vz^*}^2 &= \norm{\vz^{k} - \vz^*}^2 - 2 \inprod{\eta_k n\widehat{\mF}^k}{\vz^{k} - \vz^*} + \norm{\eta_k n \widehat{\mF}^k}^2 \\
    &= \norm{\vz^{k} - \vz^*}^2 - 2 \eta_k n \inprod{\mF\vw_\dagger^k}{\vw_\dagger^{k} - \vz^*} - 2 \eta_k^2 n^2 \inprod{\mF\vw_\dagger^k}{\mF\vz^k} \\ 
    &\phantom{= \norm{\vz^{k} - \vz^*}^2 } \ - 2 \inprod{\vr^k}{\vz^{k} - \vz^*} + \norm{\eta_k n \widehat{\mF}^k}^2 \\ 
    &\leq \norm{\vz^{k} - \vz^*}^2 - \mu \eta_k n \norm{\vz^k - \vz^*}^2 - 2 \eta_k^2 n^2 \inprod{\mF\vw_\dagger^k}{\mF\vz^k} \\ 
    &\phantom{= \norm{\vz^{k} - \vz^*}^2 } \ - 2 \inprod{\vr^k}{\vz^{k} - \vz^*} + \norm{\eta_k n \widehat{\mF}^k}^2 + 2  \mu\eta_k^3 n^3 \norm{\mF\vz^k}^2 .
\end{align*} 
Meanwhile, using the polarization identity (\Cref{lem:polarization}) and the $L$-smoothness of $\mF$ we get \begin{align*}
    - 2 \inprod{{\mF \vw_\dagger^k}}{ \mF\vz^k } &= \norm{\mF \vw_\dagger^k - \mF\vz^k}^2 - \norm{{\mF \vw_\dagger^k}}^2 - \norm{ \mF\vz^k }^2 \\
    &\leq L^2 \norm{\vw_\dagger^k - \vz^k}^2 - \norm{{\mF \vw_\dagger^k}}^2 - \norm{ \mF\vz^k }^2 \\ 
    &\leq -(1-\eta_k^2 n^2 L^2)\norm{ \mF\vz^k }^2 - \norm{{\mF \vw_\dagger^k}}^2.  
    \end{align*} Combining the two inequalities and using the definition of $\widehat{\mF}$ we obtain \begin{align*}
        \norm{\vz^{k+1} - \vz^*}^2  &\leq (1-\mu \eta_k n)\norm{\vz^{k} - \vz^*}^2 - \eta_k^2 n^2 (1-\eta_k^2 n^2 L^2)\norm{ \mF\vz^k }^2 - \eta_k^2 n^2 \norm{{\mF \vw_\dagger^k}}^2 \\
        &\phantom{\leq (1-\mu \eta_k n) \norm{\vz^{k} - \vz^*}^2}\ - 2 \inprod{\vr^k}{\vz^{k} - \vz^*} + \norm{ \eta_k n \mF \vw_\dagger^k + \vr^k}^2 + 2 \mu\eta_k^3 n^3 \norm{\mF\vz^k}^2 \\
        &\leq (1-\mu \eta_k n)\norm{\vz^{k} - \vz^*}^2 - \eta_k^2 n^2 (1-2\mu \eta_k n - \eta_k^2 n^2 L^2)\norm{ \mF\vz^k }^2 \\
        &\phantom{\leq (1-\mu \eta_k n)  \norm{\vz^{k} - \vz^*}^2}\ - 2 \inprod{\vr^k}{\vz^{k} - \vz^*} + 2 \inprod{\vr^k}{\eta_k n \mF \vw_\dagger^k}  +  \norm{\vr^k}^2 \\
        &\leq (1-\mu \eta_k n)\norm{\vz^{k} - \vz^*}^2 - \eta_k^2 n^2 (1-2\mu \eta_k n - \eta_k^2 n^2 L^2)\norm{ \mF\vz^k }^2 \\
        &\phantom{\leq (1-\mu \eta_k n)  \norm{\vz^{k} - \vz^*}^2}\ - 2 \inprod{\vr^k}{\vz^{k} - \eta_k n \mF \vw_\dagger^k - \vz^*} + \norm{\vr^k}^2 .
    \end{align*} Let us consider the inner product term in the last line above. By \Cref{lem:w-amgm-ineq} and the nonexpansiveness of the EG update (\Cref{lem:nonexpansive}), for any $\gamma_k > 0$ we have \begin{align*}
     - 2 \inprod{\vr^k}{\vz^{k} - \eta_k n \mF \vw_\dagger^k -  \vz^*} &\leq \frac{1}{\gamma_k} \norm{\vr^k}^2 + \gamma_k \norm{\vz^{k} - \eta_k n \mF \vw_\dagger^k -  \vz^*}^2 \\ 
     &\leq \frac{1}{\gamma_k} \norm{\vr^k}^2 + \gamma_k \norm{\vz^{k}- \vz^*}^2 - \gamma_k \eta_k^2 n^2 (1 - \eta_k^2 n^2  L^2) \norm{\mF\vz^k}^2 . 
    \end{align*} Plugging this back we get \begin{equation} \label{eqn:progress-of-eg}\begin{aligned}
        \MoveEqLeft \eta_k^2 n^2 (1+\gamma_k -2\mu \eta_k n - (1+\gamma_k)\eta_k^2 n^2 L^2) \norm{\mF\vz^k}^2  \\
        &\leq (1 + \gamma_k - \mu \eta_k n )\norm{\vz^{k} - \vz^*}^2 - \norm{\vz^{k+1} - \vz^*}^2 + \left( 1+\frac{1}{\gamma_k} \right) \norm{\vr^k}^2.
    \end{aligned}\end{equation} Choosing $\gamma_k = \frac{\mu \eta_k n}{2}$ completes the proof. 
\end{proof} 

\begin{prop}\label{prop:sufficient-norm-decrease}
    Let $\mF$ be a $\mu$-strongly monotone and $L$-Lipschitz operator. Then, whenever \hbox{$\eta_k \leq \frac{1}{nL\sqrt{2}}$}, it holds that \[
        \norm{\mF\vz^{k+1}} \leq  \left( 1-\frac{\mu n  \eta_k}{5} \right)\norm{\mF\vz^k} + L \norm{\vr^{k}}.
    \]
 \end{prop}
 \begin{proof}
    Let $\vz_\dagger^{k+1} \coloneqq \vz^k - \eta_k n\mF(\vz^k - \eta_k n \mF \vz^k)$, so that we have $\norm{\vz^{k+1} - \vz_\dagger^{k+1}} = \norm{\vr^k}$. Then, the $L$-smoothness of $\mF$ and \Cref{lem:norm-decrease} implies \begin{align*}
        \norm{\mF\vz^{k+1}} &\leq \norm{\mF\vz^{k+1} - \mF\vz_\dagger^{k+1}} + \norm{\mF\vz_\dagger^{k+1}} \\ 
        &\leq L \norm{\vz^{k+1} - \vz_\dagger^{k+1}} + \norm{\mF\vz_\dagger^{k+1}} \\ 
        &\leq L \norm{\vr^{k}} + \left( 1-\frac{2 \mu \eta_k n}{5} \right)^{1/2}\norm{\mF\vz^k} \\ 
        &\leq L \norm{\vr^{k}} + \left( 1-\frac{\mu \eta_k n}{5} \right)\norm{\mF\vz^k}
    \end{align*} where in the last line we apply a simple inequality $1-2x \leq (1-x)^2$ which holds for all $x \in \mathbb{R}$. 
 \end{proof} 

\begingroup 

\begin{lem} \label{lem:appx-bounded-iterates}
    Suppose that \eqref{eqn:rk1-ff} holds. %
    Say we use a constant stepsize $\eta_k = \eta$, where $\eta$ satisfies $\eta \leq \frac{1}{nL\sqrt{2}}$ %
    and \begin{equation} \label{eqn:stepsize-for-contraction}
    \eta^{a-1} n^{a-1} \leq \frac{1}{10}\min\left\{ \frac{1}{L^2}, \frac{\mu}{L (C_1 + D_1 (\norm{\mF \vz_0} + \nicefrac{V_1}{\mu L}))} \right\}. 
     \end{equation} 
    Then for any $k = 0, 1, \dots$, the following two inequalities both hold: \begin{align}
    \norm{\mF \vz^{k+1}} &\leq \left( 1 - \frac{\mu \eta n}{10}\right) \norm{\mF \vz^{k}} + \eta^a n^a L V_1, \label{eqn:induct-hyp-next}  \\
        \norm{\mF\vz^{k}} &\leq \norm{\mF \vz^0} +  \frac{V_\textsf{1}}{\mu L} \label{eqn:induct-hyp-loose}. 
    \end{align}
\end{lem}
 \begin{proof}
 For the case $k = 0$, the inequality \eqref{eqn:induct-hyp-loose} clearly holds. 
 For the remaining cases, we use strong induction on $k$. 
 More precisely, assuming that \eqref{eqn:induct-hyp-loose} holds for all $0, 1, \dots, k$, we will show that \eqref{eqn:induct-hyp-next} holds, and from that the inequality \begin{align} \label{eqn:induction-goal}
     \norm{\mF\vz^{k+1}} &\leq \norm{\mF \vz^0} + \frac{V_\textsf{1}}{\mu L} 
 \end{align} follows. 
    To this end, let us begin from noting that \Cref{prop:sufficient-norm-decrease}, \eqref{eqn:rk1-ff}, and the induction hypothesis \eqref{eqn:induct-hyp-loose} implies \begin{equation} \label{eqn:pseudo-contraction} \begin{aligned}
        \norm{\mF\vz^{k+1}} &\leq \left( 1-\frac{\mu \eta n}{5} \right)\norm{\mF\vz^k} + \eta^a n^a L \left( C_\textsf{1} \norm{\mF \vz^k} + D_\textsf{1} \norm{ \mF \vz^k}^2 + V_\textsf{1} \right) \\
        &\leq \left( 1-\frac{\mu \eta n }{5} + \eta^a n^a L C_\textsf{1} +  \eta^a n^a L D_\textsf{1} \left( \norm{\mF \vz^0} +  \frac{V_\textsf{1}}{\mu L}  \right)\right) \norm{\mF\vz^k} + \eta^a n^a LV_\textsf{1}.
    \end{aligned} \end{equation} Here, from the choice of the stepsize \eqref{eqn:stepsize-for-contraction}, we have \[
        \eta^a n^a L C_\textsf{1} +  \eta^a n^a L D_\textsf{1} \left( \norm{\mF \vz^0} + \frac{V_\textsf{1}}{\mu L}  \right) \leq \frac{\mu \eta n}{10}. 
    \] Hence, from \eqref{eqn:pseudo-contraction} we get
    \[
        \norm{\mF\vz^{k+1}} \leq \left( 1-\frac{\mu \eta n }{10} \right) \norm{\mF\vz^k} + \eta^a n^a LV_\textsf{1}.
    \] which is exactly \eqref{eqn:induct-hyp-next}. 
    Now, considering that we are assuming \eqref{eqn:induct-hyp-loose} holds for all $0, 1, \dots, k$, we must also have \eqref{eqn:induct-hyp-next} for all $0, 1, \dots, k$. 
    Thus we can unravel the recurrence to get \begin{equation} \label{eqn:example-unravel} \begin{aligned}
        \norm{\mF\vz^{k+1}} &\leq \left( 1-\frac{\mu \eta n }{10} \right) \norm{\mF\vz^k} + \eta^a n^a LV_\textsf{1} \\ 
        &\leq \left( 1-\frac{\mu \eta n }{10} \right)^2 \norm{\mF\vz^{k-1}} +\left( 1-\frac{\mu \eta n }{10} \right) \eta^a n^a LV_\textsf{1} + \eta^a n^a LV_\textsf{1} \\
        &\leq \dots \\
        &\leq \left( 1-\frac{\mu \eta n }{10} \right)^{k+1} \norm{\mF\vz^0} + \eta^a n^a LV_\textsf{1} \sum_{j=0}^k \left( 1-\frac{\mu \eta n }{10} \right)^j \\
        &\leq \norm{\mF\vz^0} + \frac{\eta^a n^a LV_\textsf{1}}{1 - \left( 1-\frac{\mu \eta n }{10} \right)} \\
        &= \norm{\mF\vz^0} +  \frac{10 \eta^{a-1} n^{a-1} L V_\textsf{1}}{\mu}. 
    \end{aligned} \end{equation}
    As \eqref{eqn:stepsize-for-contraction} also implies $10 \eta^{a-1} n^{a-1} L \leq \nicefrac{1}{L}$, we obtain \eqref{eqn:induction-goal}, as claimed. This completes the proof. 
\end{proof} 

\begin{thm}[\Cref{thm:rr-error-bound}] \label{thm:scsc-convergence-result} %
Suppose that $\mF$ is $\mu$-strongly monotone with $\mu>0$, \Cref{asmp:smoothness} holds, 
    and an optimization method whose within-epoch error satisfies \eqref{eqn:rk1-ff} and \eqref{eqn:rk2-ff} for some constant $a > 0$ is run for $K$ epochs.
    Let us define a constant \[
        \varPhi \coloneqq C_\textsf{2} + D_\textsf{2} \left( \norm{\mF \vz^0} +  \frac{V_\textsf{1}}{\mu L} \right)^2. 
        \] %
    Say we use a constant stepsize $\eta_k = \eta$, where $\eta$ is chosen as \begin{subequations}
    \begin{align}
       \eta = \min \left\{ \vphantom{\frac{2}{5nL}}\right. &{} \frac{2}{5nL}, \label{eqn:most-normal-condition} \\ 
                              &{}\quad \frac{1}{n (10 L^2)^{1/(a-1)}}, \label{eqn:for-recurrence-1}\\ 
                              &{}\quad \frac{\mu^{1/(a-1)}}{n (10 L (C_1 + D_1 (\norm{\mF \vz_0} + \nicefrac{V_1}{\mu L})))^{1/(a-1)}}, \label{eqn:for-recurrence-2} \\ 
                              &{}\quad \frac{1}{(12 \varPhi / \mu)^{1/(2a-3)} n}, \label{eqn:scsc-initial-stepsize-restriction} \\
                       &{}\quad \left. \frac{4(a-1)\log(n^{1/(2a-2)} K)}{ \mu n K} \right\} .\label{eqn:log-over-lin-stepsize}
    \end{align}
\end{subequations}
        Then for $\omega$ denoting the minimum among \eqref{eqn:most-normal-condition}--\eqref{eqn:scsc-initial-stepsize-restriction}, it holds that \begin{equation} \label{eqn:appx:scsc-meta-convergence}
          \expt \norm{\vz^K - \vz^*}^2 \leq \exp\left( - \frac{1}{2} \mu \omega n K \right)  \norm{\vz^{0} - \vz^*}^2 + \gO\left(\frac{\left(\log(n^{1/(2a-2)} K)\right)^{2a-2}}{ n  K^{2a-2}} \right)  . 
     \end{equation}
\end{thm}  
As a reminder, for \segff{} and \segrr{} we have $a = 2$, and for \segffa{} we have $a = 3$.

\begin{proof}
     Notice that \eqref{eqn:for-recurrence-1} and \eqref{eqn:for-recurrence-2} together implies \eqref{eqn:stepsize-for-contraction}, and that $\eta_k = \eta \leq \frac{2}{5nL} \leq \frac{1}{nL \sqrt{2}} < \frac{1}{nL}$. 
     So, we can utilize \eqref{eqn:rk2-ff} and \Cref{lem:appx-bounded-iterates} to get \begin{align*}
     \expt\left[\norm{\vr}^2\,\middle|\,\vz^k\right] &\leq \eta^{2a} n^{2a} C_\textsf{2} \norm{\mF \vz^k}^2 + \eta^{2a} n^{2a} D_\textsf{2} \norm{\mF \vz^k }^4 + \eta^{2a} n^{2a-1} V_\textsf{2} \\
     &\leq \eta^{2a} n^{2a} C_\textsf{2} \norm{\mF \vz^k}^2 + \eta^{2a} n^{2a} D_\textsf{2} \left( \norm{\mF \vz^0} + \frac{V_\textsf{1}}{\mu L} \right)^2 \norm{\mF \vz^k }^2 + \eta^{2a} n^{2a-1} V_\textsf{2}  \\
     &= \eta^{2a} n^{2a} \varPhi \norm{\mF \vz^k}^2  + \eta^{2a} n^{2a-1} V_\textsf{2} .   
 \end{align*} 
 Taking the conditional expectation on \eqref{eqn:progress-of-strong-eg} and applying the bound just derived, we obtain  \begin{align*}
      \MoveEqLeft \eta^2 n^2 \left( 1-\frac{3}{2}\mu \eta n - \left(1+\frac{1}{2}\mu \eta n\right) \eta^2 n^2 L^2 \right) \norm{\mF\vz^k}^2 \\
       &\leq \left( 1 - \frac{1}{2} \mu \eta n \right)\norm{\vz^{k} - \vz^*}^2 - \expt\left[ \norm{\vz^{k+1} - \vz^*}^2 \,\middle|\, \vz^k \right] \\
       &\phantom{\leq \left( 1 - \frac{1}{2} \mu \eta n \right)\norm{\vz^{k} - \vz^*}^2} \ +  \frac{2+\mu \eta n}{\mu} \left( \eta^{2a-1} n^{2a-1} \varPhi \norm{\mF \vz^k}^2  + \eta^{2a-1} n^{2a-2} V_\textsf{2} \right). 
 \end{align*}
 A simple rearrangement of the terms leads to %
 \begin{equation}\label{eqn:scsc-expected-contraction} \begin{aligned}
     \MoveEqLeft  \eta^2 n^2 \left( 1-\frac{3}{2}\mu \eta n - \left(1+\frac{1}{2}\mu \eta n\right)\eta^2 n^2 L^2 - \frac{2+\mu \eta n}{\mu} \cdot \eta^{2a-3} n^{2a-3} \varPhi \right) \norm{\mF\vz^k}^2  \\
     &\leq \left(1 - \frac{1}{2} \mu \eta n \right)\norm{\vz^{k} - \vz^*}^2 - \expt\left[ \norm{\vz^{k+1} - \vz^*}^2 \,\middle|\, \vz^k \right]+  \frac{2+\mu \eta n}{\mu} \cdot \eta^{2a-1} n^{2a-2} V_\textsf{2} .
 \end{aligned} \end{equation}
  Notice that by assuming \eqref{eqn:most-normal-condition} and \eqref{eqn:scsc-initial-stepsize-restriction}, it holds that \begin{align*}
  \MoveEqLeft \frac{3}{2}\mu \eta n  +  \left(1+\frac{1}{2}\mu \eta n\right)\eta^2 n^2 L^2 + \frac{2+\mu \eta n}{\mu} \cdot \eta^{2a-3} n^{2a-3} \varPhi \\ 
  &\leq \frac{3}{2}\cdot \frac{2}{5} + \left( 1 + \frac{1}{2} \cdot \frac{2}{5} \right) \left(\frac{2}{5}\right)^2 + \frac{12\varPhi}{5 \mu} \cdot \frac{\mu}{12 \varPhi} = \frac{124}{125},
  \end{align*} so we can guarantee that the left hand side of \eqref{eqn:scsc-expected-contraction} is nonnegative. 
  It then follows that
 \[
    \expt\left[ \norm{\vz^{k+1} - \vz^*}^2 \,\middle|\, \vz^k \right] \leq \left(1 - \frac{1}{2} \mu \eta n \right)\norm{\vz^{k} - \vz^*}^2 + \frac{2+\mu \eta n}{\mu} \cdot \eta^{2a-1} n^{2a-2} V_\textsf{2}.    
 \]
 Applying the law of total expectation, %
 from the above we obtain \[ %
    \expt \norm{\vz^{k+1} - \vz^*}^2  \leq \left(1 - \frac{1}{2} \mu \eta n \right) \expt \norm{\vz^{k} - \vz^*}^2 + \frac{2+\mu \eta n}{\mu} \cdot \eta^{2a-1} n^{2a-2} V_\textsf{2}. 
 \] %
 We can now unravel this recurrence over $k = 0, 1, \dots, K-1$ as done in \eqref{eqn:example-unravel} to get \begin{align*}
     \expt \norm{\vz^K - \vz^*}^2 &\leq \left(1 - \frac{1}{2} \mu \eta n \right) \expt \norm{\vz^{K-1} - \vz^*}^2 + \frac{2+\mu \eta n}{\mu} \cdot \eta^{2a-1} n^{2a-2} V_\textsf{2} \\ 
     &\leq \dots \\ 
     &\leq  \left(1 - \frac{1}{2} \mu \eta n \right)^K \norm{\vz^{0} - \vz^*}^2 + \frac{2+\mu \eta n}{\mu} \cdot \eta^{2a-1} n^{2a-2} V_\textsf{2} \sum_{j=0}^{K-1}\left(1 - \frac{1}{2} \mu \eta n \right)^j \\ 
     &\leq  \left(1 - \frac{1}{2} \mu \eta n \right)^K \norm{\vz^{0} - \vz^*}^2 + \frac{4+2\mu \eta n}{\mu^2 \eta n } \cdot \eta^{2a-1} n^{2a-2} V_\textsf{2} \\
     &\leq \exp\left( - \frac{1}{2} \mu \eta n K \right) \norm{\vz^{0} - \vz^*}^2 + \frac{24}{5\mu^2} \cdot \eta^{2a-2} n^{2a-3} V_\textsf{2}
 \end{align*}
 where in the last line we used the basic inequality $1+x \leq e^x$ which holds for all $x \in \rr$. %
 With the choice of the stepsize \eqref{eqn:log-over-lin-stepsize}, we arrive at \begin{equation}\label{eqn:general-rate-scsc}
     \expt \norm{\vz^K - \vz^*}^2 \leq \exp\left( - \frac{1}{2} \mu \eta n K \right)  \norm{\vz^{0} - \vz^*}^2 + \frac{24 \cdot (4a-4)^{2a-2} V_\textsf{2}}{5  \mu^{2a}} \cdot \frac{\left(\log(n^{1/(2a-2)} K)\right)^{2a-2}}{ n  K^{2a-2}}  . 
 \end{equation}

 Now, recall that $\eta$ is chosen to be the smallest one among \eqref{eqn:most-normal-condition}--\eqref{eqn:log-over-lin-stepsize}. 
 Notice that the options \eqref{eqn:most-normal-condition}--\eqref{eqn:scsc-initial-stepsize-restriction} are independent with respect to $K$, and \eqref{eqn:log-over-lin-stepsize} is the only one that depends on $K$. 
 Let us consider these two cases separately. 
 \begin{enumerate}[label=(\roman*)] 
     \item $\eta$ is chosen to be the minimum among \eqref{eqn:most-normal-condition}--\eqref{eqn:scsc-initial-stepsize-restriction}. \nopagebreak

     This is the case where we have $\eta = \omega$. Notice that the constant $\omega$ that does not depend on~$K$. 
     The inequality \eqref{eqn:general-rate-scsc} then takes the form \[
     \expt \norm{\vz^K - \vz^*}^2 \leq \exp\left( - \frac{\mu \omega n K}{2} \right)  \norm{\vz^{0} - \vz^*}^2 + \gO\left(\frac{\left(\log(n^{1/(2a-2)} K)\right)^{2a-2}}{ n  K^{2a-2}} \right). 
     \]

     \item  $\eta$ is chosen to be \eqref{eqn:log-over-lin-stepsize}, that is, $\eta = \frac{4(a-1)\log(n^{1/(2a-2)} K)}{ \mu n K}$.

     In this case, the exponential factor of the first term in the right hand side of \eqref{eqn:general-rate-scsc} reduces~to \[
        \exp\left( - \frac{1}{2} \mu \eta n K \right) = \frac{1}{nK^{2a-2}}. 
     \] Thus, the second term in \eqref{eqn:general-rate-scsc} dominates the first term, and in total \eqref{eqn:general-rate-scsc} becomes\[
     \expt \norm{\vz^K - \vz^*}^2 = \gO\left(\frac{\left(\log(n^{1/(2a-2)} K)\right)^{2a-2}}{ n  K^{2a-2}} \right). 
     \]
 \end{enumerate}

 Therefore, in both cases we have \[
     \expt \norm{\vz^K - \vz^*}^2 \leq \exp\left( - \frac{1}{2} \mu \omega n K \right)  \norm{\vz^{0} - \vz^*}^2 + \gO\left(\frac{\left(\log(n^{1/(2a-2)} K)\right)^{2a-2}}{ n  K^{2a-2}} \right)  
 \] which is exactly \eqref{eqn:appx:scsc-meta-convergence}. This completes the proof. 
 \end{proof}

 \begin{rmk}
 To compare the convergence rate of \segffa{} in the strongly monotone setting with that of \segrr{} by \citet{Emma24} more in depth, let us make an estimation on the size of $\omega$ appearing in \Cref{thm:scsc-convergence-result} when $a = 3$.
 
 To this end, we need estimates on the constants $C_\textsf{1A}$, $D_\textsf{1A}$, $V_\textsf{1A}$, $C_\textsf{2A}$, and $D_\textsf{2A}$. 
 From their definitions in \eqref{eqn:def-c1a}--\eqref{eqn:def-v1a}, \eqref{eqn:def-c2a}, and \eqref{eqn:def-d2a} we have $C_\textsf{1A} \asymp L^2$, $D_\textsf{1A} \asymp M$, $V_\textsf{1A} \asymp M+L^2$, $C_\textsf{2A} \asymp L^4$, and $D_\textsf{2A} \asymp M^2$.
In general, there is not a direct relation between $L$ and $M$. For example, recall that if all components are quadratic, then $M = 0$. Meanwhile, \citet{Gorb22EG} has argued that $M$ can be much larger than $L$ in certain cases, by providing an example where $M \asymp L^{3/2}$. 
For our purposes, however, let us allow $M$ to be even as large as $M \asymp L^2$, so that the situation is simplified into $C_\textsf{1A} \asymp D_\textsf{1A} \asymp V_\textsf{1A} \asymp L^2$ and $C_\textsf{2A} \asymp D_\textsf{2A} \asymp L^4$. 

Then, we get the estimate of \eqref{eqn:for-recurrence-2},  
\[
\frac{\mu^{1/2}}{n (10 L (C_\textsf{1A} + D_\textsf{1A} (\norm{\mF \vz_0} + \nicefrac{V_\textsf{1A}}{\mu L})))^{1/2}} \asymp \frac{\mu}{nL^2}. 
\] 
Meanwhile, as for the constant $\varPhi$ it holds that \[
     \varPhi = C_\textsf{2A} + D_\textsf{2A} \left( \norm{\mF \vz^0} +  \frac{V_\textsf{1A}}{\mu L} \right)^2 \asymp   \frac{L^6}{\mu^2},  
\] for \eqref{eqn:scsc-initial-stepsize-restriction} we have \[
    \frac{1}{(12 \varPhi / \mu)^{1/3} n} \asymp \frac{\mu}{nL^2}. 
\] As \eqref{eqn:most-normal-condition} while \eqref{eqn:for-recurrence-1} are both $\Theta(\nicefrac{1}{nL})$ and $\mu \leq L$, we essentially have $\omega \asymp \nicefrac{\mu}{nL^2}$. 
Or equivalently, for some $b = \Theta(1)$, the convergence rate \eqref{eqn:appx:scsc-meta-convergence} reads \begin{equation} \label{eqn:forced-explicit-scsc-rate}
          \expt \norm{\vz^K - \vz^*}^2 \leq \exp\left( - \frac{b \mu^2 K}{L^2}  \right)  \norm{\vz^{0} - \vz^*}^2 + \gO\left(\frac{\left(\log(n^{1/4} K)\right)^{4}}{ n  K^{4}} \right)  . 
     \end{equation}

     On the other hand, Theorem~2.1 of \citep{Emma24} states that, for some $b' = \Theta(1)$, \segrr{} exhibits a rate of \begin{equation} \label{eqn:import-rate-from-emma24}
\expt \norm{\vz^K - \vz^*}^2 \leq \exp\left( - \frac{b' \mu^2 K}{L^2}  \right)  \norm{\vz^{0} - \vz^*}^2 + \gO\left(\frac{\left(\log(n^{1/2} K)\right)^{2}}{ n  K^{2}} \right)  . 
     \end{equation}
Comparing \eqref{eqn:forced-explicit-scsc-rate} with \eqref{eqn:import-rate-from-emma24}, the exponents in the exponentially decaying term are of the same order of $-\frac{\mu^2 K}{L^2}$, so \segffa{} having a faster polynomially decaying term $\tilde{\gO}(\nicefrac{1}{nK^4})$ enjoys an improved convergence rate. 
\end{rmk}

\endgroup

\section{Convergence Rate of \segffa{} in the Monotone Setting} \label{sappx:meta-analysis}

\subsection{Star-monotonicity}
\label{appx:star-monotone}

Notice that we only used Assumptions~\ref{asmp:smoothness}~and~\ref{asmp:bounded-variance} in deriving the results in Appendices~\ref{sec:proof-segffa}~and~\ref{appx:within-epoch}, and in particular, the monotonicity assumption on $\mF$ was not necessary. 
Moreover, among the lemmata listed in \Cref{sec:toolbox}, \Cref{lem:nonexpansive} is the only one that possibly uses the (non-strongly) monotone assumption, but that lemma is not used in this section. 

In fact, as it turns out in \Cref{sec:monotone-and-som}, in the convergence analysis of \segffa, we need not fully exploit the inequality \eqref{eqn:def-strongly-monotone} provided by the monotonicity assumption. 
Rather, all the results on the performance of \segffa{} can be established with only assuming the following condition (which has been also briefly mentioned in \Cref{sect:gorb22a-asmp}). 
\begin{asmp}[Star-monotonicity]\label{asmp:star-monotone}
     Given an operator $\mF$ with a point $\vz^* \in \rr^{{d_1}+{d_2}}$ such that $\mF \vz^* = \zero$, we say that $\mF$ is star-monotone if, for any $\vz \in \rr^{{d_1}+{d_2}}$, it holds that \begin{equation}\label{eqn:def-star-monotone}
    \inprod{\mF \vz }{\vz - \vz^*} \geq 0. 
\end{equation} 
\end{asmp} 

Monotone and strongly-monotone operators are clearly star-monotone, as they satisfy \eqref{eqn:def-strongly-monotone}. 
On the other hand, there exist operators that are star-monotone but not monotone: see, \textit{e.g.}, \citep[Appendix~A.6]{Loiz21}. 

Recall that when $\mF$ is monotone, \Cref{asmp:solution} is equivalent to assuming the existence of a point $\vz^*$ that satisfies $\mF \vz^* = \zero$. 
Hence, after simply replacing the optimality condition %
in \Cref{asmp:solution} with $\mF \vz^* = \zero$, our convergence analyses not only will show that our \segffa{} finds an optimum on monotone problems, but also that it can be also used to find stationary points in ``star-monotone'' problems, allowing the objective function $f$ to be nonconvex-nonconcave.  

Star-monotonicity is also known as the \emph{variational stability condition} \citep{Hsie20}, and has much been studied in the literature. For further details on star-monotonicity, we refer to \citep{Loiz21, Hsie20} and the references therein.

\subsection{Convergence Analysis of \segffa{} in the (Star-)Monotone Setting} \label{sec:monotone-and-som}

Let us in particular consider \segffa. As in the previous section, we focus only on the iterates $\{\vz_0^k\}_{k \geq 0}$, so again, we omit the subscript $0$ unless necessary, and simply write $\vz^k$ instead of~$\vz_0^k$.

Decompose the update across the epoch into a deterministic EG update plus a noise, as %
\begin{equation} \label{eqn:appx-ffa-with-error}  \begin{aligned}
\vw_\dagger^k &\coloneqq \vz^k - \eta_k n \mF \vz^k, \\
 \vz^{k+1}  &= \vz^k - \eta_k n \widehat{\mF}^k.
\end{aligned} \end{equation} for $\widehat{\mF}^k$ defined by the equation \begin{equation} \label{eqn:appx-ffa-Fhat}
\eta_k n \widehat{\mF}^k = \eta_k n \mF \vw_\dagger^k + \vr^k.  %
\end{equation}

\begin{lem} %
   Let $\mF$ be a (star-)monotone operator with a point $\vz^*$ that satisfies $\mF\vz^* = \zero$, and suppose that \Cref{asmp:smoothness} holds. 
   Then for any $\eta_k > 0$ and $\gamma_k > 0$, it holds that \begin{equation} \label{eqn:descent-of-eg}\begin{aligned} 
       0 &\leq \norm{\vz^k - \vz^*}^2  - \frac{1}{1+\gamma_k} \norm{\vz^{k+1} - \vz^*}^2  - \eta_k^2 n^2 (1-\eta_k^2 n^2 L^2)\norm{\mF \vz^k}^2   + \frac{1}{\gamma_k}  \norm{\vr^k}^2. 
  \end{aligned}\end{equation}
\end{lem}
\begin{proof} \allowdisplaybreaks
   By \eqref{eqn:appx-ffa-with-error} and \eqref{eqn:appx-ffa-Fhat} we get \begin{align*}
       \norm{\vz^{k+1} - \vz^*}^2 &= \norm{\vz^{k} - \eta_k n \widehat{\mF}^k - \vz^*}^2 \\
       &= \norm{\vz^k - \vz^*}^2 - 2\inprod{\eta_k n \widehat{\mF}^k}{\vz^k - \vz^*} + \norm{\eta_k n \widehat{\mF}^k}^2  \\ 
       &= \norm{\vz^k - \vz^*}^2 - 2\inprod{\eta_k n \mF \vw_\dagger^k}{\vw_\dagger^k - \vz^*} - 2\inprod{\eta_k n \mF \vw_\dagger^k}{\vz^k - \vw_\dagger^k} - 2\inprod{\vr^k}{\vz^k - \vz^*} \\ 
       &\phantom{=} \qquad + \norm{\eta_k n \mF \vw_\dagger^k}^2 + 2\inprod{\vr^k}{\eta_k n \mF \vw_\dagger^k} + \norm{\vr^k}^2 \\ 
       &= \norm{\vz^k - \vz^*}^2 - 2\eta_k n \inprod{\mF \vw_\dagger^k}{\vw_\dagger^k - \vz^*} - 2\inprod{\eta_k n \mF \vw_\dagger^k}{\eta_k n \mF \vz^k} \\ 
       &\phantom{=} \qquad + \norm{\eta_k n \mF \vw_\dagger^k}^2 - 2\inprod{\vr^k}{\vz^k -\eta_k n \mF \vw_\dagger^k - \vz^*} + \norm{\vr^k}^2 \\ 
       &= \norm{\vz^k - \vz^*}^2 - 2\eta_k n \inprod{\mF \vw_\dagger^k}{\vw_\dagger^k - \vz^*} - 2\eta_k^2 n^2 \inprod{\mF \vw_\dagger^k}{\mF \vz^k} \\ 
       &\phantom{=} \qquad + \eta_k^2 n^2 \norm{\mF \vw_\dagger^k}^2 - 2\inprod{\vr^k}{\vz^{k+1} - \vz^*} - \norm{\vr^k}^2. 
   \end{align*} 
   We now bound the inner products. On one hand, by the polarization identity (\Cref{lem:polarization}) and the $L$-smoothness of $f$, we have \begin{align*}
      -2 \inprod{\mF \vw_\dagger^k}{\mF \vz^k} &= \norm{\mF \vw_\dagger^k - \mF \vz^k}^2 - \norm{\mF \vw_\dagger^k}^2 - \norm{\mF \vz^k}^2 \\ 
      &\leq L^2 \norm{- \eta_k n \mF \vz^k}^2 - \norm{\mF \vw_\dagger^k}^2 - \norm{\mF \vz^k}^2 \\ 
      &= -(1-\eta_k^2 n^2 L^2)\norm{\mF \vz^k}^2  - \norm{\mF \vw_\dagger^k}^2. 
   \end{align*} 
   On the other hand, by the weighted AM-GM inequality (\Cref{lem:w-amgm-ineq}), for any number $a_k \in (0, 1)$ it holds that \[
       - 2\inprod{\vr^k}{\vz^{k+1} - \vz^*} \leq \frac{1}{a_k}\norm{\vr^k}^2 + a_k \norm{\vz^{k+1} - \vz^*}^2. 
   \] Using these two bounds, we get \begin{align*}
       \norm{\vz^{k+1} - \vz^*}^2 &\leq \norm{\vz^k - \vz^*}^2 - 2\eta_k n \inprod{\mF \vw_\dagger^k}{\vw_\dagger^k - \vz^*} - \eta_k^2 n^2 (1-\eta_k^2 n^2 L^2)\norm{\mF \vz^k}^2  \\ 
       &\phantom{\leq} \quad - \eta_k^2 n^2 \norm{\mF \vw_\dagger^k}^2 + \eta_k^2 n^2 \norm{\mF \vw_\dagger^k}^2 + a_k\norm{\vz^{k+1} - \vz^*}^2 + \left( \frac{1}{a_k} - 1 \right) \norm{\vr^k}^2. 
   \end{align*} Choosing $a_k = \frac{\gamma_k}{1+\gamma_k}$ and rearranging the terms, we obtain \begin{equation}\label{eqn:vi-gap-bound}\begin{aligned}
       2\eta_k n \inprod{\mF \vw_\dagger^k}{\vw_\dagger^k - \vz^*} &\leq \norm{\vz^k - \vz^*}^2  - \frac{1}{1+\gamma_k} \norm{\vz^{k+1} - \vz^*}^2 \\ 
       &\phantom{\leq} \qquad - \eta_k^2 n^2 (1-\eta_k^2 n^2 L^2)\norm{\mF \vz^k}^2   + \frac{1}{\gamma_k}  \norm{\vr^k}^2.
   \end{aligned} \end{equation}
   The left hand side of \eqref{eqn:vi-gap-bound} is nonnegative by the star-monotonicity of $\mF$ \eqref{eqn:def-star-monotone}, and the claimed inequality follows. 
\end{proof}

Now we show that choosing the appropriate stepsizes leads to $\norm{\mF \vz^k}$ being bounded uniformly over~$k$. 

\begin{prop}\label{prop:bounded-iterates-m}
     Let $\mF$ be a (star-)monotone operator with a point $\vz^*$ that satisfies $\mF\vz^* = \zero$, and suppose that Assumptions~\ref{asmp:smoothness}~and~\ref{asmp:bounded-variance} hold.
    Say we are using \textsf{SEG-FFA}, or any optimization method whose within-epoch error satisfies \eqref{eqn:rk1-ffa} and \eqref{eqn:rk2-ffa}.
    Let the sequence of stepsizes $\{\eta_k\}_{k \geq 0}$ be nonincreasing, with \begin{equation}\label{eqn:cubic-sum}
        S  \coloneqq \sum_{k=0}^\infty \eta_k^3 n^3 L^3 < \infty.         
    \end{equation} 
    Suppose that initial stepsize $\eta_0$ is chosen sufficiently small so that %
         \begin{equation} \label{item:lr-cond2-m} 
               \eta_0^2 n^2 L^2 + \frac{3 \eta_0 n C_{\textsf{1A}}^2}{  L^3 } + \frac{3 \eta_0  n  D_{\textsf{1A}}^2}{ L } \cdot e^{S}  \left( \norm{\vz^{0} - \vz^*}^2 + \frac{6 S V_{\textsf{1A}}^2}{L^6} \right) \leq 1 
         \end{equation}
         for constants $C_\textsf{1A}$, $D_\textsf{1A}$, and $V_\textsf{1A}$ defined in \eqref{eqn:def-c1a}--\eqref{eqn:def-v1a}. 
    Then for all $k \geq 0$, \begin{equation}\label{eqn:gradient-norm-bound-m}
        \norm{\mF \vz^k}^2 \leq e^{S} L^2 \left( \norm{\vz^{0} - \vz^*}^2 + \frac{6 S V_{\textsf{1A}}^2}{L^6} \right). 
     \end{equation} 
 \end{prop} 
\begin{proof}  
    We use induction on $k$, to establish a stronger inequality  \begin{equation} \label{eqn:bounded-iterates-m}\begin{aligned} 
            \norm{\vz^{k} - \vz^*}^2 &\leq  e^{S}  \left( \norm{\vz^{0} - \vz^*}^2 + \frac{6 S V_{\textsf{1A}}^2}{L^6} \right). 
       \end{aligned}\end{equation}    
    To see that \eqref{eqn:bounded-iterates-m} indeed implies \eqref{eqn:gradient-norm-bound-m}, notice that by the $L$-smoothness of $f$ it holds that \[
        \norm{\mF \vz^k}^2  =  \norm{\mF \vz^k - \mF \vz^*}^2 \leq L^2 \norm{\vz^k - \vz^*}^2.     
    \]
    
    For the case when $k = 0$, as $S > 0$, it is clear that \eqref{eqn:bounded-iterates-m} holds. 
    Now suppose that \eqref{eqn:bounded-iterates-m} holds for some $k \geq 0$. 
    Applying Young's inequality on \eqref{eqn:rk1-ff} leads to\[ 
        \norm{\vr^k}^2 \leq 3 \eta_k^6 n^6 \left( C_{\textsf{1A}}^2 \norm{\mF \vz^k}^2 + D_{\textsf{1A}}^2 \norm{ \mF \vz^k}^4 + V_{\textsf{1A}}^2 \right). 
    \]
Applying this bound on $\norm{\vr^k}^2$ on \eqref{eqn:descent-of-eg}, we obtain \begin{equation} \label{eqn:special-case-descent-lemma-m} \begin{aligned}
    \MoveEqLeft[8] \eta_k^2 n^2 \left(1-\eta_k^2 n^2 L^2 - \frac{3 \eta_k^4 n^4 C_{\textsf{1A}}^2}{\gamma_k} - \frac{3 \eta_k^4 n^4 D_{\textsf{1A}}^2}{\gamma_k}\norm{\mF \vz^k}^2\right)\norm{\mF \vz^k}^2 \\ 
    &\leq \norm{\vz^k - \vz^*}^2  - \frac{1}{1+\gamma_k} \norm{\vz^{k+1} - \vz^*}^2  + \frac{3 \eta_k^6 n^6 V_{\textsf{1A}}^2}{\gamma_k} .
 \end{aligned} \end{equation} 
 Choose $\gamma_k = \eta_k^3 n^3 L^3$. 
 Notice that \eqref{item:lr-cond2-m} implies $\eta_0 n L \leq 1$, henceforth $\eta_k \leq \eta_0 \leq \nicefrac{1}{nL}$. This, with the induction hypothesis \eqref{eqn:gradient-norm-bound-m}, implies
 \begin{align*}
   \MoveEqLeft \eta_k^2 n^2 L^2 + \frac{3 \eta_k^4 n^4 C_{\textsf{1A}}^2}{\gamma_k} + \frac{3 \eta_k^4 n^4 D_{\textsf{1A}}^2}{\gamma_k}\norm{\mF \vz^k}^2  \\
   &= \eta_k^2 n^2 L^2 + \frac{3 \eta_k  n  C_{\textsf{1A}}^2}{ L^3} + \frac{3 \eta_k  n  D_{\textsf{1A}}^2}{ L^3}\norm{\mF \vz^k}^2 \\
   &\leq \eta_0^2 n^2 L^2 + \frac{3 \eta_0  n  C_{\textsf{1A}}^2}{ L^3} + \frac{3 \eta_0  n  D_{\textsf{1A}}^2}{ L^3}\norm{\mF \vz^k}^2 \\
   &\leq \eta_0^2 n^2 L^2 + \frac{3 \eta_0  n  C_{\textsf{1A}}^2}{ L^3} + \frac{3 \eta_0  n  D_{\textsf{1A}}^2}{ L^3} \cdot e^{S} L^2 \left( \norm{\vz^{0} - \vz^*}^2 + \frac{6 S V_{\textsf{1A}}^2}{L^6} \right) \\
   &\leq 1 . 
 \end{align*} That is, the left hand side of \eqref{eqn:special-case-descent-lemma-m} becomes nonnegative. 
 Then it is immediate from \eqref{eqn:special-case-descent-lemma-m} that  \begin{align*}
    \norm{\vz^{k+1} - \vz^*}^2 &\leq \left( 1+ \gamma_k \right)\norm{\vz^k - \vz^*}^2 + \frac{3 \eta_k^6 n^6 \left( 1+ \gamma_k \right) V_{\textsf{1A}}^2}{\gamma_k} \\ 
    &\leq \left( 1+ \eta_k^3 n^3 L^3 \right)\norm{\vz^k - \vz^*}^2 + \frac{6 \eta_k^3 n^3 V_{\textsf{1A}}^2}{ L^3}.
 \end{align*} Using \Cref{lem:recurrence} to unravel this recurrence relation, we obtain \begin{align*}
    \norm{\vz^{k+1} - \vz^*}^2 &\leq \left( \prod_{j=0}^k \left( 1+ \eta_j^3 n^3 L^3 \right) \right) \left( \norm{\vz^{0} - \vz^*}^2 + \sum_{j=0}^k \frac{6 \eta_j^3 n^3 V_{\textsf{1A}}^2}{ L^3} \right) \\ 
    &\leq e^{\sum_{j=0}^k \eta_j^3 n^3 L^3}  \left( \norm{\vz^{0} - \vz^*}^2 + \frac{6 V_{\textsf{1A}}^2}{L^6} \sum_{j=0}^k \eta_j^3 n^3 L^3 \right) \\ 
    &\leq e^{S}  \left( \norm{\vz^{0} - \vz^*}^2 + \frac{6 S V_{\textsf{1A}}^2}{L^6} \right)
 \end{align*} which shows that \eqref{eqn:bounded-iterates-m} also holds when $k$ is replaced by $k+1$. This completes the proof.
\end{proof}

\begin{thm}[Formal version of~\Cref{thm:cc-ffa}]
Let $\mF$ be a (star-)monotone operator with a point $\vz^*$ that satisfies $\mF\vz^* = \zero$, and suppose that Assumptions~\ref{asmp:smoothness}~and~\ref{asmp:bounded-variance} hold.
Say that we are using \segffa{}, or any optimization method whose within-epoch error satisfies \eqref{eqn:rk1-ffa} and \eqref{eqn:rk2-ffa}, with $\beta_k = \eta_k = \frac{\eta_0\sqrt[3]{2}\log 2}{(k+2)^{1/3} \log(k+2)}$ and $\alpha_k = \nicefrac{\beta_k}{2}$ for $k = 0, 1, \dots$, where, for $S  \coloneqq \sum_{k=0}^\infty \eta_k^3 n^3 L^3$, the initial stepsize $\eta_0$ is chosen so that \begin{equation} \label{item:lr-cond2} 
               \eta_0^2 n^2 L^2 + \frac{3 \eta_0 n C_{\textsf{1A}}^2}{  L^3 } + \frac{3 \eta_0  n  D_{\textsf{1A}}^2}{ L } \cdot e^{S}  \left( \norm{\vz^{0} - \vz^*}^2 + \frac{6 S V_{\textsf{1A}}^2}{L^6} \right) \leq 1 
         \end{equation} %
for constants $C_\textsf{1A}$, $D_\textsf{1A}$, and $V_\textsf{1A}$ defined in \eqref{eqn:def-c1a}--\eqref{eqn:def-v1a}, and there exists a positive constant $\lambda > 0$ such that \begin{equation}
     \eta_0^2 n^2 L^2 + \frac{\eta_0 n C_{\textsf{2A}} }{ L^3} + \frac{\eta_0 n D_{\textsf{2A}} }{L} \cdot e^{S} \left( \norm{\vz^{0} - \vz^*}^2 + \frac{6 S V_{\textsf{1A}}^2}{L^6} \right) \leq 1-\lambda \label{eqn:appx-gn-is-bounded}
\end{equation} 
for constants $C_\textsf{2A}$, $D_\textsf{2A}$, and $V_\textsf{2A}$ defined in \eqref{eqn:def-c2a}--\eqref{eqn:def-v2a}.
Then for any $K \geq 1$, it holds that \begin{equation} \label{eqn:appx-meta-bound}
    \min_{k= 0, 1, \dots, K} \expt \norm{\mF \vz^k}^2  \leq \frac{(\log (K+3))^2}{(K+3)^{1/3}} \cdot \left( \frac{\norm{\vz^0 - \vz^*}^2 + \frac{3 V_\textsf{2A}}{n L^6}}{ \lambda e^{-3/2} (\sqrt[3]{2}\log 2)^2} \eta_0^2 n^2 \right) . 
\end{equation}
\end{thm}
    \begin{proof}
As the sequence of stepsizes $\{\eta_k\}_{k \geq 0}$ is nonincreasing and   \eqref{item:lr-cond2} asserts that $\eta_0 \leq \nicefrac{1}{nL}$, we can use the bounds established in \Cref{thm:rk1-ffa} and \Cref{thm:rk2-ffa}. Also, the premises required for \Cref{prop:bounded-iterates-m} are also satisfied, so the bound \eqref{eqn:gradient-norm-bound-m} holds. 
    
Setting %
$\gamma_k = \eta_k^3 n^3 L^3$ in \eqref{eqn:descent-of-eg} and then taking the conditional expectation given $\vz^k$, with using \eqref{eqn:rk2-ffa} and \eqref{eqn:appx-gn-is-bounded}, we obtain \begin{align*} 
    0 %
    &\leq  \norm{\vz^k - \vz^*}^2  - \frac{1}{1+\gamma_k}\expt\left[ \norm{\vz^{k+1} - \vz^*}^2 \,\middle|\, \vz^k \right] - \eta_k^2 n^2 (1-\eta_k^2 n^2 L^2)\norm{\mF \vz^k}^2   + \frac{1}{\gamma_k} \expt\left[ \norm{\vr^k}^2 \,\middle|\, \vz^k \right] \\ 
    &\leq  \norm{\vz^k - \vz^*}^2  - \frac{1}{1+\gamma_k}\expt\left[ \norm{\vz^{k+1} - \vz^*}^2 \,\middle|\, \vz^k \right] \\ 
    &\phantom{\leq} \qquad - \eta_k^2 n^2 (1-\eta_k^2 n^2 L^2)\norm{\mF \vz^k}^2   + \frac{1}{L^3} \left( \eta_k^{3} n^{3} C_\textsf{2A} \norm{\mF \vz^k}^2 + \eta_k^{3} n^{3} D_\textsf{2A} \norm{ \mF \vz^k}^4 + \eta_k^{3} n^{2} V_\textsf{2A} \right) \\
    &\leq  \norm{\vz^k - \vz^*}^2  - \frac{1}{1+\gamma_k}\expt\left[ \norm{\vz^{k+1} - \vz^*}^2 \,\middle|\, \vz^k \right] \\ 
    &\phantom{\leq} \qquad - \eta_k^2 n^2 \left(1-\eta_k^2 n^2 L^2 - \frac{\eta_k n C_\textsf{2A} }{L^3} - \frac{\eta_k n D_\textsf{2A} }{L^3} \norm{\mF \vz^k}^2 \right)\norm{\mF \vz^k}^2   + \frac{\eta_k^3 n^2 V_\textsf{2A}}{L^3} \notag \\ 
    &\leq  \norm{\vz^k - \vz^*}^2  - \frac{1}{1+\gamma_k}\expt\left[ \norm{\vz^{k+1} - \vz^*}^2 \,\middle|\, \vz^k \right] - \lambda \eta_k^2 n^2 \norm{\mF \vz^k}^2 + \frac{\eta_k^3 n^2 V_\textsf{2A}}{L^3}. 
\end{align*} 
By the law of total expectation, and that $\gamma_k = \eta_k^3 n^3 L^3 < 1$, from the above we get \begin{align*}
    (1+\gamma_k) \lambda \eta_k^2 n^2 \expt \norm{\mF \vz^k}^2 &\leq {(1+\gamma_k)} \expt \norm{\vz^k - \vz^*}^2 - \expt \norm{\vz^{k+1} - \vz^*}^2 + \frac{(1+\gamma_k)\eta_k^3 n^2 V_\textsf{2A}}{L^3} \\ 
    &\leq {(1+\gamma_k)} \expt \norm{\vz^k - \vz^*}^2 - \expt \norm{\vz^{k+1} - \vz^*}^2 + \frac{2 \eta_k^3 n^2 V_\textsf{2A}}{L^3}. 
\end{align*} This recurrence can be unraveled using \Cref{lem:recurrence}, giving us \begin{equation} \label{eqn:appx:meta-bound} \begin{aligned}
    \MoveEqLeft \expt \norm{\vz^{K+1} - \vz^*}^2 + \sum_{k=0}^K (1+\gamma_k) \lambda \eta_j^2 n^2 \expt \norm{\mF \vz^k}^2 \\
    &\leq \left( \prod_{k=0}^{K} (1+\gamma_k) \right) \left( \norm{\vz^0 - \vz^*}^2 + \sum_{k=0}^{K} \frac{2 \eta_k^3 n^2 V_\textsf{2A}}{L^3} \right). 
\end{aligned} \end{equation} 

For the left hand side of \eqref{eqn:appx:meta-bound}, we have
\begin{equation} \label{eqn:segffa-bound-lhs-lower} \begin{aligned}
   \expt \norm{\vz^{K+1} - \vz^*}^2 + \sum_{k=0}^K (1+\gamma_k) \lambda \eta_k^2 n^2 \expt \norm{\mF \vz^k}^2  &\geq  \lambda  \sum_{k=0}^K \eta_k^2 n^2 \expt \norm{\mF \vz^k}^2 \\ 
   &\geq  \lambda  \min_{k= 0, 1, \dots, K} \expt \norm{\mF \vz^k}^2 \ \sum_{k=0}^K \eta_k^2 n^2.
\end{aligned} \end{equation}
From \Cref{lem:cbrt-log-growth-bound}, we know that whenever $K\geq 1$, \begin{align*}
\sum_{k=0}^K \eta_k^2 n^2 &= \eta_0^2 n^2 (\sqrt[3]{2}\log 2)^2\sum_{k=0}^K \frac{1}{(k+2)^{2/3} (\log(k+2))^2}  \\
&\geq \eta_0^2 n^2 (\sqrt[3]{2}\log 2)^2 \cdot \frac{(K+3)^{1/3}}{(\log (K+3))^2}.
\end{align*}  
Meanwhile, as $x \mapsto \frac{2 (\log 2)^3}{(x+2)(\log(x+2))^3}$ is a decreasing function, we have \begin{align*}
\sum_{k=0}^\infty \frac{2 (\log 2)^3}{(k+2) (\log (k+2))^3} &\leq 1 + \frac{2 (\log 2)^3}{3 (\log 3)^3} + \int_1^\infty \frac{2 (\log 2)^3}{(x+2)(\log(x+2))^3} \,\mathrm{d}x \\ 
&\leq 1 + \frac{2 (\log 2)^3}{3 (\log 3)^3} + \frac{(\log 2)^3}{(\log 3)^2}  \quad \leq \frac{3}{2}%
\end{align*}
and thus 
\begin{align*}
S = \sum_{k=0}^\infty \eta_k^3 n^3 L^3 &=  \eta_0^3 n^3 L^3 \sum_{k=0}^\infty \frac{2 (\log 2)^3}{(k+2)  (\log(k+2))^3} \leq \frac{3}{2} \eta_0^3 n^3 L^3 \leq \frac{3}{2}. 
\end{align*} Thus, for the right hand side of \eqref{eqn:appx:meta-bound}, it holds that 
\begin{equation}\label{eqn:segffa-bound-rhs-upper}\begin{aligned}
    \left( \prod_{k=0}^{K} (1+\gamma_k) \right) \left( \norm{\vz^0 - \vz^*}^2 + \sum_{k=0}^{K} \frac{2 \eta_k^3 n^2 V_\textsf{2A}}{L^3} \right) &\leq e^{\sum_{k=0}^K \gamma_k} \left( \norm{\vz^0 - \vz^*}^2 + \sum_{k=0}^{K} \frac{2 \eta_k^3 n^2 V_\textsf{2A}}{L^3} \right) \\
    &\leq e^{S} \left( \norm{\vz^0 - \vz^*}^2 + \frac{2 S  V_\textsf{2A}}{n L^6} \right) \\
    &\leq e^{3/2} \left( \norm{\vz^0 - \vz^*}^2 + \frac{3 V_\textsf{2A}}{n L^6} \right). 
\end{aligned}\end{equation}
Therefore, from \eqref{eqn:appx:meta-bound} we get \[
\lambda \eta_0^2 n^2 (\sqrt[3]{2}\log 2)^2 \cdot \frac{(K+3)^{1/3}}{(\log (K+3))^2} \cdot \min_{k= 0, 1, \dots, K} \expt \norm{\mF \vz^k}^2 \leq e^{3/2} \left( \norm{\vz^0 - \vz^*}^2 + \frac{3 V_\textsf{2A}}{n L^6} \right). 
\]
Simply rearranging the terms gives us the desired inequality.
\end{proof}  

\begingroup 
\begin{rmk} \label{rmk:poly-stepsize-validation}
While $\eta_0$ should be chosen so that both \eqref{item:lr-cond2} and \eqref{eqn:appx-gn-is-bounded} hold, in practice, there is a way to circumvent this complication.  
Notice that in deriving the upper bound \eqref{eqn:segffa-bound-rhs-upper} of the right hand side of \eqref{eqn:appx:meta-bound}, it suffices to have $\eta_k \leq \frac{\eta_0\sqrt[3]{2}\log 2}{(k+2)^{1/3} \log(k+2)}$, and the lower bound \eqref{eqn:segffa-bound-lhs-lower} of the left hand side holds for any $\eta_k \ge 0$. 
In other words, if we have had chosen $\eta_k = \Theta\left(\nicefrac{1}{(k+1)^q}\right)$ for $q > \frac{1}{3}$ so that $S  = \sum_{k=0}^\infty \eta_k^3 n^3 L^3 <\infty$, as long as $\eta_0$ satisfies \eqref{item:lr-cond2} and \eqref{eqn:appx-gn-is-bounded}, we would still have obtained the inequality \begin{equation}\label{eqn:segffa-weaker-stepsize}
        \min_{k= 0, 1, \dots, K} \expt \norm{\mF \vz^k}^2 \leq \frac{e^{S}}{\lambda n^2\sum_{k=0}^K \eta_k^2}  \left( \norm{\vz^0 - \vz^*}^2 + \frac{2 S  V_\textsf{2A}}{n L^6} \right). 
\end{equation} In particular, if we additionally assume that $q < \frac{1}{2}$ then \[
\sum_{k = 0}^K {\eta_k^2} \asymp \sum_{k = 1}^K \frac{1}{k^{2q}}  \asymp K^{1-2q},
\] so from \eqref{eqn:segffa-weaker-stepsize} we would have obtained the convergence rate \begin{equation}\label{eqn:segffa-weaker-rate}
        \min_{k= 0, 1, \dots, K} \expt \norm{\mF \vz^k}^2 = \gO\left(\frac{1}{K^{1-2q}}\right). 
\end{equation} 

We now claim that, if one accepts a slight sacrifice of the convergence rate from $\tilde{\gO} (\nicefrac{1}{K^{1/3}})$ to $\gO\left(\nicefrac{1}{K^{1-2q}}\right)$ for $\nicefrac{1}{3} < q < \nicefrac{1}{2}$, one can simply choose the stepsizes as $\eta_k = \nicefrac{\eta_{00}}{(k+1)^q}$ for a sufficiently small $\eta_{00}$. To see why this is the case, let us fix $\eta_0$ to be a number that satisfies the inequalities \eqref{item:lr-cond2} and \eqref{eqn:appx-gn-is-bounded}. Then, because $\nicefrac{\eta_{00}}{(k+1)^q} = o\left(\frac{1}{(k+2)^{1/3}\log(k+2)}\right)$, there will exist a nonnegative integer $k_0$ such that $\eta_k \leq \frac{\eta_0\sqrt[3]{2}\log 2}{(k+2)^{1/3} \log(k+2)}$ for all $k \geq k_0$. So, by ignoring the first $k_0$ terms if necessary---that is, considering as if the $k_0$th iteration is the $0$th iteration---it follows from the discussions made above in obtaining \eqref{eqn:segffa-weaker-rate} that we get the rate of convergence $\gO\left(\nicefrac{1}{K^{1-2q}}\right)$. 

This discussion also justifies the choice of stepsizes $\eta_k = \Theta\left(\nicefrac{1}{(1+k/10)^{0.34}}\right)$ used in the experiments for the monotone setting. 
\end{rmk}
\endgroup

\section{Proof of Lower Bounds} \label{appx:lower-bounds-pf}

\subsection{Proof of the Divergence of \segus{}, \segrr{} and \segff{}} \label{ssec:divergence}

We prove the divergence of \segus{}, \segrr{} and \segff{} 
in each proposition below, using the same worst-case problem
for $n=2$. 
These constitute the proof of \Cref{thm:rr-ff-bad}.

\begin{prop}[Part of \Cref{thm:rr-ff-bad}]
\label{thm:segusdiv}
For $n = 2$, there exists a convex-concave minimax problem $f(x,y) = \frac{1}{2} \sum_{i=1}^2 f_i(x,y)$ having a monotone $\mF$,
consisting of $L$-smooth quadratic $f_i$'s satisfying Assumption~\ref{asmp:bounded-variance} with $(\rho, \sigma) = (1,0)$ such that \segus{} diverges in expectation for any choice of stepsizes $\{\alpha_t\}_{t \geq 0}$ and $\{\beta_t\}_{t \geq 0}$. That is, for all $t \geq 0$,
 \begin{equation*}
        \E\left[\norm{\vz_{t+1}}^2 \right] > \E \left[\norm{\vz_t}^2 \right],\quad
        \E\left[\norm{\mF \vz_{t+1}}^2 \right] > \E \left[\norm{\mF \vz_t}^2 \right].
\end{equation*}
\end{prop}
\begin{proof}
    We consider the case of 
    \begin{align*}
        f_1(x,y) &= -\frac{L}{4} x^2 + \frac{L}{2}xy - \frac{L}{4} y^2,\\
        f_2(x,y) &= \frac{L}{4} x^2 + \frac{L}{2}xy + \frac{L}{4} y^2,
    \end{align*}
    which result in a bilinear (and hence convex-concave) objective function
    \begin{equation}
    \label{eqn:segrrdiv-ex}
        f(x,y) = \frac{1}{2} \sum_{i=1}^2 f_i(x,y) = \frac{L}{2} xy.
    \end{equation}
    One can quickly check from the definitions of the component functions $f_1$ and $f_2$ that the corresponding saddle gradient operators are given as
    \begin{align*}
        \mF_1\vz = 
        \underbrace{\begin{bmatrix}
        -L/2 & L/2 \\ -L/2 & L/2
        \end{bmatrix}}_{:= \mA_1}
        \vz,
        \quad
        \mF_2\vz =
        \underbrace{\begin{bmatrix}
        L/2 & L/2 \\ -L/2 & -L/2
        \end{bmatrix}}_{:= \mA_2}
        \vz
        ,\quad
        \mF \vz = \begin{bmatrix} 0 & L/2 \\ -L/2 & 0\end{bmatrix} \vz
    \end{align*}
    where $\vz = (x,y) \in \R^2$.
    From the fact that $\norm{\mA_i} \leq L$ for all $i$'s, we can confirm that $f_i$'s are indeed $L$-smooth.
    As for Assumption~\ref{asmp:bounded-variance}, we can verify that
    \begin{align*}
    \frac{1}{2} \sum_{i=1}^2         \norm{\mF_i \vz - \mF\vz}^2
        =
        \frac{L^2}{4}\norm{\vz}^2
        = 
        \norm {\mF \vz} ^2,
    \end{align*}
    thus proving that our example $f$ indeed satisfies Assumption~\ref{asmp:bounded-variance} with $(\rho, \sigma) = (1,0)$.

    We now proceed to show that for this particular worst-case example $f$, \segus{} 
    diverges in expectation.
    For $t \geq 0$, the $(t+1)$-th iteration of \segus{} starts at $\vz_t$, and the algorithm uniformly chooses an index $i(t)$ from $[n]$. The algorithm then makes an update
    \begin{align*}
    \vw_t &= \vz_t - \alpha_t \mF_{i(t)} \vz_t,  \\
    \vz_{t+1} &= \vz_t - \beta_t \mF_{i(t)} \vw_{t}.
    \end{align*}
    In our worst-case example $f$, the updates can be compactly written as
    \begin{align*}
    \vz_{t+1} &= (\mI - \beta_t \mA_{i(t)} + \alpha_t \beta_t \mA_{i(t)}^2) \vz_{t}.
    \end{align*}
    Since we have $n=2$, 
    the update can be summarized as
    \begin{align*}
    \vz_{t+1} &= 
    \begin{cases}
    (\mI - \beta_t \mA_1 + \alpha_t \beta_t \mA_1^2) \vz_t & \text{ with probability $1/2$, }\\
    (\mI - \beta_t \mA_2 + \alpha_t \beta_t \mA_2^2) \vz_t & \text{ with probability $1/2$. }
    \end{cases}
    \end{align*}    
    By the definition of $\mA_1$ and $\mA_2$
    and using 
    $\mA_1^2 = \mA_2^2 = \vzero$, we can verify that
    \begin{align*}
        \mN_1 \coloneqq \mI - \beta_t \mA_1 + \alpha_t \beta_t \mA_1^2 &= 
        \begin{bmatrix}
            1 + \frac{\beta_t L}{2} & - \frac{\beta_t L}{2}\\
            \frac{\beta_t L}{2} & 1 - \frac{\beta_t L}{2}
        \end{bmatrix},\\
        \mN_2 \coloneqq \mI - \beta_t \mA_2 + \alpha_t \beta_t \mA_2^2 &= 
        \begin{bmatrix}
            1 - \frac{\beta_t L}{2} & - \frac{\beta_t L}{2}\\
            \frac{\beta_t L}{2} & 1 + \frac{\beta_t L}{2}
        \end{bmatrix}.
    \end{align*}
    From this, we notice that the expectation of $\norm{\vz_{t+1}}^2$ conditional on $\vz_t$ reads
    \begin{align*}
        \E \left[\norm{\vz_{t+1}}^2 \,\middle|\, \vz_t \right]
        &=
        \vz_t^\top
        \left (
        \frac{\mN_1^\top \mN_1 + \mN_2^\top \mN_2}{2}
        \right )
        \vz_t.
    \end{align*}
    Working out the calculations, we can check that
    \begin{equation*}
        \frac{\mN_1^\top \mN_1 + \mN_2^\top \mN_2}{2} =
        \begin{bmatrix}
        1 + \frac{\beta_t^2 L^2}{2} & 0\\
        0 & 1 + \frac{\beta_t^2 L^2}{2}
        \end{bmatrix},
    \end{equation*}
    thus resulting in
    \begin{equation*}
        \E \left[\norm{\vz_{t+1}}^2  \,\middle|\, \vz_t \right] = \left ( 1 + \frac{\beta_t^2 L^2}{2} \right ) \norm{\vz_t}^2.
    \end{equation*}
    Since this holds for all $t \geq 0$, \segus{} diverges in expectation, for any positive stepsizes $\{\alpha_t\}_{t \geq 0}$ and $\{\beta_t\}_{t \geq 0}$. The statement on $\norm{\mF \vz_t}$ follows by realizing that $\norm{\mF \vz} = \frac{L}{2} \norm{\vz}$.
\end{proof}

\begin{prop}[Part of \Cref{thm:rr-ff-bad}]
\label{thm:segrrdiv}
For $n = 2$, there exists a convex-concave minimax problem $f(x,y) = \frac{1}{2} \sum_{i=1}^2 f_i(x,y)$ having a monotone $\mF$,
consisting of $L$-smooth quadratic $f_i$'s satisfying Assumption~\ref{asmp:bounded-variance} with $(\rho, \sigma) = (1,0)$ such that \segrr{} diverges in expectation for any choice of stepsizes $\{\alpha_k\}_{k\geq 0}$ and $\{\beta_k\}_{k \geq 0}$. That is, for any $k \geq 0$,
\begin{equation*}
        \E \left[\norm{\vz_0^{k+1}}^2 \right] > \E \left[\norm{\vz_0^k}^2 \right],\quad
        \E \left[\norm{\mF \vz_0^{k+1}}^2 \right] > \E \left[\norm{\mF \vz_0^k}^2 \right].
\end{equation*}
\end{prop}

\begin{proof}
    The proof uses the same example as \Cref{thm:segusdiv}, outlined in \eqref{eqn:segrrdiv-ex}.
    We show that for this particular worst-case example $f$, \segrr{} 
    diverges in expectation.
    For $k \geq 0$, the $(k+1)$-th epoch of \segrr{} starts at $\vz_0^k$, and the algorithm randomly chooses a permutation $\tau_k: [n] \to [n]$. The algorithm then goes through a series of updates
    \begin{align*}
    \vw_{i}^k &= \vz_{i}^k - \alpha_k \mF_{\tau_k(i+1)} \vz_{i}^k,  \\
    \vz_{i+1}^{k} &= \vz_{i}^k - \beta_k \mF_{\tau_k(i+1)} \vw_{i}^k,
    \end{align*}
    for $i = 0, \dots, n-1$.
    In our worst-case example $f$, the updates can be compactly written as
    \begin{align*}
    \vz_{i+1}^{k} &= (\mI - \beta_k \mA_{\tau_k(i+1)} + \alpha_k \beta_k \mA_{\tau_k(i+1)}^2) \vz_{i}^k.
    \end{align*}
    Since we have $n=2$ and there are only two possible permutations, 
    the updates over an epoch can be summarized as
    \begin{align*}
    \vz_0^{k+1} = \vz_n^k &= 
    \begin{cases}
    (\mI - \beta_k \mA_1 + \alpha_k \beta_k \mA_1^2)(\mI - \beta_k \mA_2 + \alpha_k \beta_k \mA_2^2) \vz_0^k & \text{ with probability $1/2$, }\\
    (\mI - \beta_k \mA_2 + \alpha_k \beta_k \mA_2^2)(\mI - \beta_k \mA_1 + \alpha_k \beta_k \mA_1^2) \vz_0^k & \text{ with probability $1/2$. }\\
    \end{cases}
    \end{align*}
    By the definition of $\mA_1$ and $\mA_2$
    and using 
    $\mA_1^2 = \mA_2^2 = \vzero$, we can verify that
    \begin{align}
    \label{eqn:segrrdiv-m1}
        \mM_1 \coloneqq (\mI - \beta_k \mA_1 + \alpha_k \beta_k \mA_1^2)(\mI - \beta_k \mA_2 + \alpha_k \beta_k \mA_2^2) &\!=\! 
        \begin{bmatrix}
            1-\frac{\beta_k^2 L^2}{2} & \!\!\!- \beta_k L - \frac{\beta_k^2 L^2}{2}\\
            \beta_k L - \frac{\beta_k^2 L^2}{2} & \!\!\! 1 - \frac{\beta_k^2 L^2}{2}
        \end{bmatrix},\\
    \label{eqn:segrrdiv-m2}
        \mM_2 \coloneqq (\mI - \beta_k \mA_2 + \alpha_k \beta_k \mA_2^2)(\mI - \beta_k \mA_1 + \alpha_k \beta_k \mA_1^2) &\!=\! 
        \begin{bmatrix}
            1-\frac{\beta_k^2 L^2}{2} & \!\!\!- \beta_k L + \frac{\beta_k^2 L^2}{2}\\
            \beta_k L + \frac{\beta_k^2 L^2}{2} & \!\!\!1 - \frac{\beta_k^2 L^2}{2}
        \end{bmatrix}.
    \end{align}
    From this, we notice that the expectation of $\norm{\vz_0^{k+1}}^2$ conditional on $\vz_0^k$ reads
    \begin{align*}
        \E \left[\norm{\vz_0^{k+1}}^2 \,\middle|\, \vz_0^k \right]
        &=
        (\vz_0^k)^\top
        \left (
        \frac{\mM_1^\top \mM_1 + \mM_2^\top \mM_2}{2}
        \right )
        \vz_0^k.
    \end{align*}
    Working out the calculations, we can check that
    \begin{equation*}
        \frac{\mM_1^\top \mM_1 + \mM_2^\top \mM_2}{2} =
        \begin{bmatrix}
        1 + \frac{\beta_k^4 L^4}{2} & 0\\
        0 & 1 + \frac{\beta_k^4 L^4}{2}
        \end{bmatrix},
    \end{equation*}
    thus resulting in
    \begin{equation*}
        \E \left[\norm{\vz_0^{k+1}}^2  \,\middle|\, \vz_0^k \right] = \left ( 1 + \frac{\beta_k^4 L^4}{2} \right ) \norm{\vz_0^k}^2.
    \end{equation*}
    Since this holds for all $k \geq 0$, \segrr{} diverges in expectation, for any positive stepsizes $\{\alpha_k\}_{k \geq 0}$ and $\{\beta_k\}_{k \geq 0}$. The statement on $\norm{\mF \vz_0^k}$ follows by realizing that $\norm{\mF \vz} = \frac{L}{2}\norm{\vz}$.
\end{proof}

\begin{prop}[Part of \Cref{thm:rr-ff-bad}]
\label{thm:segffdiv}
For $n = 2$, there exists a convex-concave minimax problem $f(x,y) = \frac{1}{2} \sum_{i=1}^2 f_i(x,y)$ having a monotone $\mF$,
consisting of $L$-smooth quadratic $f_i$'s satisfying Assumption~\ref{asmp:bounded-variance} with $(\rho, \sigma) = (1,0)$ such that \segff{} diverges in expectation for any positive stepsizes $\{\alpha_k\}_{k\geq 0}$ and $\{\beta_k\}_{k\geq 0}$. That is, for any $k \geq 0$,
\begin{equation*}
        \E \left[\norm{\vz_0^{k+1}}^2 \right] > \E \left[\norm{\vz_0^k}^2 \right],\quad
        \E \left[\norm{\mF \vz_0^{k+1}}^2 \right] > \E \left[\norm{\mF \vz_0^k}^2 \right].
\end{equation*}
\end{prop}
\begin{proof}
    The proof uses the same example as \Cref{thm:segusdiv}, outlined in \eqref{eqn:segrrdiv-ex}. 
    We prove that \segff{} also diverges for this $f$. For $k \geq 0$, the $(k+1)$-th epoch of \segff{} starts at $\vz_0^k$, and the algorithm randomly chooses a permutation $\tau_k: [n] \to [n]$, as in the case of \segrr{}. The algorithm then goes through a series of updates for $i = 0, \dots, n-1$:
    \begin{align*}
    \vw_{i}^k &= \vz_{i}^k - \alpha_k \mF_{\tau_k(i+1)} \vz_{i}^k,  \\
    \vz_{i+1}^{k} &= \vz_{i}^k - \beta_k \mF_{\tau_k(i+1)} \vw_{i}^k,
    \end{align*}
    which are the same as \segrr{}; but then, it performs another series of $n$ updates, in the reverse order. For $i = n, \dots, 2n-1$,
    \begin{align*}
    \vw_{i}^k &= \vz_{i}^k - \alpha_k \mF_{\tau_k(2n-i)} \vz_{i}^k,  \\
    \vz_{i+1}^{k} &= \vz_{i}^k - \beta_k \mF_{\tau_k(2n-i)} \vw_{i}^k.
    \end{align*}
    Using the definition of $\mM_1$ and $\mM_2$ from \eqref{eqn:segrrdiv-m1} and \eqref{eqn:segrrdiv-m2}, one can verify that the $2n = 4$ updates over an epoch of \segff{} can be summarized as
    \begin{align*}
    \vz_0^{k+1} = \vz_{2n}^k &= 
    \begin{cases}
    \mM_2 \mM_1 \vz_0^k & \text{ with probability $1/2$, }\\
    \mM_1 \mM_2 \vz_0^k & \text{ with probability $1/2$. }\\
    \end{cases}
    \end{align*}
    From this, we notice that the expectation of $\norm{\vz_0^{k+1}}^2$ conditional on $\vz_0^k$ reads
    \begin{align*}
        \E \left[\norm{\vz_0^{k+1}}^2 \,\middle|\, \vz_0^k \right]
        &=
        (\vz_0^k)^\top
        \left (
        \frac{\mM_1^\top \mM_2^\top \mM_2 \mM_1 + \mM_2^\top \mM_1^\top \mM_1 \mM_2}{2}
        \right )
        \vz_0^k.
    \end{align*}
    Working out the calculations, we can check that
    \begin{equation*}
        \frac{\mM_1^\top \mM_2^\top \mM_2 \mM_1 + \mM_2^\top \mM_1^\top \mM_1 \mM_2}{2} =
        \begin{bmatrix}
        1 + 2\beta_k^6 L^6 & 0\\
        0 & 1 + 2\beta_k^6 L^6
        \end{bmatrix},
    \end{equation*}
    thus resulting in
    \begin{equation*}
        \E \left[\norm{\vz_0^{k+1}}^2  \,\middle|\, \vz_0^k \right] = \left ( 1 + 2\beta_k^6 L^6 \right ) \norm{\vz_0^k}^2.
    \end{equation*}
    Since this holds for all $k \geq 0$, \segff{} diverges in expectation, for any positive stepsizes $\{\alpha_k\}_{k\geq 0}$ and $\{\beta_k\}_{k\geq 0}$. The statement on $\norm{\mF \vz_0^k}$ follows by realizing that $\norm{\mF \vz} = \frac{L}{2}\norm{\vz}$.
\end{proof}

\subsection{Proof of Limited Convergence of \segus{} in Monotone Cases} \label{sec:appx:segus-monotone-lb}
In \cite{Diak21,Gorb22SEG}, the authors study the same-sample and independent-sample versions of \segus{}, with step sizes $\alpha_t$ and $\beta_t$ satisfying a constant ratio: $\beta_t = \gamma \alpha_t$ for $\gamma \in (0,1]$. While the authors show convergence in the monotone $\mF$ case, there is one important limitation shared by the existing analyses. 
In order to achieve $\min_{t=0,\dots,T} \E [\| \mF \vz_t \|^2] \leq \epsilon^2$ 
for an arbitrarily chosen $\epsilon$,
the algorithms must repeat the same query to the stochastic gradient oracle $b = \gO(\frac{1}{\epsilon^2})$ times at every iteration to reduce the gradient variance from $\sigma^2$ to $\frac{\sigma^2}{b}$. In other words, the convergence bounds for \segus{} in the monotone case have an additive term $\gO(\sigma^2)$ that cannot be reduced to zero by proper choices of stepsizes. Below, we prove that such a $\sigma^2$ term is in fact inevitable for any choices of stepsizes, if the ratio $\gamma$ is fixed constant. This indicates that \segus{} considered in the existing results can never converge all the way to the optimum if $b = 1$ is maintained throughout training. In contrast, our \segffa{} shows convergence in the monotone case even when $b = 1$. 
\begin{thm}
\label{thm:seguscclb}
For $n = 2$, there exists a convex-concave minimax problem $f(x,y) = \frac{1}{2} \sum_{i=1}^2 f_i(x,y)$ having a monotone $\mF$,
consisting of $L$-smooth quadratic $f_i$'s satisfying Assumption~\ref{asmp:bounded-variance} with $(\rho, \sigma) = (0,\sigma)$ such that \segus{} with any positive stepsizes $\{\alpha_t\}_{t\geq 0}$ and $\{\beta_t\}_{t\geq 0}$ satisfying $\beta_t = \gamma \alpha_t$ for $\gamma > 0$ cannot converge beyond a certain fixed constant $\Omega(\sigma^2)$. 
More concretely, for any $t \geq 0$,
\begin{equation*}
        \E \left[\norm{\mF \vz_t}^2 \right] 
        \ge \min \left \{ 
        \norm{\mF \vz_0}^2,
        \frac{\gamma \sigma^2}{2}
        \right \}
\end{equation*}
regardless of the stepsizes.
This holds for both \textbf{same-sample} and \textbf{independent-sample} \segus{}.
\end{thm}

\begin{proof}
We consider the case of 
    \begin{align*}
        f_1(x,y) &= Lxy + \nu x - \nu y,\\
        f_2(x,y) &= Lxy - \nu x + \nu y,
    \end{align*}
    which results in a bilinear (and hence convex-concave) objective function
    \[ %
        f(x,y) = \frac{1}{2} \sum_{i=1}^2 f_i(x,y) = Lxy.
    \] %
    One can quickly check from the definitions of the component functions $f_1$ and $f_2$ that the corresponding saddle gradient operators are given as
    \begin{align*}
        \mF_1\vz = 
        \underbrace{\begin{bmatrix}
        0 & L \\ -L & 0
        \end{bmatrix}}_{:= \mA}
        \vz + \nu \vone,
        \quad
        \mF_2\vz =
        \mA 
        \vz
        - \nu \vone,
        \quad
        \mF \vz = \mA \vz,
    \end{align*}
    where $\vz = (x,y) \in \R^2$.
    From the fact that $\norm{\mA} \leq L$, we can confirm that $f_i$'s are indeed $L$-smooth.
    As for Assumption~\ref{asmp:bounded-variance}, we can verify that
    \begin{align*}
    \frac{1}{2} \sum_{i=1}^2 \norm{\mF_i \vz - \mF\vz}^2
        =
    \frac{1}{2} \sum_{i=1}^2 2 \nu^2
        = 
        2 \nu^2.
    \end{align*}
    Therefore, by choosing $\nu^2 = \frac{\sigma^2}{2}$, our example $f$ indeed satisfies Assumption~\ref{asmp:bounded-variance} with $(\rho, \sigma) = (0,\sigma)$.

    The proof is outlined as follows. For the example constructed above, we will calculate the $\E [\norm{\vz_{t+1}}^2]$ and show that the expectation is identical for both same-sample and independent sample versions of \segus{}. We will then show that the update on the expected squared distance to equilibrium $\E [\norm{\vz_{t+1}}^2]$ for given $\vz_t$ can only belong to two categories: either $\norm{\vz_t}^2 \leq \E [\norm{\vz_{t+1}}^2]$ (expected squared distance increases) or $\norm{\vz_t}^2 \geq \E [\norm{\vz_{t+1}}^2] \geq \frac{\gamma\sigma^2}{2L}$ (expected squared distance shrinks but is bounded from below by a constant). Since the two cases hold for any $t \geq 0$ and any choices of $\alpha_t$ and $\beta_t = \gamma \alpha_t$, we show that the ``convergence'' can happen only up to a neighborhood of equilibrium.

    At iteration $t$, \segus{} samples component indices $i(t), j(t) \in \{1, 2\}$ for its extrapolation step and update step, respectively. In the independent-sample version $i(t)$ and $j(t)$ are independently sampled from $\textup{Unif}(\{1,2\})$, and in the same-sample version $i(t)$ is sampled uniformly at random and $j(t)$ is set to be equal to $i(t)$.
    With the indices sampled as above, \segus{} then makes an update
    \begin{align*}
        \vw_t &= \vz_t - \alpha_t \mF_{i(t)} \vz_t,  \\
        \vz_{t+1} &= \vz_t - \beta_t \mF_{j(t)} \vw_{t}.
    \end{align*}
    In our worst-case example $f$, the updates can be written as
    \begin{align*}
    \vw_t 
    &= \vz_t - \alpha_t \mA \vz_t - s_{i(t)} \alpha_t \nu \vone\\
    &= (\mI - \alpha_t \mA) \vz_t - s_{i(t)} \alpha_t \nu \vone\\
    \vz_{t+1} &= \vz_t - \beta_t \mA \vw_t - s_{j(t)} \beta_t \nu \vone\\
    &= (\mI - \beta_t \mA + \alpha_t \beta_t \mA^2)\vz_t + s_{i(t)} \alpha_t \beta_t \nu \mA \vone - s_{j(t)} \beta_t \nu \vone,
    \end{align*}
    where we defined $s_1 = +1$ and $s_2 = -1$ for simplicity of notation.

    We now calculate the expected value of $\norm{\vz_{t+1}}^2$. 
    \begin{align*}
        \norm{\vz_{t+1}}^2
        =&
        \norm{(\mI - \beta_t \mA + \alpha_t \beta_t \mA^2)\vz_t}^2
        +
        \alpha_t^2 \beta_t^2 \nu^2 \norm{\mA \vone}^2
        + 
        \beta_t^2 \nu^2 \norm{\vone}^2\\
        &+
        2 s_{i(t)} \alpha_t \beta_t \nu \langle (\mI - \beta_t \mA + \alpha_t \beta_t \mA^2)\vz_t, \mA \vone \rangle
        -
        2 s_{j(t)} \beta_t \nu \langle (\mI - \beta_t \mA + \alpha_t \beta_t \mA^2)\vz_t, \vone \rangle\\
        &-
        2 s_{i(t)} s_{j(t)} \alpha_t \beta_t^2 \nu^2 \langle \mA \vone, \vone \rangle.
    \end{align*}
    For the independent-sample case, since $s_{i(t)}$ and $s_{j(t)}$ are independent mean-zero random variables,
    \begin{align}
    \label{eq:seguscclb-1}
        \E_{i(t),j(t)} [\norm{\vz_{t+1}}^2]
        =
        \norm{(\mI - \beta_t \mA + \alpha_t \beta_t \mA^2)\vz_t}^2
        +
        \alpha_t^2 \beta_t^2 \nu^2 \norm{\mA \vone}^2
        + 
        \beta_t^2 \nu^2 \norm{\vone}^2.
    \end{align}
    In the same-sample case, $s_{i(t)} = s_{j(t)}$ is a mean-zero random variable, so
    \begin{align*}
        \E_{i(t)} [\norm{\vz_{t+1}}^2]
        =
        \norm{(\mI - \beta_t \mA + \alpha_t \beta_t \mA^2)\vz_t}^2
        +
        \alpha_t^2 \beta_t^2 \nu^2 \norm{\mA \vone}^2
        + 
        \beta_t^2 \nu^2 \norm{\vone}^2
        -
        2 \alpha_t \beta_t^2 \nu^2 \langle \mA \vone, \vone \rangle,
    \end{align*}
    but once we realize that $\langle \mA \vone, \vone \rangle = 0$, the expectation becomes identical to \eqref{eq:seguscclb-1}; hence, the rest of the analysis is the same for the two versions.
    
    We now expand and arrange the RHS of \eqref{eq:seguscclb-1}. 
    It is easy to check that
    \begin{align*}
        (\mI - \beta_t \mA + \alpha_t \beta_t \mA^2) \vz_t
        =
        \begin{bmatrix}
            1 - \alpha_t \beta_t L^2 & - \beta_t L\\
            \beta_t L & 1 - \alpha_t \beta_t L^2
        \end{bmatrix}
        \begin{bmatrix}
            x_t\\
            y_t
        \end{bmatrix}
        =
        \begin{bmatrix}
            (1-\alpha_t \beta_t L^2) x_t - \beta_t L y_t\\
            \beta_t L x_t + (1-\alpha_t \beta_t L^2) y_t
        \end{bmatrix}
    \end{align*}
    and hence
    \begin{align*}
    \norm{(\mI - \beta_t \mA + \alpha_t \beta_t \mA^2) \vz_t}^2
    &= 
    \left ( (1-\alpha_t \beta_t L^2)^2 + \beta_t^2 L^2 \right ) \norm{\vz_t}^2\\
    &=
    \left ( 1 - 2\alpha_t \beta_t L^2 + \beta_t^2 L^2 (1+\alpha_t^2 L^2)\right ) \norm{\vz_t}^2.
    \end{align*}
    From this, we get
    \begin{align*}
        \E [\norm{\vz_{t+1}}^2]
        &=
        \norm{\vz_{t}}^2 
        -
        \left ( 2\alpha_t \beta_t L^2 - \beta_t^2 L^2 (1+\alpha_t^2 L^2)\right ) \norm{\vz_{t}}^2 
        +
        2 \alpha_t^2 \beta_t^2 L^2 \nu^2
        + 
        2 \beta_t^2 \nu^2\\
        &=
        \norm{\vz_{t}}^2 
        -
        \left ( 2\alpha_t \beta_t L^2 - \beta_t^2 L^2 (1+\alpha_t^2 L^2)\right ) \norm{\vz_{t}}^2 
        + \beta_t^2 \sigma^2 (1+\alpha_t^2 L^2),
    \end{align*}
    where we used the choice $\nu_2 = \frac{\sigma^2}{2}$ as above.

    The rest of the proof proceeds as follows: we show that, regardless of $t \geq 0$ and the choices of $\alpha_t$ and $\beta_t = \gamma \alpha_t$, the expected value of $\norm{\vz_{t+1}}^2$ given $\vz_t$ can be categorized into only two cases: 
    \begin{enumerate}
        \item $\norm{\vz_t}^2 \leq \E [\norm{\vz_{t+1}}^2]$. That is, the iterate moves away from the equilibrium in expectation.
        \item $\norm{\vz_t}^2 \geq \E [\norm{\vz_{t+1}}^2] \geq \frac{\gamma \sigma^2}{2L^2}$. That is, the expected squared distance shrinks but is lower bounded by a certain constant independent of the stepsizes.
    \end{enumerate}
    Showing this immediately finishes the proof, because there is no way that any $\E[\norm{\vz_t}^2]$ can get smaller than $\min \{ \norm{\vz_0}^2, \frac{\gamma \sigma^2}{2L^2} \}$, and $\norm{\mF \vz} = L\norm{\vz}$ for our example $f$.

    The remaining proof is simple, by noticing that $\norm{\vz_t}^2 \leq \E [\norm{\vz_{t+1}}^2]$ is equivalent to
    \begin{equation}
    \label{eq:seguscclb-2}
        \left ( 2\alpha_t \beta_t L^2 - \beta_t^2 L^2 (1+\alpha_t^2 L^2)\right ) \norm{\vz_{t}}^2 
        \leq \beta_t^2 \sigma^2 (1+\alpha_t^2 L^2).
    \end{equation}
    Hence, if $\alpha_t$, $\beta_t$, and $\vz_t$ satisfies \eqref{eq:seguscclb-2}, we belong to the first category. Otherwise, we are in the second category, for which we need to additionally show $\E [\norm{\vz_{t+1}}^2] \geq \frac{\gamma \sigma^2}{2L^2}$.
    When the inequality \eqref{eq:seguscclb-2} is satisfied with the opposite sign, we must have $ 2\alpha_t \beta_t L^2 - \beta_t^2 L^2 (1+\alpha_t^2 L^2) > 0$ and
    \begin{align*}
        \norm{\vz_t}^2 &\geq \frac{\beta_t^2 \sigma^2 (1+\alpha_t^2 L^2)}{2\alpha_t \beta_t L^2 - \beta_t^2 L^2 (1+\alpha_t^2 L^2)}. %
    \end{align*}
    Also, notice that
    \begin{align*}
        2\alpha_t \beta_t L^2 - \beta_t^2 L^2 (1+\alpha_t^2 L^2) = 1-\left ( (1-\alpha_t \beta_t L^2)^2 + \beta_t^2 L^2 \right ) < 1.
    \end{align*}
    Using $2\alpha_t \beta_t L^2 - \beta_t^2 L^2 (1+\alpha_t^2 L^2) \in (0,1)$ and substituting the lower bound on $\norm{\vz_t}^2$ into the update equation, we find that
    \begin{align*}
        \E [\norm{\vz_{t+1}}^2]
        &=
        \norm{\vz_{t}}^2 
        -
        \left ( 2\alpha_t \beta_t L^2 - \beta_t^2 L^2 (1+\alpha_t^2 L^2)\right ) \norm{\vz_{t}}^2 
        + \beta_t^2 \sigma^2 (1+\alpha_t^2 L^2)\\
        &\geq
        \frac{\beta_t^2 \sigma^2 (1+\alpha_t^2 L^2)}{2\alpha_t \beta_t L^2 - \beta_t^2 L^2 (1+\alpha_t^2 L^2)}
        =
        \frac{1}{L^2 \left ( \frac{2\alpha_t}{\beta_t \sigma^2 (1+\alpha_t^2 L^2)} - 1\right )}
    \end{align*}
    Lastly, substituting $\beta_t = \gamma \alpha_t$ into the RHS gives
    \begin{align*}
        \E [\norm{\vz_{t+1}}^2] 
        \geq 
        \frac{1}{L^2 \left ( \frac{2}{\gamma \sigma^2 (1+\alpha_t^2 L^2)} - 1\right )}
        \geq
        \frac{\gamma \sigma^2}{2 L^2}.
    \end{align*}
    This finishes the proof.
\end{proof}

\begin{rmk}
    We remark that, while \Cref{thm:seguscclb} successfully shows that \segus{} as studied in \cite{Diak21,Gorb22SEG} cannot converge to an optimal point unless the batch size is increased every iteration, it does not contradict the (almost sure) convergence result of independent-sample SEG by \citet{Hsie20}. Indeed, in \citep{Hsie20}, the stepsizes $\{\alpha_t\}_{t\geq 0}$ and $\{\beta_t\}_{t\geq 0}$ are chosen so that they decay to $0$ with a \emph{different} rate and hence the corresponding ratio $\gamma$ approaches $0$, while \Cref{thm:seguscclb} considers the case where $\alpha_t$ and $\beta_t$ differ by a constant factor $\gamma$. 
\end{rmk} 

\subsection{Proof of \sgdarr{} and \segrr{} Lower Bounds} \label{sec:appx:lower-bounds}
\begin{thm}
\label{thm:segrrscsclb}
    Suppose $n \geq 2$ and $L, \mu > 0$ satisfies $L/\mu \geq 2$. 
    There exists a $\mu$-strongly-convex-strongly-concave minimax problem $f(\vz) = \frac{1}{n} \sum_{i=1}^n f_i(\vz)$ consisting of $L$-smooth quadratic $f_i$'s satisfying Assumption~\ref{asmp:bounded-variance} with $(\rho, \sigma) = (0, \sigma)$ and initialization $\vz_0^0$ such that \segrr{} with any constant stepsize $\alpha_k = \alpha > 0$, $\beta_k = \beta > 0$ satisfies
    \begin{equation*}
        \E \left[\norm{\vz^K_0 - \vz^*}^2 \right] 
        = 
        \begin{cases}
        \Omega \left ( 
        \frac{\sigma^2}{L \mu n K}
        \right )    & \text{ if }K \leq L/\mu,\\
        \Omega \left ( 
        \frac{L\sigma^2}{\mu^3 n K^3}
        \right )    & \text{ if }K > L/\mu.\\
        \end{cases}
    \end{equation*}
    where $\vz^*$ is the unique equilibrium point of $f$. For a similar choice of problem $f$ (this time with $(\rho, \sigma) = (1,\sigma)$), \sgdarr{} with any constant stepsize $\alpha_k = \alpha > 0$ satisfies
    \begin{equation*}
        \E \left[\norm{\vz^K_0 - \vz^*}^2 \right] 
        = 
        \begin{cases}
        \Omega \left ( 
        \frac{\sigma^2}{L \mu n K}
        \right )    & \text{ if }K \leq L/\mu,\\
        \Omega \left ( 
        \frac{\sigma^2}{\mu^2n^2K^2}+\frac{L\sigma^2}{\mu^3nK^3}
        \right )    & \text{ if }K > L/\mu.\\
        \end{cases}
    \end{equation*}
\end{thm}

\begin{rmk}
In \Cref{thm:segrrscsclb}, we adopt techniques from the existing lower bounds for $\sgdrr$ to prove lower bounds for the minimax algorithms $\sgdarr$ and \segrr{}.
In the literature, there are two types of lower bounds for $\sgdrr$ when $K \gtrsim L/\mu$: namely, $\Omega (\frac{1}{n^2K^2}+\frac{1}{nK^3})$ bounds for strongly convex \emph{quadratic} functions \citep{Safr20,Safr21} and $\Omega (\frac{1}{nK^2})$ bounds for strongly convex \emph{non-quadratic} functions \citep{Rajp20,Yun22,Cha23}. 
Upper bounds that match the lower bounds in $n$ and $K$ are also known, indicating that $\sgdrr$ is one of the rare examples of minimization algorithms whose tight convergence rates for quadratic vs.\@ non-quadratic functions differ, within the narrow scope of strongly convex and smooth functions. 
While it is tempting to aim for a tighter $\Omega (\frac{1}{nK^2})$ lower bound for our algorithms of interest, we note that the existing $\Omega (\frac{1}{nK^2})$ bounds for $\sgdrr$ are proven for piecewise-quadratic functions whose Hessian is \emph{discontinuous}. 
Since the discontinuous Hessian violates our Assumption~\ref{asmp:smoothness}, we instead adhere to the quadratic case to prove lower bounds $\Omega (\frac{1}{nK^3})$ for both $\sgdarr$ and \segrr{} (when $K\geq L/\mu$).
These bounds may not be the tightest possible (since they are restricted to quadratics), but they still suffice to demonstrate that $\segffa$ is provably superior to both $\sgdarr$ and \segrr{}.
\end{rmk}

\subsubsection{Existing Lower Bound for \sgdrr{}}
For the proof of lower bounds for $\sgdarr$ and \segrr{}, we utilize the results and techniques from the lower bounds proven for $\sgdrr$; thus, it would be profitable to summarize the existing result.

In case of $\sgdrr$, it is known from Theorem~2 of \citet{Safr21} that there exists a minimization problem $g(\vx)$ such that $\sgdrr$ satisfies a lower bound of $\Omega( \frac{1}{n^2K^2} + \frac{1}{nK^3})$ for large enough values of $K$. We rewrite the theorem in a version in accordance with our notation and assumptions:
\begin{thm}[Theorem~2 of \citet{Safr21}]
\label{thm:sgdrrlb}
For any $n \geq 2$ and $L, \mu > 0$ satisfying $L/\mu \geq 2$, there exists a $\mu$-strongly convex minimization problem $g(\vx) = \frac{1}{n} \sum_{i=1}^n g_i(\vx)$ consisting of $L$-smooth quadratic $g_i$'s satisfying Assumption~\ref{asmp:bounded-variance} with $(\rho, \sigma) = (1,\sigma)$ such that $\sgdrr$ using any constant stepsize $\alpha_k = \alpha > 0$ satisfies
\begin{equation*}
    \E \left[\norm{\vx^K_0 - \vx^*}^2 \right]
    = \Omega \left ( 
    \frac{\sigma^2}{L \mu n K} \cdot \min \left \{ 1, \frac{L}{\mu n K} + \frac{L^2}{\mu^2 K^2} \right \}
    \right ).
\end{equation*}
\end{thm}
The statement is equivalent to saying that for $\sgdrr$ with constant stepsize $\alpha > 0$, the bound $\Omega ( \frac{\sigma^2}{L \mu n K})$ holds for $K \lesssim L/\mu$ and $\Omega (\frac{\sigma^2}{\mu^2n^2K^2}+\frac{L\sigma^2}{\mu^3nK^3})$ for $K \gtrsim L/\mu$.

The function $g = \frac{1}{n}\sum_{i=1}^n g_i$ used in the theorem is defined by the following component functions:
\begin{equation}
\label{eqn:sgdrrlb_gi}
    g_i(\vx) = g_i(x_1, x_2, x_3)
    \coloneqq
    \frac{\mu}{2} x_1^2 + \frac{L}{2} x_2^2 + 
    \begin{cases}
        \frac{\sigma}{2} x_2 + \frac{L}{2} x_3^2 + \frac{\sigma}{2} x_3 & i \leq \frac{n}{2},\\
        - \frac{\sigma}{2} x_2 - \frac{\sigma}{2} x_3 & i > \frac{n}{2},
    \end{cases}
\end{equation}
thus making the objective function 
\begin{equation*}
\label{eqn:sgdrrlb_g}
    g(x_1, x_2, x_3) 
    \coloneqq
    \frac{\mu}{2} x_1^2 + \frac{L}{2} x_2^2 + 
    \frac{L}{4} x_3^2.
\end{equation*}
One can notice that the linear terms in $g_i$~\eqref{eqn:sgdrrlb_gi} change signs depending on $i \leq \frac{n}{2}$ or not, and handling these sign flips is the key to the proof of the lower bound.

\subsubsection{Proof of Lower Bound for \sgdarr{}}
For the $\sgdarr$ lower bound, we consider the following minimax optimization problem:
\begin{equation}
\begin{aligned}
    f(\vx,y) &= \frac{1}{n}\sum_{i=1}^n f_i(\vx,y),\text{ where }\vx \in \R^3,~y\in\R,\\
    f_i(\vx,y) &= g_i(\vx) - \frac{\mu}{2}y^2,
\end{aligned} 
\label{eqn:sgdarrlb}
\end{equation}
where $g_i$'s are from \eqref{eqn:sgdrrlb_gi}.
We need to first check if the problem instance satisfies the assumptions listed in the theorem statement. Since $f(\vx,y) = g(\vx) - \frac{\mu}{2}y^2$ and $g$ is a $\mu$-strongly convex function, $f$ is $\mu$-strongly-convex-strongly-concave as claimed. Also, it is easy to check from the definition of $g_i$ that each component $f_i(\vx,y)$ is $L$-smooth quadratic. 

Lastly, to verify Assumption~\ref{asmp:bounded-variance}, we first define $s_1, \dots, s_n$ as $s_i = 1$ for $i \leq \frac{n}{2}$ and $s_i = 0$ for $i > \frac{n}{2}$. Using this notation, The function $g_i$ can be compactly written as the following:
\begin{equation*}
    g_i(x_1, x_2, x_3) = \frac{\mu}{2} x_1^2 + \frac{L}{2} x_2^2 + \frac{\sigma}{2} (2s_i-1) x_2 + \frac{L}{2} s_i x_3^2 +\frac{\sigma}{2} (2s_i-1) x_3.
\end{equation*}
Therefore, the saddle gradient operators $\mF_i$ of $f_i$ and $\mF$ of $f$ evaluate to
\begin{equation*}
\mF_i \vz \coloneqq
\begin{bmatrix}
    \nabla g_i(\vx)\\
    \mu y
\end{bmatrix}
=
\begin{bmatrix}
    \mu x_1\\
    L x_2 + \frac{\sigma}{2} (2s_i - 1)\\
    L s_i x_3 + \frac{\sigma}{2} (2s_i - 1)\\
    \mu y
\end{bmatrix},
\quad
\mF \vz =
\begin{bmatrix}
    \mu x_1\\
    L x_2\\
    \frac{L}{2} x_3\\
    \mu y
\end{bmatrix},
\end{equation*}
which in turn yields
\begin{equation*}
    \norm{\mF_i \vz - \mF \vz}^2 = \frac{\sigma^2}{4} + \left ( \frac{L}{2} x_3 + \frac{\sigma}{2} \right )^2 
    \leq 
    \left ( \frac{L}{2} |x_3| + \sigma \right )^2
    \leq
    \left ( \norm{\mF \vz} + \sigma \right )^2
\end{equation*}
for all $i = 1, \dots, n$. This confirms that the function $f = \frac{1}{n}\sum_{i} f_i$ satisfies Assumption~\ref{asmp:bounded-variance} with $(\rho,\sigma) = (1,\sigma)$.

If we run $\sgdarr$ on this problem, the updates on $\vx$ done by $\sgdarr$ is exactly identical to what $\sgdrr$ would perform for the minimization problem $g(\vx) = \frac{1}{n}\sum_i g_i(\vx)$ with the same choices of random permutations. Therefore, after $K$ epochs of $\sgdarr$, it follows from \Cref{thm:sgdrrlb} that
\begin{equation*}
    \E \left[ \norm{\vz_0^K - \vz^*}^2 \right] \geq 
    \E \left[ \norm{\vx_0^K - \vx^*}^2 \right] 
    =\Omega \left ( 
    \frac{\sigma^2}{L \mu n K} \cdot \min \left \{ 1, \frac{L}{\mu n K} + \frac{L^2}{\mu^2 K^2} \right \}
    \right ),
\end{equation*}
which is in fact a tighter lower bound for $\sgdarr$ than what is stated in \Cref{thm:segrrscsclb}. This finishes the proof.

\subsubsection{Proof of Lower Bound for \segrr{}}
In this subsection, we prove the lower bound for \segrr{}. We will first define a new problem instance $f$ to be used here, and verify that the assumptions in the theorem statement are indeed satisfied by this new $f$. We will then spell out the update equation of \segrr{} for this example, which will serve as a basis for the case analysis that follows: we will divide the choices of stepsizes $\alpha, \beta > 0$ to four regimes and prove a lower bound for each of them. Combining the regimes will result in the desired lower bound.

For \segrr{}, we use a slightly different problem from \eqref{eqn:sgdarrlb}. This time, we consider
\begin{equation}
\begin{aligned}
    f(\vx,y) &= \frac{1}{n}\sum_{i=1}^n f_i(\vx,y),\text{ where }\vx \in \R^2,~y\in\R,\\
    f_i(\vx,y) &= \frac{L}{2} x_1^2 + \frac{L}{4} x_2^2 + \sigma (2s_i-1) x_2 - \frac{\mu}{2} y^2,
\end{aligned} 
\label{eqn:segrrlb}
\end{equation}
where $s_i = 1$ for $i \leq \frac{n}{2}$ and $s_i = 0$ for $i > \frac{n}{2}$, as defined above.

We first check if the problem~\eqref{eqn:segrrlb} satisfies the assumptions in the theorem statement. Since 
\begin{equation*}
    f(\vx,y) = \frac{L}{2} x_1^2 + \frac{L}{4} x_2^2 - \frac{\mu}{2}y^2
\end{equation*}
and $L/2 \geq \mu$ by assumption, $f$ is $\mu$-strongly-convex-strongly-concave. Also, it is straightforward to see that each $f_i$ is an $L$-smooth quadratic function.
It is left to check Assumption~\ref{asmp:bounded-variance}. 
The saddle gradient operators $\mF_i$ of $f_i$ and $\mF$ of $f$ evaluate to
\begin{equation*}
\mF_i \vz =
\begin{bmatrix}
    L x_1\\
    \frac{L}{2} x_2 + \sigma (2s_i - 1)\\
    \mu y
\end{bmatrix},
\quad
\mF \vz =
\begin{bmatrix}
    L x_1\\
    \frac{L}{2} x_2\\
    \mu y
\end{bmatrix},
\end{equation*}
which in turn yields
\begin{equation*}
    \norm{\mF_i \vz - \mF \vz}^2 
    = \sigma^2,
\end{equation*}
for all $i = 1, \dots, n$. This confirms that the function $f = \frac{1}{n}\sum_{i} f_i$ satisfies Assumption~\ref{asmp:bounded-variance} with $(\rho,\sigma) = (0,\sigma)$, as required by the theorem.%

For $k \geq 0$, the $(k+1)$-th epoch of \segrr{} starts at $\vz_0^k = (\vx_0^k, y_0^k)$ and the algorithm chooses a random permutation $\tau_k$. The algorithm then goes through a series of updates
\begin{align*}
\vw_{i}^k &= \vz_{i}^k - \alpha \mF_{\tau_k(i+1)} \vz_{i}^k,  \\
\vz_{i+1}^{k} &= \vz_{i}^k - \beta \mF_{\tau_k(i+1)} \vw_{i}^k,
\end{align*}
for $i = 0, \dots, n-1$.
For our example $f$~\eqref{eqn:segrrlb}, it can be checked that a single iteration by \segrr{} reads
\begin{align*}
\vz_{i+1}^k = 
\begin{bmatrix}
    x^k_{i+1,1}\\
    x^k_{i+1,2}\\
    y^k_{i+1}
\end{bmatrix}
=
\begin{bmatrix}
    (1-\beta L + \alpha \beta L^2) x^k_{i,1}\\
    (1-\frac{\beta L}{2} + \frac{\alpha \beta L^2}{4}) x^k_{i,2} - \beta \sigma (1-\frac{\alpha L}{2}) (2 s_{\tau_k(i+1)}-1)\\
    (1-\beta \mu + \alpha \beta \mu^2) y^k_{i}
\end{bmatrix}.
\end{align*}
Aggregating the \segrr{} updates over an entire epoch ($i = 0, \dots, n-1$) results in
\begin{align*}
    x^{k+1}_{0,1} &= (1 - \beta L + \alpha \beta L^2)^n x^k_{0,1},\\
    x^{k+1}_{0,2} &= \left (1 - \frac{\beta L}{2} + \frac{\alpha \beta L^2}{4} \right )^n x^k_{0,2}
    - \beta \sigma \left ( 1-\frac{\alpha L}{2} \right ) \underbrace{\sum_{i=1}^n (2s_{\tau_k(i)}-1) \left (1 - \frac{\beta L}{2} + \frac{\alpha \beta L^2}{4} \right )^{n-i}}_{=:\,\Phi},\\
    y^{k+1}_0 &= (1 - \beta \mu + \alpha \beta \mu^2)^n y^k_{0}.
\end{align*}
We will now square both sides of these equations above and take expectations over $\tau_k$. In doing so, there is a useful identity:
\begin{align*}
    \E[\Phi] &= \sum_{i=1}^n \E[ 2 s_{\tau_k(i)} - 1 ] \left (1 - \frac{\beta L}{2} + \frac{\alpha \beta L^2}{4} \right )^{n-i} = 0.
\end{align*}
Also, it is worth mentioning that $\tau_k$ is independent of $\vz^k_0 = (x^k_{0,1}, x^k_{0,2}, y^k_0)$. Using these facts, we can arrange the terms to obtain
\begin{align}
    \label{eqn:segrrlb_epoch1}
    (x^{k+1}_{0,1})^2 &= (1 - \beta L + \alpha \beta L^2)^{2n} (x^k_{0,1})^2,\\
    \label{eqn:segrrlb_epoch2}
    \E[(x^{k+1}_{0,2})^2] &= \left (1 - \frac{\beta L}{2} + \frac{\alpha \beta L^2}{4} \right )^{2n} \E[(x^k_{0,2})^2] + \beta^2\sigma^2 \left ( 1-\frac{\alpha L}{2} \right )^2 \E[\Phi^2],\\
    \label{eqn:segrrlb_epoch4}
    (y^{k+1}_0)^2 &= (1 - \beta \mu + \alpha \beta \mu^2)^{2n} (y^k_0)^2.
\end{align}
Based on these three per-epoch update equations above, we now divide the choices of \segrr{} stepsizes $\alpha, \beta > 0$ into the following four cases and handle them separately:
\begin{enumerate}
    \item $\alpha > \frac{1}{L}$, in which case we show that \segrr{} makes $(x^{k+1}_{0,1})^2 > (x^{k}_{0,1})^2$ hold deterministically, so that if we initialize at $x^0_{0,1} = \frac{\sigma}{\sqrt{L\mu}}$ then we have
    \begin{equation*}
        \E \left[\norm{\vz^{K}_0}^2 \right] \geq (x^{K}_{0,1})^2 > (x^0_{0,1})^2 = \frac{\sigma^2}{L\mu}.
    \end{equation*}
    \item $\alpha \leq \frac{1}{L}$ and $\beta \leq \frac{1}{\mu n K}$, in which case we show that \segrr{} initialized at $y^0_0 = \frac{\sigma}{\sqrt{L\mu}}$ suffers 
    \begin{equation*}
        \E \left[\norm{\vz^{K}_0}^2 \right] = \Omega \left ( \frac{\sigma^2}{L\mu} \right ),
    \end{equation*}
    \item $\alpha \leq \frac{1}{L}$ and $\frac{1}{\mu n K} < \beta < \frac{1}{nL}$, in which case we show that \segrr{} initialized at $x^0_{0,2}=0$ suffers
    \begin{equation*}
        \E \left[\norm{\vz^{K}_0}^2 \right] = \Omega \left ( \frac{L \sigma^2}{\mu^3 n K^3} \right ),
    \end{equation*}
    \item $\alpha \leq \frac{1}{L}$, $\beta > \frac{1}{\mu n K}$, and $\beta \geq \frac{1}{nL}$ in which case we show that \segrr{} initialized at $x^0_{0,2} = 0$ suffers
    \begin{equation*}
        \E \left[\norm{\vz^{K}_0}^2 \right] = \Omega \left ( \frac{\sigma^2}{L \mu n K} \right ).
    \end{equation*}
\end{enumerate}
Notice that the third case $\frac{1}{\mu n K} < \beta < \frac{1}{nL}$ only makes sense when $K > L/\mu$; otherwise, the third case just disappears. Hence, for the ``large epoch'' regime where $K > L/\mu$, the third case achieves the minimum error possible, so it holds that
\begin{equation*}
    \E \left[\norm{\vz^{K}_0}^2 \right] = \Omega \left ( \frac{L \sigma^2}{\mu^3 n K^3} \right ).
\end{equation*}
For the ``small epoch'' regime ($K \leq L/\mu$), the third case does not exist and the fourth case achieves the minimum, so
\begin{equation*}
    \E \left[\norm{\vz^{K}_0}^2 \right] = \Omega \left ( \frac{\sigma^2}{L \mu n K} \right ).
\end{equation*}
Combining the two cases yields the desired lower bound in the theorem statement. 
It now remains to carry out the case analysis. 

\paragraph{Case 1: $\alpha > \frac{1}{L}$.}
For this case, we use \eqref{eqn:segrrlb_epoch1} to prove divergence. Notice from $\alpha > \frac{1}{L}$ that
\begin{equation*}
    1-\beta L + \alpha \beta L^2
    = 1 + \beta L (\alpha L - 1) > 1,
\end{equation*}
regardless of $\beta > 0$. Hence, from \eqref{eqn:segrrlb_epoch1}, we get
\begin{equation*}
    \E \left[\norm{\vz^{K}_0}^2 \right] \geq (x^{K}_{0,1})^2 > (x^0_{0,1})^2.
\end{equation*}
If we initialize at $x^0_{0,1} = \frac{\sigma}{\sqrt{L\mu}}$, then this proves
\begin{equation*}
    \E \left[\norm{\vz^{K}_0}^2 \right] \geq \frac{\sigma^2}{L\mu}.
\end{equation*}

\paragraph{Case 2: $\alpha \leq \frac{1}{L}$ and $\beta \leq \frac{1}{\mu n K}$.}
For this case, we employ \eqref{eqn:segrrlb_epoch4} to show that the ``contraction rate'' is too small to make enough ``progress.''
Notice from our stepsizes that
\begin{equation*}
    1-\beta\mu + \alpha \beta \mu^2
    \geq
    1-\beta\mu
    \geq
    1-\frac{1}{nK} \geq 0.
\end{equation*}
Applying this inequality to \eqref{eqn:segrrlb_epoch4}, we have
\begin{align*}
    (y^{k+1}_0)^2 
    \geq 
    \left (1 - \frac{1}{nK} \right )^{2n} (y^k_0)^2,
\end{align*}
which in turn means that the progress over $K$ epoch is bounded from below by
\begin{equation*}
    (y^{K}_0)^2 
    \geq 
    \left (1 - \frac{1}{nK} \right)^{2nK} (y^0_0)^2
    \geq
    \frac{(y^0_0)^2}{16},
\end{equation*}
where we used our assumption that $n \geq 2$ and $K \geq 1$. Hence, if our initialization was given as $y^0_0 = \frac{\sigma}{\sqrt{L\mu}}$, then this proves
\begin{equation*}
    \E \left[\norm{\vz^{K}_0}^2 \right] \geq (y^{K}_0)^2 \geq \frac{(y^0_0)^2}{16} = \Omega \left ( \frac{\sigma^2}{L\mu} \right ).
\end{equation*}

\paragraph{Case 3: $\alpha \leq \frac{1}{L}$ and $\frac{1}{\mu n K} < \beta < \frac{1}{nL}$.}
For stepsizes in this interval, we use \eqref{eqn:segrrlb_epoch2} to derive the desired bound.
Here, it is important to characterize a lower bound on the quantity
\begin{equation*}
    \E[\Phi^2] \coloneqq
    \E \left[ \left ( \sum_{i=1}^n (2s_{\tau_k(i)}-1) \left (1 - \frac{\beta L}{2} + \frac{\alpha \beta L^2}{4} \right )^{n-i} \right )^2\right].
\end{equation*}
To this end, we can use a lemma from \citet{Safr20}, stated below:
\begin{lem}[Lemma 1 of \citet{Safr20}]
\label{lem:safranlemma1}
Let $\pi_1, \dots, \pi_n$ (for even $n$) be a random permutation of $(1, 1, \dots, 1, -1, -1, \dots, -1)$ where both $1$ and $-1$ appear exactly $n/2$ times. Then there is a numerical constant $c > 0$ such that for any $\nu > 0$,
\begin{equation*}
    \E \left[ \left ( \sum_{i=1}^n \pi_i (1-\nu)^{n-i}\right )^2 \right]
    \geq
    c \cdot \min \left \{ 1 + \frac{1}{\nu}, n^3 \nu^2 \right \}.
\end{equation*}
\end{lem}
One can notice that Lemma~\ref{lem:safranlemma1} is directly applicable to $\E[\Phi^2]$, with $\nu \leftarrow \frac{\beta L}{2} - \frac{\alpha\beta L^2}{4}$. Since 
\begin{equation*}
    \nu = \frac{\beta L}{2} - \frac{\alpha\beta L^2}{4} \leq \frac{\beta L}{2} \leq \frac{1}{2n},
\end{equation*}
we have $n^3 \nu^2 \leq \frac{1}{8\nu}$, thereby
\begin{align*}
    \min \left \{ 1 + \frac{1}{\nu}, n^3 \nu^2 \right \}
    \geq \min \left \{ \frac{1}{\nu}, n^3 \nu^2 \right \}
    = n^3 \nu^2.
\end{align*}
Therefore, Lemma~\ref{lem:safranlemma1} gives
\begin{equation}
\label{eqn:segrrlb_phi2lb}
    \E[\Phi^2] 
    \geq c n^3 \left (\frac{\beta L}{2} - \frac{\alpha\beta L^2}{4} \right )^2 
    = \frac{c \beta^2 n^3 L^2}{4} \left ( 1- \frac{\alpha L}{2} \right )^2 
    \geq \frac{c \beta^2 n^3 L^2}{16}, 
\end{equation}
where the last inequality used $\alpha \leq \frac{1}{L}$.
Applying \eqref{eqn:segrrlb_phi2lb} to \eqref{eqn:segrrlb_epoch2} and also using $( 1- \frac{\alpha L}{2} )^2 \geq \frac{1}{4}$,
\begin{align*}
    \E[(x^{k+1}_{0,2})^2] 
    &\geq \left (1 - \frac{\beta L}{2} + \frac{\alpha \beta L^2}{4} \right )^{2n} \E[(x^k_{0,2})^2] +
    \frac{c \beta^4 n^3 L^2 \sigma^2}{64}.
\end{align*}
Unrolling the inequality for $k = 0, \dots, K-1$ gives
\begin{align*}
    \E \left[ (x^{K}_{0,2})^2 \right]
    &\geq
    \left (1 - \frac{\beta L}{2} + \frac{\alpha \beta L^2}{4} \right )^{2nK} (x^0_{0,2})^2
    + 
    \frac{c \beta^4 n^3 L^2 \sigma^2}{64}
    \sum_{j=0}^{K-1}
    \left (1 - \frac{\beta L}{2} + \frac{\alpha \beta L^2}{4} \right )^{2nj} \\
    &=
    \left (1 - \frac{\beta L}{2} + \frac{\alpha \beta L^2}{4} \right )^{2nK} (x^0_{0,2})^2
    + 
    \frac{c \beta^4 n^3 L^2 \sigma^2}{64} \cdot
    \frac{ 1 - \left (1 - \frac{\beta L}{2} + \frac{\alpha \beta L^2}{4} \right )^{2nK} }
    { 1- \left (1 - \frac{\beta L}{2} + \frac{\alpha \beta L^2}{4} \right )^{2n} }.
\end{align*}
Now note that our initialization $x^0_{0,2}$ can be set to zero, which eliminates the need to think about the first term in the RHS. It is now left to bound the second term. First, by the stepsize range $\alpha \leq \frac{1}{L}$, $\beta > \frac{1}{\mu n K}$ and our assumption $L/\mu \geq 2$, we have
\begin{align*}
    \left (1 - \frac{\beta L}{2} + \frac{\alpha \beta L^2}{4} \right )^{2nK} 
    \leq 
    \left (1 - \frac{\beta L}{4}\right )^{2nK} 
    \leq 
    \left (1 - \frac{L}{4\mu n K}\right )^{2nK} 
    \leq
    e^{- \frac{L}{2\mu}}
    \leq
    e^{-1}.
\end{align*}
Next, by Bernoulli's inequality
\begin{align*}
    \left (1 - \frac{\beta L}{2} + \frac{\alpha \beta L^2}{4} \right )^{2n} 
    \geq
    \left (1 - \frac{\beta L}{2} \right )^{2n} 
    \geq
    1 - \beta n L 
    > 0.
\end{align*}
Plugging in the two inequalities to above, we obtain
\begin{align*}
    \E \left[ (x^{K}_{0,2})^2 \right]
    &\geq
    \frac{c \beta^4 n^3 L^2 \sigma^2}{64} \cdot
    \frac{ 1 - \left (1 - \frac{\beta L}{2} + \frac{\alpha \beta L^2}{4} \right )^{2nK} }
    { 1- \left (1 - \frac{\beta L}{2} + \frac{\alpha \beta L^2}{4} \right )^{2n} }\\
    &\geq
    \frac{c \beta^4 n^3 L^2 \sigma^2}{64} \cdot
    \frac{1-e^{-1}}{1 - (1-\beta n L)}
    = c' \beta^3 n^2 L \sigma^2
\end{align*}
for a numerical constant $c' > 0$. Plugging in the lower bound $\beta > \frac{1}{\mu n K}$ yields
\begin{equation*}
    \E \left[\norm{\vz^{K}_0}^2 \right] 
    \geq 
    \E \left[ (x^{K}_{0,2})^2 \right]
    =
    \Omega \left (
    \frac{L \sigma^2}{\mu^3 n K^3}
    \right ).
\end{equation*}

\paragraph{Case 4: $\alpha \leq \frac{1}{L}$, $\beta > \frac{1}{\mu n K}$, and $\beta \geq \frac{1}{nL}$.}
We again use \eqref{eqn:segrrlb_epoch2}. By noticing that the initialization $x^0_{0,2} = 0$, we can unroll \eqref{eqn:segrrlb_epoch2} for $k = 0, \dots, K-1$ to get
\begin{align}
\label{eqn:segrrlb_phi2lb3}
    \E \left[ (x^{K}_{0,2})^2 \right]
    \geq
    \frac{\beta^2 \sigma^2}{4} \E[\Phi^2]
    \sum_{j=0}^{K-1}
    \left (1 - \frac{\beta L}{2} + \frac{\alpha \beta L^2}{4} \right )^{2nj}
    \geq
    \frac{\beta^2 \sigma^2}{4} \E[\Phi^2],
\end{align}
where the last inequality holds regardless of $\beta$ because each summand with $j \geq 1$ is nonnegative. We then invoke Lemma~\ref{lem:safranlemma1} to lower bound $\E[\Phi^2]$, again with $\nu \leftarrow \frac{\beta L}{2} - \frac{\alpha\beta L^2}{4}$. 
Since 
\begin{equation*}
    \nu = \frac{\beta L}{2} - \frac{\alpha\beta L^2}{4} \geq \frac{\beta L}{4} \geq \frac{1}{4n},
\end{equation*}
we have $n^3 \nu^2 \geq \frac{1}{64\nu}$, thereby
\begin{align*}
    \min \left \{ 1 + \frac{1}{\nu}, n^3 \nu^2 \right \}
    \geq \min \left \{ \frac{1}{\nu}, n^3 \nu^2 \right \}
    \geq \frac{1}{64\nu}.
\end{align*}
Therefore, Lemma~\ref{lem:safranlemma1} gives
\begin{equation}
\label{eqn:segrrlb_phi2lb2}
    \E[\Phi^2] 
    \geq \frac{c}{64 \nu} 
    = \frac{c}{32 \beta L}\cdot \frac{1}{1 - \frac{\alpha L}{2}}
    \geq \frac{c}{32 \beta L}.
\end{equation}
Combining \eqref{eqn:segrrlb_phi2lb2} with \eqref{eqn:segrrlb_phi2lb3} gives
\begin{equation*}
    \E \left[ (x^{K}_{0,2})^2 \right]
    \geq
    \frac{c \beta \sigma^2}{128 L}
    =
    \Omega \left ( \frac{\sigma^2}{L \mu n K} \right ),
\end{equation*}
where the last step used $\beta > \frac{1}{\mu n K}$. This finishes the case analysis, hence the proof of \Cref{thm:segrrscsclb}.

\section{Additional Experiments}
\allowdisplaybreaks[0]
\label{appx:experiments}
 
To evaluate our algorithm \segffa{} as well as other baseline algorithms, we conduct numerical experiments on monotone and strongly monotone problems. Specifically, as we have mentioned in \Cref{sec:experiments-shortlist}, we consider random quadratic problems of the form \[ %
\min_{\vx \in \sR^{d_x}} \max_{\vy \in \sR^{d_y}} \; 
\frac{1}{n} 
\sum_{i=1}^{n} \ 
\begin{bmatrix}
    \vx \\ \vy
\end{bmatrix}^\top\! 
\begin{bmatrix}
    \mA_i & \mB_i \\ \mB_i^\top & -\mC_i
\end{bmatrix}
\begin{bmatrix}
    \vx \\ \vy
\end{bmatrix} - \vt_i^\top 
\begin{bmatrix}
    \vx \\ \vy
\end{bmatrix}.
 \]  
We choose $d_x = d_y = 20$ and $n=40$ for all the experiments. %
Numerical computations are done using \texttt{NumPy} \citep{NumPy} and \texttt{SciPy} \citep{SciPy}, and the plots are drawn using \texttt{Matplotlib} \citep{Matplotlib}. 
  
\subsection{Problem Constructions for Experiments in \titleref{Section}{sec:experiments-shortlist}} \label{appx:problem-sampling}

For an experiment for the monotone case, the random components are sampled as follows. We choose $\mB_i$ so that each element is an i.i.d.\ sample from a uniform distribution over the interval $[0, 1]$, and $\vt_i$ so that each element is an i.i.d.\ sample from a standard normal distribution. We chose $\mA_i$ to be diagonal matrices in the following procedure: for each $j = 1,\dots, 20$ we randomly chose a subset $\mathcal{I}_j$ of $\frac{n}{2} = 20$ indices from $[n] = \{1, \dots, 40\}$, and set the $(j, j)$-entry of $\mA_i$ to be \[
(\mA_i)_{j, j} = \begin{cases}
    2 & \text{ if $i \in \mathcal{I}_j$} \\
    -2 & \text{ otherwise}
\end{cases}.
\] We repeat the exact same procedure for $\mC_i$ as well. Notice that $\sum_{i=1}^n \mA_i = \sum_{i=1}^n \mC_i = \zero$ by design. Hence, each of the component functions will be a nonconvex-nonconcave quadratic function in general, but the objective function itself becomes a convex-concave function. 

For the experiment in the strongly monotone case, we sample $\mB_i$ and $\vt_i$ in the same way as in the monotone case, but we use different choices of $\mA_i$ and $\mC_i$ to ensure the objective function to be strongly-convex-strongly-concave. 
In particular, for each $i = 1, \dots, n$, we sample $\mA_i$ by computing $\mA_i = \mQ_i \mD_i \mQ_i^\top$, where $\mD_i$ is a random diagonal matrix whose diagonal entries are i.i.d.\ samples from a uniform distribution over the interval $[\frac{1}{2}, 1]$, and $\mQ_i$ is a random orthogonal matrix obtained by computing a {\it QR} decomposition of a $20\times 20$ random matrix whose elements are i.i.d.\ samples from a standard normal distribution. We sample $\mC_i$ by the exact same method. 

\subsection{Monotone Case \& Ablation Study on the Anchoring Step} \label{appx:monotone-experiment}

In \Cref{sec:experiments-shortlist}, we compared the empirical performance of various SEGs, namely \segffa, \segff, \segrr, and \segus. Here, as an ablation study on the anchoring technique, we additionally compare \emph{SEG-RRA} and \emph{SEG-USA}, which are each \segrr{} and \segus{} with an additional anchoring step,
respectively. 
For these two methods, we take the anchoring step after every $n$ iterations. 
We ran those methods on the same $5$ random instances used in \Cref{sec:experiments-shortlist}. For both SEG-RRA and SEG-USA, we ran the method with two different stepsize choices, namely $\alpha_k = \beta_k = \eta_k$ (inspired by the stepsize used in deterministic EG) and $\alpha_k = \nicefrac{\beta_k}2 = \nicefrac{\eta_k}2$ (the stepsize used for \segffa{}) where we again set $\eta_k = \nicefrac{\eta_0}{(1+k/10)^{0.34}}$ with $\eta_0 = \min\{0.01, \frac{1}{L}\}$. %

The results are plotted in \Cref{fig:monotone-result}.
As SEG-RRA and SEG-USA are designed to take one pass per epoch, for those methods, we compute the ratio $\frac{\norm{\mF \vz_0^t}^2}{\norm{\mF \vz_0^0}^2}$ where $t$ denotes the number of passes, and plot the geometric mean over the $5$ runs. 

From the performance of SEG-RRA with $\alpha_k = \beta_k$ and the two variants of SEG-USA, it is possible to observe that adding the anchoring step does improve the performance of the method up to a certain level, but it alone does not fully resolve the nonconvergence issue. 
On the other hand, quite interestingly, SEG-RRA with $\alpha_k = \nicefrac{\beta_k}2$ shows a hint of convergence. While its performance is slightly worse compared to \segffa, it is nonetheless still notable as it is the only other method from \segffa{} that seems to be capable of converging to an optimum. 

We conjecture that this intriguing performance of SEG-RRA with $\alpha_k = \nicefrac{\beta_k}2$ is because it achieves an \emph{``expected''} second order matching to the (deterministic) EG. Indeed, following the notations of \Cref{prop:appx:unravelling-m}, one can deduce from \Cref{prop:appx:unravelling-m} that using SEG-RRA with $\alpha  = \nicefrac{\beta}{2}$ will result in an epoch-level update of \begin{equation} \label{eqn:seg-rra}
        \vz^\sharp =  \vz_0 - \frac{\beta}{2} \sum_{j=0}^{n-1} \mT_j \vz_0 + \frac{\beta^2}{4} \sum_{j=0}^{n-1} D\mT_j(\vz_0) \mT_j \vz_0 +  \frac{\beta^2}{2} \sum_{0\leq i < j \leq n-1}  D\mT_j(\vz_0) \mT_i \vz_0 + \frac{\veps_n}{2}  
    \end{equation}
    with $\veps_n = o\left(\beta^2\right)$. Here, notice that $(\mT_0, \mT_1, \dots, \mT_{n-1}) = (\mF_{\tau(1)}, \mF_{\tau(2)}, \dots, \mF_{\tau(n)})$ for some randomly chosen permutation $\tau \in \gS_n$. Now, observe that for any two distinct $i, j \in [n]$, there are exactly $\frac{n!}{2}$ permutations in $\gS_n$ such that $i$~comes before $j$ in the sequence $\tau(1), \tau(2), \dots, \tau(n)$, and also exactly $\frac{n!}{2}$ permutations such that $j$ comes before $i$. 
    Thus, in taking the expectation over the randomness of choosing the permutation $\tau$, we get \begin{align*}
       \expt_\tau \left[ \sum_{0\leq i < j \leq n-1}  D\mT_j(\vz_0) \mT_i \vz_0 \right] &= \expt_\tau \left[ \sum_{1\leq i < j \leq n }  D\mF_{\tau(j)}(\vz_0) \mF_{\tau(i)} \vz_0 \right] \\
       &=   \frac{1}{n!} \sum_{\tau \in \gS_n} \sum_{1\leq i < j \leq n}  D\mF_{\tau(j)}(\vz_0) \mF_{\tau(i)} \vz_0 \\
     &=   \frac{1}{2} \sum_{i \neq j} D\mF_{j}(\vz_0) \mF_{i} \vz_0,
    \end{align*}
      {where in getting the third line we have used the previously made observation that for any fixed $i$ and $j$ with $i \neq j$, the term $ D\mF_{j}(\vz_0) \mF_{i} \vz_0 $ appears exactly $\frac{n!}2$ times in the sum on the second line. }
      Hence, taking the expectation with respect to the random permutation on \eqref{eqn:seg-rra} we get \begin{align*}
          \expt_{\tau} \left[\vz^\sharp\right] &=  \vz_0 - \frac{n \beta}{2} \mF \vz_0 + \frac{\beta^2}{4} \sum_{j=1}^{n} D\mF_j(\vz_0) \mF_j \vz_0 +  \frac{\beta^2}{4} \sum_{i \neq j} D\mF_{j}(\vz_0) \mF_{i} \vz_0 + \frac{1}{2}  \expt_{\tau}\left[\veps_n\right] \\ 
          &=  \vz_0 - \frac{n \beta}{2} \mF \vz_0 + \frac{\beta^2}{4} \sum_{j=1}^{n} \sum_{i=1}^{n} D\mF_j(\vz_0) \mF_i \vz_0   + \frac{1}{2}  \expt_{\tau}\left[\veps_n\right] \\ 
          &=  \vz_0 - \frac{n \beta}{2} \mF \vz_0 + \frac{n^2 \beta^2}{4}  D\mF (\vz_0) \mF  \vz_0  + \frac{1}{2}  \expt_{\tau}\left[\veps_n\right].
      \end{align*}
      Comparing this to \eqref{eqn:egplus} when $\eta_1 = \eta_2 = \nicefrac{n\beta}{2}$, we indeed see that the update rule of SEG-RRA with $\alpha = \nicefrac{\beta}2$ achieves a second-order matching \emph{on expectation} to the (deterministic) EG update with stepsize $\nicefrac{n\beta}{2}$.  

      We also conjecture that the relatively worse performance of SEG-RRA with $\alpha = \nicefrac{\beta}2$ compared to \segffa{} is because the error over an epoch is $O(\eta^3)$ only on expectation, and thus the actual error occurring in each epoch can be larger than $O(\eta^3)$. 
       Unfortunately, our convergence analysis on \segffa{} relies on the error over an epoch being $O(\eta^3)$ \emph{deterministically}  (\textit{cf.}\@~\Cref{thm:ffa-cubic-error}), hence cannot be directly applied to SEG-RRA with $\alpha = \nicefrac{\beta}2$.
      We leave the search for a theoretical explanation on this alluring performance of SEG-RRA with $\alpha = \nicefrac{\beta}2$ as a stimulating direction for future work. 

\begin{figure}[tp]
    \centering
\includegraphics[height=6cm]{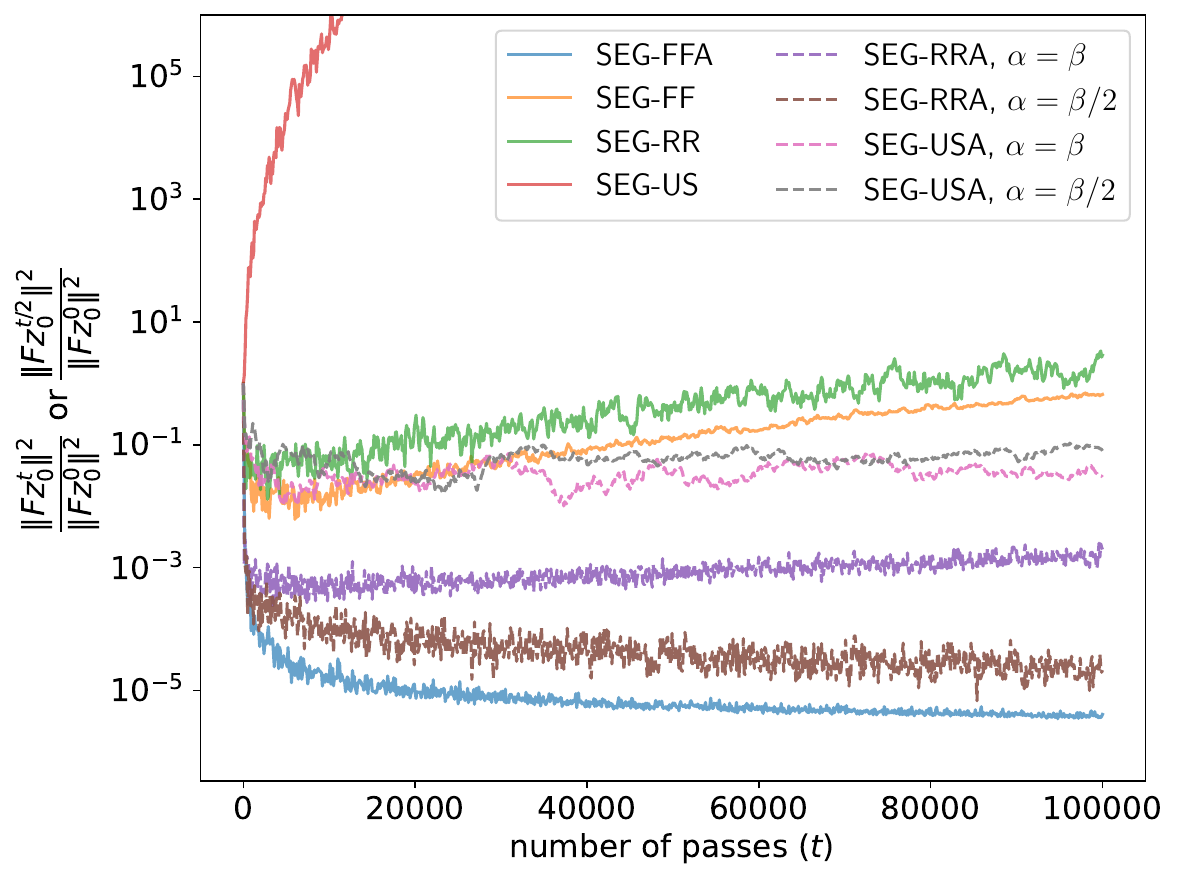}
    \caption{Experimental results in the monotone example, comparing the performance of SEG-RRA and SEG-USA with the results displayed in \Cref{fig:cc-shortlist}. Because \segffa{} and \segff{} use two passes per epoch, for those two methods, we plot $\nicefrac{\|\mF \vz_0^{t/2}\|^2}{\|\mF \vz_0^0\|^2}$.   }
    \label{fig:monotone-result}
\end{figure}

\begin{figure}[tp]
    \centering
\includegraphics[height=6cm]{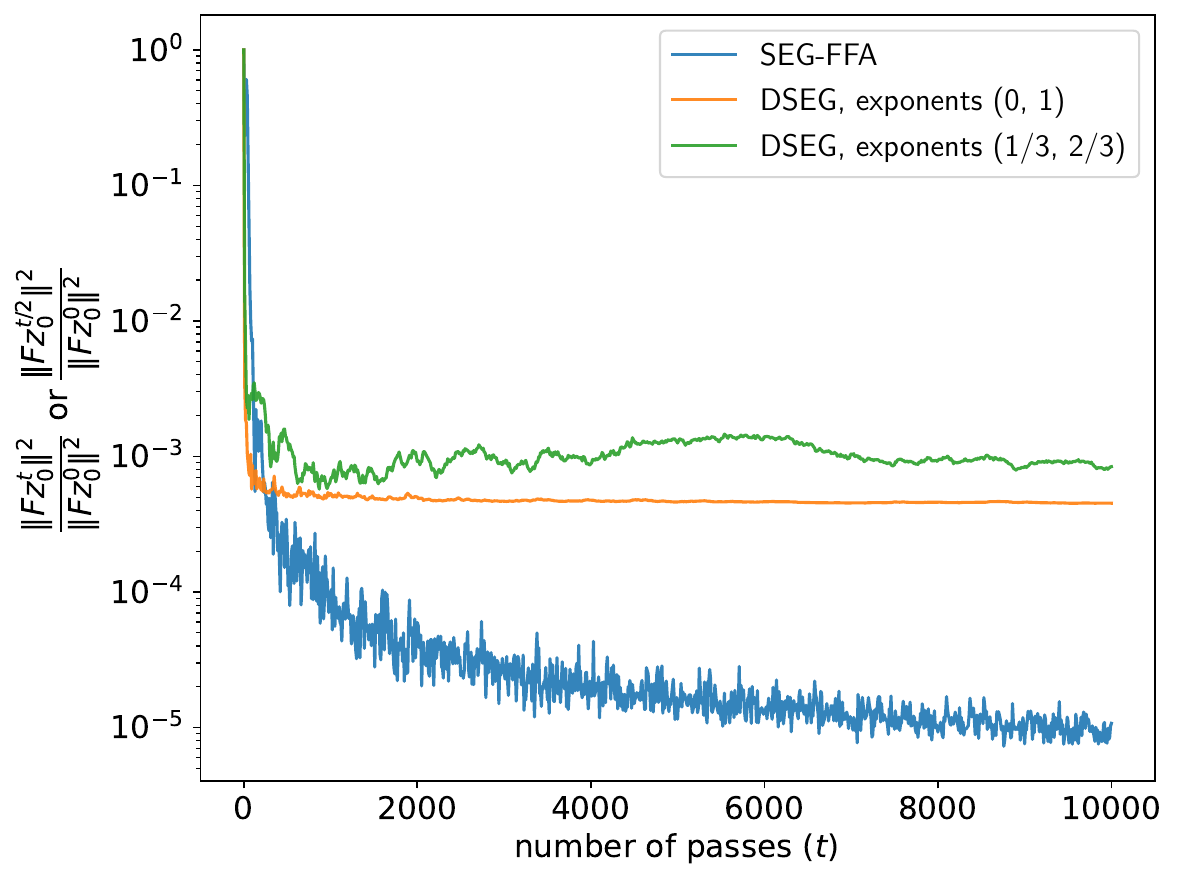}
    \caption{Experimental results in the monotone example, comparing \segffa{} and the methods proposed by \citet{Hsie20}. By the same reason as in \Cref{fig:monotone-result}, we plot $\nicefrac{\|\mF \vz_0^{t/2}\|^2}{\|\mF \vz_0^0\|^2}$ for \segffa{} only.  }
    \label{fig:hsieh-result}
\end{figure}

\subsection{Monotone Case: Comparison with \texorpdfstring{\citet{Hsie20}}{Hsieh et al.}} \label{appx:experiment-hsieh}

Let us also compare the performance of \segffa{} with the \emph{independent-sample} {double stepsize} SEG (\emph{DSEG}) by \citet{Hsie20}. Writing in terms of the finite-sum structure, the update rule of DSEG can be written as \begin{align*}
    \vw^k &\gets \vz^k - \eta_{1, k} \mF_{i(1, k)} \vz^k \\
    \vz^{k+1} &\gets \vz^k - \eta_{2, k} \mF_{i(2, k)} \vw^k 
\end{align*} where $i(1, k)$ and $i(2, k)$ are random indices that are independently drawn from $[n]$ for each $k$. The stepsizes are chosen in the form of $\eta_{1, k} = \Theta(\nicefrac{1}{k^{r_1}})$ and $\eta_{2, k} = \Theta(\nicefrac{1}{k^{r_2}})$, where setting $r_1 \leq r_2$ is the key point of DSEG. Two choices of the exponent pair $(r_1, r_2)$ proposed in \citep{Hsie20} are $(1/3, 2/3)$ for general monotone problems and $(0, 1)$ exclusively for the case when $\mF$ is affine. 

We again use the same component functions as in the previous experiment. The setup for running \segffa{} are kept the same. For DSEG, we use the default choices suggested by \citet{Hsie20}, namely $\eta_{1, k} = \nicefrac{\gamma_0}{(k+19)^{r_1}}$ and $\eta_{2, k} = \nicefrac{\eta_0}{(k+19)^{r_2}}$, where $(\gamma_0, \eta_0) = (1, 0.1)$ for the bilinear case with $(r_1, r_2) = (0, 1)$ and $(\gamma_0, \eta_0) = (0.1, 0.05)$ for the general case with $(r_1, r_2) = (1/3, 2/3)$.

The results are displayed in \Cref{fig:hsieh-result}, where the details on how the plots are drawn are the same as \Cref{fig:monotone-result}. Here we can clearly see that \segffa{} outperforms both versions of DSEG. 
 
\subsection{Strongly Monotone Case Again, with Various Stepsizes} \label{appx:ssec-scsc-longlist}
We also ran the experiment on strongly monotone problems described in \Cref{sec:experiments-shortlist}, but with changing the stepsizes. We tested six different values of $\eta_k$; we have tested with $\eta_k = a \times 10^{b}$ where $a \in \{1, 2, 5\}$ and $b \in \{-4, -3\}$. Notice that the case $\eta_k = 10^{-3}$ is exactly the experiment conducted in \Cref{sec:experiments-shortlist}. 

    \begin{figure}
    \centering 
    \begin{subfigure}[t]{0.49\textwidth}\centering 
    \includegraphics[width=\textwidth]{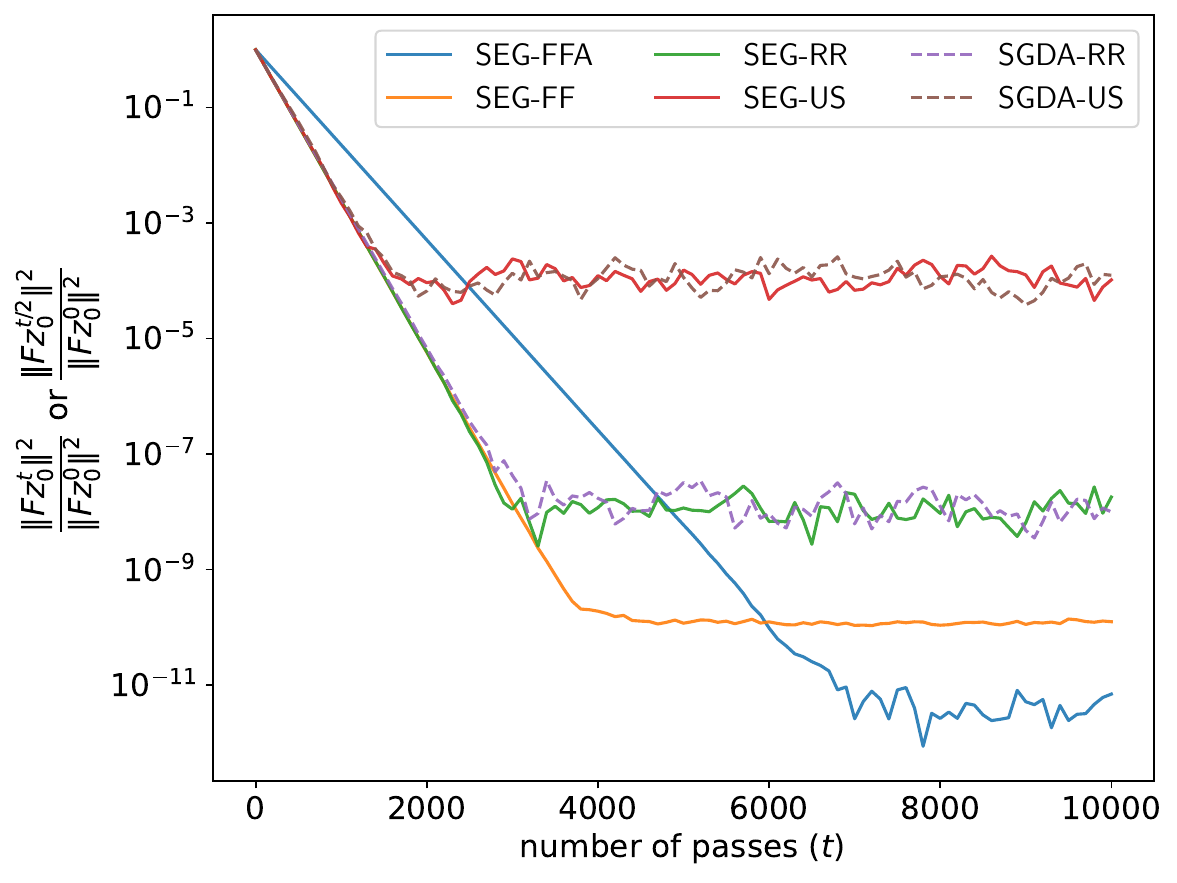}
    \caption{$\eta_k = 0.0001$}
    \label{subfig:1e-4}
    \end{subfigure} 
    \begin{subfigure}[t]{0.49\textwidth}\centering 
    \includegraphics[width=\textwidth]{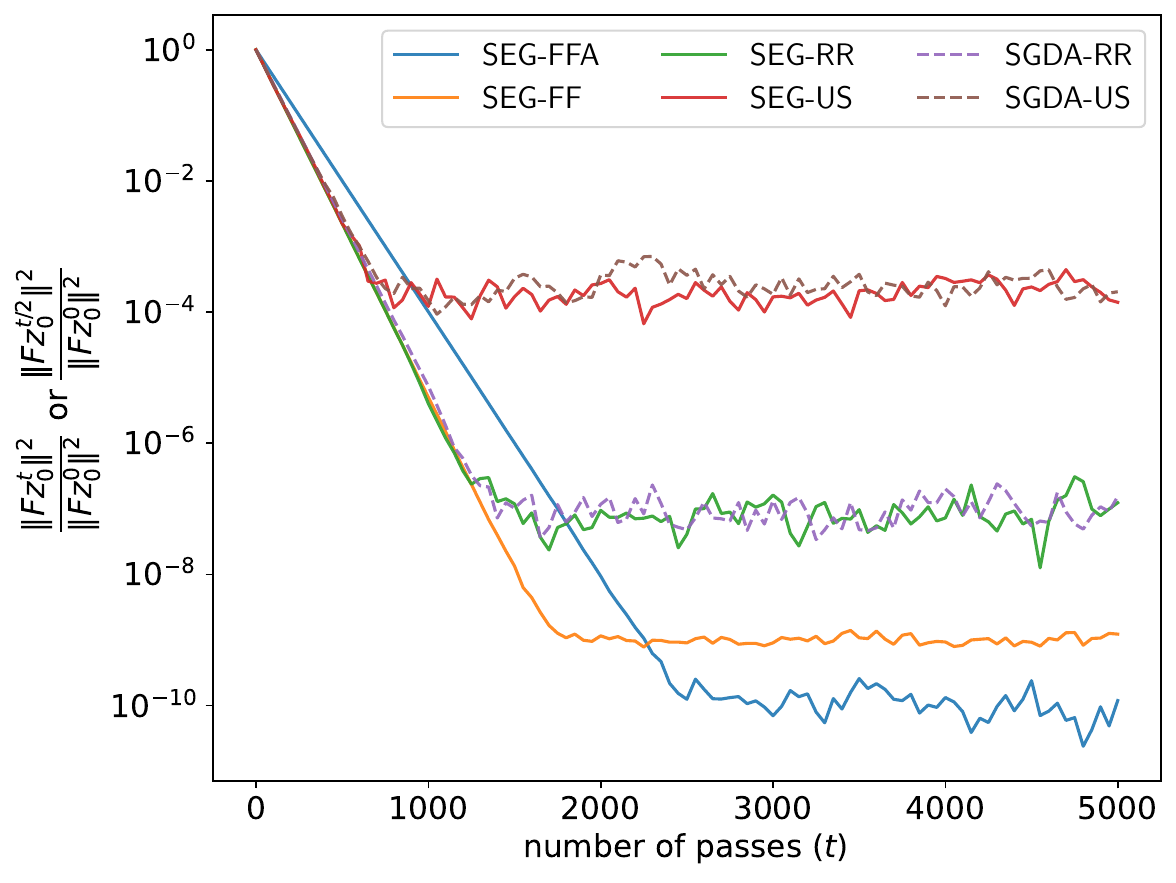}
    \caption{$\eta_k = 0.0002$}
    \label{subfig:2e-4}
    \end{subfigure} 
    \\[12.5pt]  
    \begin{subfigure}[t]{0.49\textwidth}\centering 
    \includegraphics[width=\textwidth]{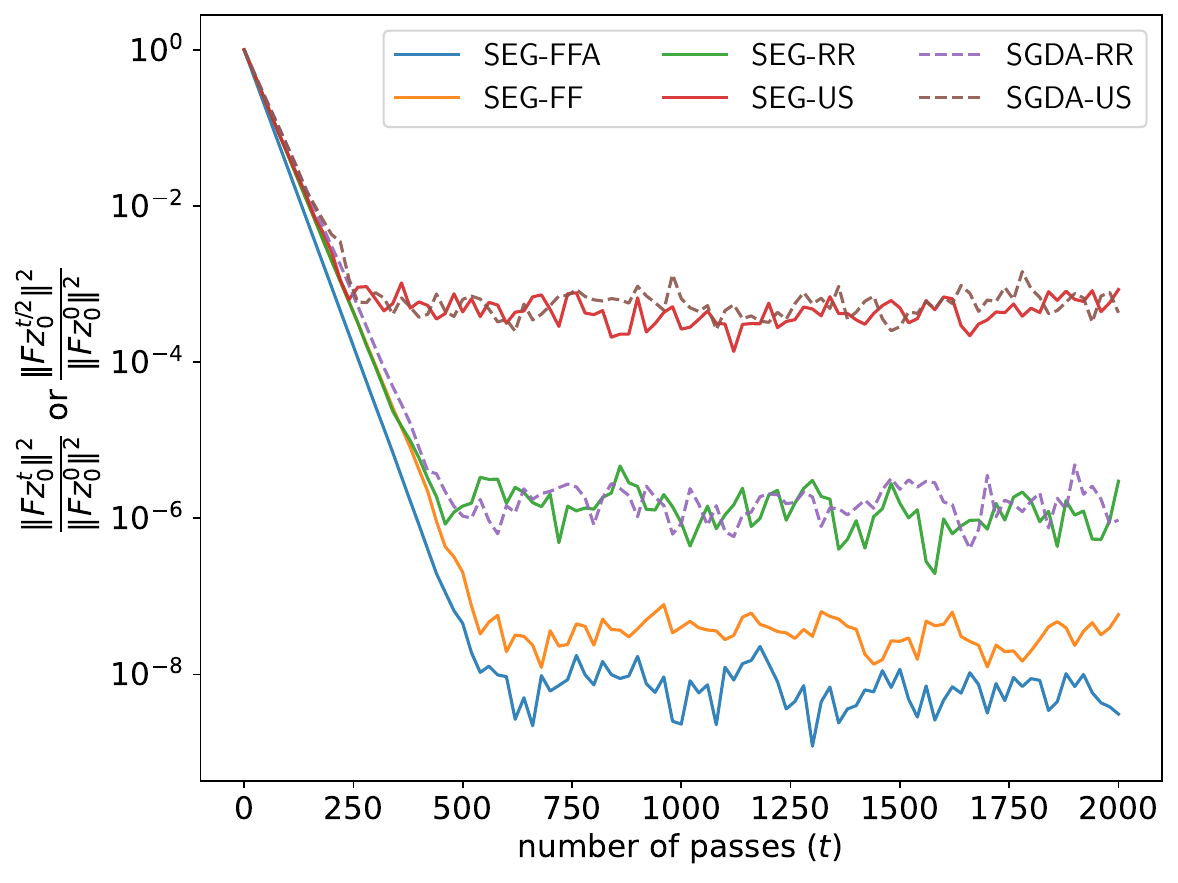}
    \caption{$\eta_k = 0.0005$}
    \label{subfig:5e-4}
    \end{subfigure} 
    \begin{subfigure}[t]{0.49\textwidth}\centering 
    \includegraphics[width=\textwidth]{figure/scsc/geom_mean_scsc1e-3}
    \caption{$\eta_k = 0.001$}
    \label{subfig:1e-3}
    \end{subfigure} 
    \\[12.5pt]
    \begin{subfigure}[t]{0.49\textwidth}\centering 
    \includegraphics[width=\textwidth]{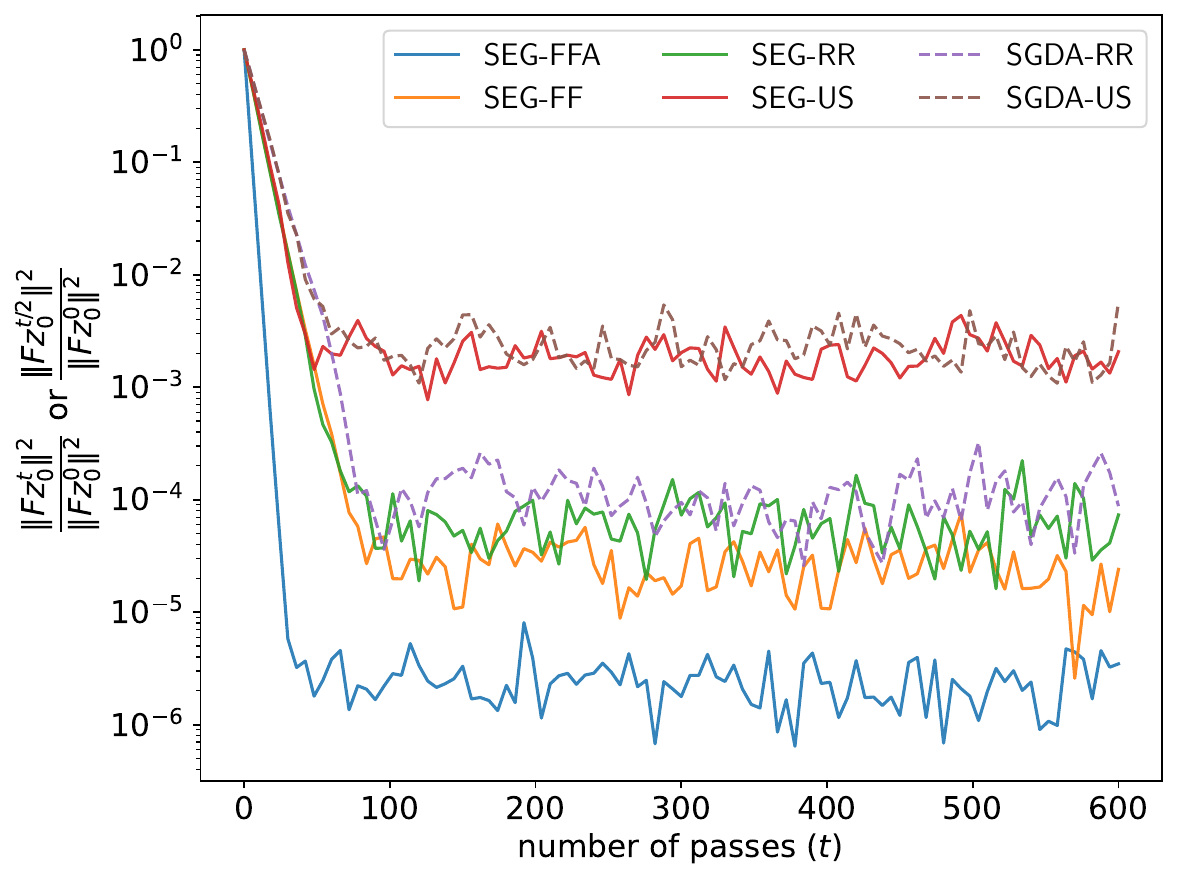}
    \caption{$\eta_k = 0.002$}
    \label{subfig:2e-3}
    \end{subfigure} 
    \begin{subfigure}[t]{0.49\textwidth}\centering 
    \includegraphics[width=\textwidth]{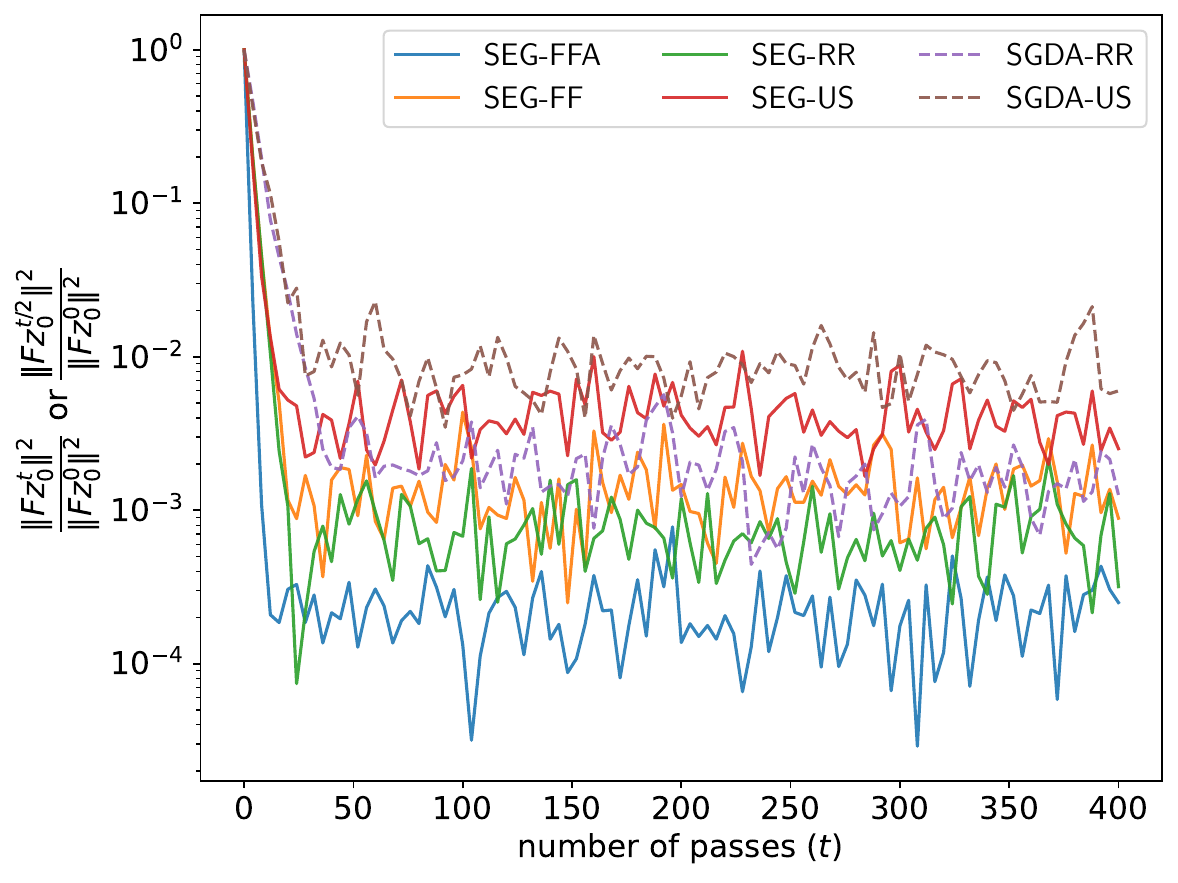}
    \caption{$\eta_k = 0.005$}
    \label{subfig:5e-3}
    \end{subfigure} 
    \caption{Experimental results on the strongly monotone problems with different stepsizes. Notice that \Cref{subfig:1e-3} is exactly the plot that is included in \Cref{sec:experiments-shortlist}. The only difference between the experiments conducted is the choice of the stepsize. }
    \label{fig:scsc-longlist}
    \end{figure}

The results are plotted in \Cref{fig:scsc-longlist}. The overall details are the same as described in \Cref{sec:experiments-shortlist}, as the only difference is the stepsize choice. We can observe that, while the initial speed of convergence may not be the fastest depending on the stepsize, \segffa{} is always the method that eventually finds the point with the smallest gradient.  
In other words, as predicted by our theoretical analyses, the supremacy of \segffa{} is in general not affected by the choice of the stepsize, as long as the chosen stepsize is reasonably small.

\newpage
\section*{NeurIPS Paper Checklist}

\begin{enumerate}

\item {\bf Claims}
    \item[] Question: Do the main claims made in the abstract and introduction accurately reflect the paper's contributions and scope?
    \item[] Answer: \answerYes{} %
    \item[] Justification: The abstract and the introduction well summarizes our theoretical results and the problem settings we are considering. %

\item {\bf Limitations}
    \item[] Question: Does the paper discuss the limitations of the work performed by the authors?
    \item[] Answer: \answerYes{} %
    \item[] Justification: While we do not have a separate "Limitations" sections, in \Cref{sec:problem-settings} we thoroughly discuss about the assumptions we have imposed. The paper is highly theoretical, hence the other factors listed in the guidelines below are either not applicable to this paper, or apparent from the statements of the theorems/lemmata/propositions and the discussions that follow. %

\item {\bf Theory Assumptions and Proofs}
    \item[] Question: For each theoretical result, does the paper provide the full set of assumptions and a complete (and correct) proof?
    \item[] Answer: \answerYes{} %
    \item[] Justification: \Cref{sec:problem-settings} is devoted for the discussions on the assumptions. Full proofs of the theorems/lemmata/propositions can be found in the appendices. %

    \item {\bf Experimental Result Reproducibility}
    \item[] Question: Does the paper fully disclose all the information needed to reproduce the main experimental results of the paper to the extent that it affects the main claims and/or conclusions of the paper (regardless of whether the code and data are provided or not)?
    \item[] Answer: \answerYes{} %
    \item[] Justification: In \Cref{appx:experiments}, we provide full explanations on how the experiments have been conducted. We have also submitted the exact code that we used for our experiments as a supplemental material. %

\item {\bf Open access to data and code}
    \item[] Question: Does the paper provide open access to the data and code, with sufficient instructions to faithfully reproduce the main experimental results, as described in supplemental material?
    \item[] Answer: \answerYes{} %
    \item[] Justification: We have submitted the exact code that we used for our experiments as a supplemental material, so that it becomes revealed to the public once our paper gets accepted. %

\item {\bf Experimental Setting/Details}
    \item[] Question: Does the paper specify all the training and test details (e.g., data splits, hyperparameters, how they were chosen, type of optimizer, etc.) necessary to understand the results?
    \item[] Answer: \answerYes{} %
    \item[] Justification: The overall settings are discussed in \Cref{appx:experiments}. The code we submit along with the paper is an exact copy of the one we used in the reported experiments, so the details not included in the paper shall be found in the code itself. %

\item {\bf Experiment Statistical Significance}
    \item[] Question: Does the paper report error bars suitably and correctly defined or other appropriate information about the statistical significance of the experiments?
    \item[] Answer:  \answerNo{} %
    \item[] Justification: Our paper is mainly theoretical, and the experiments are to demonstrate that our analyses are correct. Hence, we claim that error bars or information about the statistical significance are not necessary, and rather, the interpretations we made regarding our experiments in the relevant section(s) are enough. %

\item {\bf Experiments Compute Resources}
    \item[] Question: For each experiment, does the paper provide sufficient information on the computer resources (type of compute workers, memory, time of execution) needed to reproduce the experiments?
    \item[] Answer: \answerNo{} %
    \item[] Justification: The experiments are numerical validations of our theoretical analyses using simple quadratic functions, so they should be executable on any modern computer with a reasonable CPU. %
    
\item {\bf Code Of Ethics}
    \item[] Question: Does the research conducted in the paper conform, in every respect, with the NeurIPS Code of Ethics \url{https://neurips.cc/public/EthicsGuidelines}?
    \item[] Answer: \answerYes{} %
    \item[] Justification: We have read through the Code of Ethics, but due to the theoretical nature of the paper, there are no risks regarding ethical issues. %

\item {\bf Broader Impacts}
    \item[] Question: Does the paper discuss both potential positive societal impacts and negative societal impacts of the work performed?
    \item[] Answer: \answerNA{} %
    \item[] Justification: There are no societal impacts of this paper, as it is a theory paper. %
    
\item {\bf Safeguards}
    \item[] Question: Does the paper describe safeguards that have been put in place for responsible release of data or models that have a high risk for misuse (e.g., pretrained language models, image generators, or scraped datasets)?
    \item[] Answer: \answerNA{} %
    \item[] Justification: This paper is highly theoretical, hence poses no such risks. %

\item {\bf Licenses for existing assets}
    \item[] Question: Are the creators or original owners of assets (e.g., code, data, models), used in the paper, properly credited and are the license and terms of use explicitly mentioned and properly respected?
    \item[] Answer: \answerYes{} %
    \item[] Justification: Packages used in the experiments, \texttt{NumPy}, \texttt{SciPy}, and \texttt{Matplotlib}, are cited. No existing data nor models are used. %

\item {\bf New Assets}
    \item[] Question: Are new assets introduced in the paper well documented and is the documentation provided alongside the assets?
    \item[] Answer: \answerNA{} %
    \item[] Justification: Our paper provides novel theoretical results rather than datasets or models, hence this question is not applicable. %

\item {\bf Crowdsourcing and Research with Human Subjects}
    \item[] Question: For crowdsourcing experiments and research with human subjects, does the paper include the full text of instructions given to participants and screenshots, if applicable, as well as details about compensation (if any)? 
    \item[] Answer: \answerNA{} %
    \item[] Justification: This paper does not involve crowdsourcing nor research with human subjects. %

\item {\bf Institutional Review Board (IRB) Approvals or Equivalent for Research with Human Subjects}
    \item[] Question: Does the paper describe potential risks incurred by study participants, whether such risks were disclosed to the subjects, and whether Institutional Review Board (IRB) approvals (or an equivalent approval/review based on the requirements of your country or institution) were obtained?
    \item[] Answer: \answerNA{} %
    \item[] Justification: This paper does not involve crowdsourcing nor research with human subjects. %

\end{enumerate}

\end{document}